\documentclass{report}
\usepackage[online]{suthesis-2e}
\dept{Computer Science}
\PassOptionsToPackage{dvipsnames,table}{xcolor}
\usepackage{graphicx}          %
\usepackage{textcomp}
\usepackage[table,dvipsnames]{xcolor}  %
\usepackage{url}               %
\usepackage{booktabs}          %
\usepackage{microtype}         %

\usepackage{amsmath, amssymb, amsfonts, amsthm}  %
\usepackage{mathrsfs}%
\usepackage{nicefrac}          %
\usepackage{bm}                %
\usepackage{bbm}               %
\usepackage{xparse}
\usepackage{optidef}

\usepackage{tabularx}
\usepackage{multirow}
\usepackage{colortbl}
\usepackage{floatrow}
\usepackage{float}   

\usepackage{algorithm}
\usepackage{algpseudocode}

\usepackage{caption}
\usepackage{subcaption}        %
\usepackage{wrapfig}           %
\usepackage{adjustbox}         %

\usepackage{titletoc}
\usepackage{etoc}              %
\usepackage{multicol}          %
\usepackage{chngpage}          %
\usepackage{soul}              %
\usepackage{csquotes}          %
\usepackage{yfonts}            %
\usepackage{etoolbox}          %
\usepackage{dirtytalk}
\usepackage{mdframed} %
\usepackage{listings}%
\definecolor{light-gray}{gray}{0.95}
\usepackage{thmtools, thm-restate}  %
\usepackage{tikz}              %
\usepackage{enumitem}
\usepackage{manyfoot}%
\usepackage{lipsum}%
\usepackage{inconsolata}
\usepackage[T1]{fontenc}

\definecolor{cardinalred}{rgb}{0.549,0.082,0.082}
\definecolor{digitalred}{rgb}{0.694,0.016,0.055}
\definecolor{digitalblue}{rgb}{0.0000, 0.4235, 0.7216}
\definecolor{mydarkblue}{rgb}{0,0.08,0.45}
\usepackage{hyperref}          %
\hypersetup{ %
    pdfborder={0 0 0},
    colorlinks=true,
    linkcolor=digitalred,
    filecolor=digitalred,
    urlcolor=mydarkblue,
    citecolor=digitalblue,
    linktoc=page,
}
\usepackage[numbers,sort]{natbib}   %

\usepackage[capitalise,nameinlink]{cleveref}    
\crefformat{equation}{#2Eq.~#1#3}
\crefformat{figure}{#2Fig.~#1#3}
\crefformat{section}{#2\S#1#3}
\crefformat{subsection}{#2\S#1#3}
\crefformat{subsubsection}{#2\S#1#3}
\crefformat{assumption}{#2Assumption~#1#3}
\crefformat{assumption}{#2Assumption~#1#3}

\definecolor{flodarkpurple}{rgb}{0.288,0.1196,0.7}
\definecolor{RedOrange}{rgb}{0.9, 0.3, 0.0}

\newcommand{\rcl}[1]{\cellcolor{red!13}#1}
\newcommand{\vcl}[1]{\cellcolor{gray!15}#1}
\newcommand{\gcl}[1]{\cellcolor{green!13}#1}
\newcommand{\bcl}[1]{\cellcolor{blue!13}#1}
\newcommand{\kcl}[1]{\cellcolor{violet!13}#1}

\newcommand{\rtext}[1]{\textcolor{red!60}{#1}}
\newcommand{\vtext}[1]{\textcolor{violet!60}{#1}}
\newcommand{\gtext}[1]{\textcolor{ForestGreen!80!black}{#1}}
\newcommand{\btext}[1]{\textcolor{blue!60}{#1}}
\newcommand{\ktext}[1]{\textcolor{gray!90!black}{#1}}
\newcommand{\otext}[1]{\textcolor{RedOrange!100}{#1}}

\newcommand{\calU}{\mathcal{U}}

\newcommand{\calN}{\mathcal{N}}

\newcommand{\calL}{\mathcal{L}}

\newcommand{\calH}{\mathcal{H}}

\newcommand{\calD}{\mathcal{D}}
\newcommand{\calC}{\mathcal{C}}

\newcommand{\calA}{\mathcal{A}}

\newcommand{\E}{\mathbb{E}}

\newcommand{\R}{\mathbb{R}}

\newcommand{\prob}{\mathbb{P}}
\newcommand{\iid}{\overset{\textup{iid}}{\sim}}

\newtheorem{theorem}{Theorem}

\newtheorem{definition}{Definition}

\newtheorem{proposition}{Proposition}
\newtheorem{task}{Task}

\newcommand{\edit}[1]{{\color{black}{#1}}}
\newcommand{\STAB}[1]{\begin{tabular}{@{}c@{}}#1\end{tabular}}

\newcommand{\basemethod}{\textsc{Cupid}}
\newcommand{\qualitymethod}{\textsc{Cupid-Quality}}
\newcommand{\polinf}{\Psi_{\pi\text{-}\mathrm{inf}}}
\newcommand{\actinf}{\Psi_{a\text{-}\mathrm{inf}}}
\newcommand{\polinfest}{\widehat{\Psi}_{\pi\text{-}\mathrm{inf}}}
\DeclareMathOperator*{\argtop}{arg\text{ }top}
\makeatletter
\newcommand\footnoteref[1]{%
  \protected@xdef\@thefnmark{\ref{#1}}%
  \@footnotemark%
}
\makeatother

\newcommand{\obj}[1]{#1}

\NewDocumentCommand \action {>{\SplitList{ }} m} {\actioncall #1}
\NewDocumentCommand \actioncall {g g g g} {\text{#1}(
  \IfValueTF{#2}{#2}{}
  \IfValueTF{#3}{,#3}{}
  \IfValueTF{#4}{,#4}{}
  )
}

\DeclareMathOperator*{\argmax}{arg\,max}

\newcommand{\func}[2]{\mathop{}#1\left(#2\right)}
\newcommand{\probstap}[1]{\func{p}{#1}}
\newcommand{\given}{\;\middle|\;}
\newcommand{\Estap}[2]{\mathop{}\operatorname{E}_{#1}\left[#2\right]}
\newcommand{\Vstap}[2]{\mathop{}\operatorname{Var}_{#1}\left[#2\right]}

\lstdefinelanguage{pddl}{
    morekeywords={forall},
    otherkeywords={:goal},
    sensitive=true, %
    morecomment=[l]{;},
    morestring=[b]" %
}

\newcommand{\rone}[1]{\textcolor{black}{#1}}
\newcommand{\rtwo}[1]{\textcolor{black}{#1}}

\DeclareMathSymbol{\shortminus}{\mathbin}{AMSa}{"39}
\definecolor{flodarkpurple}{rgb}{0.288,0.1196,0.7}
\newcommand{\graytext}[1]{\textcolor{black}{#1}}

\definecolor{light-gray}{rgb}{0.8, 0.8, 0.8}
\definecolor{comment-green}{rgb}{0.435, 0.576, 0.106}
\definecolor{prompt-blue}{HTML}{2596be}
\definecolor{code-function}{HTML}{379fbe}
\definecolor{code-function}{HTML}{693da8}  %
\definecolor{code-syntax}{HTML}{0060b1}
\definecolor{code-constant}{HTML}{d86001}
\definecolor{prompt-gray}{HTML}{a7a7a7}
\definecolor{highlight}{HTML}{f8f9cb}
\definecolor{highlight}{HTML}{e3eeff}  %
\definecolor{code-perception}{HTML}{2ecc71}
\definecolor{code-control}{HTML}{ff9900}
\definecolor{code-undefined}{HTML}{ff0000}
\renewcommand\fbox{\fcolorbox{light-gray}{white}}

\NewDocumentCommand{\code}{v}{%
\texttt{\small{\textcolor{code-syntax}{#1}}}%
}

\newcommand{\gs}{\textbf{\textsc{generator-scorer}}}
\newcommand{\scgs}{\textbf{\textsc{saycan-gs}}}
\newcommand{\imgs}{\textbf{\textsc{innermono-gs}}}
\newcommand{\sh}{\textbf{\textsc{shooting}}}
\newcommand{\se}{\textbf{\textsc{greedy-search}}}
\newcommand{\ttm}{\textbf{Text2Motion}}

\newcommand{\scgsnb}{\textsc{saycan-gs}}
\newcommand{\imgsnb}{\textsc{innermono-gs}}
\newcommand{\shnb}{\textsc{shooting}}
\newcommand{\senb}{\textsc{greedy-search}}
\newcommand{\ttmnb}{Text2Motion}
\newcommand{\LH}{\textbf{LH}}
\newcommand{\LG}{\textbf{LG}}
\newcommand{\PAP}{\textbf{PAP}}

\begin{document}
\title{Deployment-Time Reliability of Learned Robot Policies}
\author{Christopher Agia}
\principaladviser{Jeannette Bohg}
\firstreader{Marco Pavone}
\secondreader{Clark Barrett}

\copyrightfalse
\beforepreface
\setcounter{page}{1}
\prefacesection{Abstract}
Significant advances in learning-based robot manipulation have produced policies capable of executing complex, dexterous behaviors, raising the prospect of deploying robots in everyday human spaces.
However, realizing this vision requires looking beyond demonstrated \textit{capability} alone and toward \textit{reliability}: during deployment, learned policies must contend with open-ended variability, distribution shift, and compounding errors that collectively undermine system performance.

This dissertation investigates how the reliability of learned robot policies can be improved at deployment time through mechanisms that operate around them. We develop three complementary classes of deployment-time mechanisms. First, we introduce runtime monitoring methods that detect impending failures by identifying inconsistencies in closed-loop policy behavior and deviations in task progress, without requiring failure data or task-specific supervision. Second, we propose a data-centric framework for policy interpretability that traces deployment-time successes and failures to influential training demonstrations using influence functions, enabling principled diagnosis and dataset curation. Third, we address reliable long-horizon task execution by formulating policy coordination as the problem of estimating and maximizing the success probability of behavior sequences, and we extend this formulation to open-ended, language-specified tasks through feasibility-aware task planning.

By centering on core challenges of deployment, these contributions advance practical foundations for the reliable, real-world use of learned robot policies.
Continued progress on these foundations will be essential for enabling trustworthy and scalable robot autonomy in the future.

\newpage
\textit{For Emily, Sandra, and Remon.}

\prefacesection{Acknowledgments}

This PhD was a collective journey as much as a personal one, and I owe heartfelt thanks to many.

First and foremost, I would like to express my deepest gratitude to my advisors, Prof. Jeannette Bohg and Prof. Marco Pavone, whose dedication to my success and unwavering commitment to excellence shaped my entire PhD experience.
From the very beginning, Jeannette generously devoted her time to strengthening my research foundations.
She has taught me how to translate ambitious aspirations into steady, tangible research progress. 
She instilled in me the passion to pursue problems I believed were meaningful and worked tirelessly to create the conditions that made that pursuit possible---facilitating collaborations, securing resources and funding, and offering thoughtful guidance at every stage. 
When I first joined Stanford, many trusted colleagues encouraged me in no uncertain terms to consider working under Jeannette's mentorship. 
Looking back, I am glad I did.

Marco's guidance has likewise had a profound and lasting impact on my development---not only as a researcher, but as a technologist more broadly. Throughout my PhD, he continually challenged me to think with greater clarity, precision, and formal rigor, while also encouraging me to consider the broader impact of my work in a rapidly evolving field. He thoughtfully created leadership opportunities that fostered growth beyond my immediate projects. 
His insistence on searching for deeper insight and nuance in all things has fundamentally reshaped how I think and how I assign value to problems---and for that, I am especially grateful.
I consider myself immensely fortunate to have had not one but two advisors who believed in my potential, supported the directions I chose to pursue, and helped me navigate the inevitable obstacles that ultimately make research so rewarding.

I would like to thank my thesis committee: Prof. Clark Barrett, Prof. Mykel Kochenderfer, and Prof. Iro Armeni. 
I appreciate Clark's service on my reading committee; his work has broadened my understanding of the many dimensions in which safety must be addressed within robotics and AI. 
My interactions with Mykel throughout my PhD were always a genuine pleasure---his questions often revealed perspectives on my research that I had not previously considered. 
I am especially glad that Iro served as chair of my defense; it was through my undergraduate research building on her PhD work that I first found my way to Stanford.
This has made for a unique and meaningful full circle.

I want to thank the wonderful staff in the Computer Science and Aeronautics and Astronautics departments---especially Jayanthi, Yaritza, Helen, Jimmy, Adele, Renee, Brian, Kassandra, Susan, and Nelly. Their guidance and support greatly simplified the logistical complexities of my PhD, particularly as an international student navigating an unfamiliar system. I am especially grateful to Jayanthi for the countless questions she answered with patience and thoroughness. 
My thanks also go to the team at the Bechtel International Center for their continued support over the years.

I am deeply grateful to the colleagues, mentors, and friends I have found in Prof. Florian Shkurti, Prof. Hani Naguib, Prof. Jiajun Wu, Edward, Toki, Krishna, Rika, Manfred, and Ran. 
Their advice provided clarity and perspective at pivotal moments in my PhD.
I feel especially fortunate to have had Toki and Edward as my first-year rotation mentors; the lessons I learned from them are reflected in nearly every project I have pursued since. 
I also want to acknowledge the steadfast encouragement that Florian, Krishna, and Hani have offered across many phases of my research journey.

I had the privilege of collaborating with an extraordinarily talented group of students at Stanford---Rohan, Jingyun, Amine, Joey, Jakob, Milan, Jacky, Kevin, Yixuan, Francesco, Guillem, Daniel, Alex, Jimmy, Matt, and Zi-ang. This work would not have been possible without them. In particular, I am grateful to Rohan and Jingyun for the many long nights they poured into our shared research; I am proud of what we accomplished together. I also want to thank Joey, who always kept an open door for long technical discussions. These shaped my research more than he would likely admit.

I am grateful to the industry scientists with whom I had the distinct pleasure of collaborating: Masha Itkina and Haruki Nishimura from the Toyota Research Institute, and Issa Nesnas and Saptarshi Bandyopadhyay from the NASA Jet Propulsion Laboratory. They welcomed me generously on a number of collaborations and guided me toward problems of greater practical significance. 
Working with them has broadened my perspective on what research can and should strive to achieve.

Teaching robotics was one of the most memorable, rewarding, and---at times---challenging parts of my PhD. 
I was incredibly fortunate to have shared this experience with such a dedicated team of course assistants: Alvin, Rohan, Aditya, Luis, Roger, and Suneel for Principles of Robot Autonomy; Jakob, Daniel, Luis, and Xilun for Test of AI \& Emerging Technologies.

Of course, I was truly lucky to have had such smart, kind, and spirited labmates in Toki, Priya, Tyler, Jingyun, Marion, Carlota, Krishnan, Claire, Kevin, and Brandon from the IPRL, and Rohan, Amine, Jakob, Daniele, John, Hugo, Daniel, Luis, Matt, Milan, Jacky, Pranit, Somrita, Spencer, Thomas, Rachel, and Alvin from the ASL---along with many others who made these years so special. 
Because of them, Stanford felt like home in more ways than one. 

Everything that made this PhD possible rests upon the unconditional love and support of my family. 
I am forever grateful to my parents, Sandra and Remon, for giving so selflessly of themselves so that I might have opportunities in life as meaningful as this one; and to my sister, Emily, for her vibrant spirit and for the wisdom beyond her years that carried me through every high and low.

Lastly, I thank my partner, Stephanie, for walking devotedly alongside me, for reminding me to see the color each day offers, and for showing me that the future is often closer than it seems. She has supported me in every imaginable way; none of this would have been possible without her.

\afterpreface

\chapter{Introduction}\label{chapter:1}

\section{Background}\label{sec:thesis-background}

Robotics is undergoing a period of consequential change. Systems once confined to repetitive automation in highly structured settings---such as factory assembly lines---are gradually shifting toward intelligent, autonomous operation in the unstructured and dynamic environments that characterize everyday human spaces, including homes, offices, and retail stores~\cite{ahn2022-saycan, wu2023tidybot, pmlr-v305-yang25b}. To operate effectively in these contexts, robots must possess more than just accurate perception of the structure and semantics of their environments, or the ability to plan sequences of actions toward a goal. They must also know how to physically interact with the world: how to execute planned actions safely, robustly, and reliably through closed-loop control. It is this capacity for reliable physical interaction that ultimately determines whether a robot can function autonomously outside carefully engineered environments.

While the robotics community has made rapid progress in general-purpose perception~\cite{conceptfusion2023, elhafsi2023semantic} and planning~\cite{driess2023palm, lin2023text2motion}---driven in large part by advances in computer vision, natural language processing, and foundation models~\cite{bommasani2021opportunities}---comparable progress in generalist robot manipulation has lagged behind. 
To date, the most compelling results have emerged through large-scale learning from demonstration~\cite{team2025gemini, intelligence2025pi05, barreiros2025careful},
where robot policies parameterized by deep neural networks are trained to map high-dimensional sensory observations to control targets via imitation of behaviors demonstrated by human operators across a diversity of manipulation tasks (e.g., folding clothes, making coffee, assembling objects). Yet current results remain limited in scope, as today’s most capable manipulation policies---those adapted from foundation models, termed Vision-Language-Action (VLA) models~\cite{rt22023arxiv}---cover only a small fraction of the behaviors required for general-purpose deployment and fall well short of human-level reliability. For instance, VLA policies remain highly sensitive to variations in object geometry, pose, and appearance, to changes in lighting and scene context, to instruction phrasing, and to spurious environmental features---resulting in failure modes which may be difficult or impossible for system designers to anticipate beforehand~\cite{hancock2025run, majumdar2025predictive, glossop2025cast, gu2025safe}. As such, the development of generalist robot manipulation policies remains a central focus of contemporary robotics research, drawing substantial investment and attention across academia, industry, and the broader technology ecosystem.

At the research frontier, much of the field’s efforts have focused on expanding the \textit{operational envelope} of generalist robot policies: The range of tasks and environments within which they can be expected to perform competently. This effort has taken multiple, complementary forms, including the development of hardware platforms that enable large-scale collection of dexterous manipulation data~\cite{Zhao-RSS-23, chi2024universal}, the transfer of internet-scale priors through foundation models such as Large Language Models (LLMs) and Vision-Language Models (VLMs)~\cite{rt22023arxiv, pmlr-v270-kim25c}, and the introduction of more sophisticated robot learning algorithms capable of training on increasingly diverse and heterogeneous data sources (e.g., teleoperated robot demonstrations~\cite{chi2023diffusion}, cross-embodiment datasets collected from different robotic platforms~\cite{o2024open}, data gathered during online deployment~\cite{intelligence2025pi06}, and off-domain data such as human videos~\cite{lepert2025masquerade}). 

While these efforts have broadened the scope of behaviors learned robot policies can exhibit in carefully controlled evaluation settings~\cite{gao2026taxonomy}, real-world deployment presents a persistent and markedly more unforgiving set of challenges for which the benefits of e.g., increased data coverage and richer priors alone are insufficient. 
That is, the open-ended variability, complexity, and non-stationarity of everyday, real-world environments ensure that, regardless of the scale or diversity of a policy's training data or the richness of its learned priors, deployed systems will eventually encounter situations that lead to erroneous or undesirable behavior~\cite{sinha2022system}. 
This observation highlights a complementary line of inquiry: alongside continued efforts to advance learned policies toward deployment readiness, we must also consider how to maintain reliability at deployment time---where failures are costly, difficult to diagnose, and often irreversible.
Herein, we consider three representative challenges that arise in the real-world deployment of robot policies:
\begin{enumerate}
    \item \textbf{Failure Mitigation:} The black-box, task-agnostic nature of modern policies---particularly those initialized from large foundation models---obscures not only whether a deployed policy is nominally within the support of its training distribution, but whether it is operating in regions of that distribution where experience is sufficiently dense to support reliable behavior. As a result, even minor and seemingly indistinguishable shifts in the environment may push a policy outside the reliable support of its training data. This lack of visibility complicates the timely detection of failures and the ability to intervene before unsafe or irreversible outcomes occur.
    \item \textbf{Model Interpretability:} The same opacity of modern policies makes it extremely difficult to understand why a policy succeeds or fails during closed-loop execution, hindering post-hoc diagnosis, systematic debugging, and principled dataset curation---processes that are essential for improving policy performance over time. 
    \item \textbf{Long-Horizon Tasks:} Learning and executing an entire multi-step task with a single policy behavior is often unreliable due to compounding prediction errors, motivating approaches that instead chain together shorter, independent behaviors with higher per-step reliability. However, even when shorter behaviors are learned reliably, composing them to accomplish new multi-step tasks often fails due to unmodeled dependencies, which may compromise the feasibility of downstream behaviors. Achieving reliable outcomes therefore requires explicitly reasoning about the probability of success of behavior sequences and coordinating policy outputs accordingly.
\end{enumerate}
Notably, addressing these challenges motivates a distinct class of methods that operate \textit{around} learned policies, explicitly accounting for the closed-loop, real-time, and safety-critical constraints inherent to robotic deployment. We argue that developing such deployment-time mechanisms---alongside continued advances in policy learning---is essential for bringing generalist robot policies out of the lab and into reliable real-world use, and it is this perspective that motivates and unifies the contributions of this dissertation.

\section{Thesis Outline}\label{sec:thesis-outline}

What it means for a learned robot policy to be reliable at deployment differs fundamentally from achieving strong performance during training or in controlled evaluation settings. Experience from adjacent domains, most notably autonomous driving, has demonstrated that reliability in the long tail of real-world deployment is not a property of any single model class (e.g., a VLA policy) nor a simple consequence of increased data scale, but rather emerges from the careful orchestration of reliability-promoting mechanisms spanning the autonomy stack and development lifecycle~\cite{webb2020waymo, goldberg2025good}. As learned manipulation policies grow in capability and are increasingly deployed in open-ended, real-world environments, the need for such deployment-time considerations becomes both unavoidable and urgent. Accordingly, this dissertation investigates how the reliability of learned robot policies---and, by extension, systems and tasks composed of them---can be enhanced at deployment through three complementary lenses: (1) detecting when policies are likely to fail, (2) interpreting policy behavior through the data that shaped it, and (3) coordinating learned policy behaviors for reliable long-horizon task execution. 
\cref{part:1} of the thesis addresses (1-2), while \cref{part:2} is dedicated to (3). 

\subsection{Policy Monitoring and Interpretability (\cref{part:1})}

\cref{chapter:sentinel} addresses the challenge of detecting when a learned robot policy is likely to fail during deployment, even when operating under seemingly nominal conditions. We present a runtime monitoring framework for closed-loop, generative manipulation policies that detects failures online without requiring failure data or task-specific supervision. Our framework judiciously integrates runtime monitors to capture both low-level, real-time inconsistencies in policy behavior and higher-level deviations in task progress. Through simulation and real-world experiments, this chapter demonstrates that structuring runtime monitoring around a hierarchy of specialized detectors---aligned with different granularities of policy failure---enables earlier and more reliable intervention than any single monitoring approach alone.
\cref{chapter:cupid} addresses a complementary challenge: while monitoring can prevent unsafe outcomes, it does not explain why a policy behaves as it does. This chapter introduces a data-centric perspective on policy interpretability, focusing on how individual training demonstrations influence closed-loop behavior at deployment. We propose a methodology based on influence functions that connects deployment-time performance metrics---such as task success or failure rates---to examples in the training dataset. By enabling counterfactual reasoning about data, this chapter shows how greater interpretability mechanisms can support principled dataset curation, diagnose spurious correlations, and systematically improve policy reliability under distribution shift.

\subsection{Policy Coordination and Planning (\cref{part:2})}

\cref{chapter:stap} examines a central limitation of behavior-centric policy learning: reliable execution of individual behaviors does not guarantee reliable outcomes when those behaviors are composed over time to accomplish real-world, long-horizon tasks. 
The key insight is that success in such settings depends on how individual behaviors interact when executed in sequence, which motivates the need for explicit coordination.
We formalize behavior coordination as the problem of estimating and maximizing the joint success probability of a given sequence of learned policy behaviors using models produced during policy training.
Building on this foundation, \cref{chapter:text2motion} introduces a language-based task and motion planning framework. 
This framework extends our behavior coordination formulation to open-ended goals expressed in natural language, using LLMs to search for candidate behavior sequences and feasibility-aware estimates to guide their selection. 
Together, these approaches show that explicitly reasoning about the success probability of coordinated policy behaviors is critical for reliable operation in complex, real-world manipulation scenarios.

Collectively, these contributions advance the reliability of learned robot policies by addressing a complementary set of challenges inherent to deployment. While not exhaustive, these challenges lie on the critical path toward trustworthy and scalable robot autonomy. Looking ahead, several promising directions for future research remain, which we outline in \cref{chapter:conclusion}.

\section{Publications}\label{sec:thesis-publications}

Parts of this dissertation are adapted from the following previously published works\footnote{For papers with multiple authors, Christopher Agia led or co-led all aspects of the research, including conception, methodology, experiments, analysis, and writing. Co-first authorship is indicated with an asterisk (*).}:

\begin{itemize}
    \item \cite{taps-2022} \textbf{Christopher Agia}*, Toki Migimatsu*, Jiajun Wu, and Jeannette Bohg. 
    \textit{STAP: Sequencing Task-Agnostic Policies}. 
    In IEEE International Conference on Robotics and Automation, 2023. 
    \copyright{} 2023 IEEE. Reprinted, with permission, from the above publication.

    \item \cite{lin2023text2motion} \textbf{Christopher Agia}*, Kevin Lin*, Toki Migimatsu, Marco Pavone, and Jeannette Bohg. 
    \textit{Text2motion: From Natural Language Instructions to Feasible Plans}. 
    Autonomous Robots, 2023. 
    \copyright{} 2023 Springer Nature. Reproduced with permission from Springer Nature.

    \item \cite{agia2024unpacking} \textbf{Christopher Agia}, Rohan Sinha, Jingyun Yang, Ziang Cao, Rika Antonova, Marco Pavone, and Jeannette Bohg. 
    \textit{Unpacking Failure Modes of Generative Policies: Runtime Monitoring of Consistency and Progress}. 
    Proceedings of the 8th Conference on Robot Learning, 2024. 
    \copyright{} 2024 PMLR.

    \item \cite{agia2025cupid} \textbf{Christopher Agia}, Rohan Sinha, Jingyun Yang, Rika Antonova, Marco Pavone, Haruki Nishimura, Masha Itkina, and Jeannette Bohg. 
    \textit{CUPID: Curating Data your Robot Loves with Influence Functions}. 
    Proceedings of the 9th Conference on Robot Learning, 2025. 
    \copyright{} 2025 PMLR.
\end{itemize}

\part{Policy Monitoring and Interpretability}\label{part:1}

\chapter{Monitoring Policies for Runtime Failure Detection}\label{chapter:sentinel}

\section{Introduction}\label{sec:sentinel-intro}

The central objective of this thesis is to investigate how the reliability of learned robot policies can be improved during real-world deployment. As discussed in \cref{chapter:1}, recent advances in robot learning---particularly the adaptation of generative, pretrained foundation models to large and heterogeneous robotics datasets---have enabled policies to exhibit increasingly rich and flexible behavior across a wider variety of tasks, environments~\cite{rt22023arxiv}, and embodiments~\cite{o2024open}.

This chapter focuses on a challenge that persists even as such capabilities improve: when robot policies are deployed in real-world settings---where conditions cannot be perfectly controlled---they will inevitably encounter \textit{out-of-distribution} (OOD) scenarios in which their behavior becomes unreliable and difficult to predict~\cite{argall2009survey}. Moreover, as such environmental shifts may be subtle and their relation to the \textit{in-distribution}, reliable regime of a policy's experience is often unclear---particularly for policies initialized from internet-pretrained foundation models---failures may arise abruptly and without clear warning. Consequently, methods that monitor learned policies at deployment time and provide early indicators of failure are essential for enabling scalable and reliable real-world robotics deployment.

\begin{figure}[ht]
    \includegraphics[width=\linewidth]{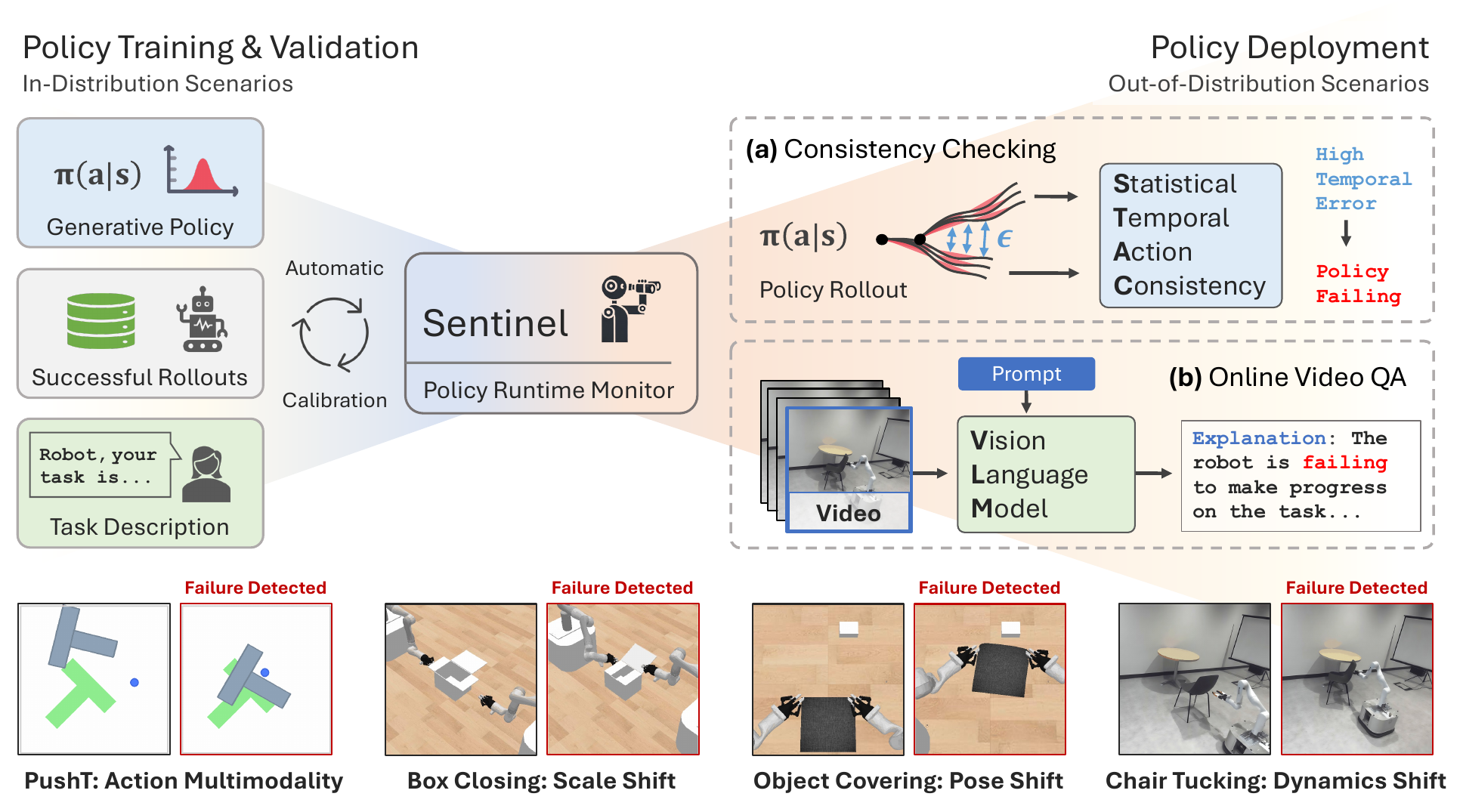}
    \caption[Overview of Sentinel]{\small
        We present \textbf{Sentinel}, a runtime monitor that detects unknown failures of generative robot policies at deployment time.
        Constructing Sentinel requires only a set of successful policy rollouts and a description of the task, from which it detects diverse failures by monitoring \textbf{(a)} the temporal consistency of action-chunk distributions generated by the policy and \textbf{(b)} the task progress of the robot(s) through video QA with Vision-Language Models. More details can be found on the Sentinel website: \url{https://sites.google.com/stanford.edu/sentinel}.
    }
    \label{fig:sentinel-teaser}
\end{figure}

Identifying when a learned model behaves unreliably is typically framed as an OOD detection problem, for which a taxonomy of methods exist~\cite{sinha2022system,yang2021generalized}.
While these methods can signal distribution shift~\cite{richter2017safe,RuffVandermeulenEtAl2018} or quantify uncertainty~\cite{GalZoubin2016, LakshminarayananPritzelEtAll2017} \textit{w.r.t.} individual input-output samples, they do not fully characterize closed-loop policy failures that arise from multiple, time-correlated prediction errors along a trajectory rollout.
Action multimodality further complicates the failure detection problem: that is, actions sampled from multimodal generative policies can vary greatly from one timestep to the next, leading to complex runtime behaviors and, by extension, diverse failure modes compared to previous model-free policies~\cite{pmlr-v164-mandlekar22a,zhang2018deep}. 
Therefore, the special case of generative robot policies necessitates the design of new failure detectors suited to their multimodal characteristics and closed-loop operational nature in deployment.

In this chapter, we present \textbf{Sentinel}, a runtime monitoring framework that splits the task of detecting generative policy failures into two complementary categories.
The first is the detection of failures in which the policy exhibits erratic behavior as characterized by its temporal inconsistency.
For example, the robot may collide with its surroundings if the policy's action distributions contain conflicting action modes across time.
To detect erratic failures, we propose to evaluate \textit{how much} a generative policy's action distributions are changing across time using \underline{S}tatistical measures of \underline{T}emporal \underline{A}ction \underline{C}onsistency (STAC).
The second category is the detection of failures in which the policy is temporally consistent but struggles to make progress on its task.
For example, the robot can stall in place or drift astray if the policy produces constant outputs.
We propose to detect task progression failures (undetectable by STAC) zero-shot with Vision-Language Models (VLMs), which can distinguish off-nominal behavior when prompted to reason about the robot's progress in a video question answering setup. 
Notably, one would want to catch erratic failures (the first category) fast, whereas task progression failures (the second category) do not require immediate intervention. 

Our contributions are three-fold: 1) A formulation of failure detection for generative policies that splits failures into two complementary categories, thus admitting the use of specialized detectors toward system-level performance increases (i.e., a \textit{divide-and-conquer} strategy); 
2) We propose STAC, a novel temporal consistency detector that tracks the statistical consistency of a generative policy's action distributions to detect erratic failures; %
3) We propose the use of VLMs to monitor the task progress of a policy over the duration of its rollout, and we offer practical insights for their use as failure detectors.
Provided with only a set of successful policy rollouts and a natural language description of the task, \textbf{Sentinel} (which runs STAC and the VLM monitor in parallel) detects over 97\% of unknown failures exhibited by diffusion policies~\cite{chi2023diffusion} across simulated and real-world robotic mobile manipulation domains.

\section{Related Work}\label{sec:sentinel-related-work}	

Advances in \textbf{robot imitation learning} include new policy architectures~\cite{chi2023diffusion,yang2023equivact,bharadhwaj2023roboagent,goyal2023rvt}, hardware innovations for data collection~\cite{Zhao-RSS-23,chi2024universal,fu2024mobile}, community-wide efforts to scale robot learning datasets~\cite{khazatsky2024droid,o2024open,walke2023bridgedata}, and training high-capacity behavior policies on these datasets~\cite{octo_2023,rt22023arxiv,rt12022arxiv,pmlr-v270-kim25c}.
Of recent interest is the use of generative models~\cite{lecun2006tutorial,vaswani2017attention,ho2020denoising,song2020score} to effectively learn from heterogeneous and multimodal datasets of human demonstrations~\cite{chi2023diffusion,o2024open,octo_2023,pmlr-v270-kim25c}.
\textit{Generative policies} thereby learn to represent highly multimodal distributions from which diverse robot actions can be sampled.
While state-of-the-art generative policies demonstrate remarkable performance, their inherent multimodality results in more stochastic runtime behavior than that of previous model-free policies~\cite{pmlr-v164-mandlekar22a,zhang2018deep,rahmatizadeh2018vision,florence2019self,pmlr-v80-haarnoja18b,nair2020awac}.
In this chapter, we focus on characterizing the behavior of generative robot policies for failure detection.

Despite recent progress, it is well known that \textbf{learned policies may fail} beyond their training distribution~\cite{argall2009survey,sinha2022system,yang2021generalized}, in part due to compounding prediction errors on states induced by the policy~\cite{ross2010efficient,ross2011reduction}.
As such, a recent work proposes to bound the performance of imitation learned policies prior to deployment~\cite{vincent2024generalizable}.
Other works propose to retrain the policy on OOD states using corrective supervision from humans~\cite{bajcsy2018learning,kelly2019hg,mandlekar2020human,pmlr-v164-hoque22a,cui2023no}.
Notably, these methods apply \textit{after} failures have occurred, maintaining the need for runtime monitors that detect policy failures and prevent their downstream consequences.
Thus, our focus can be viewed as complementary to methods that learn post hoc from corrective feedback.

The existing literature on \textbf{out-of-distribution detectors} and \textbf{runtime monitoring} for learned models is highly diverse, spanning multiple categories of methods. Model-based methods (e.g., \cite{MarcoMorleyEtAl2023, LiuDassEtAl2024}) are not directly applicable to the model-free policies we consider.
Some methods only pursue failure modes that are known \textit{a priori}~\cite{FaridSnyderEtAl2022, DaftryZengEtAl2016, RabieeBiswas2019, finonet}, \edit{whereas we seek to detect unknown failures at deployment time.}
Many OOD detection works detect dissimilarity from training data via reconstruction~\cite{richter2017safe, RuffKauffmanEtAl2021} or embedding similarity~\cite{LeeLeeEtAl2018, RuffVandermeulenEtAl2018}, however, observational differences may not always result in policy failure.
Other methods directly quantify epistemic uncertainty~\cite{GalZoubin2016,LakshminarayananPritzelEtAll2017,LiuLinEtAl2020, SharmaAzizanEtAl2021}, but come with considerable computational expense or may not be suitable for autoregressive, generative policy architectures, e.g., diffusion policies~\cite{chi2023diffusion}.
Several works monitor symbolic states to detect manipulation failures~\cite{migimatsu2022symbolic, hegemann2022learning}, but assume access to multiple sensor modalities (e.g., vision, haptic, proprioception) for symbolic state estimation.
Most related to our approach are algorithms that perform consistency checks across sensor modalities~\cite{lee2021detect} and time~\cite{AntonanteSpivakEtAl2021}.
Different from these, we directly monitor the consistency of a learned policy's action distributions and its task progress to detect closed-loop failure.

There is a growing interest in the use of \textbf{Foundation Models}~\cite{bommasani2021opportunities} toward increasing robustness in robotic systems. 
Large Language Models are used to detect anomalies~\cite{elhafsi2023semantic, sinha2024real} and to replan under execution failures~\cite{sinha2024real,pmlr-v205-huang23c,pmlr-v229-liu23g,skreta2024replan}. 
Reward models in the form of visual representations~\cite{ma2023vip,cui2022can} or VLMs~\cite{yang2023robot} could be repurposed for failure detection by thresholding predicted rewards.
However, \cite{yang2023robot} shows that additional fine-tuning is required to obtain reliable reward estimates.
\citet{pmlr-v232-du23b} fine-tunes a VLM for episode-level success classification using a human annotated dataset on the order of $10^5$ trajectories.
In contrast to this work, we 1) focus on zero-shot assessment with VLMs, 2) seek to detect policy failure amidst task execution, and 3) consider the system-level role of VLMs operating alongside policy-level failure detectors, and as such, assign each detector to a specified category of failure.

\section{Problem Setup}\label{sec:sentinel-problem-setup}

\paragraph{Failure Detection}
Our goal is to detect when a generative robot policy $\pi(a | s)$ fails to complete its task.
The policy operates within a Markov Decision Process (MDP) with a finite horizon $H$, but it may terminate upon completion of the task at an earlier timestep.
Given an initial state $s_0$ representative of a new test scenario, executing the policy for $t$ timesteps produces a trajectory $\tau_t = (s_0, a_0, \ldots, s_t)$.
We define \textbf{failure detection} as the task of detecting whether a trajectory rollout $\tau_H$ constitutes a failure at the earliest possible timestep $t$.
To do so, we aim to construct a failure detector $f(\tau_t) \rightarrow \{\texttt{ok}, \texttt{failure}\}$ that, at each timestep $t$, can provide a classification as to whether the policy \textit{will fail} if it continues executing for the remaining $H-t$ timesteps of the MDP.
Note that the failure detector makes its assessment solely based on the history of observed states and sampled actions up to the current timestep $t$.

The failure detector may contain parameters that require calibration, such as a detection threshold (as in \cref{sec:sentinel-temporal-consistency}).
Therefore, we assume a scenario in which the policy $\pi$ is first trained, then validated on test cases where it is expected to perform reliably.
This validation process yields a small dataset of $M$ successful trajectories $\mathcal{D}_\tau = \{\tau^i\}_{i=1}^M$ that can be used to calibrate the failure detector $f$ (if it contains parameters).
Intuitively, the dataset $\mathcal{D}_\tau$ characterizes nominal policy behavior within or near the distribution of states it has been trained on, which helps to ground the assessment of potentially OOD trajectories at test time. 

We measure failure detection performance in terms of true positive rate (TPR), true negative rate (TNR), and detection time (DT). We count a true positive if the failure detector raises a warning at any timestep in a trajectory where the policy fails, the earliest of which counts as the detection time. 
We count a true negative if the failure detector never raises a warning in a trajectory where the policy succeeds. 
We refer to \cref{appx:sentinel-evaluation-protocol} for supporting definitions of policy failure and key performance metrics.

\paragraph{Policy Formulation}
\begin{wrapfigure}{rt}{0.34\textwidth}
  \centering
  \vspace{-14pt}
  \includegraphics[width=1.0\textwidth]{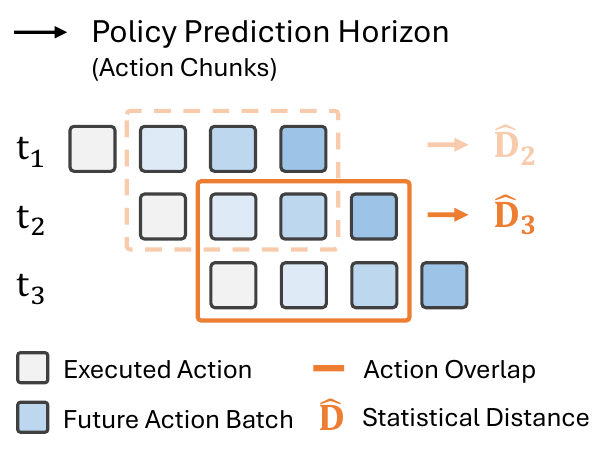}
  \caption[Action chunk overlap during policy rollout (Sentinel)]{\small 
    Action sequence prediction overlap during policy rollout.
    }
  \label{fig:sentinel-action-overlap}
\end{wrapfigure}

We consider the setting where the policy $\pi$ is stochastic and predicts a sequence of actions (also referred to as an \textit{action chunk}~\cite{Zhao-RSS-23}) for the next $h$ timesteps.
That is, the action sequence sampled at the $t$-th timestep, $a_t \sim \pi(\cdot | s_t)$, consists of $h$ actions $a_t := (a_{t|t}, a_{t+1|t}, \ldots, a_{t+h-1|t})$, where the notation $a_{t+i|t}$ denotes the action prediction for time $t+i$ generated at timestep $t$ (as in~\cite{BorrelliBemporadMorari2017}).
The actions $a_{\cdot|t} \in \mathcal{A}$ may correspond to e.g., end-effector poses or velocities. 
To control the robot, we sample an action sequence and execute the first $k<h$ actions, $a_{t:t+k|t}$, after which we re-evaluate the policy at timestep $t+k$. 
We visualize this receding horizon rollout for $k=1$ in \cref{fig:sentinel-action-overlap}.
Notably, $a_t$ and $a_{t+k}$ contain actions that temporally overlap for $h-k$ timesteps (i.e., at $a_{t+k:t+h-1|t}$ and $a_{t+k:t+h-1|t+k}$). 

Several recently proposed policy architectures achieve state-of-the-art performance by sampling action sequences using generative models~\cite{chi2023diffusion,Zhao-RSS-23,octo_2023}, to which our approach is generically applicable. Here, we specify the failure detection problem for diffusion policies (DP)~\cite{chi2023diffusion}, which a) are stable to train and b) address action multimodality by representing the policy distribution with a denoising diffusion probabilistic model (DDPM)~\cite{ho2020denoising}. 
We note that the computationally intensive, iterative nature of the denoising process makes it challenging to directly apply existing OOD detection methodologies (e.g., \cite{SharmaAzizanEtAl2021}) to diffusion policies for failure detection, thus motivating several of our design decisions.
Further details on the training and properties of these models are provided in \cref{appx:sentinel-diffusion-policy}.

\begin{figure}
    \includegraphics[width=\linewidth]{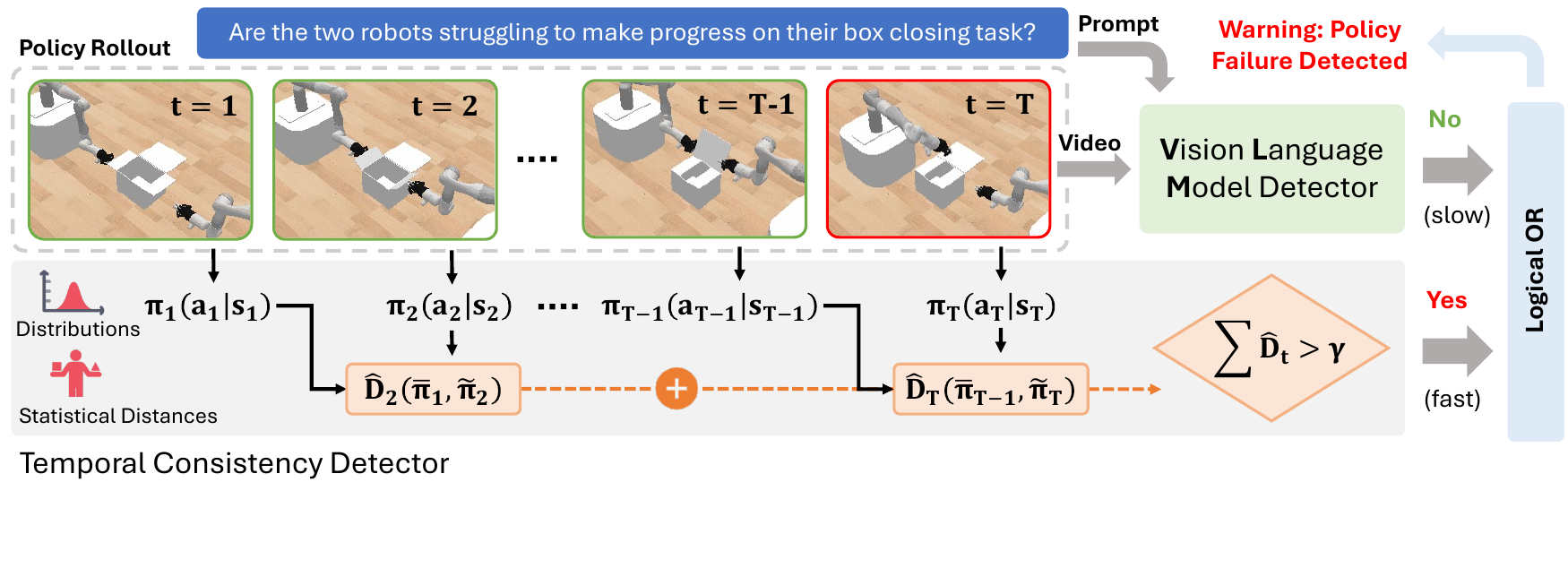}
    \vspace{-32pt}
    \caption[System architecture for failure detection (Sentinel)]{\small
        \textbf{Overview of Sentinel.} 
        The images depict a policy rollout for timesteps $t=1,\ldots, T$. Temporal Consistency Detector: At each timestep $t$, the state $s_t$ is passed to the generative policy to obtain action distributions $\pi_t$ between which statistical distances $\hat{D}_t$ are computed to measure temporal consistency. The statistical distances are summed up to the current timestep $T$ (as in Eq.~\ref{eq:sentinel-cum-score-fn}) and thresholded by $\gamma$ to detect policy failure. Vision-Language Model (VLM) Detector: The VLM classifies whether the policy is failing to make progress on its task given a video up to timestep $T$ and a description of the task. Execution stops if either detector raises a warning.
    }
    \label{fig:sentinel-full-system}
\end{figure}

\section{Proposed Approach: Sentinel}\label{sec:sentinel-approach}

The failure behavior of a generative policy by OOD conditions can be highly diverse, and we therefore argue that the desiderata for a failure detector may vary between qualitative types of failures.
Accordingly, we propose to split the failure detection task into two complementary failure categories.

The first is the detection of failures resulting from erratic policy behavior, which may cause a robot to end up in states that are difficult or costly to reset from, knock over objects, or lead to safety hazards. Therefore, it is important to detect erratic behavior as quickly as possible (\cref{sec:sentinel-temporal-consistency}).
The second category is the detection of failures in which the policy struggles to make progress on its task (hereafter referred to as \textit{task progression failures}) but does so in a temporally consistent manner. 
For example, the policy may confidently place an object in the wrong location. 
Here, we must observe the robot over a longer period of time to identify that the policy is not making progress towards task completion (\cref{sec:sentinel-vlm-assessment}).

The \textbf{key insight} of our approach is that by defining one failure category as the complement of the other, it becomes trivial to combine failure detectors to form an accurate overall detection pipeline whilst satisfying the requirements of each failure category. 
Our full pipeline, \textbf{Sentinel}, is shown in \cref{fig:sentinel-full-system}. 

\subsection{STAC: Detecting Erratic Failures with Temporal Consistency}\label{sec:sentinel-temporal-consistency}

When a policy operates in nominal, in-distribution settings, it should complete its task in a temporally consistent manner. 
For example, a policy may plan to avoid an obstacle on the right or on the left, but not jitter between the two options.
Moreover, as noted in \cite{chi2023diffusion}, training a diffusion policy that predicts action sequences rather than individual actions encourages temporal consistency between action predictions.

Therefore, we propose to construct a quantifiable measure of temporal action consistency to detect whether the policy is behaving erratically, and hence, is likely to fail at the task. 
However, the multimodal distributional nature of DPs makes it difficult to directly compare two sampled actions $a_t \sim \pi(a_t | s_t)$ and $a_{t+k} \sim \pi(a_{t+k} | s_{t+k})$, e.g., throughout execution.
This is because the actions may differ substantially along their prediction horizon when the policy commits or switches between different action modes, or simply due to randomness in sampling. 
Instead, we quantify erratic policy behavior with \underline{s}tatistical measures of \underline{t}emporal \underline{a}ction \underline{c}onsistency (\textbf{STAC}, which we term our approach).

Let $\bar\pi_t := \pi(a_{t+k:t+h-1|t} | s_t)$ and $\tilde{\pi}_{t+k} := \pi(a_{t+k:t+h-1|t+k} | s_{t+k})$ be the marginal action distributions of the temporally overlapping actions between timesteps $t$ and $t+k$. 
We compute the temporal consistency between two contiguous timesteps $t$ and $t+k$ as $\hat{D}(\bar\pi_t, \tilde\pi_{t+k}) \geq 0$, where $\hat{D}$ denotes the chosen statistical distance function (e.g., maximum mean discrepancy, KL-divergence)\footnote{Due to the iterative denoising procedure of the diffusion policy, analytically computing a distance $D(\bar\pi_t, \tilde\pi_{t+k})$ is challenging, as evaluating the densities of $\bar\pi_t$ and $\tilde\pi_{t+k}$ requires marginalizing out the intermediate diffusion steps and the non-overlapping actions. 
Instead, we approximate $D$ with its empirical counterpart $\hat{D}$ by sampling a batch of action sequences (parallelized on a GPU) at each timestep $t$ and $t+k$ rather than a single action sequence.}.
In addition, we propose to take the cumulative sum of statistical distances along a trajectory as a measure of the overall temporal consistency in a policy rollout. At each policy-inference timestep $t = jk$ with $j\in \{0,1,\dots\}$, we compute the temporal consistency score as
\begin{align}\label{eq:sentinel-cum-score-fn}
    \eta_t := \sum_{i=0}^{j-1} \hat{D}(\bar\pi_{ik}, \tilde\pi_{(i+1)k}).
\end{align}
Computing the consistency score in a cumulative manner has two advantages over thresholding the distance at each timestep individually. Firstly, it allows us to detect cases where the temporal consistency metric $\hat D$ is marginally larger than usual throughout the episode (e.g., jitter). Secondly, it allows us to detect instances where the policy is temporally inconsistent more often than in nominal scenarios. 

\begin{wrapfigure}{rt}{0.31\textwidth}
  \centering
  \vspace{-14pt}
  \includegraphics[width=1.0\textwidth]{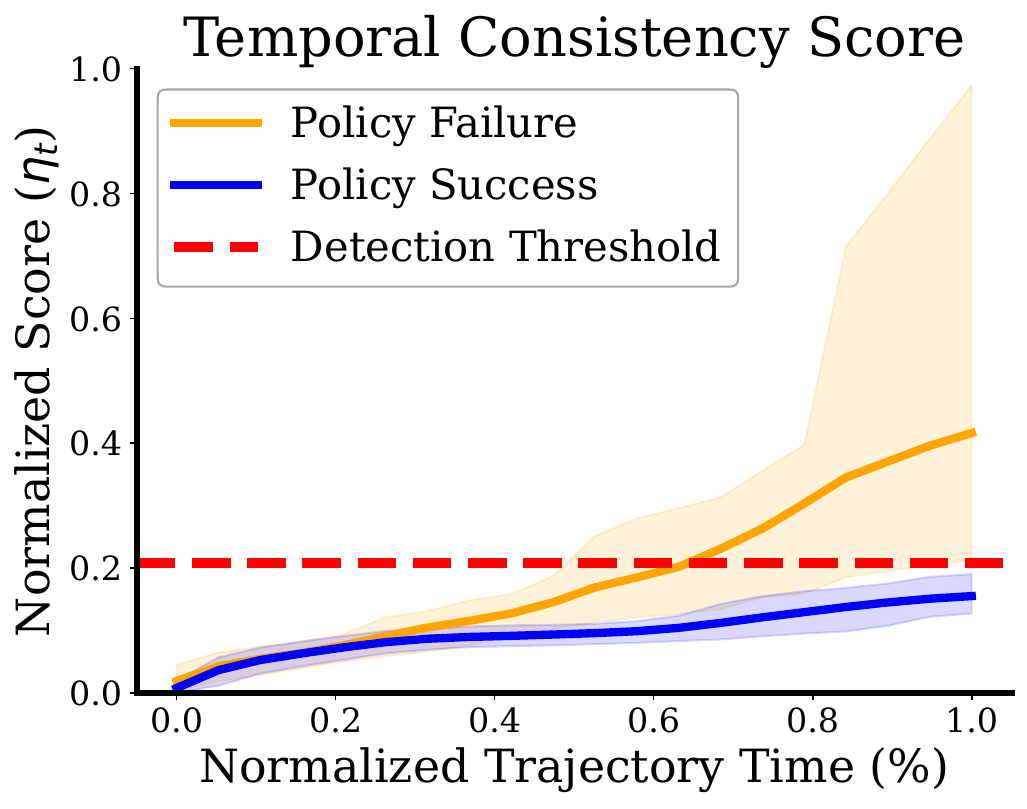}
  \caption[Detection threshold for temporal consistency score (Sentinel)]{\small
    Temporal consistency scores grow faster when the policy fails. Error bars indicate the 5-th and 95-th score quantiles.}
  \label{fig:sentinel-error-result}
\end{wrapfigure}

At runtime, we raise a failure warning at the moment that $\eta_t$ exceeds a failure detection threshold $\gamma$, which we calibrate offline using the validation dataset of successful trajectories $\calD_\tau$. 
Here, we first compute the cumulative temporal consistency scores throughout the entirety of the lengths $H_i \leq H$ of trajectories in $\calD_\tau$, yielding $\{\eta^i_{H_i}\}_{i=1}^M$.
Then, we set the threshold $\gamma$ to the $1 - \delta$ quantile of $\{\eta^i_{H_i}\}_{i=1}^M$, where $\delta \in (0,1)$ is a hyperparameter. 
Intuitively, this ensures that the false positive rate (FPR)---the probability that we raise a false alarm and terminate on any trajectory that is i.i.d. with respect to $\calD_\tau$---remains low, such that any warnings are likely failures. 
We can formalize this intuition using recent results from conformal prediction~\cite{angelopoulos2021gentle, LuoZhaoSample2023}:

\begin{proposition}[STAC has low FPR]\label{cor:sentinel-stac}
    Let $P_\tau$ denote the distribution of success trajectories in the validation dataset $\calD_\tau = \{\tau^i\}_{i=1}^M \iid P_\tau$.
    Then, the \emph{FPR}---the probability of raising a false alarm at any point during an i.i.d. test trajectory $\tau \sim P_\tau$ of length $H' \leq H$---is at most $\delta$:
    \begin{equation*}
        \mathrm{FPR} := \prob_{P_\tau} \big(\exists \ 0 \leq t \leq H' \ \mathrm{s.t.} \ \eta_t > \gamma \big) \leq \delta.
    \end{equation*}
\end{proposition}
We refer to \cref{appx:sentinel-stac} for additional details on STAC and \cref{appx:sentinel-derivations} for a full statement and proof of \cref{cor:sentinel-stac}.

\subsection{Detecting Task Progression Failures with VLMs}\label{sec:sentinel-vlm-assessment}
A policy operating in out-of-distribution settings may not always fail by exhibiting erratic behavior that we can detect with STAC (\cref{sec:sentinel-temporal-consistency}). 
For example, suppose the policy confidently commits to the wrong plan or produces approximately constant outputs. 
Detecting such failures requires an understanding as to whether or not the policy is progressing on its task, which necessitates a more comprehensive analysis of the robot's behavior within the context of its task specification. 
Therefore, we propose to use VLMs to monitor the task progress of the policy by providing them with the robot's image observations up to the current timestep as a video. We do so because recent work has shown that high-capacity VLMs possess robotics relevant knowledge and contextual reasoning abilities~\cite{nasiriany2024pivot,gao2023physically,pmlr-v232-du23b,li2024visionlanguage,yang2023robot}.

We formulate the detection of task progression failures as a chain-of-thought (CoT)~\cite{wei2022chain}, video question answering (QA) task~\cite{antol2015vqa,xu2016msr}, reflecting current best practices in prompting.
To capture a notion of \textit{task progress}, the VLM must reason across time and in the context of the policy's task.
Thus, we construct a prompt that contains a description of the task and the VLM's role as a runtime monitor.
We query the VLM online using the text prompt and the history of observed images (i.e., a video) $I_{0:t} := (I_0, I_{\nu k}, I_{2\nu k}, \ldots, I_t)$ up to the current timestep $t$, where $\nu$ determines the frequency of the images relative to the execution horizon $k$ of the DP (\cref{sec:sentinel-problem-setup}). 
Differentiating between partial progress and task failure can be ambiguous for a slow moving robot, and thus, we also specify the current elapsed time $t$ and the time limit for the task $H$. 
This enables the VLM to gauge whether the rate at which the robot is executing will result in a timely task completion.
After performing a CoT analysis, the VLM concludes with a classification in $\{\texttt{ok}, \ \texttt{failure}\}$.
For additional details on the VLM and prompt, please see \cref{appx:sentinel-vlm}.

At the time of writing, cloud-querying a state-of-the-art VLM for video QA incurs significant latency (e.g., \texttt{GPT-4o}'s mean response time was $14.0\mathrm{s}$). However, we emphasize that VLM inference latency is a lesser concern for detecting task progression failures because they are likely to occur at longer timescales and do not require urgent intervention. 
In contrast, we assign the rapid detection of erratic failures to STAC (\cref{sec:sentinel-temporal-consistency}).
Notably, the fast and slow detection requirements of our complementary failure categories mean that STAC and the VLM can operate at different timescales, offering potential benefits such as reduced costs and a lower likelihood of false positives when they run in parallel (\cref{fig:sentinel-full-system}).

\section{Experiments}\label{sec:sentinel-experiments}

We conduct a series of experiments to test our failure detection framework. 
These experiments take place in both simulation and the real world (\cref{fig:sentinel-teaser}), and host an extensive list of baselines.
We refer to \cref{appx:sentinel-experiments} for a detailed description of our environments, hardware setup, baselines, and evaluation protocol.

\textbf{Environments.} 
We include the \textbf{PushT} domain from \cite{chi2023diffusion} to evaluate the detection of failures under action multimodality.
The \textbf{Close Box} and \textbf{Cover Object} domains involve two mobile manipulators and thus present the challenge of a high-dimensional, 14 degree-of-freedom action space.
We additionally conduct hardware experiments with a mobile manipulator for a nonprehensile \textbf{Push Chair} task.
This task presents greater dynamic complexity than the simulation domains~\cite{ruggiero2018nonprehensile}.
At test time, we generate OOD scenarios by randomizing a) the scale and dimensions of objects in \textbf{PushT} and \textbf{Close Box} and b) the pose of the object in \textbf{Cover Object} and \textbf{Push Chair} beyond the policy's demonstration data.

\textbf{Baselines.} 
We evaluate Sentinel (i.e., both STAC and the VLM) against baselines representative of multiple methodological categories in the OOD detection literature~\cite{sinha2022system}. 
Intuitively, these categories represent different formulations of the failure detector's score function, responsible for computing the per-timestep scores that are then summed to compute the trajectory score as in \cref{eq:sentinel-cum-score-fn}.
We consider score functions based on the embedding similarity of observed states \textit{w.r.t.} $\mathcal{D}_\tau$ \cite{LeeLeeEtAl2018}, the reconstruction error of actions sampled from the DP~\cite{graham2023denoising}, and the output variance of the DP. 
To strengthen the comparison, we introduce a new baseline that uses the DDPM loss (\cref{eq:sentinel-ddpm-loss}) on re-noised actions sampled from the DP as the failure detector's score function.
Where applicable, we implement temporal consistency variants of these baselines to ablate the design of STAC. 
Further details on these baselines are provided in \cref{appx:sentinel-baselines}.

\textbf{Evaluation Protocol.} 
We train a DP for each environment and use standard settings for the DP's prediction and execution horizon~\cite{chi2023diffusion}. 
We use the same calibration and evaluation protocol across all failure detection methods.
That is, we calibrate detection thresholds to the 95-th quantile of scores in a dataset $\mathcal{D}_\tau = \{\tau^i\}_{i=1}^M$ of $M = 50$ in-distribution rollouts for each simulated task and $M = 10$ in-distribution rollouts for the real-world task.
Finally, we report standard detection metrics including TPR, TNR, Mean Detection Time, Accuracy, and Balanced Accuracy, following the definitions in \cref{sec:sentinel-problem-setup}.

\section{Results}\label{sec:sentinel-results}

\textbf{STAC detects diffusion policy failures in multimodal domains.}
\cref{fig:sentinel-pusht-result} (Left) compares STAC against the best performing method of each baseline category in the \textbf{PushT} domain. 
Here, STAC is the only method to achieve a balanced accuracy of over 90\%, indicating that temporal consistency (or lack thereof) is strongly correlated with success (or failure).
Alternative output metrics, such as the DP's output variance, do not perform well because both successes and failures can exhibit high-variance outputs in multimodal domains.
Interestingly, the embedding similarity approach performs strongly, which indicates that state dissimilarity \textit{w.r.t.} the calibration dataset happens to be correlated with failure in this domain.

\begin{figure}
  \includegraphics[width=0.8\textwidth]{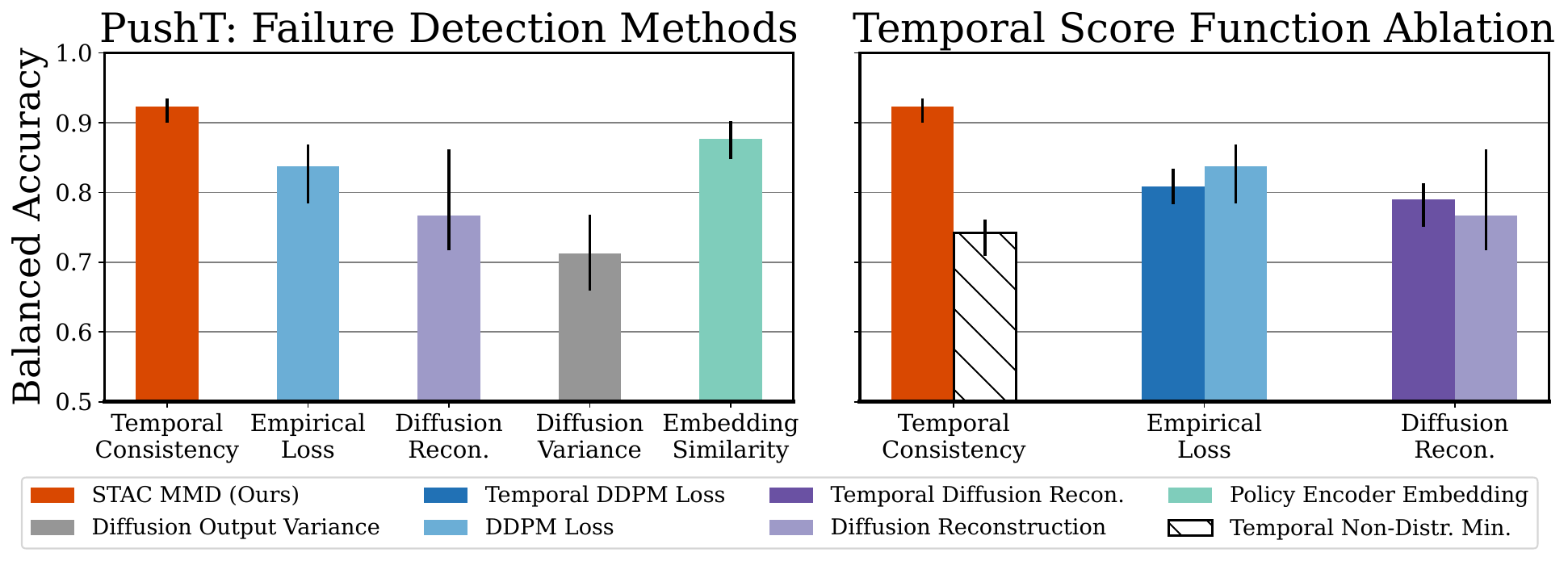}
  \caption[Failure detection results on simulated PushT diffusion policy (Sentinel)]{\small
        \textbf{Detecting failures in PushT.}
        Left: Our failure detector (\otext{STAC}) which measures the temporal consistency of a generative policy outperforms several families of out-of-distribution detectors. 
        Right: The best performance comes from measuring temporal consistency with statistical distance functions; augmenting baselines with temporal consistency does not always increase their performance.
    }
    \label{fig:sentinel-pusht-result}
\end{figure}

\textbf{Statistical measures of action similarity enable temporal consistency detection.}
\cref{fig:sentinel-pusht-result} (Right) ablates the design decisions of STAC.
First, we observe that augmenting baselines with temporal consistency will at most marginally increase their performance. 
Second, using a non-statistical distance function (e.g., min. distance) to measure temporal action consistency performs worse than the baselines because it omits action multimodality.
Thus, it is the combination of statistical distance functions with temporal consistency that yields the best result. 
We refer to \cref{appx:sentinel-stac-ablations} for an extended ablation of STAC.

\begin{table*}[t]
    \centering
    \caption[Failure detection results on simulated box closing diffusion policy (Sentinel)]{\small
        \textbf{Detecting erratic policy failures in the Close Box domain.} Results are averaged over 3 random seeds. Our temporal consistency detector, STAC, accounts for when a policy fails (\btext{high true positive rate}) and when it generalizes to out-of-distribution test cases (\btext{high true negative rate}), in contrast to embedding-based methods that associate state atypicality with policy failure (\rtext{low true negative rate}). Select baselines that \ktext{accurately detect} erratic policy failures in this domain experience a decrease in performance under multimodal conditions (i.e., PushT, as shown in \cref{fig:sentinel-pusht-result}), whereas STAC continues to exhibit \gtext{strong performance} across multiple domains. VLMs \vtext{must reason} over video to attain high true negative rates, as is necessary to combine them with STAC (see \cref{fig:sentinel-full-system-result}).
        Sentinel, which runs STAC and the VLM monitor in parallel, detects 100\% of erratic policy failures in this domain.
    } 
    \label{tab:sentinel-erratic-failures}
    \adjustbox{max width=\textwidth}{\begin{tabular}{clcccccccc|cccc}
        \toprule
        & \textbf{Category 1: Erratic Failures} & \multicolumn{3}{c}{\textbf{Close Box: In-Distribution}} & & \multicolumn{3}{c}{\textbf{Close Box: Out-of-Distribution}} & & & \multicolumn{3}{c}{\textbf{Close Box: Combined}} \\
        & & \multicolumn{3}{c}{(Policy Success Rate: 91\%)} & & \multicolumn{3}{c}{(Policy Success Rate: 41\%)} & & & \multicolumn{3}{c}{(Policy Success Rate: 67\%)} \\
        \cmidrule{3-5} \cmidrule{7-9} \cmidrule{12-14}
        & \textbf{Failure Detector} & TPR $\uparrow$ & TNR $\uparrow$ & Det. Time (s) $\downarrow$ & & TPR $\uparrow$ & TNR $\uparrow$ & Det. Time (s) $\downarrow$ & & & TPR $\uparrow$ & TNR $\uparrow$ & Accuracy $\uparrow$ \\
        \midrule
        \multirow{6}{*}{\STAB{\rotatebox[origin=c]{90}{\small \textbf{Diffusion}}}}
            & Temporal Non-Distr. Min. & 1.00 & 0.97 & 5.00 & & 1.00 & 0.27 & 12.35 & & & 1.00 & 0.77 & 0.85 \\
            & Diffusion Recon.~\cite{graham2023denoising} & 0.33 & 0.95 & 13.60 & & 0.40 & 1.00 & 17.08 & & & 0.37 & 0.96 & 0.76 \\
            & Temporal Diffusion Recon. & 1.00 & 0.96 & 8.47 & & 0.92 & 1.00 & 15.75 & & & \vcl{0.92} & \vcl{0.97} & \vcl{\underline{0.95}} \\
            & DDPM Loss (\cref{eq:sentinel-ddpm-loss}) & 1.00 & 0.90 & 8.27 & & 1.00 & 0.94 & 14.54 & & & \vcl{1.00} & \vcl{0.91} & \vcl{0.94} \\
            & Temporal DDPM Loss & 1.00 & 0.95 & 7.53 & & 1.00 & 0.37 & 13.66 & & & 1.00 & 0.79 & 0.86 \\
            & Diffusion Output Variance & 0.33 & 0.94 & 14.00 & & 0.28 & 1.00 & 17.27 & & & 0.26 & 0.96 & 0.72 \\
        \midrule
        \multirow{3}{*}{\STAB{\rotatebox[origin=c]{90}{\footnotesize \textbf{Embed.}}}}
            & Policy Encoder & 0.25 & 0.98 & 16.27 & & \rcl{1.00} & \rcl{0.00} & \rcl{1.59} & & & 0.94 & 0.70 & 0.78 \\
            & CLIP Pretrained & 1.00 & 0.95 & 15.73 & & \rcl{1.00} & \rcl{0.00} & \rcl{8.20} & & & 1.00 & 0.68 & 0.79 \\
            & ResNet Pretrained & 1.00 & 0.95 & 17.87 & & \rcl{1.00} & \rcl{0.00} & \rcl{15.51} & & & 1.00 & 0.68 & 0.79 \\
        \midrule
        \multirow{3}{*}{\STAB{\rotatebox[origin=c]{90}{\small \textbf{STAC}}}}
            & STAC For. KL (Ours) & 1.00 & 0.90 & 6.60 & & \bcl{0.99} & \bcl{0.85} & \bcl{14.04} & & & 0.99 & 0.89 & 0.92 \\
            & STAC Rev. KL (Ours) & 1.00 & 0.95 & 7.60 & & \bcl{0.93} & \bcl{0.97} & \bcl{15.12} & & & \gcl{0.93} & \gcl{0.96} & \gcl{\underline{0.95}} \\
            & \textbf{STAC MMD*} (Ours) & 1.00 & 0.94 & 7.20 & & \bcl{0.99} & \bcl{0.93} & \bcl{14.72} & & & \gcl{0.99} & \gcl{0.94} & \gcl{\textbf{0.96}} \\
        \midrule
        \multirow{2}{*}{\STAB{\rotatebox[origin=c]{90}{\small \textbf{VLM}}}}
            & GPT-4o Image QA & 1.00 & 0.00 & 23.20 & & 1.00 & 0.00 & 23.20 & & & \kcl{1.00} & \kcl{0.00} & \kcl{0.29} \\
            & \textbf{GPT-4o Video QA*} (Ours) & 1.00 & 0.89 & 21.20 & & 0.69 & 0.95 & 21.02 & & & \kcl{0.77} & \kcl{0.91} & \kcl{0.87} \\        
        \midrule
        \midrule
        \multicolumn{2}{l}{\textbf{Sentinel (STAC MMD* + GPT-4o Video QA*)}} & 1.00 & 0.86 & 5.47 & & 1.00 & 0.90 & 14.25 & & & 1.00 & 0.87 & 0.91 \\
        \bottomrule
    \end{tabular}}
\end{table*}

\textbf{STAC accounts for OOD failures and generalization.} 
Results on the \textbf{Close Box} domains are shown in \cref{tab:sentinel-erratic-failures}.
STAC attains the highest accuracy in aggregate (96\%). 
However, two of our newly proposed baselines---using the DDPM loss (\cref{eq:sentinel-ddpm-loss}) and a temporal reconstruction variant of \cite{graham2023denoising}---also perform well, perhaps due to a decrease in action multimodality relative to \textbf{PushT}.
Notably, we find that embedding similarity methods conflate OOD states with policy failure, resulting in false positives when the policy succeeds OOD.
In contrast, STAC effectively differentiates OOD successes from failures.

\textbf{VLMs must reason across time.}
In \cref{tab:sentinel-erratic-failures}, we find that a state-of-the-art VLM (\texttt{GPT-4o}) struggles to identify task success when given only a single image. Instead, it must observe the robot over the extent of a policy rollout to more accurately reason about task progression and changes in state (resulting in a $91\%$ TNR). 
While erratic failures are time-sensitive, they are visually more subtle and thus difficult for the VLM to detect (77\% TPR). The robot takes more obviously wrong actions (e.g., stalling, drifting astray) in task progression failures (\cref{fig:sentinel-full-system-result}). As expected, the VLM has a significantly slower detection time relative to STAC, further highlighting STAC's value at quickly detecting erratic behavior.

\begin{wrapfigure}{rt}{0.44\textwidth}
  \centering
  \vspace{-14pt}
  \includegraphics[width=1.0\textwidth]{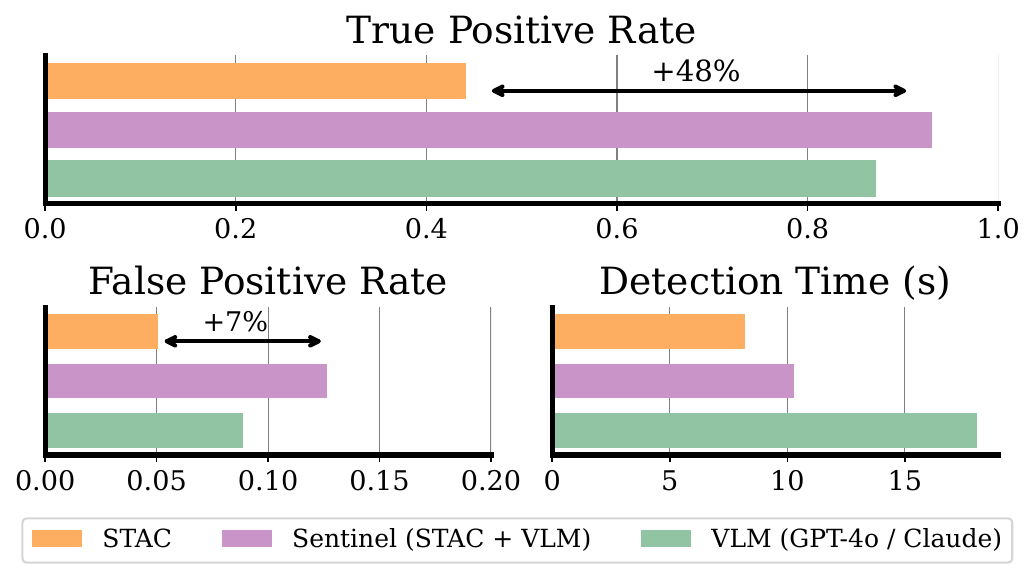}
  \caption[Failure detection results on simulated object covering diffusion policy (Sentinel)]{\small
      \textbf{Detecting task progression failures.} Combining VLMs with STAC yields an accurate overall detector (\vtext{Sentinel}) for both task progression and erratic failures (\cref{tab:sentinel-erratic-failures}).
      See \cref{appx:sentinel-extended-vlm-results} for extended results and analysis.
  }
  \label{fig:sentinel-full-system-result}
\end{wrapfigure}

\textbf{Sentinel: Combining STAC and VLMs for system-level performance increases.}
We evaluate our failure detectors on distribution shifts that primarily lead to task progression failures in the \textbf{Cover Object} and \textbf{Close Box} domains.
The result is shown in \cref{fig:sentinel-full-system-result}. 
STAC achieves a low TPR (44\%) when the policy fails in a temporally consistent manner, whereas the VLM (\texttt{GPT-4o} for \textbf{Close Box}, \texttt{Claude} for \textbf{Cover Object}) accurately detects task progression failures. 
As a result, their combination (Sentinel) achieves a 93\% TPR whilst incurring only a 7\% increase in FPR.
The rise in detection time indicates that both STAC (fast) and the VLM (slow) are contributing to the detection of failures. 

\textbf{Sentinel detects real-world, generative policy failures.}
We evaluate Sentinel on the \textbf{Push Chair} task across 10 successful and 10 failed policy rollouts.
The results are shown in \cref{tab:sentinel-push-chair}.
When calibrated on only 10 successful in-distribution rollouts, STAC detects 80\% of policy failures and raises only one false alarm (90\% TNR).
The VLM exhibits stronger performance in the real world (90\% TPR, 100\% TNR) than it does in the simulation domains, perhaps because real-world images constitute a lesser domain gap for visual reasoning compared to images rendered in simulation. 
Overall, Sentinel achieves a 95\% detection accuracy, highlighting its efficacy for detecting failures in real-world robotic settings.

\begin{wraptable}{r}{0.44\textwidth}
    \vspace{-10pt}
    \centering
    \caption[Failure detection results on real-world chair tucking diffusion policy (Sentinel)]{\small
        \textbf{Detecting real-world failures.} Sentinel demonstrates strong failure detection performance on the real-world Push Chair task, achieving an overall accuracy of 95\%.
    }
    \label{tab:sentinel-push-chair}
    \adjustbox{max width=\textwidth}{
        \begin{tabular}{lccc}
            \toprule
            \textbf{Failure Detector} & \textbf{TPR $\uparrow$} & \textbf{TNR $\uparrow$} & \textbf{Det. Time (s) $\downarrow$} \\
            \midrule
            Diffusion Output Variance & 0.60 & 0.90 & 10.67 \\
            Temporal Non-Distr. Min. & 0.70 & 0.80 & 9.52 \\
            \textbf{STAC Rev. KL} (Ours) & 0.80 & 0.90 & 9.83 \\
            \textbf{GPT-4o Video QA} (Ours) & 0.90 & 1.00 & 12.89 \\
            \midrule
            \midrule
            \rowcolor{green!13}
            \textbf{Sentinel (STAC + GPT-4o)} & 1.00 & 0.90 & 9.60 \\
            \bottomrule
        \end{tabular}
    }
\end{wraptable}

\textbf{Discussion.}
Holistic analysis of \cref{tab:sentinel-erratic-failures}, \cref{fig:sentinel-full-system-result}, \edit{and \cref{tab:sentinel-push-chair}} shows that we can easily combine STAC and the VLM to yield a performant overall detector for both erratic and task progression failures, particularly because both detectors achieve a high overall TNR.
Since all the baselines may 1) show low accuracy on either of the erratic failure domains (i.e., the multimodal \textbf{PushT} and \textbf{Close Box} domains) or 2) yield a low TNR, it is unclear how to combine them with other detectors in a way that outperforms Sentinel. 

\section{Conclusion}\label{sec:sentinel-conclusion}

In this chapter, we investigate the problem of failure detection for generative robot policies.
We propose Sentinel, a runtime monitor that splits the failure detection task into two categories: 1) Erratic failures, which we detect by measuring the statistical change of a policy's action distributions over time; 2) task progression failures, where we use Vision-Language Models to assess whether the policy is consistently taking actions that do not solve the task. 
Our results highlight the importance of targeting complementary failure categories with specialized detectors.
Future work includes the use of Sentinel to monitor high-capacity policies~\cite{octo_2023,pmlr-v270-kim25c}, inform data collection, and accelerate policy learning.

\section{Limitations and Future Work}\label{sec:sentinel-limitations}
We summarize the limitations of our approach and opportunities for future investigation:
\begin{enumerate}
    \item \textbf{Failure categories.} While categorizing erratic and task progression failures leads to accurate detection of failures across the domains we consider, these two failure categories may not be exhaustive. In the future, introducing additional categories or further partitioning existing ones might provide a broader coverage of failures, allow for more efficient failure detection, and inform mitigation strategies.
    \item \textbf{Detection guarantees.} Furthermore, our approach does not provide formal guarantees on detecting failures. However, providing such guarantees would require data of both successful and unsuccessful policy rollouts to calibrate the detector~\cite{LuoZhaoSample2023}.
    \item \textbf{Combining detectors.} Although our detectors attain low false positive rates in aggregate, taking the union of their predictions may, in the worst case, increase the risk of false alarms. Thus, exploring more sophisticated ways to combine complementary failure detectors is a possible point of extension.
    \item \textbf{Failure prediction.} Finally, our approach is not targeted at predicting failures before they occur, but instead focuses on detecting failures as they occur.
\end{enumerate}

\chapter{Understanding Policy Behavior through Training Data}\label{chapter:cupid}

\section{Introduction}\label{sec:cupid-introduction}

In \cref{chapter:sentinel}, we examined how policy reliability can be improved through deployment-time failure detection, enabling systems to identify when learned behaviors are likely to break down. While such mechanisms are essential for mitigating unsafe outcomes, they offer limited insight into \textit{why} failures occur. This chapter addresses this limitation from a data-centric perspective, motivated by the observation that a policy’s deployment-time behavior is fundamentally shaped by the quality and composition of its training data---yet this relationship often remains opaque. Accordingly, we seek to elucidate how individual training demonstrations contribute to downstream outcomes, such as closed-loop task success or failure, thereby providing principled insight into the underlying causes of policy behavior.

More broadly, while some of the largest breakthroughs in deep learning have emerged from architectural innovations, data often remains an underrecognized yet critical driver of a model's overall performance.
The success of scaling vision and language models has been followed by a rising interest in data attribution~\cite{grosse2023studying, park2023trak, engstrom2024dsdm}---methods that causally link model behavior to training data---and in automatic data curation algorithms~\cite{lee2021deduplicating, tirumala2023d4, albalak2024survey}, grounded in the idea that not all data points contribute equally, or even positively, to a model's performance. As the robotics community continues to scale imitation learning and robotics datasets become increasingly diverse~\cite{o2024open, khazatsky2024droid}, developing a deeper understanding of (i) how demonstration data shapes policy behavior and (ii) how we can extract maximum utility from training datasets will be imperative to advancing policy performance toward reliable, open-world deployment.

Curating data for robot imitation learning has been the focus of several recent works~\cite{kuhar2023learning, hejna2025robotdatacurationmutual, chen2025curating}.
A common approach retains demonstrations deemed most valuable under a heuristic, \textit{task-agnostic quality} metric, resulting in a smaller dataset curated offline~\cite{hejna2025robotdatacurationmutual}. 
This approach typically rests on the implicit assumption that the designed quality metric aligns well with the policy’s downstream performance---an assumption that may not hold uniformly across diverse robotics tasks.
While recent efforts attempt to learn \textit{performance-correlated} heuristics using online policy experience~\cite{chen2025curating}, they do not establish strong causal links between training data and policy behavior. 
As a result, these methods risk misattributing the root cause of policy success or failure with respect to the training data~\cite{de2019causal}.

\begin{figure}[t]
    \includegraphics[width=1.0\linewidth]{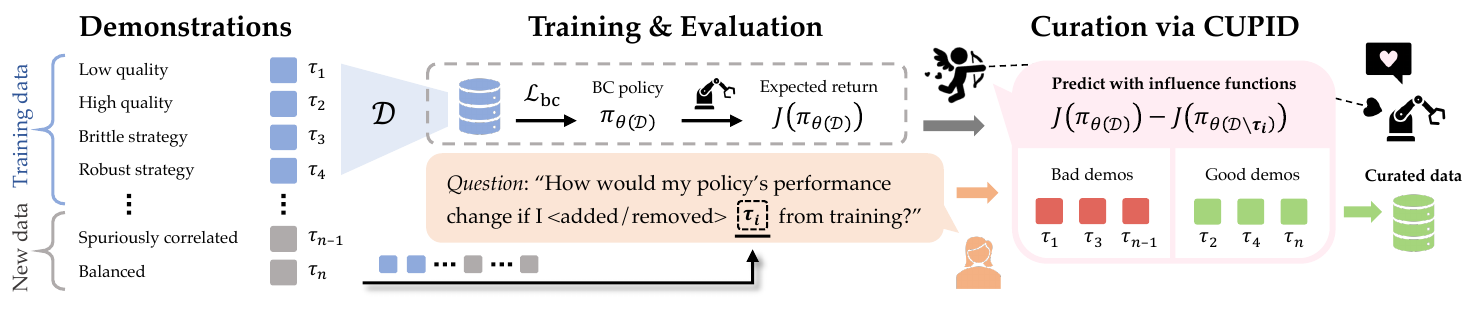}
    \caption[Overview of CUPID]{\small
        We present \textbf{CUPID}, a robot data curation method that leverages influence functions to predictively answer counterfactual questions about the effect of each demonstration on downstream policy performance.
        More details can be found on the CUPID website: \url{https://cupid-curation.github.io}.
    }
    \label{fig:cupid-teaser}
\end{figure}

In this chapter, we formally define data curation in imitation learning as the problem of identifying which expert demonstrations maximally contribute to the policy’s expected return. 
We then introduce \basemethod{} (\textsc{Cu}rating \textsc{P}erformance-\textsc{I}nfluencing \textsc{D}emonstrations), a data curation method that directly targets this objective by leveraging influence functions~\cite{koh2017understanding, koh2019accuracy}---a technique popularized in the data attribution literature~\cite{hammoudeh2024training}---to identify which demonstrations influenced a policy’s predictions during closed-loop execution.
We show that a demonstration’s influence on expected return decomposes into a tractable sum over its state-action transitions and can be efficiently approximated using a \texttt{REINFORCE}-style estimator~\cite{sutton1999policy} given a set of policy rollouts. 
Ranking demonstrations by their estimated performance impact facilitates curation in two settings: (a) filtering existing demonstrations from training sets and (b) selecting high-impact demonstrations from newly collected or pre-collected data---whereas prior work focuses solely on filtering~\cite{hejna2025robotdatacurationmutual, chen2025curating}. 
Finally, while our approach offers a general and effective standalone signal for curating demonstration data, we investigate its combined use with task-agnostic quality metrics (also derived from influence scores), identifying conditions under which the integration of performance- and quality-based metrics strengthens or weakens overall curation performance.

Our contributions are three-fold: (1) We formulate robot data curation as the problem of valuating demonstrations in accordance with their downstream impact on policy performance; (2) We propose \basemethod{}, a novel approach for curating imitation learning datasets based on influence functions, causally linking demonstrations to the policy’s expected return; (3) We characterize the conditions under which the integration of task-agnostic quality metrics strengthens performance-based data curation, providing practical insights into when such integration is beneficial.
Extensive simulation and hardware experiments show that curation with \basemethod{} significantly improves policy performance in mixed-quality regimes, even when using only a fraction of the training data.
Moreover, it identifies robust strategies under test-time distribution shifts and can disentangle spurious correlations in training data that hinder generalization---all by observing policy outcomes alone, without requiring additional supervision.

\section{Related Work}\label{sec:cupid-related-work}

\textbf{Data Curation in Robotics.} 
Assembling larger and more diverse datasets has been central to scaling efforts in robot imitation learning~\cite{o2024open, khazatsky2024droid, rt12022arxiv, rt22023arxiv, octo_2023, pmlr-v270-kim25c, black2410pi0}, yet how to extract greater utility from these datasets remains an open question.
Several works have explored data augmentation~\cite{mandlekar2023mimicgen, yu2023scaling, mandi2022cacti, smith2024steer, pmlr-v270-zawalski25a} and mixture optimization~\cite{pmlr-v270-hejna25a}.
Only recently has attention shifted to valuating individual demonstrations for data curation~\cite{kuhar2023learning, hejna2025robotdatacurationmutual, chen2025curating}.
\citet{hejna2025robotdatacurationmutual} estimate demonstration quality offline via mutual information---without considering policy performance---while \citet{chen2025curating} train classifiers to distinguish successful and failed rollouts across policy checkpoints. In contrast, we directly measure the causal influence of each demonstration on the policy's expected return, providing a signal that (a) does not require observing both successes and failures, (b) uses only a single policy checkpoint, (c) is robust to spurious correlations in the policy's rollout distribution, and (d) naturally extends to selecting new data, whereas~\cite{hejna2025robotdatacurationmutual, chen2025curating} only filter existing data.
Concurrent work includes DataMIL~\cite{dass2025datamil}, which uses datamodels to select from large multi-task datasets with an offline metric, whereas we focus on single-task curation with an influence measure that directly reflects closed-loop returns from online policy rollouts.

\textbf{Data Attribution outside Robotics.}
Data attribution methods model the relationship between training data and learned behavior, with applications in model interpretability~\cite{park2023trak, shah2023modeldiff}, data valuation~\cite{ghorbani2019data, choe2024your}, machine unlearning~\cite{georgiev2024attribute}, and more~\cite{madry2024icml}.
Recent work has focused on improving the accuracy of data attribution methods~\cite{basu2021influence, bae2022if, ilyas2025magic}, such as influence functions~\cite{koh2017understanding, koh2019accuracy}, and extending them to increasingly complex generative architectures~\cite{grosse2023studying, zheng2023intriguing, georgiev2023journey}.
A related line of research explores improving language model pre-training~\cite{engstrom2024dsdm} and fine-tuning~\cite{xia2024less, liu2024tsds, engstrom2025optimizing} through data selection.
However, these settings typically assume aligned training and evaluation objectives (i.e., prediction loss) and access to test-time labels.
In contrast, robot imitation learning involves an objective mismatch: policies are trained via supervised learning but evaluated through closed-loop environment interactions, where task success depends on many sequential predictions and ground-truth action labels are unavailable at test-time.

\section{Background: Data Attribution via Influence Functions}\label{sec:cupid-background}
At a high-level, \textbf{the goal of data attribution} methodologies is to explicitly relate model performance and behavior to the training data, so that we can answer \emph{counterfactual} questions about the contribution of training samples towards test-time predictions. 
Consider a standard supervised learning setting, where we fit model parameters $\theta$ on a given training dataset $\mathcal{D} := \{z^1, \ldots, z^n\}$ of input-label pairs $z^i=(x^i, y^i) \in \mathcal{Z}$ with $\theta(\mathcal{D}) =  \arg \min_{\theta'} \{\calL(\theta'; \calD) := \frac{1}{n} \sum_{i=1}^n \ell(z^i; \theta')\}$. Moreover, let $f(\hat{z}; \theta) \in \R$ be any chosen performance metric on a test sample $\hat{z} = (\hat{x}, \hat{y}) \in \mathcal{Z}$ given model parameters $\theta$ (e.g., cross-entropy loss for a classifier). Then, a data attribution method $\Psi^{\mathrm{out}}: \mathcal{Z} \times \mathcal{Z} \rightarrow \mathbb{R}$ aims to approximate the change in the performance metric $f$ if we were to exclude sample $z^i$ from the model's training data. That is, we aim to design $\Psi^{\mathrm{out}}$ such that $\Psi^{\mathrm{out}}(\hat{z}, z^i) \approx  f\left(\hat{z}; \theta(\mathcal{D} \setminus \{z^i\})\right) - f(\hat{z}; \theta(\mathcal{D}))$.

\textbf{The influence function} is a data attribution technique that approximates $\Psi^{\mathrm{out}}$ \emph{without} retraining any models~\cite{hammoudeh2024training}. 
Consider perturbing the training objective as $\calL_{\epsilon, z}(\theta'; \calD) := \calL(\theta'; \calD) + \epsilon \ell(z, \theta'),$ where we add an infinitesimal weight $\epsilon$ on the loss of some sample $z$ to $\calL$. The \emph{influence function} estimates the change in the performance metric $f$ as a function of $\epsilon$ with a first-order Taylor approximation as
\begin{align}
    \Psi_{\text{inf}}(\hat{z}, z)  &:= \frac{df(\hat{z}; \theta)}{d \epsilon}\bigg|_{\epsilon=0} = -\nabla_\theta f(\hat{z}; \theta(\mathcal{D}))^\top  H_{\theta}^{-1}  \nabla_\theta \ell(z; \theta(\mathcal{D})),
        \label{eq:cupid-influence-fn}
\end{align}
where $H_{\theta} = \frac{1}{n} \sum_{i=1}^{n} \nabla_\theta^2 \ell(z^i; \theta(\mathcal{D}))$ denotes the Hessian of the training loss\footnote{\label{fn:cupid-track} To reduce the computational cost of \cref{eq:cupid-influence-fn}, we use TRAK~\cite{park2023trak}, which leverages random projections and a Gauss-Newton Hessian approximation for efficient influence estimation. This also makes the influence function amenable to the non-smooth, non-convex loss functions in practical deep learning problems, so we assume \cref{eq:cupid-influence-fn} is well-defined throughout this chapter.}~\cite{koh2017understanding}. Therefore, we can use the influence function to directly approximate the \emph{leave-one-out} influence $\Psi^{\mathrm{out}}$ of a sample $z^i\in\calD$ as $\Psi^{\mathrm{out}}_{\mathrm{inf}}(\hat{z}, z^i) := -\frac{1}{n}\Psi_{\mathrm{inf}}(\hat{z}, z^i)$. In addition, for $z \not \in \calD$ we similarly define the \emph{add-one-in} influence as $\Psi^{\mathrm{in}}_{\mathrm{inf}}(\hat{z}, z):= \frac{1}{n} \Psi_{\mathrm{inf}}(\hat{z}, z) \approx f(\hat{z}; \theta(\calD \cup \{z\})) - f(\hat{z}; \theta(\calD))$ with $z$ excluded from the Hessian $H_{\theta}$.

\section{Problem Formulation}\label{sec:cupid-formulation}

\textbf{Imitation Learning (IL):}
Our objective is to understand how demonstration data contributes to closed-loop performance in robot imitation learning. Thus, we consider a Markov Decision Process $\langle \mathcal{S}, \mathcal{A}, \mathcal{T}, R, \rho_0 \rangle$ with state space~$\mathcal{S}$, action space~$\mathcal{A}$, transition model~$\mathcal{T}$, reward model~$R$, initial state distribution~$\rho_0$, and finite horizon~$H$. We train a policy~$\pi_\theta$ to minimize a behavior cloning (BC) objective, i.e., $\theta = \arg\min_{\theta'} \{\mathcal{L}_{\text{bc}}(\theta'; \mathcal{D}) := \frac{1}{|\calD| H}\sum_{\xi^i\in\calD}\sum_{(s, a) \in \xi^i} \ell(s, a; \pi_{\theta'})\}$, using a dataset of $n$ expert demonstrations $\mathcal{D} = \{\xi^1, \ldots, \xi^n\}$. Each demonstration $\xi^i = ((s^i_0, a_0^i), \ldots, (s_H^i, a_H^i))$ consists of a state-action trajectory where the robot successfully completes the task. 
We treat a trajectory $\tau = (s_0, a_0, \ldots, s_{H})$ as either a \emph{success} or a \emph{failure}, corresponding to the binary returns $R(\tau) = 1$ and $R(\tau) = -1$ respectively. 

Therefore, in IL, we train the policy $\pi_\theta$ to match the distribution of successful behaviors in $\calD$, rather than directly maximize its expected return $J(\pi_\theta) := \E_{p(\tau|\pi_\theta)}[R(\tau)]$.
As a result, the policy's performance is intimately linked to the relative suboptimality of the demonstration data---a function of its quality and composition---not just to validation losses, model capacity, or bias-variance tradeoffs. This makes it extremely challenging to systematically improve performance. 
Recent works underscore that simply scaling demonstration collection may result in datasets that contain substantial redundancies and behaviors that may actually harm policy performance, even though $R(\xi^i) = 1$ for all demonstrations $\xi^i \in \calD$~\cite{belkhale2023data}.

\textbf{Robot Data Curation:} 
While several recent works propose intuitive measures of quality to curate data, we find that such heuristics can misalign with how deep models actually learn, sometimes even worsening test-time performance compared to randomly choosing samples
(see \cref{sec:cupid-experiments}). Therefore, we first formally define robot data curation as the problem of identifying demonstration data that maximizes the policy's closed-loop performance.
In particular, assume that we have a \emph{base policy} $\pi_{\theta}$ trained on the demonstration data $\calD$. We consider two settings that are essential to a policy debugging toolchain. The first is that of \emph{data filtering}, where our goal is to identify and remove redundant or harmful demonstrations from $\calD$ that may be limiting the performance of the base policy $\pi_\theta$.

\begin{task}[Filter-$k$ demonstrations]\label{task:filter-k}
    Let $\Xi^-_k = \left\{ S \subseteq \mathcal{D} \;\middle|\; |S| = k \right\}$ denote all possible $k$-demonstration subsets of the training dataset $\mathcal{D} = \{\xi^1, \ldots, \xi^n\}$, where $k \leq n$. Determine which $k$ demonstrations should be removed from $\mathcal{D}$ to maximize policy performance with respect to the task objective $J$. That is, find 
    \begin{align*}
        S^\star =\;& \arg \max _{S \in \Xi^-_k} J(\pi_{\theta}) \quad \mathrm{s.t.} \quad \theta = \arg \min_{\theta'} \ \mathcal{L}_{\mathrm{bc}}(\theta'; \mathcal{D} \setminus S).
    \end{align*}
\end{task}
The second is that of \emph{data selection}, where we seek to guide the subselection of new demonstration data to maximally improve our base policy, given a fixed budget.
\begin{task}[Select-$k$ demonstrations]\label{task:select-k}
    Let $\Xi^+_k = \left\{ S \subseteq \mathcal{H} \;\middle|\; |S| = k \right\}$ denote all possible $k$-demonstration subsets of a holdout dataset $\mathcal{H} = \{\xi^1, \ldots, \xi^{n'}\}$, where $k \leq n'$. Determine which $k$ demonstrations should be added to $\mathcal{D}$ from $\mathcal{H}$ to maximize policy performance with respect to the task objective $J$. That is, find 
    \begin{align*}
        S^\star =\;& \arg \max_{S \in \Xi^+_k} J(\pi_{\theta}) \quad \mathrm{s.t.} \quad \theta = \arg \min_{\theta'} \ \mathcal{L}_{\mathrm{bc}}(\theta'; \mathcal{D} \cup S). 
    \end{align*} 
\end{task}
In \cref{task:select-k}, we consider the problem of identifying the most impactful trajectories from a newly collected batch of demonstrations or from an existing pre-collected dataset, akin to performing quality control.

\textbf{Policy Testing \& Evaluation:} To make progress on \cref{task:filter-k} and \cref{task:select-k}, we assume access to a small dataset of $m$ rollouts $\mathcal{D}_\tau = \{\tau^1, \ldots, \tau^m\} \iid p(\tau | \pi_\theta)$ of the base policy $\pi_\theta$ along with their associated returns $\{R(\tau^1), \ldots, R(\tau^m)\}$ to estimate $J(\pi_\theta)$. This aligns with how we currently evaluate policies in practice~\cite{vincent2024generalizable}, despite lacking principled strategies to leverage evaluations towards BC policy improvement. 

\begin{figure}[t]
    \centering
    \includegraphics[width=\linewidth]{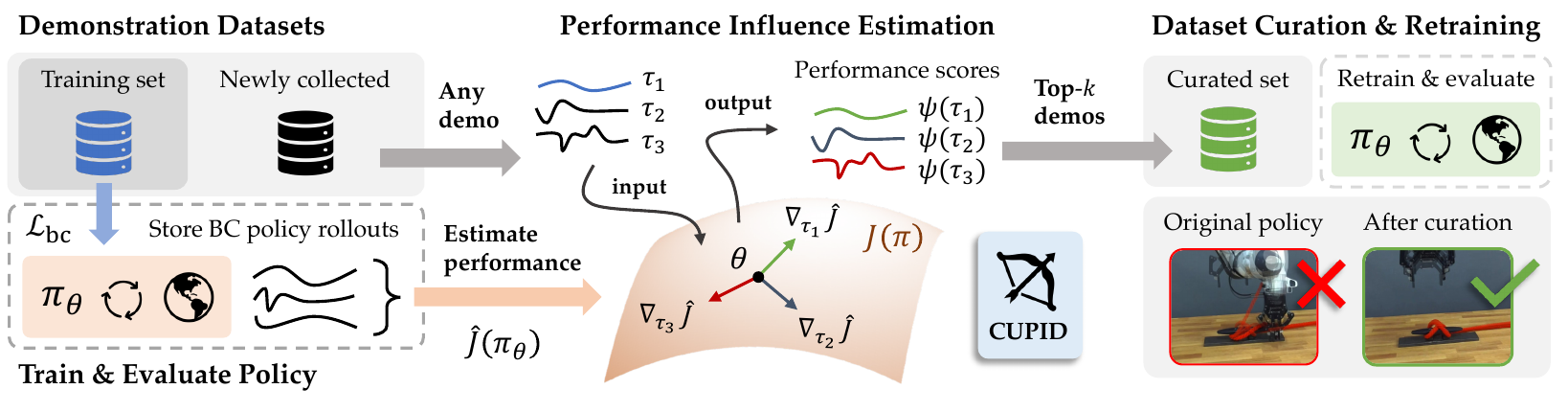}
    \caption[System architecture for data curation (CUPID)]{\small
        \textbf{Data curation with \basemethod{}.} Upon training a policy on a set of demonstrations using behavior cloning, we evaluate it online to collect closed-loop rollout trajectories and estimate the policy’s expected return. \basemethod{} ranks demonstration based on their measured influence on this performance estimate and selects the top-$k$. Thus, curating with \basemethod{} results in a dataset of demonstrations that most strongly influences closed-loop policy success.
    }
    \label{fig:cupid-method-overview}
\end{figure}

\section{CUPID: Curating Performance-Influencing Demonstrations}\label{sec:cupid-method}

In the literature, the desiderata for curating demonstration data appears diverse and often context specific, with recent works valuating demonstrations upon heuristic measures of similarity~\cite{chen2025curating}, compatibility~\cite{gandhi2023eliciting}, uncertainty~\cite{cui2019uncertainty}, and information gain~\cite{hejna2025robotdatacurationmutual}. \textbf{The key insight} of our approach is that solving curation problems, i.e., \cref{task:filter-k} and \cref{task:select-k} (\cref{sec:cupid-formulation}), requires causally connecting training data to the policy's closed-loop performance. 
Therefore, we first adapt techniques from data attribution, as defined in \cref{sec:cupid-background}, to directly compute the influence of a training demonstration on the performance of a policy. This allows us to use our \emph{performance influence} to directly curate data in alignment with our objectives.

\subsection{Demonstration-Performance Influence}\label{sec:cupid-performance-influence}

While existing data attribution methods can trace validation losses back to the training set $\calD$ for curation purposes, the BC loss is not always reflective of a policy's closed-loop performance~\cite{ross2011reduction}. Thus,
we must first develop an analogous notion of the influence function to capture the impact of a \emph{demonstration trajectory} on the \emph{closed-loop performance} of an imitation learning policy.
To do so, we group the BC training objective into trajectory-level losses by introducing  $\ell_{\mathrm{traj}}(\xi ; \pi_{\theta'}) := \frac{1}{H}\sum_{(s,a) \in \xi}\ell(s, a; \pi_{\theta'})$, so that $\calL_{\mathrm{bc}}(\theta'; \calD) = \frac{1}{|\calD|} \sum_{\xi^i \in \calD}\ell_{\mathrm{traj}}(\xi^i ; \pi_{\theta'})$. 
We now formally define the \emph{performance influence} of a demonstration as the application of the influence function (see \cref{eq:cupid-influence-fn}) on the policy's expected return:

\begin{definition}[Performance Influence]\label{def:polinf} Let $\xi$ be a demonstration of interest. Suppose we train a policy $\pi_\theta$ to minimize the perturbed BC objective $\calL_{\mathrm{bc}}^{\epsilon, \xi}(\theta' ; \calD) := \calL_{\mathrm{bc}}(\theta'; \calD) + \epsilon \ell_{\mathrm{traj}}(\xi; \pi_{\theta'})$. Then, demonstration $\xi$'s  \emph{\textbf{performance influence}} is the derivative of the policy's expected return $J(\pi_\theta)$ with respect to the weight $\epsilon$. That is,
\begin{align}\label{eq:cupid-polinf}
\polinf(\xi) &:= \frac{dJ(\pi_\theta)}{d\epsilon}\bigg|_{\epsilon=0} \nonumber = -\nabla_\theta J(\pi_\theta)^\top H_{\mathrm{bc}}^{-1} \nabla_\theta \ell_{\mathrm{traj}}(\xi; \pi_\theta),
\end{align}
where $H_{\mathrm{bc}} := \nabla^2_\theta\calL_{\mathrm{bc}}(\theta ; \calD)$ denotes the Hessian of the BC objective.
\end{definition}
In essence, \cref{def:polinf} enables us to predictively answer the counterfactual: ``How would the policy's expected return change if we upweighted---or by negating, downweighted---the loss on a demonstration $\xi$ during training?''
While \cref{def:polinf} neatly aligns with the standard definition of the influence function in \cref{eq:cupid-influence-fn}---using $J$ as the performance metric and $\ell_{\mathrm{traj}}$ as the demonstration-level loss---we distinguish the \emph{performance influence} from the standard influence function~\cite{koh2017understanding} for two key reasons: (1) The performance influence attributes the \emph{outcome} of a policy's sequential decisions to time-series demonstrations, whereas the existing techniques discussed in \cref{sec:cupid-background} only relate an individual labeled prediction to a single training sample; (2) We cannot directly compute $\polinf$ because the policy's expected return $J(\pi_\theta)$ depends on the unknown transition dynamics and reward function.
To alleviate these challenges, we show that we can decompose the \emph{performance influence} into influence scores of individual action predictions, which we define as the \textit{action influence}.
\begin{definition}[Action Influence]\label{def:actinf} The \emph{\textbf{action influence}} of a state-action pair $(s,a)$ on a test state-action pair $(s',a')$ is the influence of $(s,a)$ on the policy's log-likelihood $ \log \pi_\theta(a'|s')$. That is,
    \begin{equation}\label{eq:cupid-actinf}
    \actinf((s',a'), (s,a)) := -\nabla_\theta \log\pi_\theta(a'|s')^\top H_{\mathrm{bc}}^{-1}\nabla_\theta \ell(s, a; \pi_\theta).
\end{equation}
\end{definition}
The advantage of the \emph{action influence} is that we can easily compute the quantities in \cref{eq:cupid-actinf} given the policy weights $\theta$ and the training demonstrations $\calD$, e.g., using the attribution methods discussed in \cref{sec:cupid-background}.
However, we emphasize that computing \emph{action influences} over state-action samples from a policy rollout $\tau \sim p(\tau | \pi_\theta)$ only tells us what demonstration data led to the policy taking those actions, without ascribing value to the resulting outcome (e.g., success or failure).
We now show that the performance influence decomposes into the sum of individual action influences, weighted by the trajectory return $R(\tau)$.

\begin{proposition}\label{prop:cupid-polinf}
Assume that $\theta(\calD) = \arg\min_{\theta'} \calL_{\mathrm{bc}}(\theta'; \calD)$, that $\calL_{\mathrm{bc}}$ is twice differentiable in $\theta$, and that $H_{\mathrm{bc}} \succ 0$ is positive definite (i.e., $\theta(\calD)$ is not a saddle point)\footnoteref{fn:cupid-track}. Then, it holds that\footnote{Note that the fraction $1/H$ appears from the assumption that all trajectories have equal length, which we make purely for notational simplicity without loss of generality. We refer to \cref{appx:cupid-proof-length} for the variable length case.}
\begin{equation}\label{eq:cupid-policy-inf-deriv}
    \polinf(\xi) = \E_{\tau\sim p(\tau|\pi_\theta)}\bigg[\frac{R(\tau)}{H} \sum_{(s',a')\in\tau}\sum_{(s,a)\in\xi}\actinf\big((s',a'),(s,a)\big)\bigg].
\end{equation}
\end{proposition}
In brief, we prove \cref{prop:cupid-polinf} using the log-derivative trick underlying policy gradient methods~\cite{sutton1999policy, williams1992simple} to decompose $\polinf$ into $\actinf$ (see \cref{appx:cupid-proof} for proof).
Because \cref{prop:cupid-polinf} relates the performance influence to the average action influence that a demonstration $\xi$ has on the closed-loop distribution of policy rollouts, \cref{prop:cupid-polinf} directly provides a method to estimate $\polinf$: \newline
\textbf{Estimate $\polinf$:} First, evaluate the policy $\pi_\theta$ online to gather a set of rollouts $\mathcal{D}_\tau = \{\tau^1, \ldots, \tau^m\} \iid p(\tau | \pi_\theta)$ and their associated returns $\{R(\tau^1), \ldots, R(\tau^m)\}$. Then, construct an empirical estimate of the performance influence $\polinfest$ using \cref{eq:cupid-policy-inf-deriv}, by averaging action influences across the rollouts in $\calD_\tau$. 

We illustrate the performance influence in \cref{fig:cupid-method-overview} and summarize its estimation procedure in \cref{alg:cupid-polinf-est}.

\begin{algorithm}[H]
\small
\caption{Performance Influence}\label{alg:cupid-polinf-est}
\begin{algorithmic}[1]
\State \textbf{Input:} Policy $\pi_\theta$, training data $\calD$, demonstration $\xi$, data attribution method $\Psi$
\State Collect rollouts $\mathcal{D}_\tau = \{\tau^1, \ldots, \tau^m\} \iid p(\tau | \pi_\theta)$ and returns $\{R(\tau^1), \ldots, R(\tau^m)\}$
\State Use $\Psi, \calD$ to compute $\actinf((s',a'), (s,a))$ for all $(s',a') \in \calD_\tau$, $(s,a) \in \xi$
\State Estimate $\widehat{\Psi}_{\pi\text{-}\mathrm{inf}}(\xi) := \frac{1}{m}\sum_{\tau^i \in \calD_\tau} \frac{R(\tau^i)}{H} \sum_{(s',a') \in \tau^i} \sum_{(s,a) \in \xi} \actinf((s',a'), (s,a))$.
\State \textbf{Output:} Estimated performance influence $\widehat{\Psi}_{\pi\text{-}\mathrm{inf}}(\xi)$
\end{algorithmic}
\end{algorithm}

\subsection{Data Curation with Performance Influence}\label{sec:cupid-methods-curation}

In this section, we leverage the performance influence $\polinf,$ which we developed in \cref{sec:cupid-performance-influence}, to curate data towards the filtering and selection tasks (\cref{task:filter-k} and \cref{task:select-k}) defined in \cref{sec:cupid-formulation}. In particular, we use the estimates of $\polinf$ to make the following first-order Taylor approximations on the \emph{leave-one-out} and \emph{add-one-in} influence (as defined in \cref{sec:cupid-background}) of a demonstration trajectory as
\begin{minipage}{\linewidth}
\small
\begin{equation*}
    \polinf^{\mathrm{out}}(\xi) := -\frac{\polinfest(\xi)}{|\calD|} \approx J(\pi_{\theta(\calD \setminus \{\xi\})}) - J(\pi_{\theta(\calD)}), \quad
    \polinf^{\mathrm{in}}(\xi):= \frac{\polinfest(\xi)}{|\calD|} \approx  J(\pi_{\theta(\calD \cup \{\xi\})}) -J(\pi_{\theta(\calD)}).
\end{equation*}
\end{minipage} \\

Then, we use the \emph{leave-one-out} and \emph{add-one-in} influences to counterfactually estimate the change in expected return when removing or adding a set of demonstrations $S$ with a linear approximation as 
$\Delta \widehat{J}(\pi_{\theta(\calD \setminus S)}) \propto \frac{1}{|S|}\sum_{\xi\in S}\polinf^{\mathrm{out}}(\xi)$ and $\Delta \widehat{J}(\pi_{\theta(\calD \cup S)}) \propto \frac{1}{|S|}\sum_{\xi\in S}\polinf^{\mathrm{in}}(\xi)$. As a result, optimally curating data under our approximate linear model on policy performance simply entails selecting the least influential demonstrations from the training data $\calD$---in the case of data filtering---or selecting the most influential demonstrations from a new set of demonstrations $\calH$---in the case of data selection:

\begin{minipage}{0.47\textwidth}
  \textbf{\cref{task:filter-k}: Filter-$k$ Demonstrations}
  \begin{equation}\label{eq:cupid-prune-topk}
    S^\star_{\mathrm{out}} = \argtop\text{-}k\big(\{\polinf^{\mathrm{out}}(\xi^i) : \xi^i \in \mathcal{D}\}\big),
  \end{equation}
\end{minipage}
\hfill
\begin{minipage}{0.47\textwidth}
  \textbf{\cref{task:select-k}: Select-$k$ Demonstrations}
  \begin{equation}\label{eq:cupid-select-topk}
    S^\star_{\mathrm{in}} = \argtop\text{-}k\big(\{\polinf^{\mathrm{in}}(\xi^i) : \xi^i \in \mathcal{H}\}\big).
  \end{equation}
\end{minipage} \\

We note that by linearly approximating policy performance changes using $\polinf$, we construct what is commonly termed a (linear) \emph{datamodel}~\cite{ilyas2022datamodels}. 
As shown in NLP \cite{engstrom2024dsdm}, using such first-order approximations for data curation can often greatly improve model performance over manual notions of quality. 

\subsection{Additional Quality Metrics}\label{sec:cupid-methods-quality}

In \cref{sec:cupid-performance-influence}, we constructed a method to estimate $\polinf$ from a dataset of policy rollouts $\calD_\tau$ by relying on policy gradient methods.
Therefore, the estimated performance influence $\widehat{\Psi}_{\pi\text{-}\mathrm{inf}}$ becomes increasingly noisy as we reduce the number of rollouts $m$ to evaluate the policy---akin to the high variance problem of the \texttt{REINFORCE} algorithm. 
To complement the analysis in \cref{sec:cupid-performance-influence}, we explore the integration of a \emph{reward-agnostic, heuristic} demonstration quality metric based on the action influence scores $\actinf$: 
\begin{equation}\label{eq:cupid-quality-influence-fn}\small
    \Psi_{\text{qual}}(\xi; \mathcal{D}_\tau) := \frac{1}{m}\sum_{\tau \in \mathcal{D}_\tau} \max_{\tiny{(s',a') \in \tau}} \min_{(\tiny{s,a) \in \xi}} \actinf\big((s',a'), (s,a)\big) - \min_{\tiny{(s',a') \in \tau}} \max_{\tiny{(s,a) \in \xi}} \actinf\big((s',a'), (s,a)\big).
\end{equation}
We base the quality score \cref{eq:cupid-quality-influence-fn} on the intuition that we should penalize demonstrations containing outlier or noisy influence scores \cite[Sec. 5.2]{koh2017understanding}, \cite{hejna2025robotdatacurationmutual}.
As such, we posit that this heuristic can reduce variance on tasks requiring precise motion, yet introduce bias uncorrelated with performance in other settings.
Thus, in \cref{sec:cupid-experiments}, we investigate when the quality score can complement $\polinf$ to curate data by taking their convex combination, $\alpha \polinf + (1-\alpha)\Psi_{\mathrm{qual}}$, ablating $\alpha=1$ (\basemethod{}) and $\alpha = 1/2$ (\qualitymethod{}).

\section{Experiments}\label{sec:cupid-experiments}

We conduct a series of experiments to test the efficacy of \basemethod{} alongside state-of-the-art baselines for robot data curation. These experiments take place across three simulated tasks from the RoboMimic benchmark suite~\cite{pmlr-v164-mandlekar22a} and three real-world tasks with a Franka FR3 manipulator (see \cref{fig:cupid-franka-dp-results}). These tasks comprise a taxonomy of settings where data curation may benefit policy performance. For a detailed description of our tasks, datasets, baselines, evaluation protocol, and hardware setup, please refer to \cref{appx:cupid-experiments} 

\textbf{Evaluation.} We study the filter-$k$ (\cref{task:filter-k}) and select-$k$ (\cref{task:select-k}) curation tasks wherever applicable. For statistical significance, we start filter-$k$ and select-$k$ from random $\sim2/3$ and $\sim1/3$ subsets in RoboMimic (300 demonstrations per task total), and random $\sim9/10$ and $\sim4/10$ subsets on Franka tasks (120-160 demonstrations per task total), respectively. We use the official convolutional-based diffusion policy implementation~\cite{chi2023diffusion} for all tasks to measure the effect of curation on a state-of-the-art policy architecture. 
Details on the influence function computation for diffusion models are provided in \cref{appx:cupid-method}.
We also consider the official $\pi_0$ implementation~\cite{black2410pi0} for real-world tasks. To reflect practical constraints, we limit the rollout budget (i.e., the number of rollouts in $\mathcal{D}_\tau = \{\tau^i\}_{i=1}^m$ a curation algorithm may use, as described in \cref{sec:cupid-formulation}) to $m = 100$ and $m = 25$ for simulated and real-world tasks, respectively. We report policy success rates over 500 rollouts averaged over the last 10 policy checkpoints for simulated tasks, and 25 rollouts performed with the last checkpoint for real-world tasks.

\textbf{Baselines.}
We consider baselines from several methodological categories: DemInf~\cite{hejna2025robotdatacurationmutual}---applicable only to filter-$k$ (\cref{task:filter-k})---curates data offline (i.e., without rollouts) by maximizing mutual information, promoting diverse and predictable demonstrations; Demo-SCORE~\cite{chen2025curating} trains binary classifiers to distinguish states from successful and failed rollouts, retaining demonstrations with a high average state success probability; Success Similarity is a custom method that ranks demonstrations by their average state similarity to successful rollouts; Random chooses demonstrations uniformly at random; Oracle approximates an upper bound on performance by curating data with privileged access to ground-truth demonstration labels, e.g., indicating demonstration quality, strategy robustness, or other properties.

\subsection{Setting 1: Improving Policy Performance in Mixed-Quality Regimes}\label{sec:cupid-exp-quality}

We first study curation of mixed-quality datasets, where training on lower-quality demonstrations may degrade policy performance~\cite{pmlr-v164-mandlekar22a, hejna2025robotdatacurationmutual}. 
We use the ``Lift,'' ``Square,'' and ``Transport'' tasks from RoboMimic's multi-human (MH) task suite, which provides ground-truth quality labels for demonstrations. 
The ``Lift'' and ``Square'' tasks contain three quality tiers \{“low”, “medium”, “high”\}, while the more complex bi-manual ``Transport'' task contains six quality tiers \{“low-low”, “low-medium”, \ldots\}. 
On hardware, we design the “Figure-8” task (\cref{fig:cupid-franka-dp-results}(a)), where the robot must tie a simplified cleat hitch---a knot that follows a figure-8 pattern---requiring precise manipulation of a deformable rope. 

\begin{figure}[t]
    \centering
    \includegraphics[width=\linewidth]{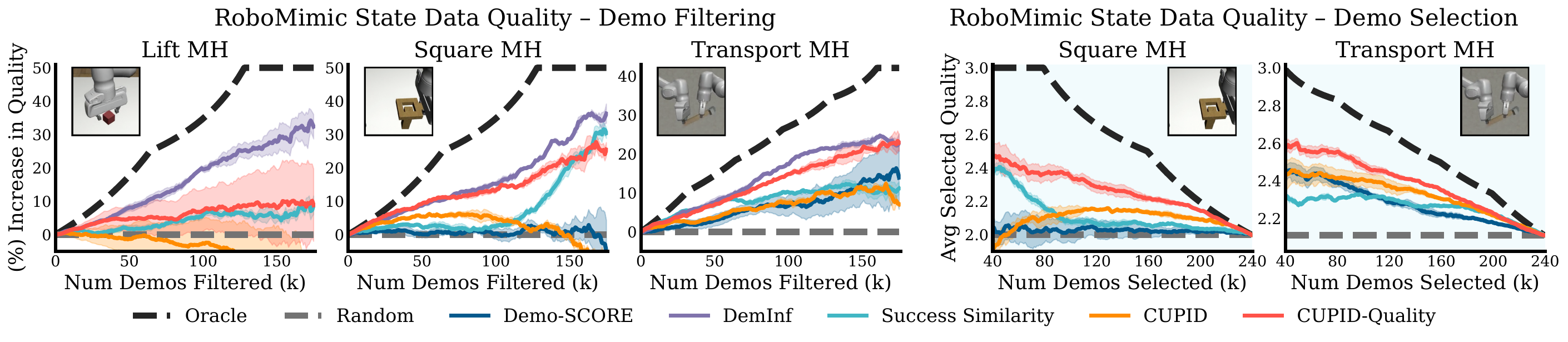}
    \includegraphics[width=\linewidth]{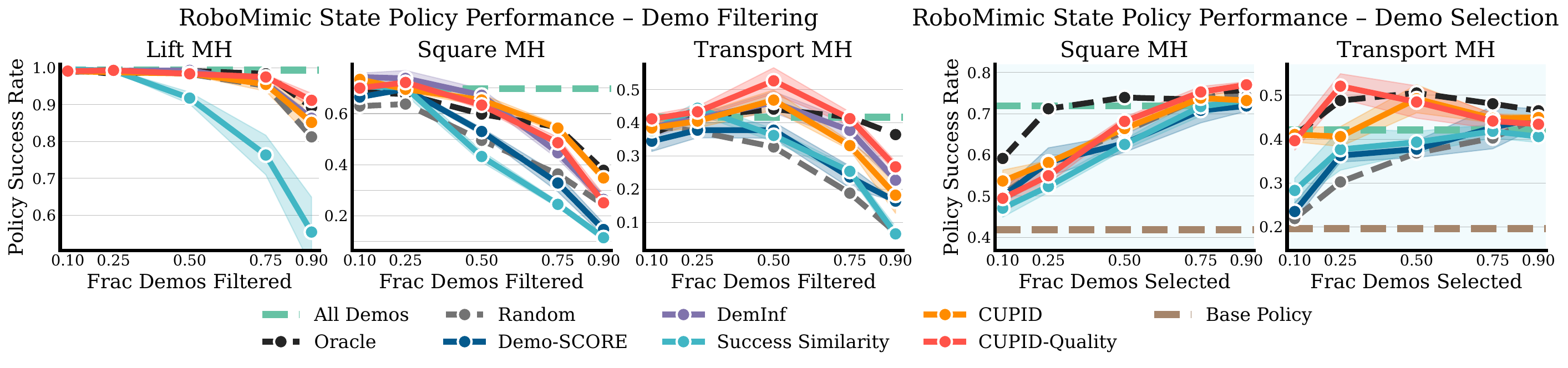}
    \caption[Data curation results on simulated RoboMimic diffusion policy (CUPID)]{\small
        RoboMimic mixed-quality curation results. \textbf{Top: Data Quality.} Baselines often prioritize demonstration quality (e.g., DemInf~\cite{hejna2025robotdatacurationmutual}), but higher demonstration quality does always translate to higher policy success rates. In contrast, \basemethod{} targets demonstrations that most strongly contribute to downstream policy performance. \textbf{Bottom: Policy Performance.} Diffusion policies trained on data curated by \basemethod{} achieve higher success rates than baselines, despite using demonstrations of perceived lower quality. Although combining performance and quality measures (\qualitymethod{}) yields the best policies on mixed-quality datasets, quality measures can degrade performance in other settings (see \cref{fig:cupid-franka-dp-results}).
        Results are averaged over 3 random seeds (500 policies trained across settings). Success rates are computed over 50 rollouts from the last 10 checkpoints (500 rollouts total).
    }
    \label{fig:cupid-robomimic-dp-results}
\end{figure}

\textbf{RoboMimic analysis.} \cref{fig:cupid-robomimic-dp-results} presents the RoboMimic benchmark results: the top row shows data quality trends for filter-$k$ and select-$k$ across varying $k$, while the bottom row reports success rates of diffusion policies trained on the corresponding curated datasets. As expected, we first observe that DemInf---which targets demonstration quality---curates datasets of the highest overall quality by RoboMimic's ground-truth labels for filter-$k$ (top row, \cref{fig:cupid-robomimic-dp-results}). However, policies trained on data curated by \basemethod{} 
consistently match or outperform those of DemInf (bottom row, \cref{fig:cupid-robomimic-dp-results}). 
This indicates that human perception of demonstration quality does not necessarily correspond to data that maximizes downstream policy success.
Second, we find the state similarity heuristics employed by Demo-SCORE and Success Similarity to be relatively ineffective in challenging mixed-quality regimes, where successful and failed rollouts exhibit similar states. 
Lastly, \qualitymethod{}, which evenly balances demonstration quality and downstream performance impact (\cref{sec:cupid-methods-quality}), attains the highest policy success rates---surpassing the Oracle in 3/5 cases, and achieving an even higher success rate than the official diffusion policy~\cite{chi2023diffusion} on ``Transport MH'' while using fewer than (i) 33\% of the original 300 demonstrations and (ii) 10\% of the model parameters. 
We provide an extended discussion of the RoboMimic results in \cref{appx:cupid-results-robomimic-discussion}.

\textbf{Figure-8 analysis.} 
\cref{fig:cupid-franka-dp-results}(a) shows diffusion policy results on the real-world ``Figure-8'' task. First, \basemethod{} improves over the base policy's success rate by 38\% (averaged over filtering and selection). Second, as in RoboMimic, \qualitymethod{} further strengthens curation performance, corroborating the utility of quality metrics (\cref{eq:cupid-quality-influence-fn}) in mixed-quality regimes. 
As shown in \cref{fig:cupid-franka-dp-distr-filter-dataset-results}(a) (filtering; see \cref{appx:cupid-select-distr} for selection), both \basemethod{} and \qualitymethod{} successfully retain high-quality demonstrations, whereas baselines such as Demo-SCORE discard some in favor of lower-quality demonstrations. Overall, training on lower-quality demonstrations appears to adversely affect policy performance on the ``Figure-8'' task.

\begin{figure}[t]
    \centering
    \includegraphics[width=\linewidth]{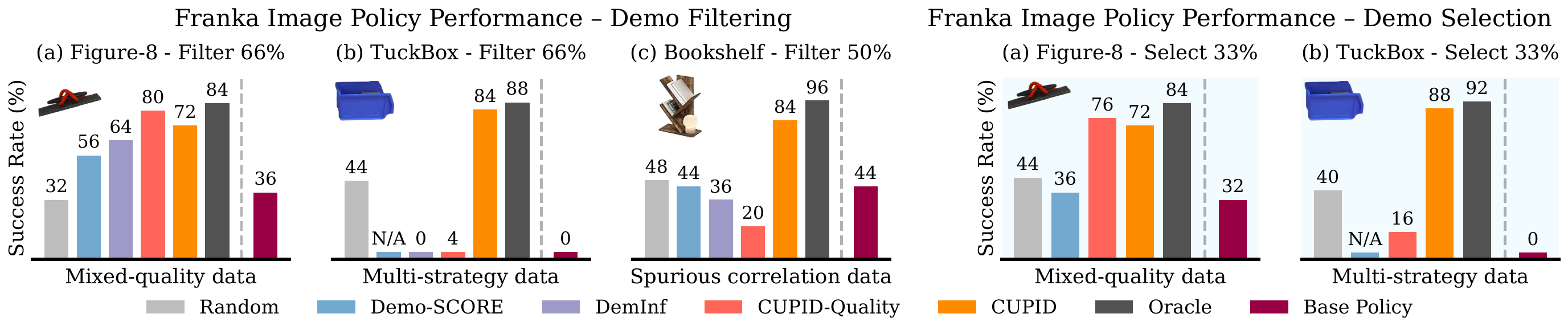}
    \includegraphics[width=\linewidth]{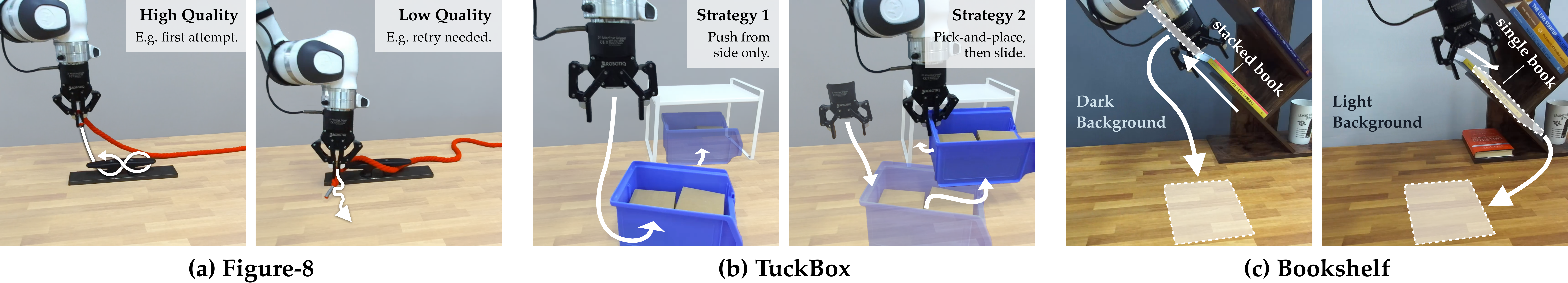}
    \caption[Data curation results on real-world diffusion policy (CUPID)]{\small
        \textbf{Franka real-world diffusion policy performance.} \basemethod{}, which curates demonstrations \textit{w.r.t.} policy performance, improves success rates on mixed-quality datasets, identifies robust strategies, and disentangles spurious correlations that hinder performance. Although quality measures (e.g., DemInf, \qualitymethod{}) help in mixed-quality settings (Figure-8; \cref{fig:cupid-robomimic-dp-results}), they degrade performance when higher-quality demonstrations induce brittle strategies at test time (TuckBox), or when quality is not the primary factor limiting policy success (Bookshelf). Overall, curating data based on performance (\basemethod{}) maintains robustness across these settings. 
    }
    \label{fig:cupid-franka-dp-results}
\end{figure}

\subsection{Setting 2: Identifying Robust Test-time Strategies from Policy Failures}\label{sec:cupid-exp-strategies}

Heterogeneous imitation learning datasets may contain multiple strategies for solving a task, some of which can fail under distribution shifts at deployment. We design a real-world ``TuckBox'' task, where a robot must tuck a recycling bin under a receptacle by (i) sliding or (ii) first repositioning it via pick-and-place (see \cref{fig:cupid-franka-dp-results}(b)). The dataset contains a 2:1 ratio of sliding to pick-and-place demonstrations, making sliding the dominant strategy. At test time, we induce an imperceptible distribution shift by altering the bin's mass distribution, rendering sliding unreliable. 
In this setting, curation aims to rebalance the dataset to promote strategies that are more robust to unforeseen shifts at deployment.

\textbf{TuckBox analysis.} \cref{fig:cupid-franka-dp-results}(b) shows the diffusion policy results on ``TuckBox.'' Due to the strategy imbalance, the base policy exclusively exhibits the sliding behavior, resulting in a 100\% failure rate under the distribution shift. This immediately invalidates the use of Demo-SCORE, which requires both successful and failed rollouts. In contrast, \basemethod{} does not require observing successes: by linking failures to the demonstrations that influenced them, curating with \basemethod{} yields a policy that exhibits increased pick-and-place behavior, performing comparably (84\%-88\% success rate) to the Oracle. 
In contrast, both DemInf and \qualitymethod{} incorrectly associate the higher-variance pick-and-place demonstrations with lower quality, resulting in more uniform filtering across strategies (see \cref{fig:cupid-franka-dp-distr-filter-dataset-results}(b)). As a result, policies trained on data curated by these baselines default to the brittle sliding strategy at deployment.

\begin{figure}[t]
    \centering
    
    \begin{subfigure}[b]{\linewidth}
        \centering
        \includegraphics[width=0.95\linewidth]{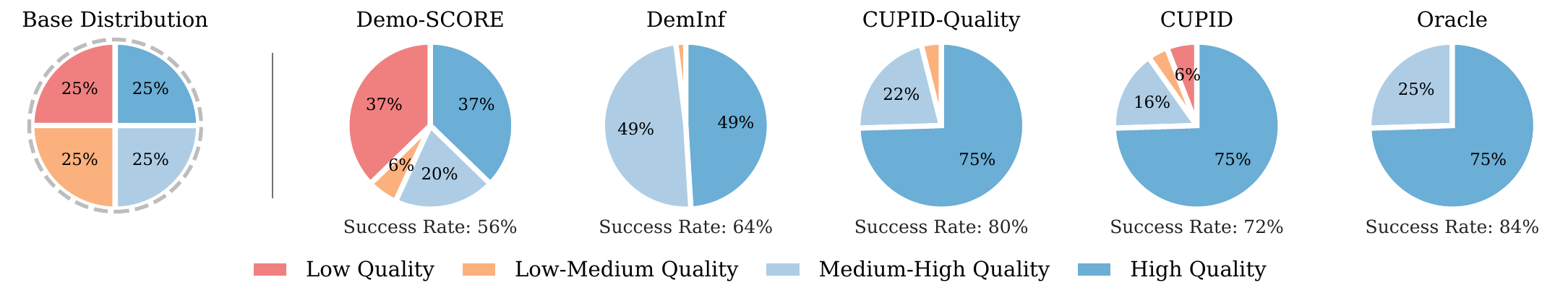}
        \vspace{-6pt}
        \caption{\footnotesize 
            \textbf{Figure-8:} Distribution of curated demonstrations after \textit{filtering} 66\%. Higher-quality demos are better.
        }
        \label{fig:cupid-franka-dp-distr-filter-dataset-results-figure8}
    \end{subfigure}

    \vspace{18pt}

    \begin{subfigure}[b]{\linewidth}
        \centering
        \includegraphics[width=0.95\linewidth]{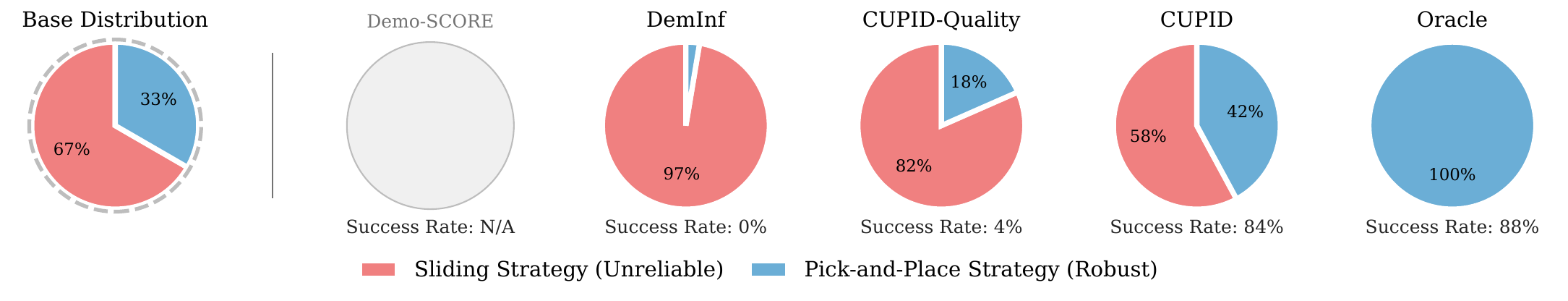}
        \vspace{-6pt}
        \caption{\footnotesize
            \textbf{TuckBox:} Distribution of curated demonstrations after \textit{filtering} 66\%. Pick-and-place demos are better.
        }
        \label{fig:cupid-franka-dp-distr-filter-dataset-results-tuckbox}
    \end{subfigure}
    
    \vspace{18pt}
    
    \begin{subfigure}[b]{\linewidth}
        \centering
        \includegraphics[width=0.95\linewidth]{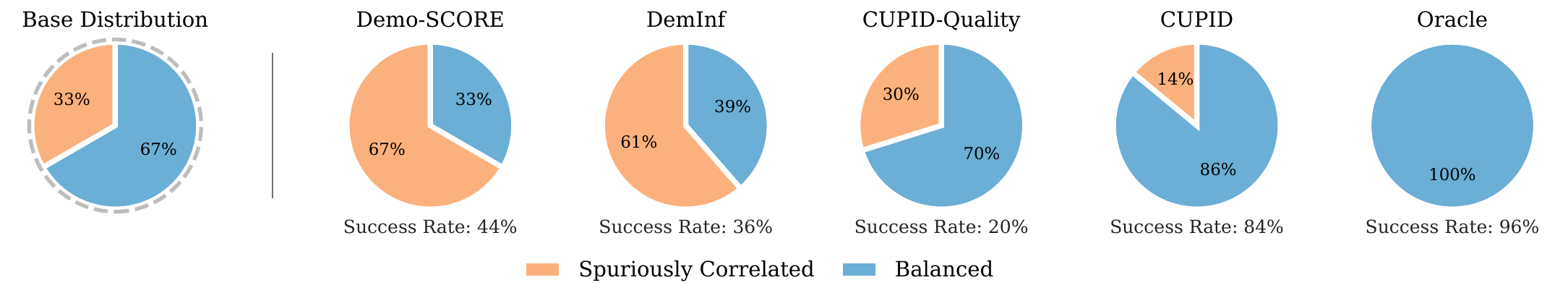}
        \vspace{-6pt}
        \caption{\footnotesize
            \textbf{Bookshelf:} Distribution of curated demonstrations after \textit{filtering} 50\%. Balanced data is better.
        }
        \label{fig:cupid-franka-dp-distr-filter-dataset-results-bookshelf}
    \end{subfigure}
    
    \vspace{8pt}
    
    \caption[Curated dataset distributions for demo filtering on real-world tasks (CUPID)]
    {\small
        \textbf{Franka diffusion policy curated dataset distributions for filtering (\cref{task:filter-k}).} \basemethod{} filters out lower-quality demonstrations (Figure-8), brittle strategies (TuckBox), and spuriously correlated examples (Bookshelf), improving policy performance across tasks. While curation heuristics employed by baselines may be effective in some cases (e.g., DemInf and \qualitymethod{} in Figure-8), they can lead to suboptimal pruning in others.
    }
    \label{fig:cupid-franka-dp-distr-filter-dataset-results}
\end{figure}

\subsection{Setting 3: Disentangling Spurious Correlations in Demonstration Data}\label{sec:cupid-exp-correlations}

Spurious correlations in training data may cause a policy to rely on non-causal features, hindering generalization to variations in the input or task~\cite{de2019causal}. We design a real-world ``Bookshelf'' task, where a robot must extract a target book via (i) horizontal or (ii) vertical pulling motion, depending on whether another book is stacked above the target book. While both strategies are equally represented in the training set, each co-occurs more frequently with a certain background color (see \cref{fig:cupid-franka-dp-results}(c)).
At evaluation, we test the policy under slight variations in the number and position of distractor books, while keeping the white background fixed---the correlate associated with the horizontal pulling behavior.

\textbf{Bookshelf analysis.} Diffusion policy results are shown in \cref{fig:cupid-franka-dp-results}(c). 
The base policy achieves only a 44\% success rate, as the presence of the white background often causes the policy to extract the target book horizontally despite another book being stacked on top (causing it to fall). 
Interestingly, by training classifiers to distinguish failed from successful states, Demo-SCORE appears to misattribute failure to the presence of rollout correlates (the stacked book) rather than causal factors (the white background). 
In contrast, \basemethod{} attains an 84\% success rate by identifying demonstrations that causally drive failure---in this case, horizontal pulling motion with a white background---enabling dataset rebalancing that mitigates the effect of spurious correlations (see \cref{fig:cupid-franka-dp-distr-filter-dataset-results}(c)). 
As in \cref{sec:cupid-exp-strategies}, DemInf and \qualitymethod{} incorrectly prioritize the lower-variance horizontal pulling motion, yielding negligible performance gains.

\section{Discussion and Ablations}\label{sec:cupid-discussion}

\subsection{How Is Curation Performance Affected by Properties of the Data and Task?}\label{sec:cupid-discussion-properties}

Our mixed-quality curation experiments (\cref{fig:cupid-robomimic-dp-results} and \cref{fig:cupid-franka-dp-results}(a)) reveal that while curation strengthens performance on ``Transport MH'' and ``Figure-8'' (i.e., a fraction of the demonstrations harm policy performance), removing almost \textit{any} demonstration degrades performance on ``Square MH'' (i.e., all demonstrations appear important). 
In contrast, only about 15\% of the dataset is necessary to maximize performance on ``Lift MH''  (i.e., the dataset is highly redundant)\footnote{Note that \cref{fig:cupid-robomimic-dp-results} does not include select-$k$ curation results for ``Lift MH'' because the base policy already achieves a 100\% success rate, leaving no further room for improvement by selecting additional demonstrations.}. 
These results indicate that the potential benefits of data curation depend on properties of both the data and the task. 
For example, one possible hypothesis is that curation is most effective in complex, precision-critical settings (e.g., ``Transport MH''), whereas for tasks with greater tolerance for error (e.g., ``Lift MH''), state-of-the-art policies~\cite{chi2023diffusion} appear less sensitive to---and may even benefit from---training on lower-quality demonstrations.

\subsection{How Many Policy Rollouts Are Required for Effective Curation with \basemethod{}?}\label{sec:cupid-discussion-rollouts}

\basemethod{} uses a \texttt{REINFORCE}-style estimator to compute the performance influence of each demonstration (\cref{eq:cupid-policy-inf-deriv}) for curation.
Thus, the accuracy of estimated performance influences depends on the number of policy rollouts.
While \texttt{REINFORCE}~\cite{sutton1999policy} often yields high-variance gradient estimates under limited rollout budgets, e.g., in reinforcement learning contexts~\cite{greensmith2004variance}, we highlight that our curation objective imposes a lower fidelity requirement: since curation with \basemethod{} involves top-$k$ selection (\cref{sec:cupid-methods-curation}), it suffices to rank helpful demonstrations above harmful ones (requiring fewer rollouts) rather than to estimate performance influence precisely (requiring many rollouts). 
As shown in \cref{fig:cupid-robomimic-state-data-quality-cupid-rollout}, the ranking of demonstrations stabilizes with approximately $m \in [25, 50]$ rollouts on ``Lift MH'' and ``Square MH,'' and $m \in [50, 100]$ rollouts on ``Transport MH.'' Similarly, we use only $m = 25$ rollouts for our real-world Franka tasks (\cref{fig:cupid-franka-dp-results}). These results support the practicality of \basemethod{} under realistic rollout budgets, while noting that more complex tasks (e.g., ``Transport MH'') may benefit from a greater number of rollouts.

\begin{figure}[H]
    \centering
    \includegraphics[width=\linewidth]{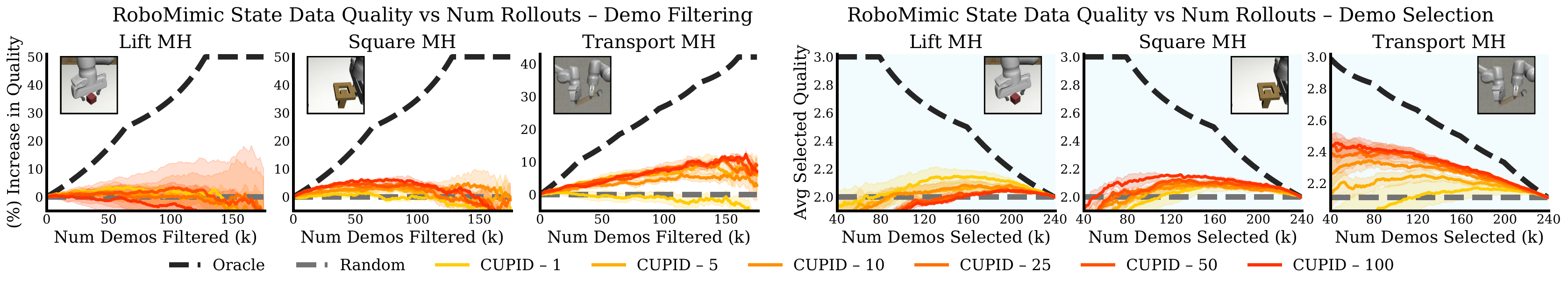}
    \caption[Ablation result on the number of policy rollouts (CUPID)]{\small
        \textbf{\basemethod{} ablation on the number of policy rollouts.} 
        Performance influences (\cref{eq:cupid-policy-inf-deriv}) converge with $m \in [25, 50]$ rollouts on ``Lift MH'' and ``Square MH'' (yielding similar quality trends), and $m \in [50, 100]$ rollouts on ``Transport MH,'' validating the practical applicability of \basemethod{} under realistic rollout budgets. 
        Curation is performed on diffusion policies.
        Results are averaged over 3 random seeds.
        Errors bars represent standard error.
    }
    \label{fig:cupid-robomimic-state-data-quality-cupid-rollout}
\end{figure}

\subsection{Can Data Curated for Single-Task Policies Strengthen Generalist Policy Performance?}\label{sec:cupid-discussion-pi0-transfer}

\begin{wrapfigure}{rt}{0.42\textwidth}
    \centering
    \vspace{-14pt}
    \includegraphics[width=1.0\textwidth]{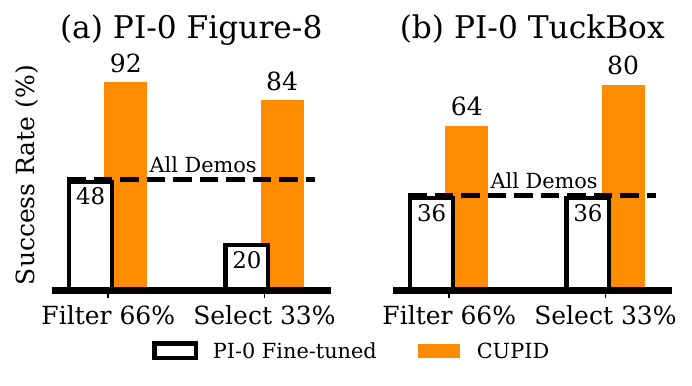}
    \caption[Data curation results on real-world PI-0 multi-task policy (CUPID)]{\small 
        Data curated for single-task diffusion policies improves $\pi_0$~\cite{black2410pi0} post-training performance. 
        Additional results in \cref{fig:cupid-franka-pi0-results-appx}.
    }
    \vspace{-4pt}
    \label{fig:cupid-franka-pi0-transfer-results}
\end{wrapfigure}

In \cref{fig:cupid-franka-pi0-transfer-results}, we show that datasets curated by \basemethod{} for single-task diffusion policies can significantly improve the fine-tuned performance of a generalist Vision-Language-Action (VLA) model, $\pi_0$~\cite{black2410pi0}.
While \basemethod{} is, in principle, tailored to the specific policy used during rollouts, it consistently identifies low-quality, stochastic behaviors in ``Figure-8'' and unreliable strategies in ``TuckBox'' (\cref{fig:cupid-franka-dp-distr-filter-dataset-results})---both intrinsic properties of the data. 
Filtering these poorer demonstrations (or selecting better ones) is thereby likely to improve the performance \textit{any} policy.
This highlights a promising direction to alleviate the computational cost of \basemethod{} in large-scale settings: use smaller, single-task policies to curate datasets for larger, multi-task models. 

\textit{VLA Robustness.} \cref{fig:cupid-franka-pi0-transfer-results} also suggests that scaling the pre-training of VLA models does not inherently enable them to leverage their generalist knowledge to, e.g., \textit{ignore} low-quality behaviors or brittle strategies in demonstration data. That is, data curation still appears important for VLA post-training. 

\section{Conclusion}\label{sec:cupid-conclusion}

In this chapter, we study the problem of data curation for robot imitation learning. We present \basemethod{}, a novel data curation method that uses influence functions to measure the causal impact of a demonstration on the policy's closed-loop performance. Our results highlight the general utility of performance-based curation for two key curation tasks---filtering existing training demonstrations and subselecting new demonstrations---and across diverse curation settings, where a policy's test-time performance varies with the choice of training data. 

Among many key problems in robotics, it is inherently difficult to develop strong intuitions about how training data influences downstream policy behavior, and to delineate why a policy trained exclusively on expert demonstration data would exhibit suboptimal performance at deployment. We hope this investigation spurs continued interest in pursuit of these questions.

\section{Limitations and Future Work}\label{sec:cupid-limitations}

We summarize the limitations of our approach and opportunities for future investigation:
\begin{enumerate}
    \item \textbf{Curation tasks.} The curation tasks considered in this chapter (\cref{task:filter-k} and \cref{task:select-k}) aim to curate performance-maximizing datasets for a specified filtering or selection quantity of demonstrations $k$. Determining the suitable quantity of demonstrations to curate represents a possible point of extension.
    \item \textbf{Data properties.} Critically, future work should further investigate how properties of the data dictate the extent to which curation can improve policy performance, as discussed in \cref{sec:cupid-discussion-properties}. 
    \item \textbf{Data explainability.} Our methods focus on curating existing demonstrations as a first step. However, future work may seek to interpret the properties of influential demonstrations to actively inform subsequent data collection efforts---for example, by providing instructions to data collectors. 
    \item \textbf{Selection methods.} While the \emph{greedy} selection procedures used in \cref{eq:cupid-prune-topk} and \cref{eq:cupid-select-topk} are tractable to optimize and often improve over quality- and similarity-based measures~\cite{engstrom2024dsdm}, they ignore the interactions between demonstrations in the curated set~\cite{koh2019accuracy, ilyas2022datamodels}. This can temper performance gains when the size of the curated set is large. Future work should investigate higher-order approximations that consider the joint diversity of the curated dataset, as is common in the active learning literature (e.g., \cite[Sec. 4.3]{Settles2012}).
    \item \textbf{Larger datasets.} Estimating performance influences over the full demonstration dataset incurs a computational cost comparable to that of policy training. Reducing this expense in large-scale settings is an important future direction. For example, one could approximate group effects~\cite{koh2019accuracy} via random sampling or limit influence estimation to smaller data subsets identified using coarse-grained heuristics. 
    \item \textbf{Estimator variance.} Finally, although we observe stable performance from \basemethod{} across curation settings, the use of the \texttt{REINFORCE} estimator may result in high variance influence scores, e.g., when the number of policy rollouts is small. In such settings, variance reduction techniques, such as those typically used in reinforcement learning~\cite{greensmith2004variance}, may further improve the fidelity of our influence scores.  
\end{enumerate}

\part{Policy Coordination and Planning}\label{part:2}

\chapter{Coordinating Policy Sequences for Long-Horizon Tasks}\label{chapter:stap}

\section{Introduction}\label{sec:stap-intro}

In the preceding chapters, we examined how deployment-time monitoring and data-centric interpretability mechanisms can help safeguard learned policies, diagnose failures, and improve the reliability of individual behaviors. However, many everyday tasks---such as rearranging a table, shelving kitchenware, or packing objects---are inherently \textit{long-horizon}: they require the successful execution of a sequence of behaviors rather than a single action in isolation. This chapter focuses on a critical challenge that arises when independently learned behaviors are composed. That is, subtle dependencies between behaviors emerge when they are sequenced, potentially rendering downstream behaviors infeasible. As a result, achieving reliable outcomes for multi-step, real-world tasks requires robots to explicitly reason about and coordinate these dependencies at deployment time.

Consider the example in \cref{fig:stap-teaser}, where the robot needs to retrieve an object outside of its workspace by first using an L-shaped hook to pull the target object closer. How the robot picks up the hook affects whether the target object will be reachable. %
Traditionally, planning actions to ensure the geometric feasibility of a sequential manipulation task is handled by motion planning~\cite{lavalle:2006,toussaint2015-lgp,toussaint2018differentiable}, which typically requires full observability of the environment state and knowledge of its dynamics. %
Learning-based approaches~\cite{Kaelbling1996ReinforcementLA,argall2009survey,Schaal1999-SCHIIL-2} can acquire skills without this privileged information. 
However, using independently learned skills to perform unseen long-horizon manipulation tasks is an unsolved problem. 
The skills could be myopically executed one after another to solve a simpler subset of tasks, but solving more complex tasks requires planning with these skills to ensure the feasibility of the entire skill sequence.

\begin{figure}[t]
    \centering
    \includegraphics[width=\linewidth]{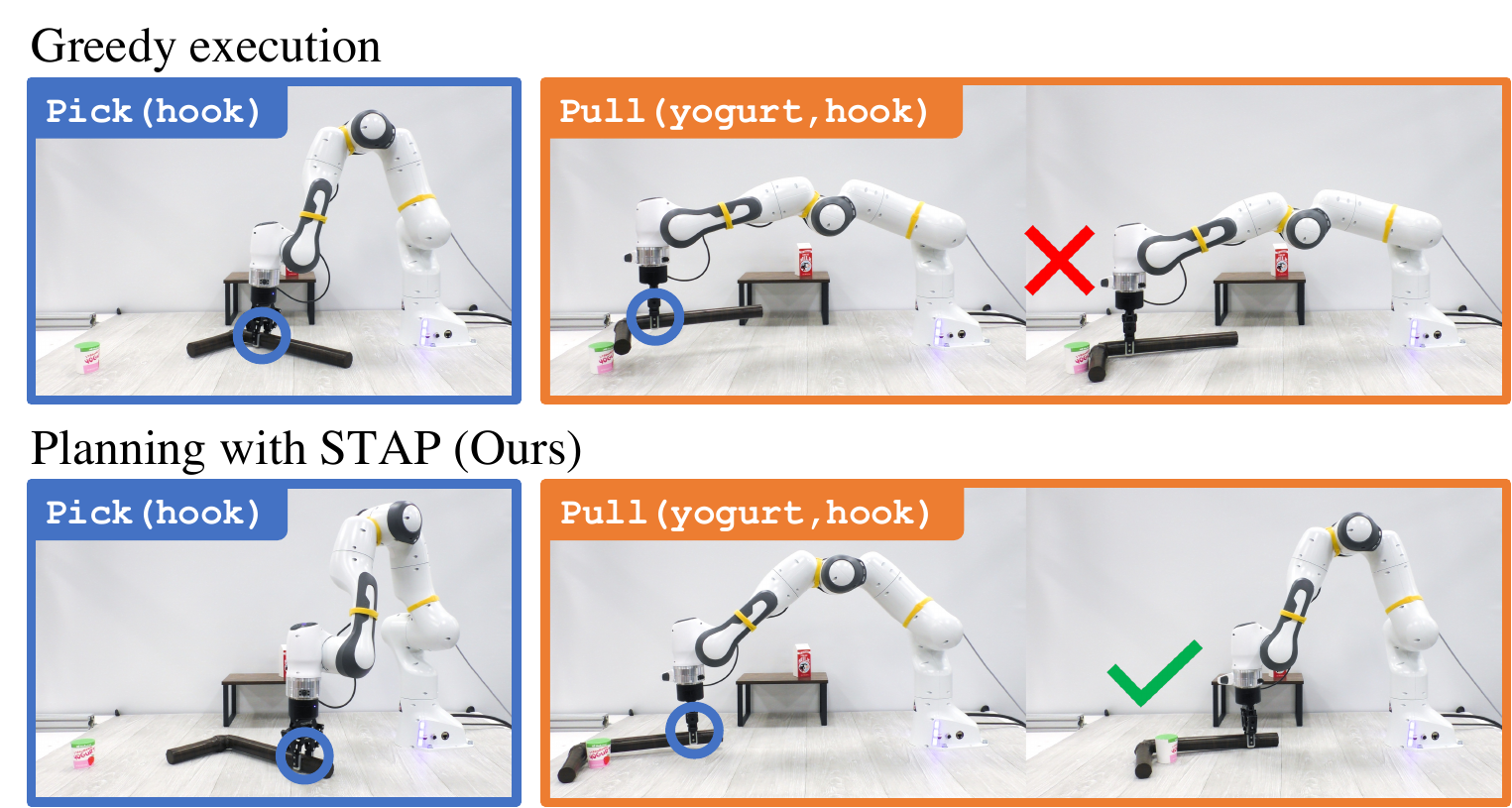}
    \caption[Overview of STAP]{\small
        Sequential manipulation tasks often contain geometric dependencies between actions. In this example, the robot needs to use the hook to pull the block into its kinematic workspace so it is close enough to pick up. The top row shows how greedy execution of skills results in the robot picking up the hook in a way that prevents it from reaching the block. We present a method for planning with skills to maximize long-horizon success without the need to train the skills on long-horizon tasks.
        More details can be found on the STAP website: \url{https://sites.google.com/stanford.edu/stap}.
        \copyright{} 2023 IEEE.
    }
    \label{fig:stap-teaser}
\end{figure}

Prior work focuses on sequencing skills at \textit{train time} to solve a small set of sequential manipulation tasks~\cite{xu2021-daf, dalal2021-raps}. 
To contend with long-horizons, these methods often learn skills~\cite{da2012learning} that consist of a policy and parameterized manipulation primitive~\cite{felip2013manipulation}. The policy predicts the parameters of the primitive, thereby governing its motion.
Such methods are \textit{task-specific} in that they need to be trained on skill sequences that reflect the tasks they might encounter at test time.
In our framework, we assume that a task planner provides a novel sequence of skills at \textit{test time} that will then be grounded with parameters for manipulation primitives through optimization.
This makes our method \textit{task-agnostic}, as skills can be sequenced to solve long-horizon tasks not seen during training.

At the core of our method, {\em Sequencing Task-Agnostic Policies} (STAP), we use Q-functions to optimize the parameters of manipulation primitives in a given sequence.
Policies and Q-functions for each skill are acquired through off-the-shelf Reinforcement Learning. 
We then define a planning objective to maximize all Q-functions in a skill sequence, ensuring its geometric feasibility.
To evaluate downstream Q-functions of future skills, we learn a dynamics model that can predict future states. 
We also use {\em Uncertainty Quantification\/} (UQ) to avoid visiting states that are {\em Out-Of-Distribution\/} (OOD) for the learned skills. 
We train all of these components independently per skill, making it easy to gradually expand a library of skills without the need to retrain existing ones.

Our contributions are three-fold: we propose 
1) a framework to train an extensible library of task-agnostic skills, 
2) a planning method that optimizes arbitrary sequences of skills to solve long-horizon tasks,
and 3) a method to solve Task and Motion Planning (TAMP) problems with learned skills. 
In extensive experiments, we demonstrate that planning with STAP promotes long-horizon success on tasks with complex geometric dependencies between actions.
We also demonstrate that our framework works on a real robot.

\section{Related Work}\label{sec:stap-related}

\subsection{Robot Skill Learning}
How to represent and acquire composable manipulation skills is a widely studied problem in robotics. 
A broad class of methods uses Learning from Demonstration (LfD)~\cite{argall2009survey}. 
Dynamic Movement Primitives (DMPs) \cite{schaal2006-dmp,pastor2009-dmp,khansari2011-dmp} are a form of LfD that learns the parameters of dynamical systems encoding movements~\cite{ijspeert2002movement,matsubara2010learning,ude2010task,kulvicius2011joining}.
More recent extensions integrate DMPs with deep neural networks to learn more flexible policies~\cite{bahl2020-ndp,bahl2021-hndp}---for instance, to build a large library of skills from human video demonstrations~\cite{concept2robot-2021}.
Skill discovery methods instead identify action patterns in offline datasets~\cite{shankar2019-discovering} and either distill them into policies~\cite{shankar2020-learning,ajay2020-opal} or extract skill priors for use in downstream tasks~\cite{singh2020-parrot,pertsch2020-spiral}. 
Robot skills can also be acquired via active learning~\cite{da2014active}, Reinforcement Learning (RL)~\cite{rajeswaran2017-dext-manip,kalashnikov2018-qtopt,kalashnikov2021-mtopt,sharma2019-dynamics,lu2022-awopt}, and offline RL~\cite{chebotar2021-actionable}. 

An advantage of our planning framework is that it is agnostic to the types of skills employed, requiring only that it is possible to predict the probability of the skill's success given the current state and action.
Here, we learn skills~\cite{da2012learning} that consist of a policy and a parameterized manipulation primitive~\cite{felip2013manipulation}. 
The actions output by the policy are the parameters of the primitive determining its motion. 
In STAP, we will use the Q-functions of the policy to optimize suitable parameters~\cite{kalashnikov2018-qtopt,concept2robot-2021} for a sequence of manipulation primitives.

\subsection{Long-Horizon Robot Planning}
Once manipulation skills have been acquired, using them to perform sequential manipulation tasks remains an open challenge.
\cite{kappler2015data,kaelbling2017learning,ames2018learning,migimatsu2022symbolic} propose data-driven methods to determine the \textit{symbolic} feasibility of skills and only control their timing, while we seek to ensure the \textit{geometric} feasibility of skills by controlling their trajectories.
Other techniques rely on task planning~\cite{wu2021-embr,huang2019-crsp}, subgoal planning~\cite{simeonov2020-subgoal-skills}, or meta-adaptation~\cite{xu2018-ntp, huang2019-ntg} to sequence learned skills to novel long-horizon goals.  
However, the tasks considered in these works do not feature rich geometric dependencies between actions that necessitate motion planning or skill coordination.

The options framework~\cite{sutton1999between} and the parameterized action
Markov Decision Process (MDP)~\cite{masson2016-pamdps} train a high-level policy to engage low-level policies~\cite{bacon2017-option,nachum2018-hiro} or primitives~\cite{dalal2021-raps,chitnis2020-efficient,vuong2021learning,nasiriany2022-augmenting} towards long-horizon goals.
\cite{shah2021-vfs} proposes a hierarchical RL method that uses the value functions of lower-level policies as the state space for a higher-level RL policy.
Our approach is also related to model-based RL methods which jointly learn dynamics and reward models to guide planning~\cite{finn2017-dvf,chua2018-pets,hafner2019-planet}, policy search~\cite{janner2019-mbpo,hafner2019-dreamer}, or combine both~\cite{xie2020-skill,sekar2020-plan2explore}. 
While these methods demonstrate that policy hierarchies and model-based planning can enable RL to solve long-horizon problems, they are typically trained in the context of a single task.
In contrast, we seek to \textit{plan} with lower-level skills to solve tasks never seen before.

Closest in spirit to our approach is that of \citet{xu2021-daf}, Deep Affordance Foresight (DAF), which proposes to learn a dynamics model, skill-centric affordances (value functions), and a skill proposal network that serves as a higher-level RL policy. 
We identify several drawbacks with DAF: 
first, because DAF relies on multi-task experience for training, generalizing beyond the distribution of training tasks may be difficult; 
second, the dynamics, affordance models, and skill proposal network need to be trained synchronously, which complicates expanding the current library of trained skills;
third, their planner samples actions from uniform random distributions, which prevents DAF from scaling to high-dimensional action spaces and long horizons. 
STAP differs in that our dynamics, policies, and affordances (Q-functions) are learned independently per skill. 
Without any additional training, we combine the skills at planning time to solve unseen long-horizon tasks. 
We compare our method against DAF in the planning experiments (\cref{sec:stap-experiments-planning}).

\subsection{Task and Motion Planning}

TAMP solves problems that require both symbolic and geometric reasoning~\cite{garrett2021integrated, toussaint2015-lgp}. 
DAF learns a skill proposal network to replace the typical task planner in TAMP, akin to~\cite{wang2022-gtp}.
Another prominent line of research learns components of the TAMP system, often from a dataset of precomputed solutions~\cite{wang2018active,driess2020-dvr,chitnis2021camps,silver2021-ploi,driess2021-lgr,wang2021learning}.
The problems we consider involve complex geometric dependencies between actions that are typical in TAMP.
However, STAP only performs geometric reasoning and by itself is not a TAMP method. 
We demonstrate in experiments (\cref{sec:stap-experiments-tamp}) that STAP can be combined with symbolic planners to solve TAMP problems.

\section{Problem Setup}\label{sec:stap-problem}

\subsection{Long-Horizon Planning}\label{sec:stap-planning-tamp}

Our objective is to solve long-horizon manipulation tasks that require sequential execution of learned skills.
These skills come from a skill library $\mathcal{L} = \{\psi^1, \dots, \psi^K\}$, where each skill $\psi^k$ consists of a parameterized manipulation primitive~\cite{felip2013manipulation} $\phi^k$ and a learned policy $\pi^k$.
A primitive $\func{\phi^k}{a^k}$ takes in parameters $a^k$ and executes a series of motor commands on the robot, while a policy $\func{\pi^k}{a^k \given s^k}$ is trained to predict a distribution of suitable parameters $a^k$ from the current state $s^k$.
For example, the $\actioncall{Pick}{a}{b}$ skill may have a primitive which takes as input an end-effector pose and executes a trajectory to pick up object $\obj{a}$, where the robot first moves to the commanded pose, closes the gripper to grasp $\obj{a}$, and then lifts $\obj{a}$ off of $\obj{b}$.
The learned policy $\pi^k$ for this skill will then try to predict end-effector poses to pick up $\obj{a}$.

We assume access to a high-level planner that computes \textit{plan skeletons} (i.e. skill sequences) to achieve a high-level goal.
STAP aims to solve the problem of turning plan skeletons into geometrically feasible \textit{action plans} (i.e. parameters for each manipulation primitive in the plan skeleton).

STAP is agnostic to the choice of high-level planner. 
For instance, it can be used in conjunction with Planning Domain Definition Language (PDDL) \cite{mcdermott1998pddl} task planners to perform hierarchical TAMP~\cite{kaelbling2011hierarchical}.
In this setup, the task planner and STAP will be queried numerous times to find multiple plan skeletons grounded with optimized action plans.
STAP will also evaluate each action plan's probability of success (i.e. its geometric feasibility).
After some termination criterion is met, such as a timeout, the candidate plan skeleton and action plan with the highest probability of success is returned.

\subsection{Task-Agnostic Policies}

We aim to learn policies $\{\pi^1, \dots, \pi^K\}$ for the skill library $\mathcal{L}$ that can be sequenced by a high-level planner in arbitrary ways to solve any long-horizon task.
We call these policies task-agnostic because they are not trained to solve a specific long-horizon task.
Instead, each policy $\pi^k$ is associated with a skill-specific contextual bandit (i.e. a single timestep MDP)
\begin{align}\label{eq:stap-skill-mdp}
    \mathcal{M}^k = \left(\mathcal{S}^k, \mathcal{A}^k, T^k, R^k, \rho^k\right),
\end{align}
where $\mathcal{S}^k$ is the state space, $\mathcal{A}^k$ is the action space, $\func{T^k\!}{s'^k \given s^k, a^k}$ is the transition model, $\func{R^k\!}{s^k, a^k, s'^k}$ is the binary reward function, and $\func{\rho^k\!}{s^k}$ is the initial state distribution.
Given a state $s^k$, the policy $\pi^k$ produces an action $a^k$, and the state evolves according to the transition model $\func{T^k\!}{s'^k \given s^k, a^k}$. 
Thus, the transition model encapsulates the execution of the manipulation primitive $\phi^k$ (\cref{sec:stap-planning-tamp}). 

A long-horizon domain is one in which each timestep involves the execution of a single policy, and it is specified by 
\begin{align}
    \overline{\mathcal{M}}
        &= \left(\mathcal{M}^{1:K}, \overline{\mathcal{S}}, \overline{T}^{1:K}, \overline{\rho}^{1:K}, \Gamma^{1:K}\right), \label{eq:stap-long-horizon-domain}
\end{align}
where $\mathcal{M}^{1:K}$ is the set of MDPs whose policies can be executed in the long-horizon domain, $\overline{\mathcal{S}}$ is the state space of the long-horizon domain, $\func{\overline{T}^k\!}{\overline{s}' \given \overline{s}, a^k}$ is an extension of dynamics $\func{T^k\!}{s'^k \given s^k, a^k}$ that models how the entire long-horizon state evolves with action $a^k$, $\func{\overline{\rho}^k}{\overline{s}}$ is an extension of initial state distributions $\func{\rho^k}{s^k}$ over the long-horizon state space, and $\Gamma^k: \overline{\mathcal{S}} \rightarrow \mathcal{S}^k$ is a function that maps from the long-horizon state space to the state space of policy $k$. We assume that the dynamics $\func{T^k\!}{s'^k \given s^k, a^k}, \func{\overline{T}^k\!}{\overline{s}' \given \overline{s}, a^k}$ and initial state distributions $\func{\rho^k\!}{s^k}, \func{\overline{\rho}^k\!}{\overline{s}^k}$ are unknown.

Note that while the policies may have different state spaces $\mathcal{S}^k$, policy states $s^k$ must be obtainable from the long-horizon state space $\overline{\mathcal{S}}$ via $s^k = \func{\Gamma^k\!}{\overline{s}}$. 
This is to ensure that the policies can be used together in the same environment to perform long-horizon tasks. 
In the base case, all the state spaces are identical and $\Gamma^k$ is simply the identity function. 
Another case is that $\overline{s}$ is constructed as the concatenation of all $s^{1:K}$ and $\func{\Gamma^k\!}{\overline{s}}$ extracts the slice in $\overline{s}$ corresponding to $s^k$.

\section{Sequencing Task-Agnostic Policies}\label{sec:stap-taps}

Given a task in the form of a sequence of skills to execute, our planning framework constructs an optimization problem with the policies, Q-functions, and dynamics models of each skill.
Solving the optimization problem results in parameters for all manipulation primitives in the skill sequence such that the entire sequence's probability of success is maximized.

We formalize our planning methodology in this section and outline its implementation in \cref{sec:stap-planning}. 
Lastly, we describe our procedure for training modular skill libraries in \cref{sec:stap-training}.

\subsection{Grounding Skill Sequences with Action Plans}
\label{sec:stap-taps-grounding}

\begin{figure}
    \centering
    \includegraphics[width=\columnwidth]{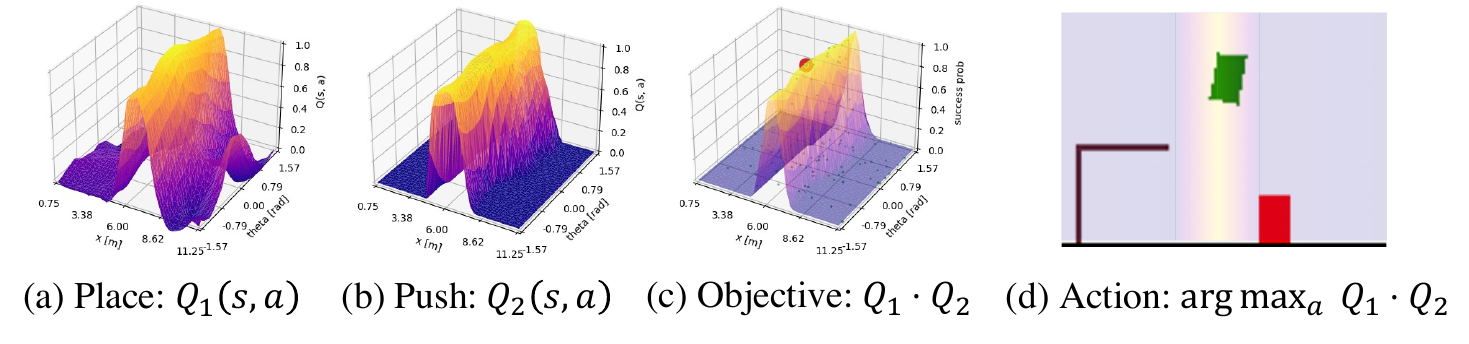}
    \caption[Coordinating policy behaviors in a 2D toy domain (STAP)]{\small
    Planning in a 2D toy domain. The agent needs to get the green block under the brown receptacle with two skills: $\action{Place}$ and $\action{Push}$ that operate on the horizontal position $x$ of the green block.
    Plots (a) and (b) show the Q-functions across $(x, \theta)$ for each skill. 
    $\action{Place}$ is only trained to get the green block on the ground, so the planner must determine $a=x$ \textit{s.t.} $\action{Push}$ is unobstructed. 
    The optimal action maximizes the probability of long-horizon task success (\cref{eq:stap-planning-objective}), approximated by the product of Q-functions in plot (c). \copyright{} 2023 IEEE.
    }
    \label{fig:stap-pybox2d}
\end{figure}

We assume that we are given a plan skeleton of skills $\tau = [\psi_1, \dots, \psi_H] \in \mathcal{L}^H$ (hereafter denoted by $\tau = \psi_{1:H}$) that should be successfully executed to solve a long-horizon task.
Let $\mathcal{M}_h$ with subscript $h$ denote the MDP corresponding to the $h$-th skill in the sequence---in contrast to $\mathcal{M}^k$ with superscript $k$, which denotes the $k$-th MDP in the skill library. 
A long-horizon task is considered successful if every skill reward $r_1, \dots, r_H$ received during execution is $1$.

Given an initial state $\overline{s}_1 \in \overline{\mathcal{S}}$, our problem is to ground the plan skeleton $\tau = \psi_{1:H}$ with an action plan $\xi = [a_1, \dots, a_H] \in \mathcal{A}_1 \times \dots \times \mathcal{A}_H$ that maximizes the probability of succeeding at the long-horizon task. This is framed as an optimization problem $\argmax_{a_{1:H}} J$, where the maximization objective $J$ is the task success probability
\begin{equation*}
    J(a_{1:H}; \overline{s}_1)
        = \probstap{r_1 = 1, \dots, r_H = 1 \given \overline{s}_1, a_{1:H}}.
\end{equation*}
Here, $r_{1:H}$ are the skill rewards received at each timestep.

With the long-horizon dynamics models $\func{\overline{T}^k\!}{\overline{s}' \given \overline{s}, a^k}$, the objective can be cast as the expectation
\begin{equation*}
    J%
        = \Estap{\overline{s}_{2:H} \sim \overline{T}_{1:H-1\!}\!}{\probstap{r_1 = 1, \ldots, r_H = 1 \given \overline{s}_{1:H}, a_{1:H}}}.
\end{equation*}
By the Markov assumption, rewards are conditionally independent given states and actions. We can express the probability of task success as the product of reward probabilities
\begin{equation*}
    J%
        = \Estap{\overline{s}_{2:H} \sim \overline{T}_{1:H-1\!}\!}{\Pi_{h=1}^H \probstap{r_h = 1 \given \overline{s}_h, a_h}}.
\end{equation*}
Because the skill rewards are binary, the skill success probabilities are equivalent to Q-values:
\begin{align*}
    \probstap{r_h = 1 \given \overline{s}_h, a_h}
        &= \Estap{\overline{s}_{h+1} \sim \overline{T}_{h}}{r_h \given \overline{s}_h, a_h} \\
        &= \func{Q_h}{\!\func{\Gamma_h}{\overline{s}_h}, a_h}.
\end{align*}
The final objective is expressed in terms of Q-values:
\begin{equation}
    J%
        = \Estap{\overline{s}_{2:H} \sim \overline{T}_{1:H-1}}{\Pi_{h=1}^{H} \func{Q_h}{\!\func{\Gamma_h}{\overline{s}_h}, a_h}}. \label{eq:stap-planning-objective}
\end{equation}

This planning objective is simply the product of Q-values evaluated along the trajectory $(\overline{s}_1, a_1, \dots, \overline{s}_H, a_H)$, where the states are predicted by the long-horizon dynamics model: $\overline{s}_2 \sim \func{\overline{T}_1}{\cdot \given \overline{s}_1, a_1}, \dots, \overline{s}_H \sim \func{\overline{T}_{H-1}}{\cdot \given \overline{s}_{H-1}, a_{H-1}}$.\footnote{One might consider maximizing the \textit{sum} of Q-values instead of the product, but this may not reflect the probability of task success. For example, if we want to optimize a sequence of ten skills, consider a plan that results in nine Q-values of $1$ and one Q-value of $0$, for a total sum of $9$. One Q-value of $0$ would indicate just one skill failure, but this is enough to cause a failure for the entire task. Compare this to a plan with ten Q-values of $0.9$. This plan has an equivalent sum of $9$, but it is preferable because it has a non-zero probability of succeeding.}

\subsection{Ensuring Action Plan Feasibility}

A plan skeleton $\tau = \psi_{1:H}$ is feasible only if, for every pair of consecutive skills $\psi_i$ and $\psi_j$, there is a non-zero overlap between the terminal state distribution of $i$ and the initial state distribution of $j$. 
More formally,
\begin{equation}
    \Estap{\overline{s}_i \sim \overline{\rho}_i, a_i \sim \mathcal{A}_i, \overline{s}_j \sim \overline{\rho}_j}{\func{\overline{T}_i\!}{\overline{s}_j \given \overline{s}_i, a_i}}
        > 0, \label{eq:stap-transition-feasibility}
\end{equation}
where $\overline{\rho}_i$ and $\overline{\rho}_j$ are the initial state distributions for skills $\psi_i$ and $\psi_j$, respectively, and $a_i$ is uniformly distributed with respect to action space $\mathcal{A}_i$ for skill $\psi_i$. 
Given a state $\overline{s}_i \sim \overline{\rho}_i$, it is part of the planner's job to determine an action $a_i$ that induces a valid subsequent state $\overline{s}_j \sim \overline{\rho}_j$ if one exists.
Failing to do so constitutes an OOD event for skill $\psi_j$, where the state $\overline{s}_j$ has drifted beyond the region of the state space where $\func{Q_j}{\func{\Gamma_j}{\overline{s}_j}, a_j}$ is well-defined and $\psi_j$ is executable.

Neglecting state distributional shift over an action plan $\xi$ may degrade the quality of objective function $J$ with spuriously high Q-values (\cref{eq:stap-planning-objective}).
Moreover, \cref{eq:stap-transition-feasibility} cannot be explicitly computed to determine the validity of actions because the initial state distributions of all skills $\func{\overline{\rho}^k\!}{\overline{s}^k}$ are unknown.
We can detect OOD states (and actions) by performing UQ on the Q-functions $Q^k(s^k, a^k)$. 
Filtering out Q-values with high uncertainty would result in action plans $\xi$ that are robust (i.e. have low uncertainty) while maximizing the task feasibility objective. We discuss efficient methods for training UQ models on learned Q-functions in \cref{sec:stap-training-scod}.

\begin{figure*}[t]
    \centering
    \begin{minipage}{0.42\textwidth}
        \centering
        \includegraphics[width=\linewidth]{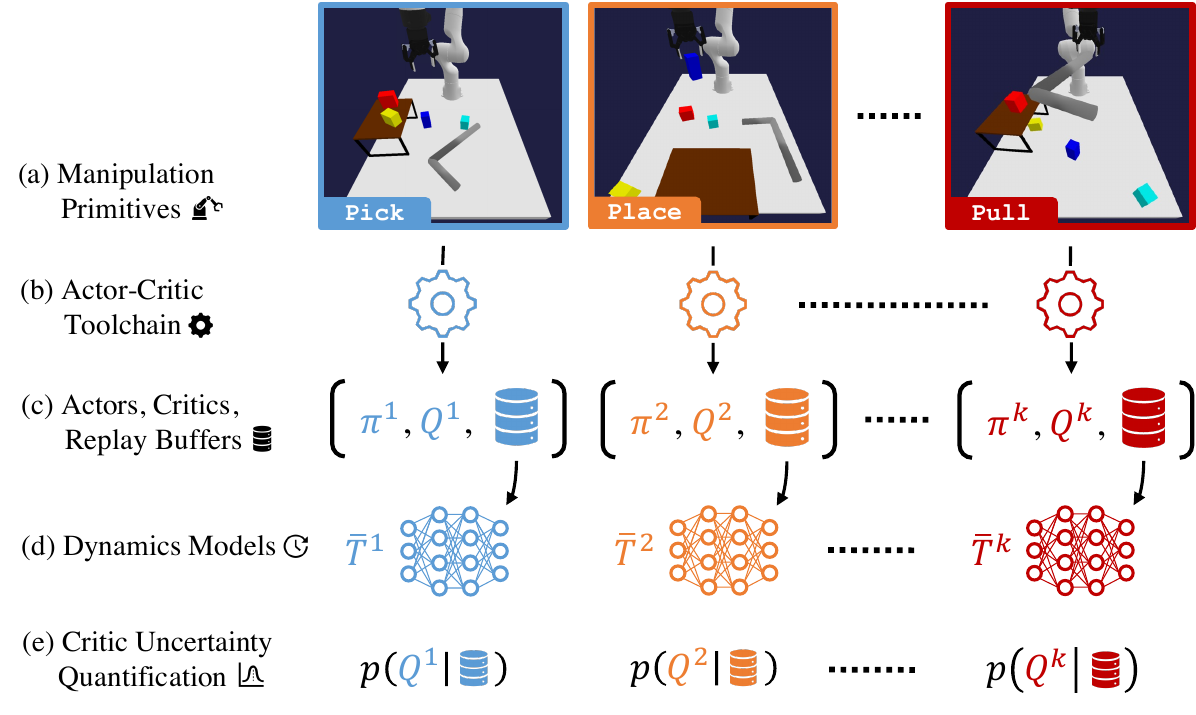}
        \label{fig:stap-taps-training}
    \end{minipage}
    \hfill
    \begin{minipage}{0.56\textwidth}
        \centering
        \includegraphics[width=\linewidth]{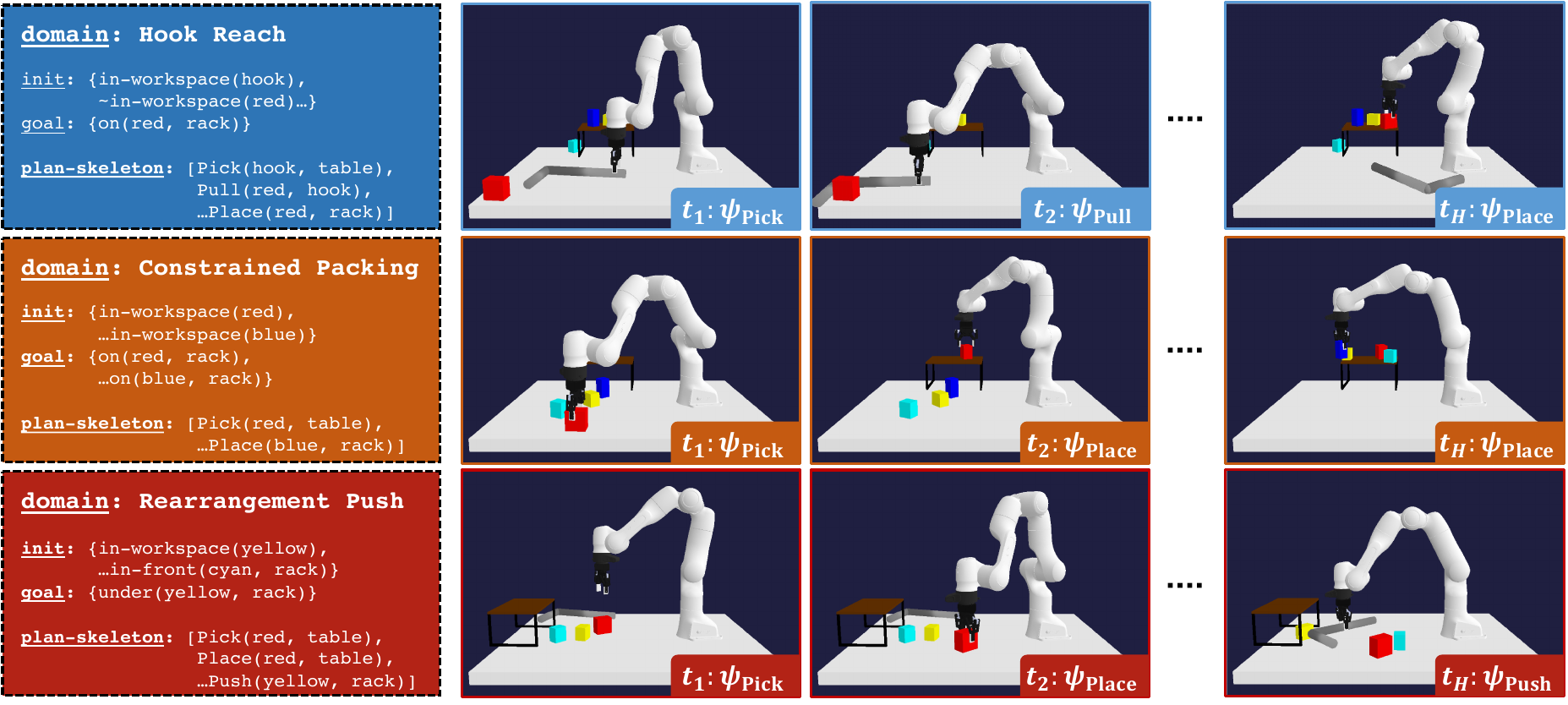}
        \label{fig:stap-taps-tasks}
    \end{minipage}
    \vspace*{-2mm}
    \caption[Policy training pipeline and policy sequence coordination tasks (STAP)]{\small \textbf{Left:} Training pipeline. We train each skill independently on single-step environments with skill-specific rewards. We use the experience collected by each skill to train (1) a dynamics model and (2) a SCOD~\cite{SharmaAzizanEtAl2021} model to predict OOD critic inputs per skill. The benefit of training each skill independently is that we can easily add skills to the library or even mix skill acquisition strategies (e.g. RL, imitation learning, and handcrafted skills). Our planning framework ensures that the skills can be composed to solve any long-horizon task even if the skills were not explicitly trained to perform those tasks. \textbf{Right:} Example evaluation tasks. We evaluate our method on 9 tasks: 3 \texttt{Hook Reach} tasks, where the robot needs to use the hook to bring objects closer, 3 \texttt{Constrained Packing} tasks, where the robot needs to place blocks on the rack, and 3 \texttt{Rearrangement Push} tasks, where the robot needs to remove obstacles to push a target block under the rack. The 9 tasks feature a range of geometric complexities and plan skeleton lengths. \copyright{} 2023 IEEE.}
    \label{fig:stap-taps-methods}
\end{figure*}

\section{Planning Action Sequences}\label{sec:stap-planning}

To find action sequences that maximize the probability of long-horizon task success (\cref{eq:stap-planning-objective}), we use sampling-based optimization techniques: shooting and cross-entropy method (CEM)~\cite{rubinstein1999-cem}. In shooting, we simply sample action plans $\xi = a_{1:H} \in \mathcal{A}_1 \times \dots \times \mathcal{A}_H$ and select the one with the highest predicted objective score. 
CEM is an extension of shooting that iteratively refines the action sampling distribution to fit a fraction of the population with the highest objective scores.

Sampling action plans from uniform distributions may be sufficient for small action spaces and short skill sequences. 
However, this strategy suffers from the curse of dimensionality and may not scale desirably to the large action spaces and long skill sequences that we consider.
Meanwhile, directly executing actions $a^k \sim \func{\pi^k\!}{\cdot | s^k}$ from policies that are trained to solve skill-specific tasks produces myopic behavior that rarely succeeds for long-horizon tasks with complex geometric dependencies between actions.

The policies can be leveraged to initialize a sampling-based search by producing an action plan that is likely to be closer to an optimal plan than one sampled uniformly at random. 
We therefore use two variants of shooting and CEM, termed policy shooting and policy CEM, which sample actions from Gaussian distributions $a^k \sim \func{\mathcal{N}}{\func{\pi^k\!}{s^k}, \sigma}$, where the mean is the action predicted by the policy and the standard deviation is a planning hyperparameter.

\section{Training Skills}\label{sec:stap-training}
\subsection{Policies and Q-functions}
\label{sec:stap-training-policies}

One of the key advantages of our approach is that the policies can be trained independently and then composed at test time to solve unseen sequential tasks. 
For each skill $\psi_k$, we want to obtain a policy $\pi^k: \mathcal{S}^k \rightarrow \mathcal{A}^k$ that solves the task specified by the skill-specific MDP $\mathcal{M}^k$ (\cref{eq:stap-skill-mdp}), along with a Q-function $Q^k$ modeling the policy's expected success
\begin{equation*}
    Q^k(s^k, a^k)
        = \Estap{s'^k \sim \func{T^k}{\cdot \given s^k, a^k}}{\!\func{R^k\!}{s^k, a^k, s'^k}}.
\end{equation*}

Our framework is agnostic to the method for acquiring the policy and Q-function. 
Many deep RL algorithms are able to simultaneously learn the policy (i.e. actor) and Q-function (i.e. critic) with unknown dynamics~\cite{lillicrap2015-ddpg,fujimoto2018-td3}. 
We therefore leverage off-the-shelf RL algorithms to learn a policy and Q-function for each skill (\cref{fig:stap-taps-methods} - Left (c)). 
In our experiments, we specifically use Soft Actor-Critic (SAC)~\cite{haarnoja2018-sac}.
For other policy acquisition methods, policy evaluation can be performed to obtain a Q-function after a policy has been learned. 

\subsection{Dynamics}
\label{sec:stap-training-dynamics}
The dynamics models are used to predict future states at which each downstream Q-function in the plan skeleton will be evaluated. 
We learn a deterministic model $\func{\overline{T}^k}{\overline{s}, a^k}$ for each skill $\psi_k$ using the single-step forward prediction loss
\begin{equation*}
    \func{L_\text{dynamics}}{\overline{T}^k; \overline{s}, a^k, \overline{s}'}
        = \left\| \func{\overline{T}^k\!}{\overline{s}, a^k} - \overline{s}' \right\|_2^2.
\end{equation*}
Each dynamics model $\overline{T}^k$ is trained on the state transition experience $(\overline{s}, a^k, \overline{s}')$ collected during the training of policy $\pi^k$ (\cref{sec:stap-training-policies}), stored in the replay buffer $\mathcal{D}^k$ (\cref{fig:stap-taps-methods} - Left (d)).
Training our dynamics models on existing state transitions is efficient and circumvents the challenges associated with learning dynamics in the context of a long-horizon task~\cite{janner2019-mbpo}.

\subsection{Uncertainty Quantification}
\label{sec:stap-training-scod}
Measuring the epistemic uncertainty over the Q-values allows us to identify when dynamics-predicted states and planned actions drift OOD for downstream critics $Q^k$.
We leverage recent advances in neural network UQ to obtain an explicit Gaussian posterior predictive distribution
\begin{equation}
    \probstap{Q^k \given s^k, a^k, \mathcal{D}^k; w^k}
        = \func{\mathcal{N}}{\mu_{Q^k}, \sigma_{Q^k}; w^k}
    \label{eq:stap-scod-posterior}
\end{equation}
with sketching curvature for OOD detection (SCOD)~\cite{SharmaAzizanEtAl2021}. 
SCOD computes the weights $w^k$ that parameterizes the posterior distribution over each critic $Q^k$ using only the experience $(\overline{s}, a^k) \sim \mathcal{D}^k$ collected over the course of training policy $\pi^k$ (\cref{fig:stap-taps-methods} - Left (e)).
An advantage of SCOD over common UQ techniques~\cite{ensemblereview2021,GalZoubin2016} is that it imposes no train-time dependencies on any algorithms used in our framework.

\section{Experiments}\label{sec:stap-experiments}
In our experiments, we test the following hypotheses:
\begin{description}
    \item[H1] Maximizing the product of learned Q-functions (\cref{eq:stap-planning-objective}) translates to maximizing long-horizon task success.
    \item[H2] Skills trained with our framework are able to generalize to unseen long-horizon tasks by optimizing \cref{eq:stap-planning-objective}.
    \item[H3] Our planning method can be combined with a task planner and UQ (\cref{eq:stap-scod-posterior}) to solve TAMP problems.
\end{description}

We evaluate our method on a 3D manipulation domain with 4 skills: $\action{Pick a b}$: pick $\obj{a}$ from $\obj{b}$; $\action{Place a b}$: place $\obj{a}$ onto $\obj{b}$; $\action{Pull a hook}$: pull $\obj{a}$ into the robot's workspace with a $\obj{hook}$; and $\action{Push a hook}$: push $\obj{a}$ with a $\obj{hook}$.

The long-horizon state space $\mathcal{\overline{S}}$ is a sequence of low-dimensional object states that contains information such as 6D poses. 
The policy state spaces $\mathcal{S}^k$ are constructed so that the first $m$ object states correspond to the $m$ arguments of the corresponding skill. 
For example, a state for the policy of $\action{Pick box rack}$ will contain first the $\obj{box}$'s state, then the $\obj{rack}$'s state, followed by a random permutation of the remaining object states.
The policy action spaces $\mathcal{A}^k$ are all 4D. 
For example, a policy-predicted action for $\action{Pick a b}$ specifies the 3D grasp position of the end-effector relative to the target $\obj{a}$ and orientation about the world $z$-axis.

Our evaluation is on 9 different long-horizon tasks (i.e. plan skeletons $\tau$). The tasks cover a range of symbolic and geometric complexities (\cref{fig:stap-taps-methods} - Right), with plan skeleton lengths ranging from 4 to 10 skills. 
Each task involves geometric dependencies between actions, which motivates the need for planning. 
We use 100 randomly generated instances (i.e. object configurations) for evaluation on each task.

\subsection{Product of Q-Functions Approximates Task Success (H1)}

We test \textbf{H1} by comparing STAP to an \textbf{Oracle} baseline that runs forward simulations with policy shooting to find action plans that achieve ground-truth task success. Our method uses learned Q-functions and dynamics to predict task success as the product of Q-functions. We expect that planning with this objective will come close to matching the task success upper bound provided by \textbf{Oracle}.

We compare several planning methods: \textbf{Policy Shooting} and \textbf{Policy CEM}, which use the learned policies to initialize the action sampling distributions (\cref{sec:stap-planning}), as well as \textbf{Random Shooting} and \textbf{Random CEM}, which use uniform action priors. 
We also compare with \textbf{Greedy}, which does not plan but greedily executes the skills. 
The evaluation metrics are ground-truth task success, sub-goal completion rate (what percentage of skills in a plan are successfully executed), and predicted task success computed from \cref{eq:stap-planning-objective}.

Due to the significant amount of time required to run forward simulations for \textbf{Oracle}, we limit the number of sampled trajectories evaluated during planning to 1000 for all methods. 
This is not enough to succeed at the most complex tasks, and thus, we evaluate on the simplest task from the \texttt{Hook Reach} and \texttt{Constrained Packing} domains.

The results from both tasks are averaged and presented in \cref{fig:stap-exp_a}. As expected, \textbf{Oracle} achieves the highest success rate, although not perfect because 1000 samples are not enough to solve all of the tasks. \textbf{Policy CEM} nearly matches \textbf{Oracle}'s success rate, which demonstrates that maximizing the product of Q-functions is a good proxy for maximizing task success. \textbf{Policy CEM} also exhibits a low success prediction error, which demonstrates that the learned Q-functions and dynamics generalize well to these unseen long-horizon tasks. Meanwhile, planning with these learned models runs 4 orders of magnitude faster than \textbf{Oracle} and does not require ground-truth knowledge about the environment state or dynamics.

\textbf{Policy Shooting} performs slightly worse than \textbf{Policy CEM}, which demonstrates CEM's strength in finding local maxima through iterative refinement. \textbf{Random CEM} and \textbf{Random Shooting} perform quite poorly, indicating that the planning space is too large (16D for these tasks) for random sampling. \textbf{Greedy} performs strongly, perhaps indicating that these simpler tasks can be solved without planning.

\begin{figure}[t]
    \centering
    \includegraphics[width=\columnwidth]{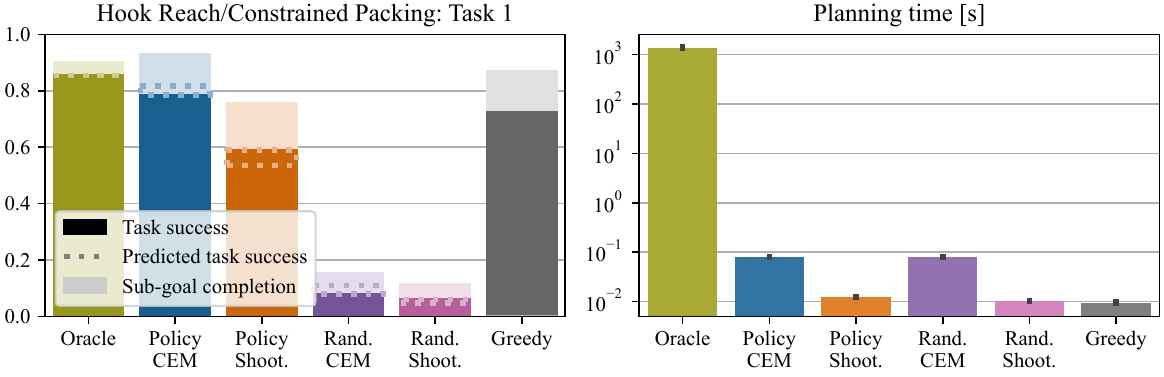}
    \caption[Ablation on sequence feasibility estimation and computation time (STAP)]{\small A small-scale experiment comparing the performance of our method to \textbf{Oracle} planning. The left plot shows the average success rates across two domains (\texttt{Hook Reach} and \texttt{Constrained Packing}). The dark bars indicate the ground truth task success, and the light bars indicate sub-goal completion rate, which measures how close the plan was to successfully completing the task. The predicted task success computed from the product of Q-values is indicated by a dotted line. Our method with \textbf{Policy CEM} is able to nearly match the success rate of \textbf{Oracle} while taking 4 orders of magnitude less time, as shown in the plot on the right. \copyright{} 2023 IEEE.}
    \label{fig:stap-exp_a}
\end{figure}

\begin{figure}[t]
    \centering
    \includegraphics[width=\columnwidth]{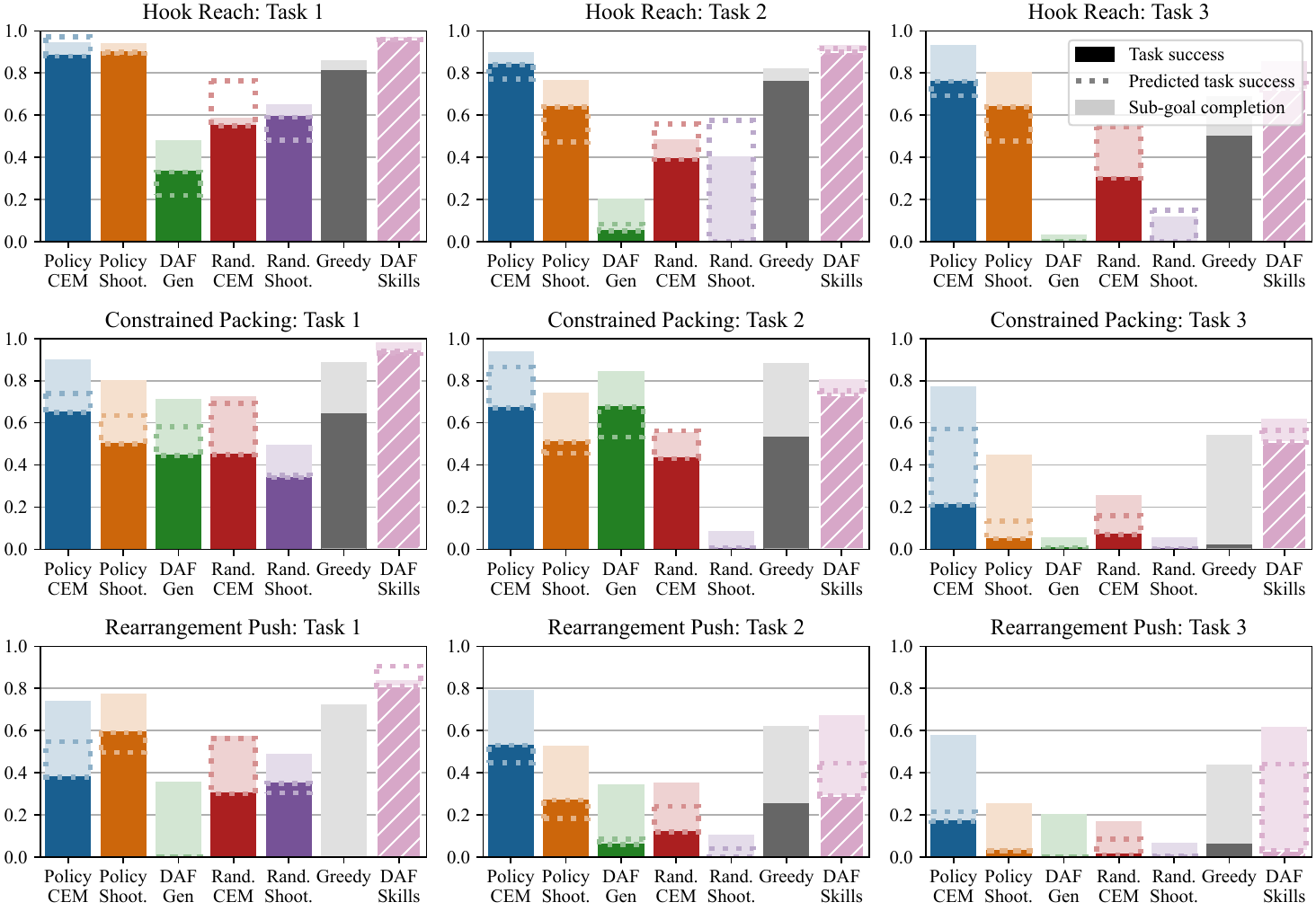}
    \caption[Policy sequence coordination results in simulation (STAP)]{\small
        Planning experiment with 3 domains, each with 3 tasks. 
        Our method with \textbf{Policy CEM} is able to generalize to all of these tasks without ever seeing them during training.
        On 6 out of the 9 tasks, our method either matches or outperforms \textbf{DAF-Skills}, which is trained directly on the evaluation task. 
        \textbf{DAF-Gen} shows the generalization performance of the \textbf{DAF-Skills} models when evaluated on unseen tasks within the same domain. \copyright{} 2023 IEEE.
    }
    \label{fig:stap-exp_abc}
\end{figure}

\subsection{STAP Skills Generalize to Long-Horizon Tasks (H2)}
\label{sec:stap-experiments-planning}

In this experiment, we test the ability of our framework to solve 9 long-horizon tasks with geometric dependencies between actions. We compare against DAF~\cite{xu2021-daf}, a state-of-the-art method for learning to solve TAMP problems. 
As task planning is outside the scope of this chapter, we omit DAF's skill proposal network and compare only to the skills trained with DAF (\textbf{DAF-Skills}), which are comprised of dynamics and affordance models.
DAF's planning objective is similar to ours, except that it evaluates the product of affordances rather than Q-functions. 
We give \textbf{DAF-Skills} the same plan skeleton $\tau$ that is given to our method and augment DAF's shooting planner with CEM for a more even comparison.

Like other model-based RL methods, DAF requires training on a set of long-horizon tasks that is representative of the evaluation task distribution. 
We therefore train one \textbf{DAF-Skills} model per task (9 total) and run evaluation on the same task. 
We also test the ability of these models to generalize to the other two tasks within the same domain (\textbf{DAF-Gen}). 
Since \textbf{DAF-Skills} is trained on its evaluation task, we expect it to perform at least as well as STAP, if not better.
However, we expect \textbf{DAF-Gen} to perform slightly worse than STAP, since the evaluation tasks differ from the training tasks, even if they are similar.
We train all models for 48 hours each and allow $1000$ samples per dimension for planning.

The results are presented in \cref{fig:stap-exp_abc}. 
Our method with \textbf{Policy CEM} achieves competitive success rates with \textbf{DAF-Skills} on 4 out of the 9 tasks and outperforms it on 2 tasks with highly complex action dependencies (\texttt{Rearrangement Push}). 
While \textbf{DAF-Gen} matches the performance of \textbf{Policy CEM} on 2 tasks, it gets relatively low success on the others. 
This indicates that skills trained on one long-horizon task may not  effectively transfer to other tasks with similar action dependencies. 
Our method of training skills in independent environments and then generalizing to long-horizon tasks via planning is efficient from a training perspective, since the same trained skills can be used for all downstream tasks.

\subsection{STAP Can be Extended for TAMP with UQ (H3)}
\label{sec:stap-experiments-tamp}
In this experiment, we combine our framework with a PDDL task planner as described in \cref{sec:stap-planning-tamp} and evaluate it on two TAMP problems. 
In \texttt{Hook Reach}, the robot needs to decide the best way to pick up a block, which may or may not be in its workspace. 
In \texttt{Constrained Packing}, the robot needs to place a fixed number of objects on the rack but is free to choose among any of the objects on the table. 
To mimic what the robot might find in an unstructured, real-world environment, some of these objects are distractor objects that are initialized in ways not seen by the skills during training (e.g. the blocks can be stacked, placed behind the robot base, or tipped over). 
The task planner may end up selecting these distractor objects for placing on the rack, but since the skills have not been trained to handle these objects, their predicted success (Q-values) may be unreliable. 
UQ is particularly important for such scenarios, so we introduce \textbf{SCOD Policy CEM}, which filters out candidate action plans with high uncertainty in the predicted task success score (\cref{eq:stap-planning-objective}).
That is, the $n$ action plans with the highest skill uncertainties (\cref{eq:stap-scod-posterior}) are not considered for execution.

The results are presented in \cref{fig:stap-tamp_results}. \textbf{Policy CEM} achieves 97\% success on the \texttt{Hook Reach} TAMP problem, while \textbf{SCOD Policy CEM} suffers a slight performance drop. However, for \texttt{Constrained Packing}, which contains OOD states, \textbf{SCOD Policy CEM} strongly outperforms the other methods. Exploring different ways to integrate UQ into our planning framework is a promising direction for future work.

\begin{figure}
    \centering
    \includegraphics[width=\columnwidth]{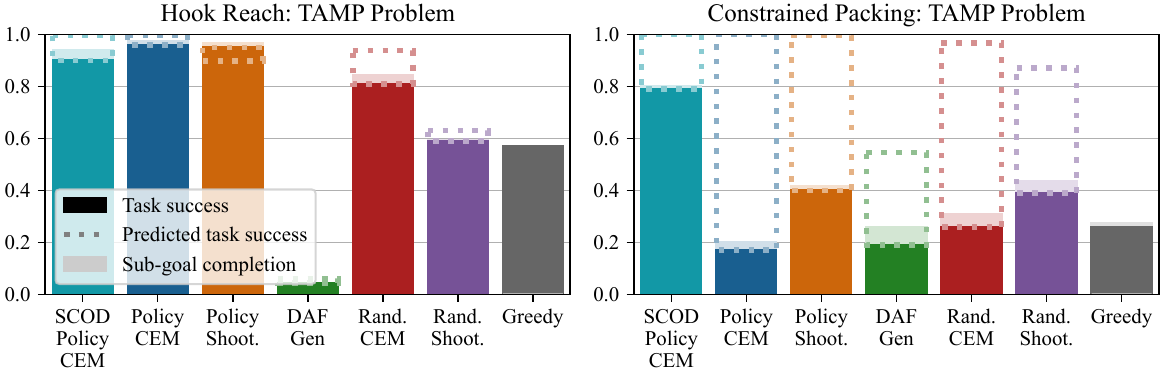}
    \caption[Policy task and motion planning results in simulation (STAP)]{\small
        Integration of our planning framework with task planning and UQ to solve TAMP problems. 
        The poor performance of the methods without SCOD in the \texttt{Constrained Packing} problem highlights the importance of UQ when attempting to solve unseen TAMP problems with learned skills.
        \copyright{} 2023 IEEE.
    }
    \label{fig:stap-tamp_results}
\end{figure}

\subsection{Real-World Sequential Manipulation}

We demonstrate that skills trained with our framework can be used to perform sequential manipulation tasks in a real robot environment. We take RGB-D images from a Kinect v2 camera and use manually tuned color thresholds to segment objects in the scene. With these segmentations, we estimate object poses using the depth image, which is then used to construct the initial environment state $\overline{s}_1$. Qualitative results are provided in the supplementary video.

\section{Conclusion}\label{sec:stap-conclusion}

We present a framework for sequencing task-agnostic policies that have been trained independently. The key to generalization is planning actions that maximize the probability of long-horizon task success, which we model using the product of learned Q-values. This requires learning a dynamics model to predict future states and using UQ to filter out OOD states that the skills do not support. The result is a library of skills that can be composed to solve arbitrary long-horizon tasks with complex geometric dependencies between actions. 

\subsection{Future Work}\label{sec:stap-future-work}
Future work includes the investigation of methods for scaling skills to high-dimensional observations, combining the library of learned skills with a set of handcrafted skills, and exploring planning objectives that capture other desirable properties of trajectories, beyond their geometric feasibility.

\chapter{Searching for Feasible Policy Sequences from Language}\label{chapter:text2motion}

\section{Introduction}\label{sec:text2motion-introduction}

In \cref{chapter:stap}, we examined a central challenge in solving real-world, long-horizon tasks with learned behavior policies: even when individual behaviors can be executed reliably, unmodeled dependencies between them may critically affect the feasibility and success of an entire sequence. To address this, we introduced a method that estimates the success probability of a given policy sequence and uses this estimate to coordinate policy outputs at deployment time. This chapter builds on that foundation by observing that, for many tasks, multiple candidate sequences of behaviors may achieve the same goal. Generating such sequences---commonly referred to as task planning or symbolic planning---and selecting those with the highest likelihood of success therefore represents an additional avenue for improving reliability. Because (1) tasks are often most naturally specified in language (e.g., ``shelve the dishes''), and (2) pretrained foundation models possess broad, internet-scale knowledge relevant to everyday activities, we explore the use of \textit{Large Language Models} as symbolic planners to flexibly generate and search over candidate policy sequences directly from natural language task descriptions.

Long-horizon robot planning is traditionally formulated as a joint symbolic and geometric reasoning problem, where the symbolic reasoner is supported by a formal logic representation (e.g. first-order logic~\cite{aeronautiques1998pddl}).
Such systems can generalize within the logical planning domain specified by experts. 
However, many desirable properties of plans that can be \rone{conveniently expressed in language by non-expert users} may be cumbersome to specify in formal logic.
Examples include the specification of user intent or preferences.

\begin{figure}[t!]
    \centering
    \includegraphics[width=0.90\textwidth]{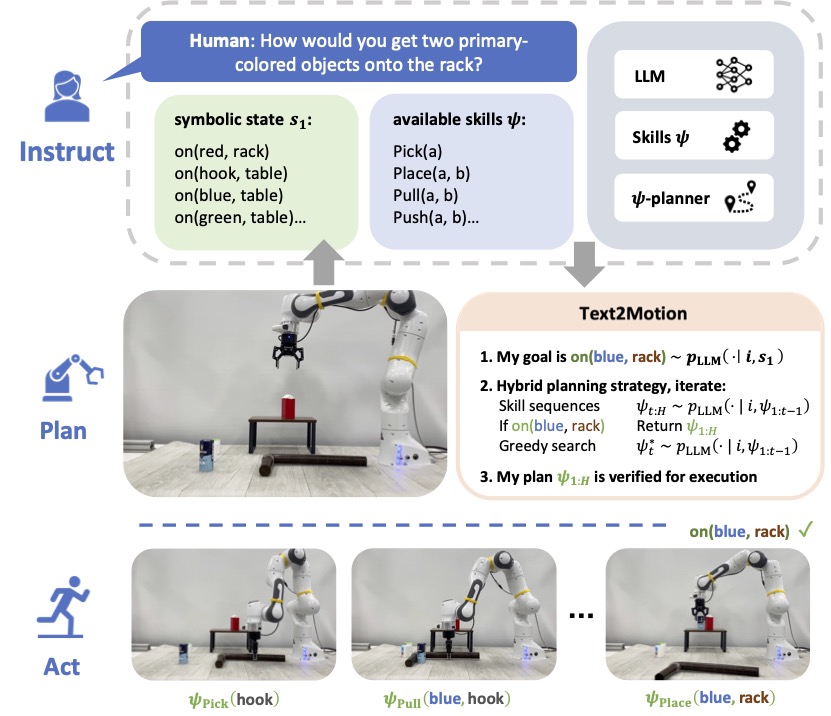}
    \caption[Overview of Text2Motion]{
        To carry out the instruction ``get two primary-colored objects onto the rack,'' the robot must apply symbolic reasoning over the scene description and language instruction to deduce what skills should be executed to acquire a second primary-colored object, after noticing that a red object is already on the rack (i.e. \graytext{on(red, rack)}). It must also apply geometric reasoning to ensure that skills are sequenced in a manner that is likely to succeed. Unlike prior work~\cite{ahn2022-saycan, pmlr-v205-huang23c} that myopically executes skills at the current timestep, \ttmnb{} constructs \textit{sequences} of skills and coordinates their geometric dependencies with geometric feasibility planning~\cite{taps-2022}. Upon planning the skill sequence \graytext{Pick(hook)}, \graytext{Pull(blue, hook)}, \graytext{Pick(blue)}, \graytext{Place(blue, rack)}, our method computes a grasp position on the hook that enables pulling the blue object into the robot workspace so that it can be successfully picked up in the next step.
        More details can be found on the Text2Motion website: \url{https://sites.google.com/stanford.edu/text2motion}.
        \copyright{} 2023 Springer Nature.
    }
    \label{fig:text2motion-teaser}
\end{figure}

The emergence of Large Language Models (LLMs)~\cite{bommasani2021opportunities} as a task-agnostic reasoning module presents a promising pathway to general robot planning capabilities.
Several recent works~\cite{ahn2022-saycan, pmlr-v205-huang23c, code-as-policies-2022, wu2023tidybot} capitalize on their ability to perform task planning for robot systems without needing to manually specify symbolic planning domains.
\rone{
Nevertheless, these prior approaches adopt myopic or open-loop execution strategies, trusting LLMs to produce correct plans without verifying them on the symbolic or geometric level.
Such strategies are challenged in long-horizon settings, where the task planning abilities of even the most advanced LLMs appear to degrade~\cite{llms-cant-plan-2022}, and the overall success of a seemingly correct task plan depends as well on how it is executed to ensure long-horizon feasibility.
Therefore, we ask in this chapter: how can we verify the correctness and feasibility of LLM-generated plans prior to execution?}

We propose \ttm{}, a language-based planning framework that interfaces an LLM with a library of learned skills and a geometric feasibility planner~\cite{taps-2022} to solve complex sequential manipulation tasks (\cref{fig:text2motion-teaser}).
Our contributions are two-fold: (i) a hybrid LLM planner that synergistically integrates shooting-based and search-based planning strategies to construct geometrically feasible plans for tasks not seen by the skills during training; and (ii) a plan termination method that infers goal states from a natural language instruction to verify the completion of plans before executing them. 
We find that our planner achieves a success rate of 82\% on a suite of challenging table top manipulation tasks, while prior language-based planning methods achieve a 13\% success rate.

\section{Related Work}\label{sec:text2motion-related-work}
\subsection{Language for Robot Planning}
\label{subsec:text2motion-language-literature}
Language is increasingly being explored as a medium for solving long-horizon robotics problems.
For instance, {\em Language-conditioned policies\/} (LCPs) are not only used to learn short-horizon skills~\cite{stepputtis2020language, jang2021bczero, concept2robot-2021, cliport-2022, perceiver-2022, vima-2022}, but also long-horizon policies~\cite{mees2022calvin, rt12022arxiv, dalal2023imitating}.
However, LCPs require expensive data collection and training procedures if they are to generalize to a wide distribution of long-horizon tasks with diverse instructions.

Several recent works leverage the generative qualities of LLMs by prompting them to predict long-horizon plans.
\cite{zeroshot-llms-2022} grounds an LLM planner to admissible action sets for task planning, \cite{silver2022pddl, liu2023llm+} explore the integration of LLMs with PDDL~\cite{aeronautiques1998pddl}, and \cite{wang2023describe, skreta2023errors} focuses on task-level replanning with LLMs.
Tangential works shift the representation of plans from action sequences to code~\cite{code-as-policies-2022, progprompt-2022, zelikman2022parsel, vemprala2023chatgpt} and embed task queries, robot actions, solution samples, and fallback behaviors as programs in the prompt. 
In contrast to these works, which primarily address challenges in task planning with LLMs, we focus on verifying LLM-generated plans for feasibility on the geometric level.

Closest in spirit to our approach are SayCan~\cite{ahn2022-saycan} and Inner Monologue (IM)~\cite{pmlr-v205-huang23c} which at each timestep score the \textit{usefulness} and \textit{feasibility} of all possible skills and execute the one with the highest score.
Termination occurs when the score of the $\texttt{stop}$ ``skill'' is larger than any other.
IM provides additional sources of feedback to the LLM in the form of skill successes and task-progress cues. 

While SayCan and IM are evaluated on a diverse range of tasks, there are several limitations that impede their performance in the settings we study.
First, by only myopically executing the next skill at each timestep, they may fail to account for geometric dependencies that exist over the extent of a skill sequence. 
For an example, see \cref{fig:text2motion-teaser}.
Second, they do not explicitly predict a multi-step plan, which prevents verification of desired properties or outcomes prior to execution.
Examples of such properties could include whether the final state induced by the plan satisfies symbolic constraints or whether the plan adheres to safety criteria.
Lastly, these methods ignore the uncertainty of skill feasibility predictions (i.e. affordances), which \cite{taps-2022} demonstrates is important when sequencing learned skills to solve long-horizon tasks.
By addressing these limitations, \ttm{} outperforms SayCan and IM by a large margin on tasks with geometric dependencies, as demonstrated in the experiments.

\subsection{Task and Motion Planning}
\label{subsec:text2motion-tamp-literature}
Task and Motion Planning (TAMP) refers to a problem setting in which a robot solves long-horizon tasks through symbolic and geometric reasoning~\cite{kaelbling2012integrated, garrett2021integrated}.
The hierarchical approach~\cite{kaelbling2011hierarchical} characterizes the most common family of solution methods.
Such methods typically employ a) a symbolic task planner~\cite{bonet2001planning, helmert2006-fd} to produce a candidate plan skeleton, and b) a motion planner to verify the plan skeleton for its geometric feasibility and compute a motion trajectory subject to robot and environmental constraints~\cite{garrett2020-pddlstream, toussaint2015-lgp, driess2019-hlgp}.

For complex tasks, classical TAMP solvers~\cite{kaelbling2011hierarchical, lagriffoul2014efficiently, toussaint2015-lgp, dantam2016incremental, bidot2017geometric, garrett2020-pddlstream} may iterate between task planning and motion planning for minutes until a plan is found.
To amortize planning costs, works learn sampling distributions~\cite{wang2018active, xu2021-daf, kim2019adversarial, kim2022representation, taps-2022}, visual feasibility heuristics~\cite{driess2020-dvr, driess2020-dvh, driess2021-lpr}, low-level controllers~\cite{driess2021-lgr, silver2022learning}, or state sparsifiers~\cite{chitnis2021camps, silver2021-ploi}, from datasets of solutions computed by classical TAMP solvers.
Another line of works learn symbolic representations for TAMP~\cite{kroemer2016learning, ames2018learning, konidaris2018skills, silver2021learning, wang2021learning, curtis2022discovering, chitnis2022learning, silver2022learning}, often from task-specific symbolic transition experience.

\rone{
As is common in TAMP, \ttm{} also assumes knowledge of task-relevant objects and their poses in order to plan feasible trajectories for long-horizon tasks.
However, central to our investigation is the use of LLMs instead of symbolic task planners often used in TAMP~\cite{garrett2021integrated}, and language as convenient medium to express tasks that may be cumbersome to specify in formal logic (e.g. user preferences~\cite{wu2023tidybot}). 
Accordingly, we address early challenges concerning the reliable use of LLMs (discussed in \cref{subsec:text2motion-language-literature}) in the long-horizon settings typically solved by TAMP.
\ttm{} thereby presents several qualitative differences from TAMP: i) the ability to interpret free-form language instructions for the construction of multi-step plans, and ii) the capacity to reason over an unrestricted set of object classes and object properties, both of which are supported by the commonsense knowledge of LLMs~\cite{brown2020language}.
We leave the extension of our framework to open-world settings (e.g. via environment exploration~\cite{chen2022open} or interaction~\cite{curtis2022-unknown}) to future work.}

\section{Problem Setup}
\label{sec:text2motion-problem-setup}
We aim to solve long-horizon sequential manipulation problems that require symbolic and geometric reasoning from a natural language instruction $i$ and the initial state of the environment $s_1$. 
\rtwo{We assume a closed-world setting, whereby the initial state $s_1$ contains knowledge of task-relevant objects and their poses as provided by an external perception system (\cref{appx-sub:text2motion-hardware-setup}).
Fulfillment of the instruction $i$ corresponds to achieving a desired goal configuration of the task-relevant objects which can be symbolically expressed with a closed set of predicates (\cref{appx-sub:text2motion-scene-descr-symbolic}).}

\subsection{LLM and Skill Library}
\label{subsec:text2motion-prelminaries}
We assume access to an LLM and a library of skills $\mathcal{L}^\psi = \{\psi^1, \ldots, \psi^N\}$.
Each skill $\psi$ consists of a policy $\pi(a | s)$ and a parameterized manipulation primitive $\phi(a)$~\cite{felip2013manipulation}, and is associated with a contextual bandit, or a single-timestep Markov Decision Process (MDP):
\begin{equation}
    \label{eq:text2motion-skill-mdp}
    \mathcal{M} = (\mathcal{S}, \mathcal{A}, T, R, \rho), 
\end{equation}
where $\mathcal{S}$ is the state space, $\mathcal{A}$ is the action space, $T(s' | s, a)$ is the transition model, $R(s, a, s')$ is the binary reward function, and $\rho(s)$ is the initial state distribution.
When a skill $\psi$ is executed, an action $a \in \mathcal{A}$ is sampled from its policy $\pi(a|s)$ and fed to its primitive $\phi(a)$, which consumes the action and executes a series of motor commands on the robot. 
If the skill succeeds, it receives a binary reward of $r$ (or $\neg r$ if it fails).
We subsequently refer to policy actions $a \in \mathcal{A}$ as \textit{parameters} for the primitive, which,  depending on the skill, can represent grasp poses, placement locations, and pulling or pushing distances (\cref{appx-sub:text2motion-learning-models}).

A timestep in our environment corresponds to the execution of a single skill.
We assume that each skill comes with a language description and that methods exist to obtain its policy $\pi(a|s)$, Q-function $Q^\pi(s, a)$, and dynamics model $T^\pi(s' | s, a)$. 
Our framework is agnostic to the approach used to obtain these models.
We also assume a method to convey the environment state $s \in \mathcal{S}$ to the LLM as natural language. 

\subsection{The Planning Objective}
\label{subsec:text2motion-planning-objective}
Our objective is to find a plan in the form of a sequence of skills $[\psi_1, \ldots, \psi_H]$ (for
notational convenience, we hereafter represent sequences
with range subscripts, e.g. $\psi_{1:H}$) that is both likely to satisfy the instruction $i$ and can be successfully executed from the environment's initial state $s_1$. 
This objective can be expressed as the joint probability of skill sequence $\psi_{1:H}$ and binary rewards $r_{1:H}$ given the instruction $i$ and initial state $s_1$:
\begin{equation}
    \begin{split}
        &p(\psi_{1:H}, r_{1:H} \mid i, s_1) \\
        &\quad\quad= p(\psi_{1:H} \mid i, s_1)\, p(r_{1:H} \mid i, s_1, \psi_{1:H}).
    \end{split} \label{eq:text2motion-tamp-score}
\end{equation}
The first term in this product $p(\psi_{1:H} \mid i, s_1)$ considers the probability that the skill sequence $\psi_{1:H}$ will satisfy the instruction $i$ from a symbolic perspective. 
However, a symbolically correct skill sequence may fail during execution due to kinematic constraints of the robot or geometric dependencies spanning the skill sequence.
We must also consider the \textit{success probability} of the skill sequence $\psi_{1:H}$ captured by the second term in this product $p(r_{1:H} \mid i, s_1, \psi_{1:H})$.
The success probability depends on the parameters $a_{1:H}$ fed to the underlying sequence of primitives $\phi_{1:H}$ that control the robot's motion:
\begin{equation}
    \label{eq:text2motion-motion-parameter-score}
    p(r_{1:H} \mid i, s_1, \psi_{1:H}) = p(r_{1:H} \mid s_1, a_{1:H}).
\end{equation}
\cref{eq:text2motion-motion-parameter-score} represents the probability that skills $\psi_{1:H}$ achieve rewards $r_{1:H}$ when executed from initial state $s_1$ with parameters $a_{1:H}$; which is independent of the instruction $i$.
If just one skill fails (reward $\neg r$), then the entire plan fails.

\subsection{Geometric Feasibility Planning}
\label{subsec:text2motion-geometric-feasibility-planning}
The role of geometric feasibility planning is to maximize the success probability (\cref{eq:text2motion-motion-parameter-score}) of a skill sequence $\psi_{1:H}$ by computing an optimal set of parameters $a_{1:H}$ for the underlying primitive sequence $\phi_{1:H}$.
This process is essential for finding plans that maximize the overall planning objective in \cref{eq:text2motion-tamp-score}.
In our experiments, we leverage Sequencing Task-Agnostic Policies (STAP)~\cite{taps-2022}.

STAP resolves geometric dependencies across the skill sequence $\psi_{1:H}$ by maximizing the product of step reward probabilities of parameters $a_{1:H}$:
\begin{equation}
    \label{eq:text2motion-taps-objective}
    a_{1:H}^* = \arg \max_{a_{1:H}} \, \Estap{s_{2:H}}{\prod_{t=1}^H p(r_t \mid s_t, a_t)},
\end{equation}
where future states $s_{2:H}$ are predicted by dynamics models $s_{t+1} \sim T^{\pi_t}(\cdot | s_t, a_t)$. 
Note that the reward probability $p(r_t \mid s_t, a_t)$ is equivalent to the Q-function $Q^{\pi_t}(s_t, a_t)$ for skill $\psi_t$ in a contextual bandit setting with binary rewards (\cref{eq:text2motion-skill-mdp}).
The success probability of the optimized skill sequence $\psi_{1:H}$ is thereby approximated by the product of Q-functions evaluated from initial state $s_1$ along a sampled trajectory $s_{2:H}$ with parameters $a^*_{1:H}$:
\begin{equation}
    \label{eq:text2motion-stap-score}
    p(r_{1:H} \mid s_1, a_{1:H}) \approx \prod_{t=1}^H Q^{\pi_t}(s_t, a^*_t).
\end{equation}

In principle, our framework is agnostic to the specific approach used for geometric feasibility planning, requiring only that it is compatible with the skill formalism defined in \cref{subsec:text2motion-prelminaries} and provides a reliable estimate of \cref{eq:text2motion-motion-parameter-score}.

\section{Methods}
\label{sec:text2motion-text2motion}

The core idea of our approach is to ensure the geometric feasibility of an LLM task plan---and thereby its correctness---by predicting the success probability (\cref{eq:text2motion-motion-parameter-score}) of learned skills that are sequenced according to the task plan.
In the following sections, we outline two strategies for planning with LLMs and learned skills: a shooting-based planner and a search-based planner.
We then introduce the full planning algorithm, \ttm{}, which synergistically integrates the strengths of both strategies.
These strategies represent different ways of maximizing the overall planning objective in \cref{eq:text2motion-tamp-score}.

\begin{figure*}
    \centering
    \includegraphics[width=0.98\textwidth]{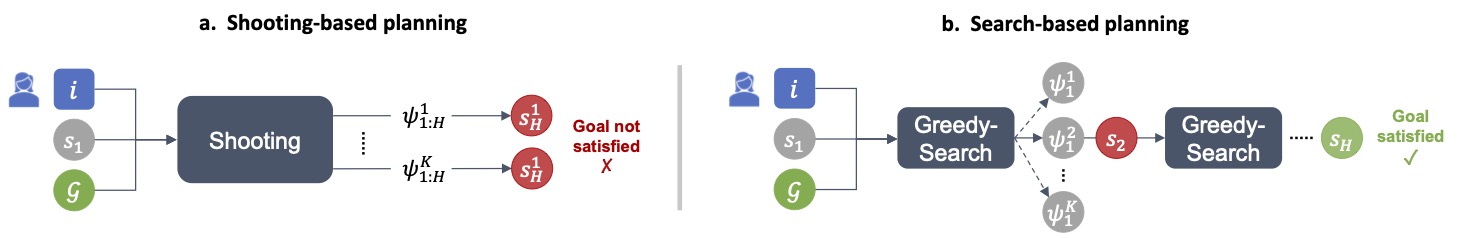}
    \caption[System architectures for shooting- and search-based policy planners (Text2Motion)]{\textbf{\shnb{} and \senb{} planning overview}. 
    Both \shnb{} and \senb{} planners use the LLM to predict the set of valid goal states given the user's natural language instruction and a description of the current state of the environment. These predicted goals are used to decide when the instruction is satisfied and planning can terminate. \textbf{Left:} The \shnb{} strategy uses the LLM to propose full skill sequences first and then runs geometric feasibility planning afterwards. As shown in the experiments, this approach fails when the space of candidate task plans is large but few skill sequences are geometrically feasible. \textbf{Right:} In the \senb{} strategy, the LLM is used to propose $K$ candidate skills with the top LLM scores. The geometric feasibility planner then evaluates the feasibility of each candidate skill, and the one with the highest product of LLM and geometric feasibility scores is selected. The successor state of this skill is predicted by the geometric feasibility planner's dynamics model. If the successor state does not satisfy any of the predicted goals, then it is given to the LLM to plan the next skill. If a goal is satisfied, then the planner returns the skill sequence for execution. By interleaving LLM task planning with geometric feasibility planning at each planning iteration, \senb{} is able to reliably find feasible plans across the different families of tasks we study in the experiments. 
    \copyright{} 2023 Springer Nature.
    }
    \label{fig:text2motion-system}
\end{figure*}

\subsection{Goal Prediction} 
\label{subsec:text2motion-goal-prediction}
Plans with high overall objective scores (\cref{eq:text2motion-tamp-score}) are not guaranteed to satisfy their instruction.
Consider the instruction ``move all the dishes from the table to the sink'' issued in an environment with two dishes on the table. 
While a plan that picks and places one of the two dishes in the sink may have a high language model likelihood and success probability, it fails to satisfy the instruction. 

The first step in all planning strategies is to convert the language instruction into a goal condition that can be checked against a candidate sequence of skills.
\rone{Given an instruction $i$, a set of objects $O$ in the scene, and a library of predicate classifiers $\mathcal{L}^\chi=\{\chi^1,\ldots,\chi^M\}$, we use the LLM to predict a set of $|\mathcal{G}|$ symbolic goal propositions $\mathcal{G} = \{g_1, \ldots, g_j\}$ that would satisfy the instruction. 
Each goal proposition $g\in\mathcal{G}$ is a set of predicates grounded over objects in the scene.
Each predicate is a binary-valued function over objects and has a one-to-one correspondence with a predicate classifier $\chi\in\mathcal{L}^\chi$ that implements the predicate (details in \cref{appx-sub:text2motion-scene-descr-symbolic}).
We define a satisfaction function $\func{F^\mathcal{G}_{\text{sat}}}{s}: \mathcal{S} \rightarrow \{0, 1\}$ which takes as input a geometric state $s$ and evaluates to $1$ if any goal proposition $g\in\mathcal{G}$ predicted by the LLM holds in state $s$.}

A sequence of skills $\psi_{1:H}$ is said to satisfy the instruction $i$ \textit{iff}:
\begin{equation}
    \label{eq:text2motion-satisfaction-function}
    \exists\: s \in s_{2:H+1}: F^\mathcal{G}_{\text{sat}}(s) = 1,
\end{equation}
where the future states $s_{2:H+1}$ are predicted by the geometric feasibility planner (see \cref{subsec:text2motion-geometric-feasibility-planning}).
If $F^\mathcal{G}_{\text{sat}}(s_t)$ evaluates to $1$ for a geometric state $s_t$ at timestep $t \le H+1$, then the planner returns the subsequence of skills $\psi_{1:t-1}$ for execution.

\begin{algorithm}
\small
\caption{Shooting-based LLM planner}\label{alg:text2motion-shooting}
\begin{algorithmic}[1]
\State \textbf{globals:} $\mathcal{L}^\psi, \mathcal{L}^\chi, \textsc{SatFunc}, \textsc{LLM}, \textsc{STAP}$
\Function{Shooting}{$i, s_1, \mathcal{G}; K$}
    \State $F^\mathcal{G}_{\text{sat}} \gets \Call{SatFunc}{\mathcal{G}, \mathcal{L}^\chi}$ \Comment{Goal checker}
    \State $\{\psi^{(j)}_{1:H}\}_{j=1}^K \gets \Call{LLM}{i, s_1, \mathcal{G}, K}$ \Comment{Gen. plans}
    \State $C = \{\,\}$ \Comment{Init. candidate set}
    \For{$j = 1 \ldots K$}
        \State $s^{(j)}_{2:H+1}, a^{(j)}_{1:H} \gets \Call{STAP}{s_1, \psi^{(j)}_{1:H}, \mathcal{L}^\psi}$ 
        \If{$F^\mathcal{G}_{\text{sat}}(s^{(j)}_t) == 1$ \textbf{for} $t \leq H+1$}
            \State $\psi^{(j)}_{1:t-1} \gets \psi^{(j)}_{1:H}[:t-1]$ \Comment{Slice plan}
            \State $C \gets C \cup \{j\}$ \Comment{Add to candidate set}
        \EndIf
        \State Compute $p_{\text{success}}^{(j)}$ via \cref{eq:text2motion-stap-score}
    \EndFor
    \State Filter OOD plans from $C$ as per \cref{eq:text2motion-ood-detection}
    \If{$C == \emptyset$}
        \State \textbf{raise} \texttt{planning failure}
    \EndIf
    \State $j^* = \argmax_{j \in C}\; p_{\text{success}}^{(j)}$
    \State \Return $\psi^{(j^*)}_{1:t-1}$ \Comment{Return best plan}
\EndFunction
\end{algorithmic}
\end{algorithm}

\subsection{Shooting-Based Planning}
\label{sec:text2motion-shooting}
The planner is responsible for finding geometrically feasible plans that satisfy the goal condition predicted by the LLM (\cref{subsec:text2motion-goal-prediction}).
To this end, the first strategy we propose is a shooting-based planner, termed \sh{} (see \cref{fig:text2motion-system}, Left), which takes a single-step approach to maximizing the overall planning objective in \cref{eq:text2motion-tamp-score}.
\sh{}'s process is further outlined in \cref{alg:text2motion-shooting}.
\sh{} requires querying the LLM only once to generate $K$ candidate skill sequences $\{\psi^{1}_{1:H}, \ldots, \psi^{K}_{1:H}\}$ in an open-ended fashion.
Each candidate skill sequence is processed by the geometric feasibility planner which returns an estimate of the sequence's success probability (\cref{eq:text2motion-stap-score}) and its predicted future state trajectory ${s_{2:H+1}}$.
Skill sequences that satisfy the goal condition (\cref{eq:text2motion-satisfaction-function}) are added to a candidate set. 
Invalid skill sequences as determined by \cref{subsec:text2motion-ood-detection} are filtered-out of the candidate set.
If the candidate set is not empty, \sh{} returns the skill sequence with the highest success probability, or raises a \texttt{planning failure} otherwise.

\subsection{Search-Based Planning}
\label{sec:text2motion-greedy-search}
We propose a second planner, \se{} (see \cref{fig:text2motion-system}, Right), which at each planning iteration ranks candidate skills predicted by the LLM and adds the top scoring skill to the running plan.

This iterative approach can be described as a decomposition of the planning objective in \cref{eq:text2motion-tamp-score} by timestep $t$:
\begin{equation}
    \label{eq:text2motion-tamp-score-decomp}
    \begin{split}
        &p(\psi_{1:H}, r_{1:H} \mid i, s_1) \\
        &\quad\quad= \prod_{t=1}^H p(\psi_t, r_t \mid i, s_1, \psi_{1:t-1},  r_{1:t-1}).
    \end{split}
\end{equation}
We define the joint probability of $\psi_t$ and $r_t$ in \cref{eq:text2motion-tamp-score-decomp} as the skill score $S_{\text{skill}}$:
\begin{equation*}
    \label{eq:text2motion-tamp-step-score}
    S_{\text{skill}}(\psi_t) 
        = p(\psi_t, r_t \mid i, s_1, \psi_{1:t-1}, r_{1:t-1}),
\end{equation*}
which we factor using conditional probabilities:
\begin{equation}
    \label{eq:text2motion-tamp-step-score-factor}
    \begin{split}
        S_{\text{skill}}(\psi_t) &= p(\psi_t \mid i, s_1, \psi_{1:t-1}, r_{1:t-1}) \\ 
        &\quad\quad\quad\quad p(r_t \mid i, s_1, \psi_{1:t}, r_{1:t-1}).
    \end{split}
\end{equation}
Each planning iteration of \se{} is responsible for finding the skill $\psi_t$ that maximizes the skill score (\cref{eq:text2motion-tamp-step-score-factor}) at timestep $t$. 

\textbf{Skill usefulness:}
The first factor of \cref{eq:text2motion-tamp-step-score-factor} captures the \textit{usefulness} of a skill generated by the LLM with respect to satisfying the instruction. 
We define the skill usefulness score $S_{\text{llm}}$:
\begin{align}
    S_{\text{llm}}(\psi_t) 
        &= p(\psi_t \mid i, s_1, \psi_{1:t-1}, r_{1:t-1}) \label{eq:text2motion-llm-step-score} \\
        &\approx p(\psi_t \mid i, s_{1:t}, \psi_{1:t-1}). \label{eq:text2motion-llm-step-score-decomp}
\end{align}
In \cref{eq:text2motion-llm-step-score-decomp}, the probability of the next skill $\psi_t$ (\cref{eq:text2motion-llm-step-score}) is cast in terms of the predicted state trajectory $s_{2:t}$ of the running plan $\psi_{1:t-1}$, and is thus is independent of prior rewards $r_{1:t-1}$. 
We refer to \cref{appx-sub:text2motion-skill-usefulness} for a detailed derivation of \cref{eq:text2motion-llm-step-score-decomp}.

At each planning iteration $t$, we optimize $S_{\text{llm}}(\psi_t)$ by querying an LLM to generate $K$ candidate skills $\{\psi_t^1, \dots, \psi_t^K\}$. 
We then compute the usefulness scores $S_{\text{llm}}(\psi_t^k)$ by summing the token log-probabilities of each skill's language description \rtwo{(visualized in \cref{subsec:text2motion-prompt-engineering})}.
These scores represent the likelihood that $\psi^k_t$ is the correct skill to execute from a language modeling perspective to satisfy instruction $i$.

\begin{algorithm}[t]
\small
\caption{Search-based LLM planner}\label{alg:text2motion-greedy-search}
\begin{algorithmic}[1]
\State \textbf{globals:} $\mathcal{L}^\psi, \mathcal{L}^\chi, \textsc{SatFunc}, \textsc{LLM}, \textsc{STAP}$
\Function{Greedy-Search}{$i, s_1, \mathcal{G}; K, d_{\text{max}}$}
    \State $F^\mathcal{G}_{\text{sat}} \gets \Call{SatFunc}{\mathcal{G}, \mathcal{L}^\chi}$ \Comment{Goal checker}
    \State $\Psi = [\,]$; $\tau = [s_1]$ \Comment{Init. running plan}
    \While{$len(\Psi) < d_{\text{max}}$}
        \State $\Psi, \tau \gets \Call{Greedy-Step}{i, s_1, \mathcal{G}, \Psi, \tau, K}$
        \If{$F^\mathcal{G}_{\text{sat}}(\tau[-1]) == 1$} 
            \State \Return $\Psi$ \Comment{Return goal-reaching plan}
        \EndIf
    \EndWhile
    \State \textbf{raise} \texttt{planning failure}
\EndFunction
\Function{Greedy-Step}{$i, s_1, \mathcal{G}, \Psi, \tau; K$}
    \State $t = len(\Psi) + 1$ \Comment{Curr. planning iteration}
    \State $\{\psi^{(j)}_t\}_{j=1}^K \gets \Call{LLM}{i, \tau, \mathcal{G}, K}$ \Comment{Gen. skills}
    \State $C = \{\,\}$ \Comment{Init. candidate set}
     \For{$j = 1 \ldots K$}
        \State $\psi^{(j)}_{1:t} \gets \Psi.append(\psi^{(j)})$
        \State $s^{(j)}_{2:t+1}, a^{(j)}_{1:t} \gets \Call{STAP}{s_1, \psi^{(j)}_{1:t}, \mathcal{L}^\psi}$ 
        \State Compute $S_{\text{llm}}(\psi^{(j)}_t)$ via \cref{eq:text2motion-llm-step-score-decomp}
        \State Compute $S_{\text{geo}}(\psi^{(j)}_t)$ via \cref{eq:text2motion-motion-step-score-decomp}
        \State $S_{\text{skill}}(\psi^{(j)}_t) \gets S_{\text{llm}}(\psi^{(j)}_t) \times S_{\text{geo}}(\psi^{(j)}_t)$
        \If{$\psi^{(j)}_t$ is not OOD} \Comment{As per \cref{eq:text2motion-ood-detection}}
            \State $C \gets C \cup \{j\}$ \Comment{Add to candidate set}
        \EndIf
    \EndFor
    \State $j^* = \argmax_{j \in C}\; S_{\text{skill}}(\psi^{(j)}_t)$
    \State \Return $\psi^{(j^*)}_{1:t}, s^{(j^*)}_{1:t+1}$ \Comment{Return running plan}
\EndFunction
\end{algorithmic}
\end{algorithm}

\textbf{Skill feasibility:} 
The second factor of \cref{eq:text2motion-tamp-step-score-factor} captures the \textit{feasibility} of a skill generated by the LLM. 
We define the skill feasibility score $S_{\text{geo}}$:
\begin{align}
    S_{\text{geo}}(\psi_t) 
        &= p(r_t \mid i, s_1, \psi_{1:t}, r_{1:t-1}) \label{eq:text2motion-motion-step-score} \\
        &\approx Q^{\pi_t}(s_t, a^*_t) \label{eq:text2motion-motion-step-score-decomp},
\end{align} 
where \cref{eq:text2motion-motion-step-score-decomp} approximates \cref{eq:text2motion-motion-step-score} by the Q-value evaluated at predicted future state $s_t$ with optimized parameter $a^*_t$, both of which are computed by the geometric feasibility planner.
We refer to \cref{appx-sub:text2motion-skill-feasibility} for a detailed derivation of \cref{eq:text2motion-motion-step-score-decomp}.

\textbf{Skill selection:}
The skill feasibility score (\cref{eq:text2motion-motion-step-score-decomp}) and skill usefulness score (\cref{eq:text2motion-llm-step-score-decomp}) are then multiplied to produce the overall skill score (\cref{eq:text2motion-tamp-step-score-factor}) for each of the $K$ candidate skills $\{\psi_t^1, \dots, \psi_t^K\}$. 
Invalid skills as determined by \cref{subsec:text2motion-ood-detection} are filtered-out of the candidate set.
Of the remaining skills, the one with the highest skill score $\psi^*_t$ is added to the running plan $\psi_{1:t-1}$.
If the predicted geometric state $s_{t+1}$ that results from skill $\psi^*_t$ satisfies the predicted goal condition (\cref{eq:text2motion-satisfaction-function}), the skill sequence $\psi_{1:t}$ is returned for execution.
Otherwise, $s_{t+1}$ is used to initialize planning iteration $t+1$.
The process repeats until the planner returns or a maximum search depth $d_{\text{max}}$ is met raising a \texttt{planning failure}.
This process is outlined in \cref{alg:text2motion-greedy-search}.

The baselines we compare to~\cite{ahn2022-saycan, pmlr-v205-huang23c} only consider the feasibility of skills $\psi^k_t$ in the current state $s_t$. 
In contrast, \se{} considers the feasibility of skills $\psi^k_t$ in the context of the planned sequence $\psi_{1:t-1}$ via geometric feasibility planning.

\begin{figure}
    \centering    
    \includegraphics[width=0.90\textwidth]{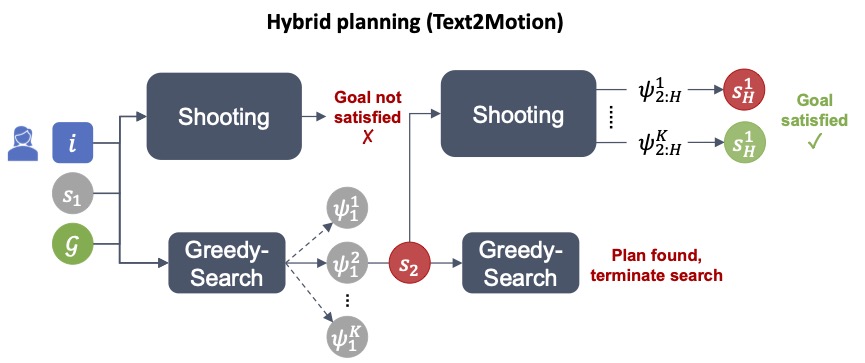}
    \caption[System architecture for proposed hybrid policy planner (Text2Motion)]{
    \textbf{Proposed hybrid planner}. After predicting goals for a given instruction, \ttmnb{} iterates the process: i) invoke \shnb{} to plan full skill sequences, and if no goal-reaching plan is found, ii) take a \senb{} step and check if executing the selected “best” skill would reach the goal.
    Note that the entire planning process occurs before execution. See \cref{fig:text2motion-system} for a visualization of the \shnb{} and \senb{} planners.
    \copyright{} 2023 Springer Nature.
    }
    \label{fig:text2motion-overall-method}
\end{figure}

\subsection{Text2Motion}
\label{subsec:text2motion-text2motion}
We present \ttm{}, a hybrid planning algorithm that inherits the strengths of both shooting-based and search-based planning strategies.
In particular, \sh{} offers efficiency when geometrically feasible skill sequences can be easily predicted by the LLM given the initial state and the instruction. 
\se{} serves as a reliable fall-back strategy that can determine what skills are feasible at the current timestep, should \sh{} fail to find a plan.
A visualization is provided in \cref{fig:text2motion-overall-method}.

At each planning iteration $t$, \ttm{} optimistically invokes \sh{} to plan $K$ candidate skill sequences.
If \sh{} raises a \texttt{planning failure}, then \ttm{} falls back to a single step of \se{}, which adds the skill $\psi^*_t$ with the highest skill score (\cref{eq:text2motion-tamp-step-score-factor}) to the running plan $\psi_{1:t-1}$.
The geometric feasibility planner predicts the state $s_{t+1}$ that would result from executing $\psi^*_t$.
If state $s_{t+1}$ satisfies the goal condition (\cref{eq:text2motion-satisfaction-function}), the skill sequence $\psi_{1:t}$ is returned for execution.
Otherwise, the next planning iteration starts by invoking \sh{} on predicted state $s_{t+1}$.
The process repeats until the planner returns or a maximum search depth $d_{\text{max}}$ is met.
\ttm{} is outlined in \cref{alg:text2motion-text2motion}.

\begin{algorithm}[t]
\algblock[TryCatchFinally]{try}{endtry}
\algcblock[TryCatchFinally]{TryCatchFinally}{finally}{endtry}
\algcblockdefx[TryCatchFinally]{TryCatchFinally}{catch}{endtry}
	[1]{\textbf{catch} #1}
	{\textbf{end try}}
\small
\caption{Text2Motion hybrid planner}\label{alg:text2motion-text2motion}
\begin{algorithmic}[1]
\State \textbf{globals:} $\mathcal{L}^\chi, \textsc{SatFunc}, \textsc{Shooting}, \textsc{Greedy-Step}$
\Function{Text2Motion}{$i, s_1, \mathcal{G}; K, d_{\text{max}}$}
    \State $F^\mathcal{G}_{\text{sat}} \gets \Call{SatFunc}{\mathcal{G}, \mathcal{L}^\chi}$ \Comment{Goal checker}
    \State $\Psi = [\,]$; $\tau = [s_1]$ \Comment{Init. running plan}
    \While{$len(\Psi) < d_{\text{max}}$}
        \try
            \State \Return $\Call{Shooting}{i, \tau, \mathcal{G}, K}$
        \catch{\texttt{planning failure}}
            \State $\Psi, \tau \gets \Call{Greedy-Step}{i, s_1, \mathcal{G}, \Psi, \tau, K}$
            \If{$F^\mathcal{G}_{\text{sat}}(\tau[-1]) == 1$} 
                \State \Return $\Psi$
            \EndIf
        \endtry
    \EndWhile
    \State \textbf{raise} \texttt{planning failure}
\EndFunction
\end{algorithmic}
\end{algorithm}

\subsection{Out-of-Distribution Detection}
\label{subsec:text2motion-ood-detection}
During planning, the LLM may propose skills that are out-of-distribution (OOD) given a state $s_t$ and optimized parameter $a^*_t$.
For instance, a symbolically incorrect skill, like \graytext{Place(dish, table)} when the \graytext{dish} is not in hand, may end up being selected if we rely on learned Q-values, since the Q-value for an OOD input can be spuriously high. 
We therefore reject plans that contain an OOD skill.

We consider a skill $\psi_t$ to be OOD if the variance of its Q-value (\cref{eq:text2motion-motion-step-score-decomp}) predicted by an ensemble~\cite{LakshminarayananPritzelEtAll2017} exceeds a calibrated threshold $\epsilon^{\psi_t}$:
\begin{equation}
    \label{eq:text2motion-ood-detection}
    \func{F_{\text{OOD}}}{\psi_t} = \mathbbm{1} \left(\Vstap{i\sim1:B}{Q^{\pi_t}_i(s_t, a_t^*)} \geq \epsilon^{\psi_t} \right),
\end{equation}
where $\mathbbm{1}$ is the indicator function and $B$ is the ensemble size. 
\rtwo{We refer to \cref{appx-sub:text2motion-ood-calibration} for details on calibrating OOD thresholds $\epsilon^{\psi}$.}

\section{Experiments}
\label{sec:text2motion-experiments}
We conduct experiments to test four hypotheses:
\begin{description}
    \item[\textbf{H1}] Geometric feasibility planning is a necessary ingredient when using LLMs and robot skills to solve manipulation tasks with geometric dependencies from a natural language instruction.
    \item[\textbf{H2}] \se{} is better equipped to solve tasks with partial affordance perception (as defined in \cref{subsec:text2motion-task-suite}) compared to \sh{}.
    \item[\textbf{H3}] \ttm{}'s hybrid planner inherits the strengths of shooting- and search-based strategies.
    \item[\textbf{H4}] \textit{A priori} goal prediction is a more reliable plan termination strategy than \texttt{stop} scoring.
\end{description}

The following subsections describe the baseline methods we compare against, details on LLMs and prompts, the tasks over which planners are evaluated, and performance metrics we report.

\subsection{Baselines}
\label{subsec:text2motion-baselines}
We compare \ttm{} with a series of language-based planners, including the proposed \sh{} and \se{} strategies.
For consistency, we use the same skill library $\mathcal{L^\psi}$, with independently trained policies $\pi$ and Q-functions $Q^{\pi}$, the OOD rejection strategy (\cref{subsec:text2motion-ood-detection}) and, where appropriate, the dynamics models $T^\pi(s, a)$ and geometric feasibility planner (\cref{subsec:text2motion-geometric-feasibility-planning}) across all methods and tasks.

\scgs{}:
We implement a cost-considerate variant of SayCan~\cite{ahn2022-saycan} with a module dubbed \gs{} (GS).
At each timestep $t$, SayCan ranks \textit{all possible} skills by $p(\psi_t \mid i, \psi_{1:t-1}) \cdot V^{\pi_t}(s_t)$, before executing the top scoring skill (Scorer).
However, the cost of ranking skills scales unfavorably with the number of scene objects $\mathcal{O}$ and skills in library $\mathcal{L}^\psi$.
\scgs{} limits the pool of skills considered in the ranking process by querying the LLM for the $K$ most useful skills $\{\psi_t^1, \dots, \psi_t^K\} \sim p(\psi_t \mid i, \psi_{1:t-1})$ (Generator) before engaging Scorer.
Execution terminates when the score of the \texttt{stop} ``skill'' is larger than the other skills.

\imgs{}:
We implement the \textit{Object + Scene} variant of Inner Monologue~\cite{pmlr-v205-huang23c} by providing task-progress scene context in the form of the environment's symbolic state.
We acquire \imgs{} by equipping~\cite{pmlr-v205-huang23c} with \gs{} for cost efficiency. 
LLM skill likelihoods are equivalent to those from \scgs{} except they are now also conditioned on the visited state history $p(\psi_t \mid i, s_{1:t}, \psi_{1:t-1})$.

\subsection{Large Language Model}
\label{subsec:text2motion-llm}
We use two pretrained language models, both of which were accessed through the OpenAI API: i) \texttt{text-davinci-003}, a variant of the InstructGPT \cite{ouyang2022training} language model family which is finetuned from GPT-3 with human feedback and ii) the Codex model \cite{chen2021evaluating} (specifically, \texttt{code-davinci-002)}. 
For the \sh{} planner, we empirically found \texttt{text-davinci-003} to be the most capable at open-ended generation of skill sequences.
For all other queries, we use \texttt{code-davinci-002} as it was found to be reliable. 
We do not train or finetune the LLMs and only use few shot prompting.

\subsection{Prompt Engineering}
\label{subsec:text2motion-prompt-engineering}
The in-context examples are held consistent across all methods and tasks in the prompts passed to the LLM. 
We provide an example of the prompt structure used to query \se{} for $K=5$ skills at the first planning iteration (prompt template is in \texttt{black} and LLM output is in \texttt{\textcolor{orange}{orange}}):

\vspace{0.3cm}
\noindent\fbox{\parbox{0.97\linewidth}{\small{\texttt{{
Available scene objects: [`table', `hook', `rack', `yellow box', `blue box', `red box']\\\\
Object relationships: [`inhand(hook)', `on(yellow box, table)', `on(rack, table)', `on(blue box, table)']\\\\
Human instruction: How would you push two of the boxes to be under the rack?\\\\
Goal predicate set: {\color{code-constant}[[`under(yellow box, rack)', `under(blue box, rack)'], [`under(blue box, rack)', `under(red box, rack)'], [`under(yellow box, rack)', `under(red box, rack)']]}\\\\
Top 5 next valid robot actions (python list): {\color{code-constant}['push(yellow box, rack)', 'push(red box, rack)', 'place(hook, table)', 'place(hook, rack)', 'pull(red box, hook)']}
}}}}}
\vspace{0.3cm}

\noindent
The prompt above chains the output of two queries together: one for goal prediction (\cref{subsec:text2motion-goal-prediction}), and another for skill generation (\cref{sec:text2motion-greedy-search}). 
To compute the skill usefulness (\cref{eq:text2motion-llm-step-score-decomp}), we replace \texttt{Top 5 next valid robot actions} with \texttt{Executed action:}, append the language description of the generated skill (e.g. \graytext{Push(yellow box, rack)}), and sum token log-probabilities. 
We provide the full set of in-context examples in the Appendix (\cref{appx-sub:text2motion-suppl-incontext-examples}).

\begin{figure}
    \centering    
    \includegraphics[width=0.80\textwidth]{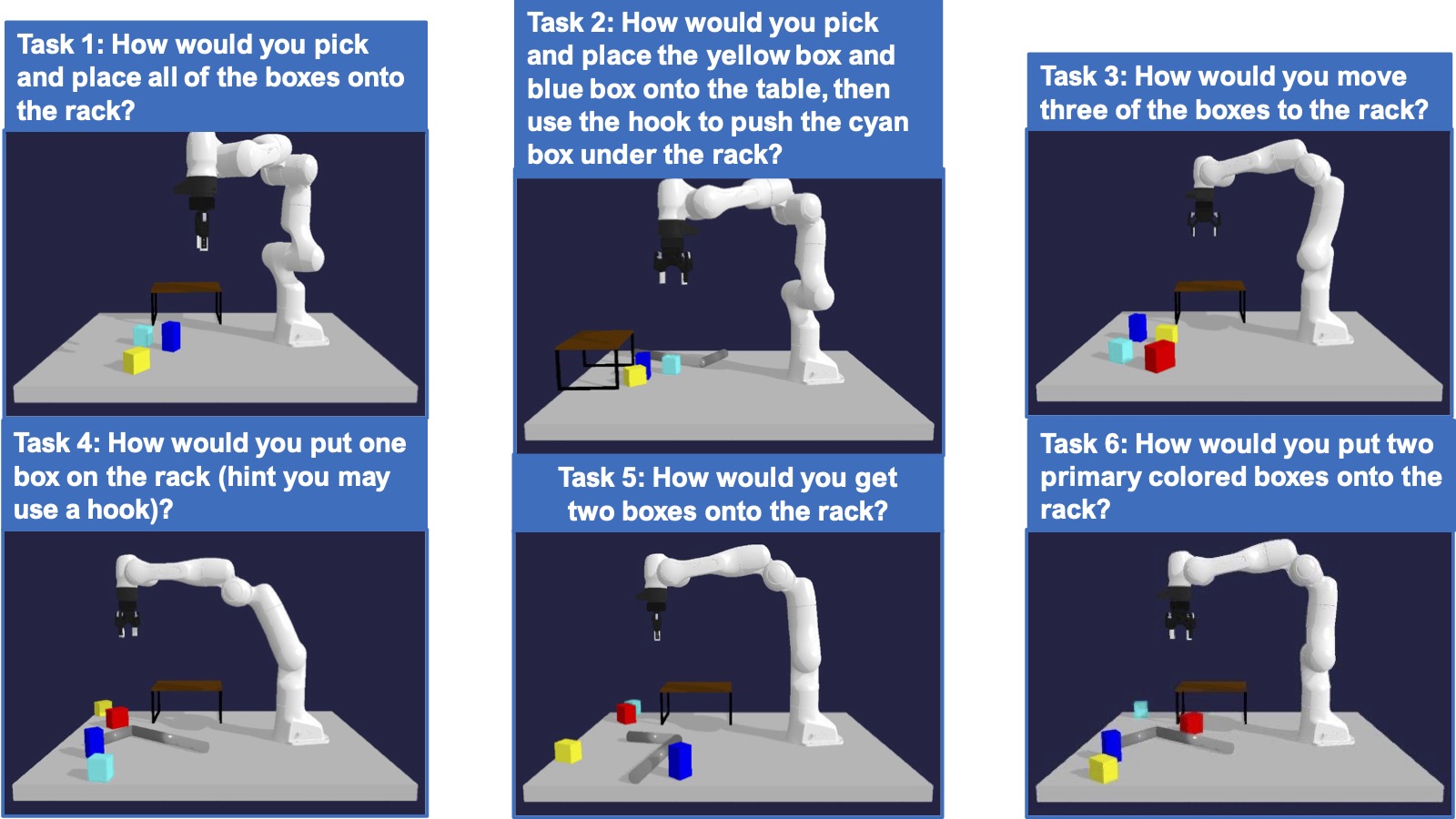}
    \vspace{8pt}
    \caption[Simulated TableEnv manipulation planning task suite (Text2Motion)]{\textbf{TableEnv manipulation evaluation task suite}. We evaluate the performance of all methods on tasks based on the above manipulation domain. The tasks considered vary in terms of difficulty and each task contains a subset of three properties: being long horizon (Tasks 1, 2, 3, 5, 6), containing lifted goals (Tasks 4, 5, 6), and having partial affordance perception (Tasks 4, 5, 6). During evaluation, we randomize the geometric parameters of each task.
    \copyright{} 2023 Springer Nature.
    }
    \label{fig:text2motion-evaluation_task_suite}
\end{figure}

\subsection{Task Suite}
\label{subsec:text2motion-task-suite}
We construct a suite of evaluation tasks (\cref{fig:text2motion-evaluation_task_suite}) in a table-top manipulation domain.
Each task includes a natural language instruction $i$ and initial state distribution $\rho(s)$ from which geometric task instances are sampled. 
For the purpose of experimental evaluation only, tasks also contain a ground-truth goal criterion to evaluate whether a plan has satisfied the corresponding task instruction.
Finally, each task contains subsets of the following properties:
\begin{itemize}
    \item \textbf{Long-horizon (LH):} Tasks that require skill sequences $\psi_{1:H}$ of length six or greater to solve. For example, Task 1 in \cref{fig:text2motion-evaluation_task_suite} requires the robot to pick and place three objects for a total of six skills. In our task suite, \textbf{LH} tasks also contain geometric dependencies that span across the sequence of skills which are unlikely to be resolved by myopically executing each skill. For example, Task 2 (\cref{fig:text2motion-evaluation_task_suite}) requires the robot to pick and place obstructing boxes (i.e. blue and yellow) to enable a collision-free push of the cyan box underneath the rack using the hook.
    \item \textbf{Lifted goals (LG):} Goals are expressed over object classes rather than object instances. For example, the lifted goal instruction ``move three boxes to the rack'' specifies an object class (i.e. boxes) rather than an object instance (e.g. the red box). This instruction is used for Task 3 (\cref{fig:text2motion-evaluation_task_suite}). 
    Moreover, \LG{} tends to correspond to planning tasks with many possible solutions. For instance, there may only be a single solution to the non-lifted instruction ``fetch me the red box and the blue box,'' but an LLM must contend with more options when asked to, for example, ``fetch any two boxes.''
    \item \textbf{Partial affordance perception (PAP):} Skill affordances cannot be perceived solely from the spatial relations described in the initial state $s_1$. For instance, Task 5 (\cref{fig:text2motion-evaluation_task_suite}) requires the robot to put two boxes onto the rack. However, the scene description obtained through predicate classifiers $\mathcal{L}^\chi$ (described in \cref{subsec:text2motion-goal-prediction}) and the instruction $i$ do not indicate whether it is necessary to use a hook to pull an object closer to the robot first. 
\end{itemize}

\subsection{Evaluation and Metrics}
\label{subsec:text2motion-evaluation-metrics}
\ttm{}, \sh{}, \se{}:
We evaluate these language planners by marking a plan as successful if, upon execution, they reach a final state $s_{H+1}$ that satisfies the instruction $i$ of a given task. 
A plan is executed only if the geometric feasibility planner predicts a state that satisfies the inferred goal conditions (\cref{subsec:text2motion-goal-prediction}).

Two failure cases are tracked: i) planning failure: the method does not produce a sequence of skills $\psi_{1:H}$ whose optimized parameters $a^*_{1:H}$ (\cref{eq:text2motion-taps-objective}) results in a state that satisfies $F^\mathcal{G}_{\text{sat}}$ within a maximum plan length of $d_{\text{max}}$; ii) execution failure: the execution of a plan that satisfies $F^\mathcal{G}_{\text{sat}}$ does not achieve the ground-truth goal of the task.

Since the proposed language planners use learned dynamics models to optimize parameters $a_{1:H}$ with respect to (potentially erroneous) future state predictions $s_{2:H}$, we perform the low-level execution of the skill sequence $\psi_{1:H}$ in closed-loop fashion. 
Thus, upon executing the skill $\psi_{t}$ at timestep $t$ and receiving environment feedback $s_{t+1}$, we call STAP~\cite{taps-2022} to perform geometric feasibility planning on the remaining planned skills $\psi_{t+1:H}$. 
We do not perform task-level replanning, which would involve querying the LLM at timestep $t+1$ for a new sequence of skills $\psi_{t+1:H}$.

\scgs{} \textbf{\&} \imgs{}:
These myopic agents execute the next best admissible skill $\psi_t$ at each timestep $t$ without looking-ahead. 
Hence, we evaluate them in a closed-loop manner for a maximum of $d_{\text{max}}$ steps.
We mark a run as a success if the agent issues the \texttt{stop} skill and the current state $s_t$ satisfies the ground-truth goal. 
Note that this comparison is advantageous for these myopic agents because they are given the opportunity to perform closed-loop replanning at the task-level (e.g. re-attempting a failed skill), whereas task-level replanning does not occur for \ttm{}, \sh{}, or \se{}.
This advantage does not lead to measurable performance gains on the challenging evaluation domains that we consider.

\textbf{Reported metrics:} 
We report success rates and subgoal completion rates for all methods.
Success rates are averaged over ten random seeds per task, where each seed corresponds to a different geometric instantiation of the task (\cref{subsec:text2motion-task-suite}).
Subgoal completion rates are computed over all plans by measuring the number of steps an \textit{oracle planner} would take to reach the ground-truth goal from a planner's final state.
To further delineate the performance of \ttm{} from \sh{} and \se{}, we also report the percentages of planning and execution failures.

\section{Results}
\label{sec:text2motion-results}

\begin{figure}[t]
    \centering
    \includegraphics[width=0.86\textwidth]{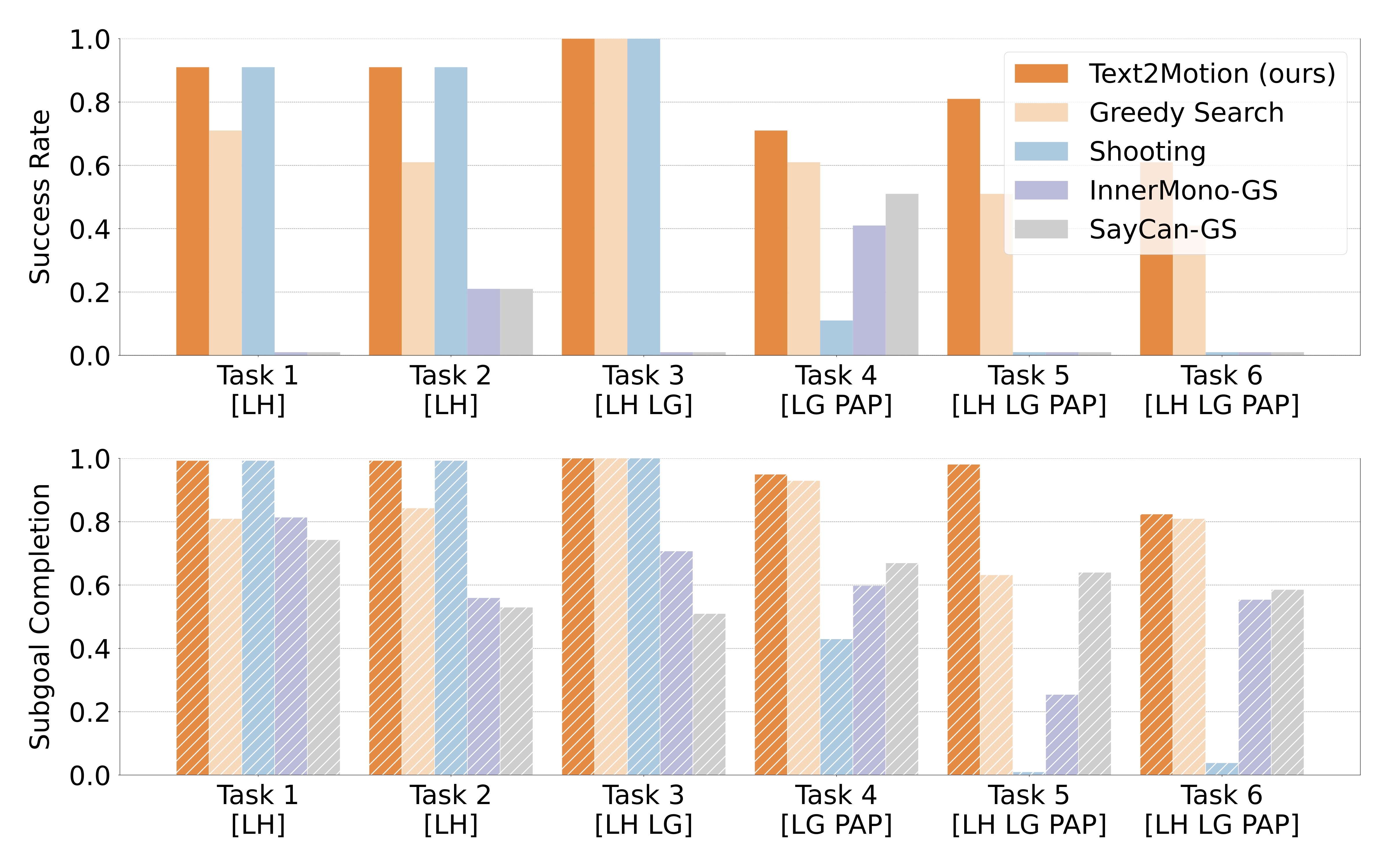}
    \caption[Policy planning results in the TableEnv simulated task suite (Text2Motion)]{
    \textbf{Results on the TableEnv manipulation domain} with 10 random seeds for each task. {\bf Top:} Our method (Text2Motion) significantly outperforms all baselines on tasks involving partial affordance perception (Task 4, 5, 6). For tasks without partial affordance perception, the methods that use geometric feasibility planning (\ttmnb{}, \shnb{}, \senb{}) convincingly outperform the methods (\scgsnb~and \imgsnb) that do not. We note that \shnb{}~performs well on the tasks without partial affordance perception as it has the advantage of outputting \textit{multiple} goal-reaching candidate plans and selecting the one with the highest execution success probability. {\bf Bottom:} Methods without geometric feasibility planning tend to have high sub-goal completion rates but very low success rates. This divergence arises because it is possible to make progress on tasks without resolving geometric dependencies in the earlier timesteps; however, failure to account for geometric dependencies results in failure of the overall task. 
    \copyright{} 2023 Springer Nature.
    }
    \label{fig:text2motion-planning-result}
\end{figure}

\subsection{Feasibility Planning Is Required to Solve Tasks with Geometric Dependencies (H1)}
\label{subsec:text2motion-components-tamp}

Our first hypothesis is that performing geometric feasibility planning on task plans output by the LLM is essential to task success. To test this hypothesis, we compare methods that use geometric feasibility planning (\ttm{}, \sh{}, \se{}) against myopic methods that do not (\scgs{} and \imgs{}).

Instructions $i$ provided in the first two planning tasks (\LH{}) allude to skill sequences that, if executed appropriately, would solve the task.
In effect, the LLM plays a lesser role in contributing to plan success, as its probabilities are conditioned to mimic the skill sequences in $i$.
On such tasks, \ttm{}, \sh{} and \se{} which employ geometric feasibility planning over skills sequences better contend with geometric dependencies prevalent in \LH{} tasks and thereby demonstrate higher success rates. 

In contrast, the myopic baselines (\scgs{} and \imgs{}) fail to surpass success rates of 20\%, despite completing between 50\%-80\% of the subgoals (\cref{fig:text2motion-planning-result}).
This result is anticipated as the feasibility of downstream skills requires coordination with earlier skills in the sequence, which these methods do not consider.
As the other tasks combine aspects of \LH{} with \LG{} and \PAP{}, it remains difficult for \scgs{} and \imgs{} to find solutions.

Surprisingly, we see that \scgs{} closely matches the performance of \imgs{}, which is additionally provided with descriptions of all states encountered during execution (as opposed to just the initial scene description).
This result suggests that explicit language feedback does not contribute to success on our tasks when considered in isolation from plan feasibility.

\begin{figure}[t]
    \centering    
    \includegraphics[width=0.95\textwidth]{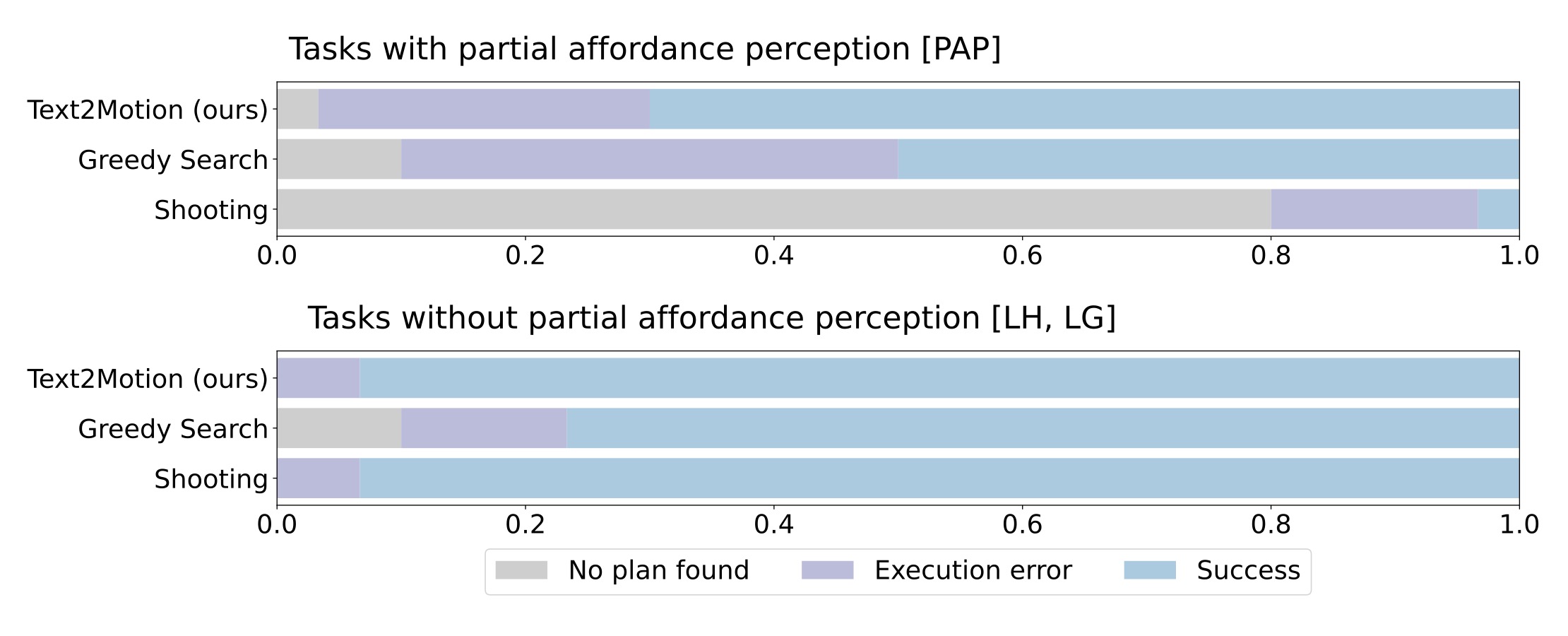}
    \caption[Ablation result on planning failure modes under partial observability (Text2Motion)]{\textbf{Failure modes of language-based planners on two categories of tasks}.
    In this plot, we analyse the various types of failure modes that occur with \ttmnb{}, \shnb{} and \senb{} when evaluated on tasks with partial affordance perception (PAP; see \cref{subsec:text2motion-task-suite} for an explanation) and tasks without partial affordance perception (non-PAP). 
    \textbf{Top}: For the PAP tasks, \shnb{} incurs many planning failures because the space of possible plans is large but only few can be feasibly executed. In contrast, \senb{} uses value functions during search to narrow down the space of plans to those that are feasible.
    \ttmnb{} relies on \senb{} as a fallback if \shnb{} fails, and thus can also contend with PAP tasks. 
    \textbf{Bottom}: For the non-PAP tasks, \shnb{} outperforms \senb{}. We attribute this difference to \shnb{}'s ability to output multiple task plans while \senb{} can only output a single plan. Finally, \ttmnb{} matches the performance of \shnb{} as it also outputs and selects among multiple task plans.
    \copyright{} 2023 Springer Nature.
    }
    \label{fig:text2motion-failure}
\end{figure}

\subsection{Search-Based Reasoning Is Required for PAP Tasks
 (H2)}
\label{subsec:text2motion-integrated-planning}
Our second hypothesis is that search-based reasoning is required to solve the \PAP{} family of tasks (defined in \cref{subsec:text2motion-task-suite}). 
We test this hypothesis by comparing \se{} and \sh{}, which represent two distinct approaches to combining symbolic and geometric reasoning to maximize the overall planning objective (\cref{eq:text2motion-tamp-score}).
\sh{} uses Q-functions of skills to optimize $K$ skill sequences (\cref{eq:text2motion-taps-objective}) \textit{after} they are generated by the LLM.
\se{} uses Q-functions as skill feasibility heuristics (\cref{eq:text2motion-motion-step-score-decomp}) to guide search \textit{while} a skill sequence is being constructed.

In the first two tasks (\LH{}, \cref{fig:text2motion-planning-result}), we find that \sh{} achieves slightly higher success rates than \se{}, while both methods achieve 100\% success rates in the third task (\LH{} + \LG{}).
This result indicates a subtle advantage of \sh{} when multiple feasible plans can be directly inferred from $i$ and $s_1$. \sh{} can capitalize on diverse orderings of $K$ generated skill sequences (including the one specified in $i$) and select the one with the highest success probability (\cref{eq:text2motion-motion-parameter-score}).
For example, Task 1 (\cref{fig:text2motion-evaluation_task_suite}) asks the robot to put three boxes onto the rack; \sh{} allows the robot to test multiple different skill sequences while \se{} only outputs a single plan.
This advantage is primarily enabled by bias in the Q-functions: \cref{eq:text2motion-stap-score} may indicate that \graytext{Place(dish, rack)} then \graytext{Place(cup, rack)} is more geometrically complex than \graytext{Place(cup, rack)} then \graytext{Place(dish, rack)}, while they are geometric equivalents.

The plans considered by \se{} at planning iteration $t$ share the same sequence of predecessor skills $\psi_{1:t-1}$.
This affords limited diversity for the planner to exploit.
However, \se{} has a significant advantage when solving the \PAP{} family of problems (\cref{fig:text2motion-planning-result}, Tasks 4-6).
Here, skill sequences with high success probabilities (\cref{eq:text2motion-motion-parameter-score}) are difficult to infer directly from $i$, $s_1$, and the in-context examples provided in the prompt.
As a result, \sh{} incurs an 80\% planning failure rate, while \se{} finds plans over 90\% of the time (\cref{fig:text2motion-failure}).
In terms of success, \se{} solves 40\%-60\% of the \PAP{} tasks, while \sh{} achieves a 10\% success rate on Task 4 (\LG{} + \PAP{}) and fails to solve any of the latter two tasks (\LH{} + \LG{} + \PAP{}).
Moreover, \sh{} does not meaningfully advance on any subgoals, unlike \scgs{} and \imgs{}, which consider the geometric feasibility of skills at each timestep (albeit, myopically).

\begin{table}[]
    \rtwo{
    \centering
    \resizebox{0.64\textwidth}{!}{%
    \begin{tabular}{lccc}
    \toprule
     Hybrid planning breakdown & Task 4 & Task 5 & Task 6 \\
    \midrule
    \% \shnb{} only  &  14\% &  0\% &  0\% \\
    \% \senb{} only &  0\% &  0\%  &  0\% \\
    \% Combination &  86\% & 100\%  & 100\% \\
    Avg. \senb{} Steps & 1.0  &  2.6 &  3.0 \\
    Avg. Plan Length & 5.0 &  7.0 &  7.0 \\
    \bottomrule \\
    \end{tabular}}
        \caption[Ablation result on hybrid planner behavior for complex tasks (Text2Motion)]{
            \textbf{Ablation on hybrid planning method.} We analyze the usage percentages of both \shnb{} and \senb{} in successful plans found by our hybrid planner (see \cref{fig:text2motion-overall-method}). We find that, as tasks increase in difficulty (Task 4, 5, 6), the majority of solutions involve a combination of both planners. This result indicates that shooting-based and search-based planning strategies play complementing roles in the success of \ttmnb{}.
            \copyright{} 2023 Springer Nature.
        }
    \label{tab:text2motion-ablation-hybrid}
    }
\end{table}

\subsection{Hybrid Planning Integrates the Strengths of Shooting-Based and Search-Based Methods (H3)}
\label{subsec:text2motion-hybrid-planning}
Our third hypothesis is that shooting-based planning and search-based planning have complementing strengths that can be unified in a hybrid planning framework.
We test this hypothesis by comparing the performance of \ttm{} against \sh{} and \se{}. 

The results are presented in \cref{fig:text2motion-planning-result}.
We find that \ttm{} matches the performance of \sh{} on tasks that do not consist of \PAP{} (Task 1, 2, 3).
This is expected because \sh{} does not exhibit planning failures on these tasks (\cref{fig:text2motion-failure}) and \ttm{} starts by invoking \sh{}, which results in their identical performance.
However, on tasks with \PAP{} (Task 4, 5, 6) we observe that \ttm{} succeeds more often than \se{}.
This suggests that interleaving \sh{} and \se{} at each planning iteration enables \ttm{} to consider a more diverse set of goal-reaching solutions.
This result is corroborated in \cref{fig:text2motion-failure}, where we see that \ttm{} incurs fewer planning and execution failures than \se{}.

\rtwo{In \cref{tab:text2motion-ablation-hybrid}, we further analyze the usage percentages of \sh{} and \se{} within successful plans executed by \ttm{}. 
The results show that, for tasks involving \PAP{}, over 90\% of solutions involve a combination of both shooting- and search-based strategies, which confirms our third hypothesis.}

\subsection{Plan Termination Is Made Reliable via Goal Prediction (H4)}
\label{subsec:text2motion-plan-termination}
Our fourth hypothesis is that predicting goals from instructions \textit{a priori} and selecting plans based on their satisfication (\cref{subsec:text2motion-goal-prediction}) is more reliable than 
scoring plan termination with a dedicated \texttt{stop} skill at each timestep.
We test this hypothesis in an ablation experiment (\cref{fig:text2motion-ablation-termination}), comparing our plan termination method to that of SayCan and Inner Monologue's, while keeping all else constant for our \se{} planner. 
We run 120 experiments (two variations, six tasks, and ten seeds each) in total on the TableEnv Manipulation task suite. 
The results in \cref{fig:text2motion-ablation-termination} suggest that, for the tasks we consider, our proposed goal prediction method leads to 10\% higher success rates than the scoring baseline. 

We also note the apparent advantages of both techniques.
First, goal prediction is more efficient than scoring \texttt{stop} as the former requires only one LLM query, whereas the latter needs to be queried at every timestep.
Second, goal prediction offers interpretability over \texttt{stop} scoring, as it is possible to inspect the goal that the planner is aiming towards prior to execution. 
Nonetheless, \texttt{stop} scoring does provide benefits in terms of expressiveness, as its predictions are not constrained to any specific output format.
This advantage, however, is not captured in our evaluation task suite, which at most require conjunctive ($\land$) and disjunctive ($\lor$) goals.
For instance, ``Stock two boxes onto the rack'' could correspond to (\graytext{on(red box, rack)} $\land$ \graytext{on(blue box, rack)}) $\lor$ (\graytext{on(yellow box, rack)} $\land$ \graytext{on(cyan box, rack)}), while in theory, \texttt{stop} scoring can represent all goals expressible in language.

\begin{figure}
    \centering        
    \includegraphics[width=0.90\textwidth]{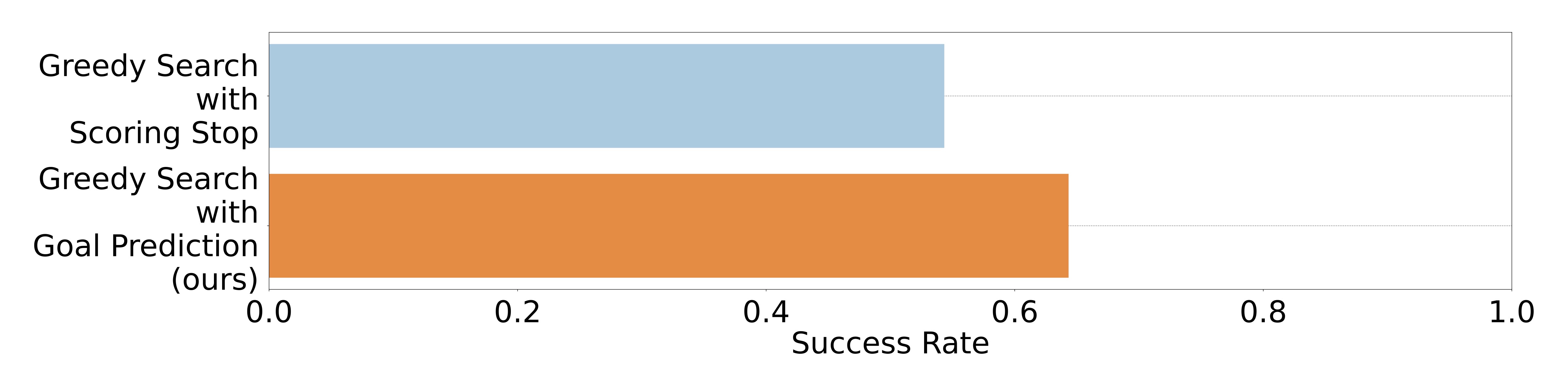}
    \caption[Ablation result on planning termination strategies (Text2Motion)]{\textbf{Ablation on termination method: goal proposition prediction vs stop scoring}. 
    We compare the performance of \senb{} using two different plan termination methods: using the LLM to predict goals \textit{a priori} (ours) and scoring a \texttt{stop} skill~\cite{ahn2022-saycan} during search. 
    We present results averaged across all six tasks and ten seeds for each variation (120 experiments in total). 
    We find that terminating planning when LLM-predicted goals are satisfied results in a 10\% boost in success rate over \texttt{stop} scoring.
    \copyright{} 2023 Springer Nature.
    }
    \label{fig:text2motion-ablation-termination}
\end{figure}

\section{Conclusion}
\label{sec:text2motion-conclusion}

We present a language-based planning framework that combines LLMs, learned skills, and geometric feasibility planning to solve long-horizon robotic manipulation tasks containing geometric dependencies.
\ttm{} constructs a task- and motion-level plan and verifies that it satisfies a natural language instruction by testing planned states against inferred goals.
In contrast to prior language planners, our method verifies that its plan satisfies the instruction before executing any actions in the environment.
\ttm{} represents a hybrid planning formalism that optimistically queries an LLM for long-horizon plans and falls back to a reliable search strategy should optimistic planning fail. 
As a result, \ttm{} inherits the strengths of both shooting-based and search-based planning formalisms.

Our results highlight the following: (i) geometric feasibility planning is important when using LLMs and learned skills to solve sequential manipulation tasks from natural language instructions; (ii) search-based reasoning can contend with a family of tasks where the space of possible plans is large but only few are feasible; (iii) shooting-based and search-based planning strategies can be synergistically integrated in a hybrid planner that outperforms its constituent parts; (iv) terminating plans based on inferred symbolic goals is more reliable than prior LLM scoring techniques.

\section{Limitations and Future Work}\label{sec:text2motion-limitations}

We summarize the limitations of our approach and opportunities for future investigation:

\begin{enumerate}
    \item \textbf{LLM likelihoods:} We observed an undesirable pattern emerge in the planning phase of \se{} and the execution phase of \scgs{} and \imgs{}, where \textit{recency bias}~\cite{zhao2021calibrate} would cause the LLM to produce unreliable likelihoods (\cref{eq:text2motion-llm-step-score-decomp}), inducing a cyclic state of repeating feasible skills. While we mitigate such failures by combining \se{} and \sh{} in the hybrid \ttm{} algorithm, leveraging calibration techniques to increase the reliability LLM likelihoods~\cite{li2022contrastive, chen2022close} may improve the performance search-based planning over long-horizons.
    \item \textbf{Skill library:} As with other methods that use skill libraries, \ttm{} is reliant on the fidelity of the learned skills, their value functions, and the ability to accurately predicted future states with dynamics models. Thus, incorporating skills that operate on high-dimensional observations (e.g. images~\cite{concept2robot-2021}) into our framework may require adopting techniques for stable long-term predictions in these observation spaces~\cite{2022-driess-compNerf}.
    \item \textbf{Runtime complexity:} \ttm{} is mainly comprised of learned components, which comes with its associated efficiency benefits. Nonetheless, runtime complexity was not a core focus of this investigation, and STAP~\cite{taps-2022} was frequently called during planning, increasing the overall planning time. LLMs also impose an inference bottleneck as each API query (\cref{subsec:text2motion-llm}) requires 2-10 seconds, but we anticipate improvements with advances in both LLM inference techniques and in methods that distill LLM capabilities into smaller, cost-efficient models~\cite{touvron2023llama}. Thus, there remains opportunities to increase the plan-time efficiency of our method, for instance, by warm starting geometric feasibility planning with solutions cached in earlier planning iterations~\cite{williams2015model}.
    \item \textbf{Closed-world assumptions:} Our framework operates in a closed-world setting (\cref{sec:text2motion-problem-setup}), where we assume to know which objects are task-relevant and the poses of objects are estimated by an external perception system at the time of receiving a language instruction. Extending \ttm{} to open-world settings may necessitate exploring~\cite{chen2022open} or interacting~\cite{curtis2022-unknown} with the environment to discover objects that are not initially observable, and training skills to support a diverse set of real-world objects~\cite{jang2021bczero, mtopt2021arxiv}. This could be accomplished by integrating \ttm{} as part of a broader planning system, where it is used to produce feasible and verified plans to subgoals, while building knowledge of the environment in unobserved or partially observable regions during the execution of those subgoals. The use of Visual Question and Answering (VQA)~\cite{zhou2020unified} and multi-modal foundation models that are visually grounded~\cite{driess2023palm, openai2023gpt4} also holds promise. Such models may support scaling \ttm{} to higher-dimensional observation spaces and potentially serve as a substitutes for closed-world components used in our framework (e.g. detecting a variable number of predicates using VQA).
\end{enumerate}


\chapter{Conclusion}\label{chapter:conclusion}

\section{Summary of Contributions and Findings}\label{sec:thesis-summary}

Continued progress in general-purpose learned policies for robot control suggests a trajectory toward increasingly broad deployment across industrial settings and everyday human environments, where reliability is critical.
This dissertation advances the perspective that deployment-time reliability is a \textit{property} that emerges from deliberately designed, complementary mechanisms that operate \textit{around} learned policies---reinforcing them at points of fragility and providing practical insight into their behavior to enable systematic improvement. 
We have examined several such mechanisms, including: (1) runtime failure detection and intervention~\cite{agia2024unpacking}, (2) data-centric interpretation of policy behavior and performance~\cite{agia2025cupid}, and (3) composition and coordination of learned behaviors for real-world, long-horizon tasks~\cite{taps-2022, lin2023text2motion}---the key contributions and findings of which we summarize below.

\textbf{Monitoring Policies for Runtime Failure Detection (Sentinel):} 
In \cref{chapter:sentinel}, we studied the problem of detecting imminent deployment-time failures of learned robot policies and introduced Sentinel, a runtime monitoring framework for generative policy architectures that detected failures without requiring any failure data. The central insight was that policy failures manifested across multiple timescales---ranging from local behavioral inconsistencies to global stagnation in task progress---and were therefore best captured by a hierarchy of specialized, complementary detectors. Across both simulation and real-world experiments, Sentinel detected 18\% more failures than any individual monitoring approach alone, while enabling earlier and more reliable intervention.

\textbf{Understanding Policy Behavior through Training Data (CUPID):}
In \cref{chapter:cupid}, we investigated the opaque relationship between closed-loop policy behavior and training data. We introduced CUPID, a data-centric interpretability framework that leveraged influence functions to estimate how individual training demonstrations affected a policy’s expected return. The key insight was that deployment-time performance could be meaningfully interpreted---and systematically improved---by attributing outcomes such as task success or failure to specific training examples, enabling counterfactual reasoning about data quality and composition. Empirically, CUPID reliably identified the demonstrations most influential to policy performance and achieved state-of-the-art results while using less than 50\% of the available training data.

\textbf{Coordinating Policy Sequences for Long-Horizon Tasks (STAP):}
In \cref{chapter:stap}, we examined how individually learned behaviors often failed when composed to solve long-horizon tasks, due to unmodeled dependencies that compromise sequence feasibility. To address this, we introduced STAP, a framework that coordinated sequences of learned policies by explicitly estimating and optimizing their joint success probability at deployment time. The key insight was that this joint success probability could be captured as the product of Q-functions---models learned during policy training---whose optimization automatically resolved dependencies between behaviors. STAP significantly improved long-horizon task success in complex environments and achieved up to four orders of magnitude faster planning than baselines that relied on simulation for future prediction.

\textbf{Searching for Feasible Policy Sequences from Language (Text2Motion):}
In \cref{chapter:text2motion}, we studied the problem of solving long-horizon manipulation tasks specified in natural language, where multiple candidate behavior sequences may satisfy the goal semantically but differ in feasibility.
We introduced Text2Motion, a language-based task and motion planning framework that used LLMs to generate candidate policy sequences and feasibility-aware estimates from STAP to guide their selection. The key insight was that interleaving high-level language-guided search with low-level feasibility estimation was essential for reliably solving tasks with large combinatorial search spaces. Empirically, Text2Motion significantly outperformed prior language-based planning approaches, achieving a 69\% improvement in success rate across a suite of long-horizon manipulation tasks.

Together, these contributions constitute an articulated exploration into improving the reliability of learned policies at deployment. Each addresses a distinct yet complementary challenge that inevitably arises when robot systems transition from controlled laboratory settings to the open-ended variability of the real world.

\section{Directions for Future Work}\label{sec:thesis-future-work}
There remain many promising avenues for future investigation. Several reflect key extensions of the methodologies developed in this dissertation---monitoring, interpreting, and coordinating policies---and are discussed in the future work sections of their respective chapters. We focus here on a set of new directions that have received comparatively less attention in the study of policy reliability.

\textbf{Robotics World Models:} 
A core challenge for policy reliability is that training-time metrics such as validation loss are often weak predictors of closed-loop performance at deployment~\cite{ross2011reduction} (e.g., policies with lower validation loss do not necessarily achieve higher task success on a real robot). However, directly evaluating policy behavior in closed loop is expensive and inherently unscalable~\cite{barreiros2025careful}, since finite test conditions cannot reflect the breadth of variability a policy will encounter over its deployment lifetime. 
High-fidelity world models offer a promising path toward more reliable policies by enabling low-cost, closed-loop evaluation of policy behavior in imagination, reducing reliance on real-world testing~\cite{quevedo2025evaluating, veorobotics2025}. 
Furthermore, building controllable world models---as in autonomous driving research~\cite{rempe2022generating, tan2025promptable, diffusionrenderer, lu2025infinicube, mao2025dreamdrive}---could support systematic stress-testing of policies under rare, unexpected, or safety-critical scenarios, enabling targeted robustness evaluation and policy improvement~\cite{chen2025reimagination}.
Recent work has begun exploring the integration of world models in the deployment loop to steer~\cite{WuY1-RSS-25} or filter policy actions~\cite{pmlr-v305-seo25a}, while more expressive models that capture distributions over future states under partial observability could enable policies to anticipate and avoid failure modes before they occur. 
Given the computational cost of forward prediction, such approaches may benefit from fast-slow architectures that selectively invoke expensive look-ahead when policy uncertainty is high.

\textbf{Hierarchical Policy Architectures:} 
Hierarchical design has long played a central role in robotics, with its benefits clearly demonstrated in autonomous driving~\cite{karkus2023diffstack, wang2025alpamayo} and task and motion planning systems~\cite{kaelbling2011hierarchical, garrett2021integrated}. By stratifying decision-making and control across multiple levels of abstraction and timescales, hierarchical architectures enable coordinated, intelligent, and robust runtime behavior~\cite{migimatsu2023long}. While earlier efforts toward general-purpose policies unified high-level reasoning and low-level control within a single model (e.g., reasoning-based VLAs~\cite{Belkhale-RSS-24, pmlr-v270-zawalski25a}), more recent work has gravitated toward system-1–system-2 architectures~\cite{team2025gemini15}, which pair a slower, high-level planning policy with a fast, reactive control policy. Such architectures may offer several advantages for deployment-time reliability. First, high-capacity system-2 models can support reliability-enhancing functions such as failure monitoring~\cite{ICLR2025_70a06501}, online correction~\cite{lin2025failsafe}, and fallback intervention~\cite{sinha2024real}---capabilities that have already shown promise in handling unexpected and safety-critical scenarios~\cite{pmlr-v305-ganai25a}. Second, as discussed in \cref{chapter:cupid}, policy failures are often difficult to diagnose and remediate; modular architectures may improve interpretability by isolating failures to specific components (e.g., planning versus control), enabling more targeted improvement. 
Lastly, while the approaches discussed in \cref{chapter:stap} and \cref{chapter:text2motion} employed hierarchies for sequencing single-step primitive policies, integrating more advanced, closed-loop policies such as VLAs for task and motion planning remains an open problem. 

\textbf{Explainability and Mechanistic Interpretability:} 
While \cref{chapter:cupid} established a quantitative link between deployment-time policy behavior and the training data that influenced policy predictions, such approaches do not surface conceptual, human-interpretable insight into why a policy succeeds or fails. This limits our ability to derive actionable guidance from closed-loop evaluation data, which may be used to efficiently debug and improve policies through fine-grained data curation~\cite{agia2025cupid}, determine what data should be collected next~\cite{zha2025guiding}, or construct evidence-based safety cases prior to deployment. Early work on explainability has begun to characterize policy failure modes~\cite{sagar2025batcave, sagar2024mystery}, and initial mechanistic interpretability approaches have explored concept-based steering of policy behavior at deployment time~\cite{haon2025mechanistic}. However, these efforts only scratch the surface of the types of questions such methods could address. 
Promising future work lies in identifying new reliability-critical problems that explainability and interpretability are well suited to, and in developing new ways to apply these tools to learned policies---beyond post hoc analysis or manual concept probing. 

\section{Concluding Remarks}\label{sec:thesis-concluding-remarks}

This dissertation casts deployment as an emerging problem category in robot learning---one whose importance will grow as policy capabilities mature and robots move beyond controlled laboratory settings. We have argued that reliability is a central desideratum for learned robot policies and demonstrated that it can be systematically advanced by addressing complementary challenges arising from the uncertainty, diversity, and complexity of real-world environments. Looking ahead, emerging tools and models open promising research directions for further strengthening deployment-time reliability, moving the field toward a future of safe and trustworthy robot deployment at scale.

\appendix

\chapter{Additional Details: Sentinel}\label{chapter:sentinel-appx}

\section*{Appendix Overview: Unpacking Failure Modes of Generative Policies}
The appendix offers additional details with respect to the implementation of our failure detection framework (\cref{appx:sentinel-method}), the experiments conducted (\cref{appx:sentinel-experiments}), along with extended results and analysis (\cref{appx:sentinel-results}), and finally, supporting derivations (\cref{appx:sentinel-derivations}) for our proposed failure detectors.
\textbf{Qualitative results} and a \textbf{video abstract} are made available at \url{https://sites.google.com/stanford.edu/sentinel}.

\startcontents[sections]
\printcontents[sections]{l}{1}{\setcounter{tocdepth}{2}}

\section{Method Details: Sentinel}\label{appx:sentinel-method}
As shown in \cref{fig:sentinel-full-system}, the \textbf{Sentinel} runtime monitoring framework consists of the parallel operation of two complementary failure detectors, each assigned to the detection of a particular failure category of generative policies.
The first is a temporal consistency detector that monitors for erratic policy behavior via \underline{s}tatistical \underline{t}emporal \underline{a}ction \underline{c}onsistency (STAC) measures.
The second is a Vision Language Model (VLM) that monitors for failure of the policy to make progress on its task.
In this section, we provide additional details \textit{w.r.t.} the implementation of STAC (\cref{appx:sentinel-stac}) and the VLM runtime monitor (\cref{appx:sentinel-vlm}). 

\subsection{Temporal Consistency Detection with STAC}\label{appx:sentinel-stac}
\paragraph{Background}
To summarize \cref{sec:sentinel-problem-setup}, STAC assumes the use of a stochastic policy $\pi$ that, at each policy-inference timestep $t$, predicts an action sequence for the next $h$ timesteps as $a_{t:t+h-1|t} \sim \pi(\cdot | s_t)$, executes the first $k$ actions $a_{t:t+k|t}$, before re-evaluating the policy at timestep $t+k$. 
Between two contiguous inference timesteps $t$ and $t+k$, sampled action sequences $a_{t+k:t+h-1|t}$ and $a_{t+k:t+h-1|t+k}$ (both in $\R^{(h-k) \times |\calA|}$) overlap for $h-k$ timesteps.
At a high-level, STAC seeks to quantify \textit{how much} a generative policy's action distributions are changing over time.
It does this by computing statistical distances between the distributions of overlapping actions, i.e., given $\bar{\pi}_{t}:=\pi(a_{t+k:t+h-1|t} | s_{t})$ and $\tilde{\pi}_{t+k}:=\pi(a_{t+k:t+h-1|t+k} | s_{t+k})$, we compute $D(\bar{\pi}_{t}, \tilde{\pi}_{t+k})$.

\paragraph{Hypothesis} 
Our central hypothesis is that large statistical distances correlate with downstream policy failure.
Intuitively, a predictive policy can be likened to possessing an internal world model that simulates how robot actions affect environment states.
When the policy is in distribution, we expect this world model to be accurate, thus resulting in smaller statistical distances.
More concretely, if the policy's internal model of state $s_{t+k}$ at timestep $t$ coincides with the actual observed state $s_{t+k}$ at timestep $t+k$, the distribution of actions $\tilde{\pi}_{t+k}$ should be well-represented by the distribution $\bar{\pi}_{t}$.
As a result, the distance $D(\bar{\pi}_{t}, \tilde{\pi}_{t+k})$ will be small (for the right choice of statistical distance function $D$).
Conversely, when the policy is out of distribution (OOD), its internal model of state $s_{t+k}$ at timestep $t$ may be inaccurate, yielding a divergence between 
$\bar{\pi}_{t}$ and $\tilde{\pi}_{t+k}$ and a larger statistical distance.  

\paragraph{Implementation Details} 
As mentioned in \cref{sec:sentinel-temporal-consistency}, we propose to approximate $D(\bar{\pi}_{t}, \tilde{\pi}_{t+k})$ with an empirical distance function $\hat{D}$ instead of computing it analytically, as doing so presents the challenge of marginalizing out both the non-overlapping actions (between timesteps $t$ and $t+k$) and the intermediate steps of the diffusion process~\cite{sohl2015deep}.
We found the following approximations to work well in practice:
\begin{itemize}
    \item Maximum Mean Discrepancy (MMD) with radial basis function (RBF) kernels. We compute
    \begin{align*}
        \hat{D}(\bar{\pi}_{t}, \tilde{\pi}_{t+k}) 
            &= \E_{a_t, a_t' \sim \bar{\pi}_{t}}\left[k(a_t, a_t')\right] + \E_{a_{t+k}, a_{t+k}' \sim \tilde{\pi}_{t+k}}\left[k(a_{t+k}, a_{t+k}')\right] \\
            &- 2 \E_{a_t \sim \bar{\pi}_{t}, a_{t+k} \sim \tilde{\pi}_{t+k}}\left[k(a_t, a_{t+k})\right],\quad \text{where}\quad k(x, y;\; \beta_1) = \exp\left(-\frac{||x-y||^2}{\beta_1}\right).
    \end{align*}
    Here, $k: \R^{(h-k) \times |\calA|} \times \R^{(h-k) \times |\calA|} \rightarrow \R$ computes the similarity between two overlapping action sequences, and $\beta_1$ denotes the bandwidth of the RBF kernel.
    The expectations are taken over a batch of $B$ action sequences sampled from the generative policy. 
    \item Forward KL-divergence via Kernel Density Estimation (KDE) of the policy distributions:
    \begin{align*}
        \hat{D}(\bar{\pi}_{t}, \tilde{\pi}_{t+k}) = \E_{a_{t+k} \sim \tilde{\pi}_{t+k}}\left[\log\;\ \frac{p(a_{t+k})}{q(a_{t+k})} \right],
    \end{align*}
    where $p$ and $q$ are KDEs of $\tilde{\pi}_{t+k}$ and $\bar{\pi}_{t}$ fit on a batch of $B$ action sequences sampled from each policy distribution, respectively. As before, we use Gaussian RBF kernels of the form $k(x, y;\; \beta_2)$, where $\beta_2$ denotes the bandwidth of the RBF kernels used for KDE.
    \item Reverse KL-divergence via KDE of the policy distributions:
    \begin{align*}
        \hat{D}(\bar{\pi}_{t}, \tilde{\pi}_{t+k}) = \E_{a_{t} \sim \bar{\pi}_{t}}\left[\log\;\ \frac{p(a_{t})}{q(a_{t})} \right],
    \end{align*}
    where $p$ and $q$ are KDEs of $\bar{\pi}_{t}$ and $\tilde{\pi}_{t+k}$ fit on a batch of $B$ action sequences sampled from each distribution, respectively, and all other parameters follow the forward KL definitions.
\end{itemize}

\paragraph{Hyperparameters}
The batch size $B$, MMD bandwidth $\beta_1$, and KDE bandwidth $\beta_2$ are hyperparameters that we select for a given environment. 
As expected, we found that larger batch sizes are necessary for accurate mean embeddings and density estimates in domains with higher degrees of multimodality (e.g., $B = 256$ action sequences for \textbf{PushT} and \textbf{Push Chair}).
We also found that using either default settings or dynamic calibration techniques are sufficient to obtain suitable MMD and KDE bandwidth parameters $\beta_1$ and $\beta_2$, respectively.
For example, setting $\beta_2$ in proportion to the maximum eigenvalue of the covariance of overlapping actions $a_{t+k:t+h-1|\cdot}$ sampled from $\bar{\pi}_{t}$ and $\tilde{\pi}_{t+k}$ worked well in multimodal domains.
Further details on hyperparameters are provided in \cref{tab:sentinel-stac-hyperparams}.

\begin{table*}[h!]
    \centering
    \caption[Hyperparameters for temporal consistency detection (Sentinel)]{\small
        \textbf{Hyperparameters settings} for temporal consistency detection with STAC.
    } 
    \label{tab:sentinel-stac-hyperparams}
    \adjustbox{max width=\textwidth}{\begin{tabular}{|l|ccc|}
        \toprule
        \textbf{Hyperparameters} & \textbf{PushT ($\uparrow$ Multimodality)} & \textbf{Close Box \& Cover Object ($\downarrow$ Multimodality)} & \textbf{Push Chair ($\uparrow$ Multimodality)}\\
        \midrule
        MMD + KDE batch size ($B$) & 256 & 32 & 256 \\
        MMD bandwidth ($\beta_1$) & Median Heuristic~\cite{gretton2005kernel,muandet2017kernel} & $1.0/|\calA|$ & Median Heuristic \\
        KDE bandwidth ($\beta_2$) & $\sqrt{\lambda_{\max}(\mathrm{Cov}(a_{t+k:t+h-1|\cdot}))}$ & $1.0$ & $\sqrt{\lambda_{\max}(\mathrm{Cov}(a_{t+k:t+h-1|\cdot}))}$\\
        \midrule
        Policy action space ($\calA$) & Linear Velocity & 2 x (Linear + Angular Velocity) & 1 x (Linear + Angular Velocity) \\
        Policy prediction horizon ($h$) & 16 & 16 & 16 \\
        Policy execution horizon ($k$) & 8 & 4 & 4 \\
        \bottomrule
    \end{tabular}}
\end{table*}

\paragraph{Additional Design Choices} 
There are several additional settings that one could adjust to increase STAC's detection performance on their task. 
First, filtering components of the policy's action space that are either noisy or discrete can increase the quality of the statistical distance score function. 
For example, the policy's action space in our robotic manipulator domains include end-effector linear and angular velocities, as well as a binary gripper command. However, when computing statistical distances, we omit all binary gripper commands. 
Next, reducing the execution horizon $k$ of the generative policy to compare action distributions that are closer in time can mitigate excessively large statistical distances in highly dynamic or stochastic environments.
Likewise, comparing action distributions over a shorter prediction horizon $h$ may be suitable if the tails of predicted action sequences e.g., exhibit high variance. 

\subsection{Runtime Monitoring with Vision Language Models}\label{appx:sentinel-vlm}
As described in \cref{sec:sentinel-vlm-assessment}, we formulate the detection of task progression failures as a chain-of-thought (CoT)~\cite{wei2022chain}, video question answering (Video QA)~\cite{xu2016msr} task with VLMs. 
Below, we provide details on the implementation of our VLM runtime monitor and the prompt templates used in our experiments.

\paragraph{Vision Language Models}
In extended experiments (\cref{appx:sentinel-extended-vlm-results}), we include variants of the VLM runtime monitor based on several models: OpenAI's GPT-4o, Anthropic's Claude 3.5 Sonnet, and Google's Gemini 1.5 Pro~\cite{reid2024gemini}.
At the time of writing, these represent the current state-of-the-art VLMs for complex, multimodal reasoning tasks.
We use consistent prompts across all models, however, we slightly vary the implementation of the monitor to reflect the suggested best practices of each VLM.

\paragraph{Implementation Details} 
We propose to query the VLM online with a parsed text prompt describing the runtime monitoring task and the history of images (i.e., a video) $I_{0:t} := (I_0, I_{\nu k}, I_{2\nu k}, \ldots, I_t)$ captured by the robot's camera system up to the current timestep $t$.
Here, the hyperparameter $\nu$ specifies the frequency of the images relative to the execution horizon $k$ of the generative policy (\cref{sec:sentinel-problem-setup}) for generality, as the video may be captured at a much higher frame rate than the policy's execution rate. 
In experiments, simply setting $\nu \in \{1, 2\}$ provided sufficient granularity for the VLM to identify the motion of the robot.

By making non-blocking API calls, the VLM runtime monitor can operate at relatively high frequencies.
For example, the VLM can be queried at each policy-inference timestep $t = jk$ for $j \in \{0, 1, \dots\}$ (i.e., at STAC's detection frequency) to provide a failure classification.
However, depending on the task, doing so may neither be necessary nor desirable for two reasons.
First, because task progression failures are likely to occur at longer timescales than erratic failures, querying the VLM at a reduced frequency might provide time for meaningful changes in state to occur.
In turn, this would reduce redundancy in the VLM's predictions e.g., if no meaningful changes in state occurred since the last time it was queried.
Second, while STAC---a statistical output monitor---can be evaluated at negligible cost (computationally and monetarily), querying state-of-the-art, closed-source VLMs may come with considerable expense. 
Since task progression failures are unlikely to require immediate intervention (in contrast to erratic failures), querying the VLM less often than STAC could be preferable.
In experiments, we queried the VLM to detect task progression failures twice per episode.

The prompt template consists of three parts: 1) a brief description of the VLM's role as the runtime monitor of a robotic manipulator system, which it must execute by analyzing the attached video; 2) a description of the robot's task, the total amount of time that has elapsed\footnote{Online runtime monitoring requires the VLM to differentiate whether a) the robot is still in progress of executing the task correctly (i.e., partial progress) or whether b) the robot will fail to complete the task (e.g., by stalling in a partially completed state). 
Differentiating between partial progress and task failure can be ambiguous for a slow moving robot, and thus, providing the VLM with the current elapsed time serves as a reference to gauge whether or not the rate at which the robot is executing the task will result in a timely task completion.}, and time limit for the task (corresponding to the MDP horizon $H$ in \cref{sec:sentinel-problem-setup}); 3) instructions to elicit a CoT response, ensuring that the VLM describes and analyzes the motion of the robot and all task-relevant objects and outputs a classification that can be easily parsed.
To remove ambiguity over the expected behavior of the robot and what constitutes task completion, we make sure that the task description is sufficiently detailed:

\begin{mdframed}[backgroundcolor=light-gray, roundcorner=10pt, leftmargin=1, rightmargin=1, innerleftmargin=14, innertopmargin=10,innerbottommargin=0, outerlinewidth=1, linecolor=light-gray]
\begin{lstlisting}[linewidth=\columnwidth,breaklines=true]
task_descriptions:
  cover: "hide the white box by covering it with the black blanket. The white box is located somewhere in front of the two robot arms and does not move. The black blanket starts directly in between the two robot arms"
  close: "close the white box by folding in the two smaller white side lids and the bigger white back lid. The white box is located in between the two robot arms and does not move. The robots should concurrently approach the side lids and push both side lids up, followed by approaching the back lid and folding up the back lid with both arms, without grasping the lids with the grippers"
  push_chair: "push the black chair into the circular table. The black chair starts directly in front of the robot. The robot should push black chair in a relatively straight line, without the chair rotating to the left or to the right, so that the seat of the chair is properly tucked under the circular table"
\end{lstlisting}
\end{mdframed}

We elicit a four-step CoT response from the VLM that 1) generates a set of task-relevant questions whose correct answers would fully characterize the motion of the robot and all task-relevant objects in the video, 2) answers the task-relevant questions while providing fine-grained visual details, 3) analyzes the questions, answers, and elapsed time to identify whether the robot is making progress towards task completion within the episode time limit, and 4) concludes with an overall classification in $\{\texttt{ok}, \ \texttt{failure}\}$. 
Interestingly, we found that prompting the VLM to generate its own questions instead of manually specifying them leads to more accurate descriptions of the video and ensuing predictions.

\paragraph{Prompt Template (Video QA)}
\begin{mdframed}[backgroundcolor=light-gray, roundcorner=10pt, leftmargin=1, rightmargin=1, innerleftmargin=14, innertopmargin=12,innerbottommargin=0, outerlinewidth=1, linecolor=light-gray]
\begin{lstlisting}[linewidth=\columnwidth,breaklines=true]
You are the runtime monitor for an autonomous mobile manipulator robot capable of solving common household tasks. A camera system captures a series of image frames (i.e., a video) of the robot executing its current task online. The image frames are captured at approximately 1Hz. As a runtime monitor, your job is to analyze the video and identify whether the robot is a) in progress of executing the task or b) failing to execute the task, for example, by acting incorrectly or unsafely.

The robot's current task is to {DESCRIPTION}. The robot may take up to {TIME_LIMIT} seconds to complete this task. The current elapsed time is {TIME} seconds.

Format your output in the following form:
[start of output]
Questions: First, generate a set of task-relevant questions that will enable you to understand the full, detailed motion of the robot and all task-relevant objects from the beginning to the end of the accompanying video.
Answers: Second, precisely answer the generated questions, providing fine-grained visual details that will help you accurately assess the current state of progress on the task.        
Analysis: Assess whether the robot is clearly failing at the task. Since the video only represents the robot's progress up to the current timestep and the robot moves slowly, refrain from making a failure classification unless the robot is unlikely to complete the task in the allotted time. Explicitly note the amount of time that has passed in seconds and compare it with the time limit (e.g., x out of {TIME_LIMIT} seconds). Finally, based on the questions, answers, analysis, and elapsed time, decide whether the robot is in progress, or whether the robot will fail to complete its task in the remaining time (if any).
Overall assessment: {CHOICE: [ok, failure]}
[end of output]

Rules:
1. If you see phrases like {CHOICE: [choice1, choice2]}, it means you should replace the entire phrase with one of the choices listed. For example, replace the entire phrase '{CHOICE: [A, B]}'  with 'B' when choosing option B. Do NOT enclose your choice in '{' '}' brackets. If you are not sure about the value, just use your best judgement.
2. Do NOT forget to conclude your analysis with an overall assessment. As indicated above with '{CHOICE: [ok, failure]}', your only options for the overall assessment are 'ok' or 'failure'.
3. Always start the output with [start of output] and end the output with [end of output].
\end{lstlisting}
\end{mdframed}

\subsubsection{Prompt Ensembling}\label{appx:sentinel-prompt-ensemble}
The Video QA failure detection task requires comprehensive and detailed reasoning of over potentially long sequences of images, which, at the time of writing, is a challenge for even the most capable VLMs.
As such, we can expect the performance of the VLM runtime monitor to degrade when it is deployed in domains that are visually OOD \textit{w.r.t.} the VLM's training data (e.g., images rendered in simulation or captured from unusual camera poses). 
In these settings, the VLM runtime monitor may provide a reasonable but imperfect signal on task progression failures, resulting in misclassifications.

To strengthen the reliability of our VLM runtime monitor, we propose a simple prompt ensembling strategy~\cite{pitis2023boosted}, whereby we construct multiple Video QA prompts, query the VLM with each prompt, and take the overall failure classification to be the \textit{majority vote} of the predictions across all prompts. 
Intuitively, if the failure detectors associated with each individual prompt are fairly accurate to begin with, the resulting majority-vote detector will have an even higher probability of correctness.

In experiments, we only found it necessary to use prompt ensembling in the \textbf{Cover Object} domain.
We construct two variants of the Video QA prompt (3 prompts total), each of which follow a similar CoT structure while including additional information to diversify the VLM's reasoning. 
The first prompt variant, \textbf{Video QA + Success Video}, includes a video of a successful policy rollout for the current task.
This allows the VLM to distinguish off-nominal policy behavior at test time from nominal policy behavior illustrated in the example video.
The second prompt variant, \textbf{Video QA + Goal Images}, includes example images of the scene at the end of successful policy rollouts, which serve as a visual reference for task completion. 
In accordance with the assumptions of our framework, these prompt variants only require data associated with policy success to detect unknown failures at test time. 

\paragraph{Prompt Template (Video QA + Success Video)}
\begin{mdframed}[backgroundcolor=light-gray, roundcorner=10pt, leftmargin=1, rightmargin=1, innerleftmargin=14, innertopmargin=12,innerbottommargin=0, outerlinewidth=1, linecolor=light-gray]
\begin{lstlisting}[linewidth=\columnwidth,breaklines=true]
You are the runtime monitor... # Same as Video QA

To inform your analysis, you will be provided with an example video that shows the full motion of the robot and all task-relevant objects when the task is successfully executed. For example, the last image frame in the example video will show what the scene should look like at the end of a successsfully executed task. By comparing the current video with the example video, you may be able to visually distinguish when the robot is failing at the task versus when it is making steady progress or has completed.

The robot's current task is... # Same as Video QA

Questions: First, generate a set of task-relevant questions that will enable you to understand the full, detailed motion of the robot and all task-relevant objects from the beginning to the end of the accompanying video. In addition, generate questions that will enable you to identify any key similarities or differences between the current video and the example success video.
Answers: Second, precisely answer... # Same as Video QA
\end{lstlisting}
\end{mdframed} 
\vspace{6pt}

\paragraph{Prompt Template (Video QA + Goal Images)}
\begin{mdframed}[backgroundcolor=light-gray, roundcorner=10pt, leftmargin=1, rightmargin=1, innerleftmargin=14, innertopmargin=12,innerbottommargin=0, outerlinewidth=1, linecolor=light-gray]
\begin{lstlisting}[linewidth=\columnwidth,breaklines=true]
You are the runtime monitor... # Same as Video QA

To inform your analysis, you will be provided with several example images that show what the scene (i.e., the robot and all task-relevant objects) should look like at the end of a successfully executed task. By comparing the last few image frames of the current video with these example images, you may be able to visually distinguish when the robot is failing at the task versus when it is making steady progress or has completed.

The robot's current task is... # Same as Video QA

Questions: First, generate a set of task-relevant questions that will enable you to understand the full, detailed motion of the robot and all task-relevant objects from the beginning to the end of the accompanying video. In addition, generate questions that will enable you to identify any key similarities or differences between the current video and the example images of successfully completed tasks.
Answers: Second, precisely answer... # Same as Video QA
\end{lstlisting}
\end{mdframed}

\section{Experiment Details}\label{appx:sentinel-experiments}

\subsection{Environments}\label{appx:sentinel-environments}
We provide additional details on the environments used to evaluate Sentinel. 
These environments vary in terms of their data distribution (e.g., multimodal or high-dimensional actions) and support different types of distribution shift (e.g., object scale, pose, dynamics), under which the behavior of generative diffusion policies can be methodically studied. 
The environments are visualized in \cref{fig:sentinel-environments}.

\subsubsection{Simulation Domains}

\paragraph{PushT Domain}
The policy is tasked with pushing a planar ``T''-shaped object into a goal configuration. 
A trajectory is considered successful if the overlap between the ``T''-shaped object and its goal exceeds 90\% within 300 environment steps.
The action space is the 2-DoF linear velocity of the end-effector. 
We generate OOD test scenarios by non-uniformly randomizing the scale and dimensions of the ``T''-shaped object beyond the randomizations contained in the policy's demonstration data.
The policy tends to fail by converging to a locally optimal configuration, where the ``T'' overlaps with its goal but in an incorrect orientation.
Since the task can be solved in a number of ways, we include this domain to evaluate the performance of various score functions in the presence of action multimodality.
We refer to \cite{chi2023diffusion} for the process of generating demonstration data in this domain.
    
\paragraph{Close Box Domain} The policy is tasked with closing a box that has three lids. 
A trajectory is considered successful if all three lids are closed within 120 environment steps (24 seconds).
The action space is the 14-DoF linear + angular velocities and gripper commands for the end-effectors of two mobile manipulators. 
Demonstration data is generated by an oracle policy that sets a series of waypoints for the end-effectors based on the initial state. 
We generate OOD test scenarios by non-uniformly randomizing the scale of the box beyond the randomizations contained in the policy's demonstration data.
The policy tends to fail erratically when the robots e.g., collide with the box or its lids, however, task progression failures may also occur.
This domain is primarily used to evaluate the detection of erratic policy failures on a bi-manual robotic system with a high-dimensional action space.

\paragraph{Cover Object Domain} The policy is tasked with covering a rigid object with a cloth. 
A trajectory is considered successful if over 75\% of the object is covered by the cloth within 64 environment steps (13 seconds). 
The action space and process of generating demonstration data is identical to that of \textbf{Close Box}. 
We generate OOD test scenarios by non-uniformly randomizing the position of the object beyond the randomizations contained in the policy's demonstration data.
The policy tends to fail by releasing the cover before reaching the object.
Hence, this domain is used to evaluate the detection of task progression failures, where reasoning over longer durations is required to assess task progress.

\begin{figure}
    \includegraphics[width=\linewidth]{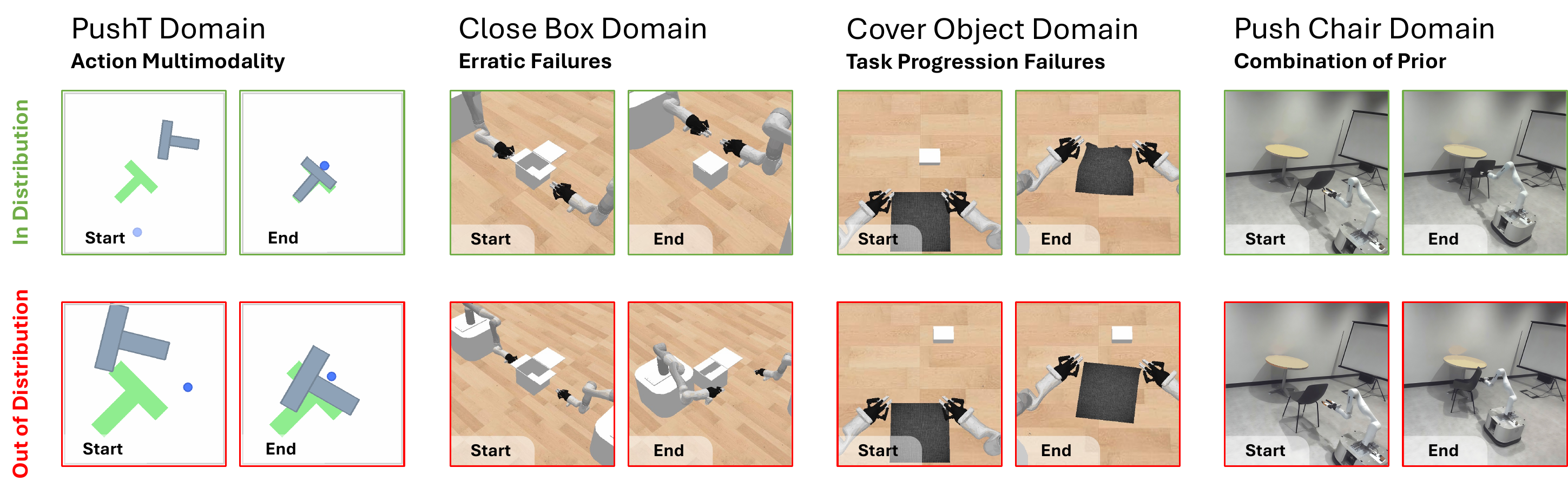}
    \caption[Simulated and real-world evaluation tasks for failure detection (Sentinel)]{\small
        \textbf{Evaluation Domains.} We evaluate our failure detection framework across three simulation domains and one real-world domain. These domains provide coverage over different data distributions (e.g., action multimodality, high-dimensional actions) and modes of generative policy failure. For example, generative policies tend to fail erratically in the \textbf{Close Box} domain, but smoothly in the \textbf{Cover Object} domain. 
        An effective failure detector should be performant across multiple domains, which entails coverage over heterogeneous failure modes.
    }
    \label{fig:sentinel-environments}
\end{figure}

\subsubsection{Real-World Domains}\label{appx:sentinel-environments-real}

\paragraph{Mobile Robot Setup}
We use a holonomic mobile base equipped with a Kinova Gen3 7-DoF arm. 
A single ZED 2 camera is fixed in the workspace to capture visual observations for the generative policy. 
The ZED 2 camera first generates a partial-view point cloud of the environment, from which we segment task-relevant objects using the Grounded Segment Anything Model~\cite{ren2024grounded} based on a natural language description related to the task. 
The segmented point cloud serves as the visual input to the policy.
Additionally, we use a motion capture system to track the pose of the mobile base. 
During evaluation, the policy processes the point clouds, predicts a sequence of 16 actions, of which the first 4 are executed on the robot.
The mobile manipulator robot then maneuvers its arm according to these commands, adjusting the pose of the base if the end-effector moves outside a pre-defined workspace.

\paragraph{Push Chair Domain} 
The policy is tasked with tucking a chair into a table using a single-arm mobile manipulation platform.
A trajectory is considered successful if the seat of the chair is properly tucked under the table by the end of the policy rollout.
The action space is the 7-DoF linear + angular velocities and gripper command for the end-effector of the mobile manipulator robot. 
Demonstration data is extracted from human videos: we use an off-the-shelf hand detection model~\cite{pavlakos2024reconstructing}, an object segmentation model~\cite{cheng2023tracking,ren2024grounded}, and a stereo-to-depth model to extract human hand poses and object point clouds from a subsampled set of frames in each of the 15 human demonstration videos. 
We generate OOD test scenarios by randomizing the initial pose of the chair beyond the randomizations contained in the demonstration data. 
The policy tends to fail erratically if the chair rotates away in either direction when pushed, but such failures are also visually apparent. 
Therefore, this task is used to test the efficacy of both STAC and the VLM runtime monitor in a dynamically complex~\cite{ruggiero2018nonprehensile}, real-world setting.

\subsection{Diffusion Policies}\label{appx:sentinel-diffusion-policy}

We train a diffusion policy (DP) for each environment, using 200 demonstrations for the \textbf{PushT} domain, 50 demonstrations for each of the \textbf{Close Box} and \textbf{Cover Object} domains, and 15 demonstrations for the real-world \textbf{Push Chair} domain.
In a DP, actions are generated by iteratively denoising an initially random action $a^N_t \sim \calN(0, 1)$ over $N$ steps as $a^N_t, \ldots, a^0_t$, where $a^i_t$ with a superscript $i$ denotes the generated action sequence at the $i$-th denoising iteration.
In an imitation learning setting, the DP's noise prediction network $\epsilon_\theta$ is trained to predict the \edit{random noise $\epsilon^i$} added to actions drawn from a dataset of expert demonstrations $\mathcal{D}_\mathrm{train}$ by minimizing
\begin{equation}\label{eq:sentinel-ddpm-loss}
    \mathcal{L}_\mathrm{ddpm} := \E_{(s, a^0)\sim \mathcal{D_\mathrm{train}}, \epsilon^i, i} \left[||\epsilon^i - \epsilon_\theta(\sqrt{\bar\alpha_i} a^0 + \sqrt{1 - \bar{\alpha}_i}\epsilon^i, s, i)||^2 \right],
\end{equation}
where the constants $\bar \alpha_i$ depend on the chosen noise schedule of the diffusion process.

To increase the salience of distribution shift \textit{w.r.t.} the position and scale of objects, we use point clouds as inputs to the policy instead of RGB images (i.e., a 5\% increase in object scale may not be salient in an image). 
For simulation experiments, we use a diffusion policy architecture identical to the original paper~\cite{chi2023diffusion} except for the visual encoder, where we substitute the ResNet-based encoder for a PointNet-based one: a 4-layer PointNet++ encoder~\cite{qi2017pointnet++} with hidden dimension 128. 
The output of this encoder is concatenated with the proprioceptive inputs and then fed to the noise prediction network.
For real-world experiments, we use the recently proposed EquiBot diffusion policy architecture~\cite{yang2024equibot}, which additionally incorporates SIM(3) equivariance into the diffusion process.
We use EquiBot to evaluate our failure detectors on a current state-of-the-art approach for learning generative policies in the real world.
All diffusion policies produce an action over $N=100$ denoising iterations. 
Unless otherwise specified, we use standard settings for the prediction $h$ and execution horizon $k$ of the diffusion policy.

\subsection{Baselines}\label{appx:sentinel-baselines}
We outline the implementation details of our baselines as introduced in \cref{sec:sentinel-experiments}.
First, with the exception of the VLM runtime monitors, all evaluated failure detection methods consist of computing a score $S(\cdot)$ at each policy-inference timestep in a rollout, taking the cumulative sum of scores up to the current timestep $t$, and then checking if the cumulative sum exceeds a calibrated threshold to detect policy failure.
As such, the baselines differ in their \textit{score function}, i.e., how they compute the per-timestep scores that are then summed and thresholded.
Intuitively, a good score function should be well-correlated with policy failure, that is, it should output small values when the policy is succeeding and large ones when it is failing.
For example, \cref{fig:sentinel-error-result} demonstrates that STAC holds this property.
We baseline against an extensive suite of score functions, some of which we newly introduce for the case of generative diffusion policies, and others that are common in the OOD detection literature~\cite{sinha2022system}.

\subsubsection{STAC Baselines (Policy-Level Monitors)}\label{appx:sentinel-stac-baselines}
\begin{itemize}
    \item \textbf{Policy Encoder Embedding} quantifies the dissimilarity of the current point cloud observation $o_t$ \textit{w.r.t.} to the point clouds in the calibration dataset of $M$ successful policy rollouts $\calD_\tau = \{\tau^i\}_{i=1}^M$ (as described in \cref{sec:sentinel-problem-setup}) within the embedding space of the policy's encoder (here, $o_t$ denotes the point cloud input to the policy, which includes the point cloud at the current and previous timestep).
    More concretely, let $E$ be the policy's encoder, $z_t = E(o_t)$ be the current point cloud embedding, and $\calD_{z} = E(\calD_\tau)$ be the embeddings of all point clouds contained in the calibration dataset. We compute the per-timestep score as the Mahalanobis distance
    \begin{equation}\label{eq:sentinel-mahal-embeddings}
        S(z_t;\;\; \calD_{z}) = \sqrt{(z_t - \mu_z)^T\; \Sigma_{zz}^{-1}\; (z_t - \mu_z)},
    \end{equation}
    where $\mu_z$ is the mean and $\Sigma_{zz}$ is the covariance of the embeddings in $\calD_{z}$. 
    At test time, we raise a failure warning if the cumulative score $\eta_t$ exceeds a calibrated detection threshold $\gamma$
    \begin{equation*}
        \eta_t > \gamma,\quad \text{where}\quad \eta_t = \sum_{i=0}^{j-1} S(z_i;\;\; \calD_{z}), \;\;\;\; t=jk.
    \end{equation*}
    Here, $\gamma$ is set to the $1 - \delta$ quantile of cumulative scores computed over the calibration trajectories $\{\eta^i_{|\tau^i|}\}_{i=1}^M$, where $\tau^i \in \calD_\tau$.
    Importantly, when computing the calibration scores $\eta^i$, we do so in a \textit{leave-trajectory-out} fashion: i.e., for a point cloud $o_t \in \tau^i$ where $\tau^i \in \calD_\tau$, we compute the per-timestep score as $S(E(o_t);\;\; E(\calD_\tau \setminus \tau^i))$. This ensures that the dissimilarity of observation $o_t$ is computed \textit{w.r.t.} trajectories other than its own, which a) aligns with how scores are computed at test time and b) ensures that calibration scores are not trivially low.

    We experimented with alternatives to the Mahalanobis distance in \cref{eq:sentinel-mahal-embeddings}, substituting it with top-$k$ scoring for $k \in \{1, 5, 10\}$ based on cosine similarity and L2 distance metrics.
    However, we found the Mahalanobis distance to be the most stable.
    We also evaluated variants of this baseline that compute the dissimilarity of the full policy state $s_t$ (including both the point cloud embedding and the robots' end-effector poses), but found equivalent performance. 
        
    \item \textbf{CLIP Pretrained Embedding} quantifies the dissimilarity of the current image observation $I_t$ \textit{w.r.t.} to the images in the calibration dataset $\calD_\tau = \{\tau^i\}_{i=1}^M$ within the embedding space of a pretrained CLIP encoder~\cite{radford2021learning}. The score function (\cref{eq:sentinel-mahal-embeddings}) and calibration process are identical to those of \textbf{Policy Encoder Embedding}. 
    Importantly, the encoder used here is trained with a representation learning objective, which results in a structured embedding space and more interpretable embedding similarity scores. In our experiments, we use the open-source \texttt{clip-vit-base-patch32} version of CLIP without any fine-tuning.
    
    \item \textbf{ResNet Pretrained Embedding} is identical to \textbf{CLIP Pretrained Embedding}, except it quantifies image-space dissimilarity using embeddings from a ResNet18 pretrained model~\cite{he2016deep}.
    
    \item \textbf{Temporal Non-Distributional Minimum} is similar to STAC (\cref{appx:sentinel-stac}) in that it seeks to compute a consistency score between overlapping actions $a_{t+k:t+h-1|t}$ and $a_{t+k:t+h-1|t+k}$ sampled from the generative policy at contiguous policy-inference timesteps $t$ and $t+k$, respectively. However, it does so by using a non-statistical distance function. 
    In particular, this baseline computes the per-timestep temporal consistency score at timestep $t+k$ as
    \begin{equation*}
        S(s_{t+k}) = \underset{b \in \{1..B\}}{\min} \;\;\left\lVert a_{t+k:t+h-1|t} - a^b_{t+k:t+h-1|t+k} \right\rVert, 
    \end{equation*}
    where $a^b_{t+k:t+h-1|t+k} \sim \tilde{\pi}_{t+k}(\cdot | s_{t+k})$.
    That is, we sample a batch of $B$ action sequences at timestep $t+k$, compute their L2 distances \textit{w.r.t.} the overlapping actions of the previously executed action sequence $a_{t+k:t+h-1|t}$, and return the L2 distance associated with the most similar action sequence. Intuitively, this baseline attempts to find the closest action sequence at timestep $t+k$ to the previously executed action sequence, while STAC attempts to quantify how well the action distribution $\tilde{\pi}_{t+k}$ at timestep $t+k$ is represented in the distribution $\bar{\pi}_{t}$ at timestep $t$. 
    The values of $B$ are provided in \cref{tab:sentinel-stac-hyperparams}.
    The calibration and runtime procedures of this baseline are identical to those of STAC (\cref{sec:sentinel-temporal-consistency}).
    
    \item \textbf{Diffusion Reconstruction} adapts the diffusion-based OOD detection approach of \citet{graham2023denoising} for the case of diffusion policies. Specifically, this baseline computes the reconstruction error on re-noised action sequences sampled from the diffusion policy as
    \begin{equation}\label{eq:sentinel-diffusion-recon-score}
        S(s_t) = \E_{a^0\sim \pi(\cdot | s_t), \epsilon^i, i} \left[\left\lVert a^0 - \epsilon^{i:0}_\theta(\sqrt{\bar\alpha_i} a^0 + \sqrt{1 - \bar{\alpha}_i}\epsilon^i, s_t)\right\rVert^2 \right],
    \end{equation}
    where $\epsilon^{i:0}_\theta$ denotes the reverse diffusion process from the $i$-th denoising iteration to the $0$-th iteration, resulting in the reconstructed action. We approximate the expectation in \cref{eq:sentinel-diffusion-recon-score} over a batch of $B = 256$ action sequences sampled from the diffusion policy, each re-noised for $i \in \{5, 10, 25, 50\}$ forward diffusion steps (also referred to as \textit{reconstruction depths}). 
    We experimented with several sets of reconstruction depths and found comparable performance. 
    We note that this baseline comes with significant computational expense as it needs to perform the denoising process multiple times: i.e., if we would like to compute $R$ reconstructions, this baseline is approximately $R$ times more expensive than a single reverse diffusion process.
    The calibration and runtime procedures of this baseline are identical to those of STAC (\cref{sec:sentinel-temporal-consistency}).
    
    \item \textbf{Temporal Diffusion Reconstruction} is a temporal variant of \textbf{Diffusion Reconstruction} that also computes the reconstruction error on re-noised action sequences sampled from the diffusion policy, but reconstructs the action sequences conditioned on the previous state $s_t$ as
    \begin{equation*}
        S(s_t, s_{t+k}) = \E_{a_{t+k:t+h-1|t+k}^0\sim \tilde{\pi}_{t+k}, \epsilon^i, i} \left[\left\lVert \hat{a}^0 - \epsilon^{i:0}_\theta(\sqrt{\bar\alpha_i} \hat{a}^0 + \sqrt{1 - \bar{\alpha}_i}\epsilon^i, s_t)\right\rVert^2 \right].
    \end{equation*}
    Here, $\hat{a}^0$ denotes the action sequence over which reconstructions are computed, concatenating the first $k$ (executed) actions sampled at timestep $t$ with following $h-k$ (predicted) actions sampled at timestep $t+k$: that is, $\hat{a}^0 = a_{t:t+k|t} \oplus a_{t+k:t+h-1|t+k}^0$. 
    This step is necessary to ensure that the denoising process conditioned on $s_t$ only considers actions within the policy's prediction horizon. 
    This baseline represents an alternative form of temporal consistency.
    Intuitively, it asks whether action sequences sampled at timestep $t+k$ would also be sampled at timestep $t$, to which the answer is likely yes if the policy is in distribution, and likely no if the policy is OOD---because the marginal distributions conditioned on $s_t$ versus $s_{t+k}$ may be different.
    The hyperparameters of this baseline follow those of \textbf{Diffusion Reconstruction}.
    
    \item \textbf{DDPM Loss} computes the empirical DDPM loss on re-noised action sequences sampled from the diffusion policy as
    \begin{equation*}
        S(s_{t}) = \E_{a^0\sim \pi(\cdot | s_t), \epsilon^i, i} \left[\left\lVert \epsilon^i - \epsilon_\theta(\sqrt{\bar\alpha_i} a^0 + \sqrt{1 - \bar{\alpha}_i}\epsilon^i, s_t, i)\right\rVert^2 \right].
    \end{equation*}
    Here, the expectation is taken over a batch of $B = 256$ sampled action sequences and $10$ sampled denoising iterations $i \sim \calU[0, N)$, where $N$ is the total number of denoising iterations (\cref{appx:sentinel-diffusion-policy}). 
    We can think of this baseline as a more efficient version of \textbf{Diffusion Reconstruction}, since it directly quantifies the diffusion policy's performance on its training task without the need to reconstruct actions over numerous denoising iterations. 
    The calibration and runtime procedures of this baseline are identical to those of STAC (\cref{sec:sentinel-temporal-consistency}).
    
    \item \textbf{Temporal DDPM Loss} is a temporal variant of \textbf{DDPM Loss} that also computes the empirical DDPM loss on re-noised action sequences sampled from the diffusion policy, but does so conditioned on the previous state $s_t$ as
    \begin{equation*}
        S(s_t, s_{t+k}) = \E_{a_{t+k:t+h-1|t+k}^0\sim \tilde{\pi}_{t+k}, \epsilon^i, i} \left[\left\lVert \epsilon^i - \epsilon_\theta(\sqrt{\bar\alpha_i} \hat{a}^0 + \sqrt{1 - \bar{\alpha}_i}\epsilon^i, s_t, i)\right\rVert^2 \right],
    \end{equation*}
    where $\hat{a}^0 = a_{t:t+k|t} \oplus a_{t+k:t+h-1|t+k}^0$ (as defined in \textbf{Temporal Diffusion Reconstruction}).
    The hyperparameters of this baseline follow those of \textbf{DDPM Loss}, over which it is expected to offer advantages via temporal consistency.

    \item \textbf{Diffusion Output Variance} computes the variance over $B$ action sequences sampled from the diffusion policy and thresholds it \textit{w.r.t.} the $1-\delta$ quantile of sample variances computed over the calibration dataset $\calD_\tau$. 
    This baseline reflects an alternative output metric to temporal consistency that can be monitored to detect policy failure. 
    While computing output variances might bear resemblance to ensemble methods~\cite{LakshminarayananPritzelEtAll2017}, we note that this approach does not quantify epistemic model uncertainty. Doing so would require training multiple diffusion policies and performing inference with each at test time, which we avoid due to computational expense.
    The hyperparameters of this baseline are identical to those of STAC (see \cref{tab:sentinel-stac-hyperparams}). 
        
\end{itemize}
\paragraph{Discussion on Baselines} 
First, we highlight that the embedding-based approaches predict failure solely based on the dissimilarity or atypicality of the current state.
Hence, these baselines are not \textit{policy aware}: they may raise failure warnings for states that are dissimilar from those contained in the calibration dataset $\calD_\tau$ without understanding how the policy behaves in those states. 
In some cases, the policy may still succeed or generalize to minor distribution shifts in state, causing the detection performance of these baselines to significantly diminish.
The reconstruction-based approaches may account for the generalization characteristics of the policy but come with computational expense, which may prohibit their use in real-time settings.
The DDPM loss approaches present the next best alternative to STAC, as their score functions coincide with the diffusion policy's training task and can be computed at negligible computational cost. 
However, we note that the DDPM loss baseline is specific to diffusion policies, whereas STAC is agnostic to the generative policy formulation.

\subsubsection{VLM Baselines (Task-Level Monitors)}\label{appx:sentinel-vlm-baselines}
As described in \cref{sec:sentinel-vlm-assessment}, we propose to monitor the task progress of a generative policy by zero-shot prompting a VLM to analyze a video of the robot's execution up to the current timestep.
We contrast the performance of our Video QA approach with a variation, Image QA, that queries the VLM using only $I_t$, the image recorded at the current timestep $t$, rather than the full video $I_{0:t}$
This baseline is used to evaluate the importance of video-based reasoning compared to single images. 
We construct the Image QA prompt by minimally modifying the Video QA prompt (\cref{appx:sentinel-vlm}) as shown below:

\paragraph{Prompt Template (Image QA)}
\begin{mdframed}[backgroundcolor=light-gray, roundcorner=10pt, leftmargin=1, rightmargin=1, innerleftmargin=14, innertopmargin=12,innerbottommargin=0, outerlinewidth=1, linecolor=light-gray]
\begin{lstlisting}[linewidth=\columnwidth,breaklines=true]
You are the runtime monitor for an autonomous mobile manipulator robot capable of solving common household tasks. A camera system captures image frames (at approximately 1Hz) of the robot executing its current task online. As a runtime monitor, your job is to analyze the most recent image frame and identify whether the robot is a) in progress of executing the task or b) failing to execute the task, for example, by acting incorrectly or unsafely.

The robot's current task is to {DESCRIPTION}. The robot may take up to {TIME_LIMIT} seconds to complete this task. The current elapsed time is {TIME} seconds.

Format your output in the following form:
[start of output]
Questions: First, generate a set of task-relevant questions that will enable you to thoroughly analyze the image frame and identify what actions the robot has taken so far.
Answers: Second, precisely answer the generated questions, providing fine-grained visual details that will help you accurately assess the current state of progress on the task.
Analysis: Assess whether the robot is clearly failing at the task. Since the image frame only represents the robot's progress up to the current timestep and the robot moves slowly, refrain from making a failure classification unless the robot takes unsafe actions or is unlikely to complete the task in the allotted time. Explicitly note the amount of time that has passed in seconds and compare it with the time limit (e.g., x out of {TIME_LIMIT} seconds). Finally, based on the questions, answers, analysis, and elapsed time, decide whether the robot is in progress, or whether the robot will fail to complete its task in the remaining time (if any).
Overall assessment: {CHOICE: [ok, failure]}
[end of output]

Rules:
1. If you see phrases like {CHOICE: [choice1, choice2]}, it means you should replace the entire phrase with one of the choices listed. For example, replace the entire phrase '{CHOICE: [A, B]}'  with 'B' when choosing option B. Do NOT enclose your choice in '{' '}' brackets. If you are not sure about the value, just use your best judgement.
2. Do NOT forget to conclude your analysis with an overall assessment. As indicated above with '{CHOICE: [ok, failure]}', your only options for the overall assessment are 'ok' or 'failure'.
3. Always start the output with [start of output] and end the output with [end of output].
\end{lstlisting}
\end{mdframed}

\subsection{Evaluation Protocol}\label{appx:sentinel-evaluation-protocol}

\subsubsection{Definition: Policy Failure}\label{appx:sentinel-policy-failure}
Consider a policy $\pi(a | s)$ that operates within a finite-horizon Markov Decision Process (MDP): a 5-tuple $\langle \mathcal{S}, \mathcal{A}, T, R, H \rangle$, where $\mathcal{S}$ and $\mathcal{A}$ are the state and action spaces, $T(s' | s, a)$ is the transition model, $R(s, a, s')$ is the reward model, and $H$ is the MDP horizon.
Given an initial state $s_0$ representative of a new test scenario, executing the policy for $t$ timesteps produces a trajectory $\tau_t = (s_0, a_0, \ldots, s_t)$.
The trajectory's \textit{return} is defined as the cumulative sum of rewards: $R(\tau_t) = \sum_{t'=0}^{t-1}R(s_{t'}, a_{t'}, s_{t'+1})$.

We define \textbf{policy failure} simply in terms of task completion. 
More formally, given a defined success threshold $R_\tau$, the policy fails if the return on its trajectory $\tau_t$ does not exceed the success threshold within the MDP horizon: $R(\tau_t) < R_\tau$ where $t \geq H$.
In the simplest case, the success threshold $R_\tau$ equals $1$, and the reward model $R(s, a, s')$ equals $1$ \textit{iff} the task is complete at state $s'$.
For example, if the robot is tasked with picking up a cup and receives a reward of $1$ only once the cup is firmly grasped.
In experiments, we adhere to this definition of policy failure and threshold trajectory returns as $R(\tau_H) < R_\tau$ to compute ground-truth labels for whether or not a policy failed in a trajectory $\tau_H$.

\textbf{Relation to failure detection:} In \cref{sec:sentinel-problem-setup}, we provide a definition of the failure detection task---to detect whether a trajectory $\tau_H$ constitutes a policy failure at the earliest possible timestep $t$---that is different from detecting the specific timestep at which (or before) the policy ``fails.'' 
Doing so removes the need to manually specify task-specific failure criteria required to e.g., label each timestep in a trajectory. 
While the goal of our failure detectors is thus to flag failure episodes, it is still beneficial to catch failures at the earliest possible timestep, which is why a) our above definition of policy failure is formulated in terms of a partial trajectory $\tau_t$ up to the current timestep $t \leq H$ of the MDP, b) we propose an online detection scheme that monitors for failure at each timestep $t$ based on the trajectory up to the current timestep (i.e., $f(\tau_t) \rightarrow \{\texttt{ok}, \texttt{failure}\}$), and c) we report the detection time as a performance metric.

\subsubsection{Constructing the Calibration Dataset} 
Calibrating STAC and its baselines requires a small dataset of successful policy rollouts $\calD_\tau = \{\tau^i\}_{i=1}^M$, which provide grounding on the nominal, in-distribution behavior of the policy.
This allows us to evaluate the test-time behavior of a potentially failing policy \textit{w.r.t.} its known nominal behavior.

\paragraph{Calibration Data Quality}
We found it important to ensure the quality of trajectories $\tau^i \in \calD_\tau$.
Specifically, trajectories in which the policy succeeds but in an undesired or unacceptable manner should not be used for calibration.
For example, the policy may solve the \textbf{Close Box} task (\cref{fig:sentinel-environments}) but damage the lids of the box in the process. 
Including such a trajectory in the calibration dataset would define this behavior as \textit{nominal} and degrade the sensitivity of the detectors at test time. 
Returning to our example, the detectors may not raise a failure warning if the policy damages a box at test time.

\paragraph{Collecting the Calibration Dataset}
In practice, such a calibration dataset could be collected during a policy validation phase prior to deployment.
For instance, we collect $M = 50$ successful policy rollouts for each simulation domain, manually filtering episodes where the policy succeeded with unacceptable behavior (e.g., with jitter). 
We hypothesize that the performance of the detectors \textit{w.r.t.} the number of rollouts $M$ is task specific.
For example, a smaller calibration dataset may be sufficient for tasks with low variability (i.e., in a single, structured environment), while a larger dataset may be necessary if the policy is to be deployed at scale. 
We note, however, that increasing the calibration dataset size may be desirable to achieve stronger conformal guarantees on the detector's FPR (as derived in \cref{appx:sentinel-derivations}).

\paragraph{Calibrating on Demonstration Data} 
Finally, in attempt to eliminate the need to collect an additional calibration dataset of successful policy rollouts, we experimented with variants of STAC that directly calibrated on trajectories contained in the policy's demonstration dataset.
However, doing so led to a significant increase in the detector's FPR.
We attribute this to the well-known covariate shift problem for imitation learned policies~\cite{ross2010efficient,ross2011reduction}. 
That is, their prediction error increases quadratically on states induced under the policy, causing the detectors' to mistake successful test-time rollouts for failures. 

\subsubsection{Testing and Evaluation}
Instead of evaluating the failure detectors online (i.e., during policy rollouts), we collect several test datasets of policy rollouts, which consist of both successes and failures. 
Each trajectory is labeled either success or failure by thresholding the return at the final state of the episode (as detailed in \cref{appx:sentinel-policy-failure}).
We then perform offline evaluation of the failure detectors by invoking them at each timestep of the trajectory, which allows us to identify the first timestep at which the detectors issue a warning.

\subsubsection{Reported Metrics}\label{appx:sentinel-reported-metrics}
We expand on the metrics outlined in \cref{sec:sentinel-problem-setup}. 
We first define a \textit{positive} as a trajectory where the policy fails and a \textit{negative} as a trajectory where the policy succeeds.
A true positive is counted if the failure detector raises a warning at any timestep in a trajectory where the policy fails. 
A true negative is counted if the failure detector never raises a warning in a trajectory where the policy succeeds. 
The definitions for false positive and false negative follow accordingly.
Detection time is defined as the earliest timestep in which the failure detector raises a warning in a trajectory where the policy fails. 

In our experiments, we report true positive rate (TPR), true negative rate (TNR), false positive rate (FPR), detection time (DT), accuracy, and balanced accuracy.
TPR measures the number of true positives (detected failures) over total number of positives (failures). 
TNR measures the number of true negatives (detected successes) over total number of negatives (successes).
FPR measures the number of false positives (false alarms) over the total number of negatives (successes). 
Accuracy and balanced accuracy account for both the TPR and TNR of the detector. 
However, we report balanced accuracy when the test set contains a non-negligible imbalance of positive and negative trajectories.

\section{Additional Results}\label{appx:sentinel-results}

\subsection{Ablation Experiments on STAC}\label{appx:sentinel-stac-ablations}

\textbf{Does STAC's performance depend on the policy's prediction and execution horizon?}

We conduct an ablation study on the \textbf{PushT} domain to test how the performance of STAC varies with respect to the prediction horizon $h$ and execution horizon $k$ of the diffusion policy.
Together, the prediction and execution horizons determine the number of temporally overlapping action components (i.e., between $a_{t+k:t+h-1|t}$ and $a_{t+k:t+h-1|t+k}$) that are statistically compared by STAC, while the execution horizon governs how far apart in time the action distributions $\bar{\pi}_{t}$ and $\tilde{\pi}_{t+k}$ are generated. 

\begin{wrapfigure}{rt}{0.37\textwidth}
  \centering
  \vspace{-14pt}
  \includegraphics[width=1.0\textwidth]{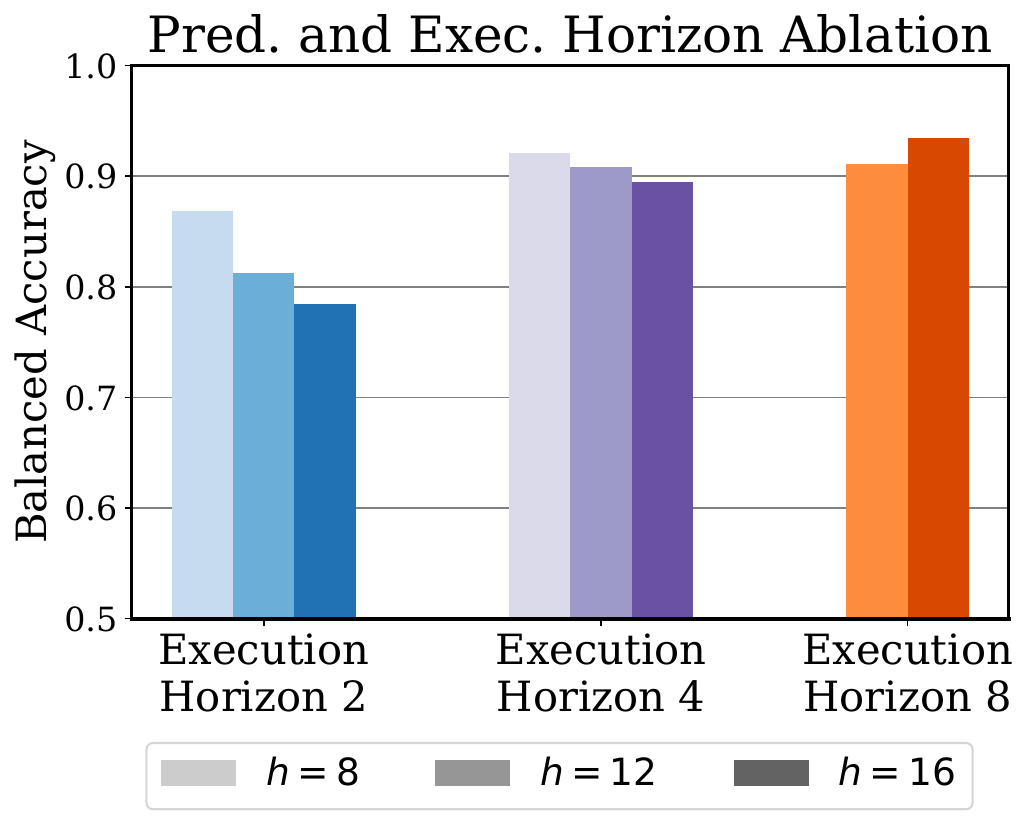}
  \caption[Ablation result on policy prediction and execution horizon for failure detection (Sentinel)]{\small
        Performance variation of STAC subject to different policy prediction and execution horizons in \textbf{PushT}.
    }
  \label{fig:sentinel-pusht-abl-result}
\end{wrapfigure}

The result is shown in \cref{fig:sentinel-pusht-abl-result}.
We find that STAC (MMD) performs comparatively across execution horizons of $k = 4$ and $k = 8$, but performs best with the standard diffusion policy settings of $k = 8$ and $h = 16$ (used for the main result in \cref{fig:sentinel-pusht-result}).
The detector's performance drops when using the smallest execution horizon of $k = 2$.
We attribute this to the relatively small amount of environment change that occurs within two execution steps, which causes $\bar{\pi}_{t}$ and $\tilde{\pi}_{t+k}$ to be similarly distributed and leads to overly conservative statistical distances. 
This is reflected in our results, where the detectors attain $>95$\% TNRs across various execution horizons, but using $k = 2$ leads to a significant drop in TPR to $61$\%, while $k = 4$ and $k = 8$ attain TPRs of $78$\% and $95$\%, respectively.

\begin{table}{
    \centering
    \caption[Ablation result on policy prediction horizon for failure detection (Sentinel)]{\small
        STAC ablation on policy prediction horizon $h$.
    }
    \label{tab:sentinel-horizon-ablation}
    \adjustbox{max width=\textwidth}{
        \begin{tabular}{clcccc}
            \toprule
            & \textbf{Domain} & \textbf{Pred. Horizon ($h$)} & \textbf{TPR $\uparrow$} & \textbf{TNR $\uparrow$} & \textbf{Accuracy $\uparrow$} \\
            \midrule
            \multirow{3}{*}{\STAB{\rotatebox[origin=c]{90}{\small \textbf{Sim.}}}}
            & Close Box & 8 & 0.92 & 0.94 & 0.93 \\
            & Close Box & 12 & 0.88 & 1.00 & 0.93 \\
            & Close Box & 16 & 0.96 & 1.00 & 0.98 \\
            \midrule
            \multirow{3}{*}{\STAB{\rotatebox[origin=c]{90}{\small \textbf{Real}}}}
            & Push Chair & 8 & 1.00 & 0.80 & 0.90 \\
            & Push Chair & 12 & 0.80 & 0.80 & 0.80 \\
            & Push Chair & 16 & 0.80 & 0.90 & 0.85 \\
            \bottomrule
        \end{tabular}
        }
    }
\end{table}

To further ablate the choice of policy prediction horizon, we conduct a similar study on the simulated \textbf{Close Box} and real-world \textbf{Push Chair} domains.
The result is shown in \cref{tab:sentinel-horizon-ablation}, where we find that STAC is quite robust to the choice of prediction horizon, while the best result is achieved by using the standard setting of $h=16$. 

Overall, STAC's performance may vary with the policy's execution horizon, but is relatively stable across choices of the policy's prediction horizon.
The fact that we calibrate STAC and deploy it with the same prediction horizon may normalize the influence of this parameter at test time.

\textbf{How does STAC's performance vary with the choice of statistical distance function?}

\begin{wrapfigure}{rt}{0.35\textwidth}
  \centering
  \vspace{-14pt}
  \includegraphics[width=1.0\textwidth]{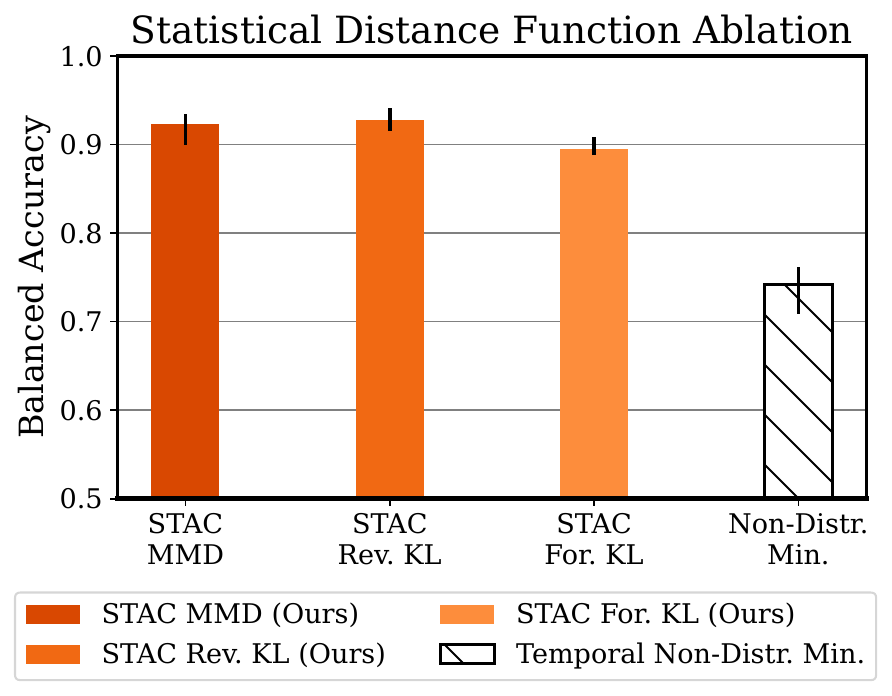}
  \caption[Ablation result on statistical distance function for failure detection (Sentinel)]{\small
        Performance variation of STAC based on the choice of statistical distance function in \textbf{PushT}.
    }
  \label{fig:sentinel-pusht-abl-result-functions}
\end{wrapfigure}

\cref{fig:sentinel-pusht-abl-result-functions} ablates STAC's performance across various choices of statistical distance functions in the \textbf{PushT} domain.
Here, we find that STAC performs comparably across common choices like maximum mean discrepancy (MMD) with RBF kernels and KL-divergence via kernel density estimation (details in \cref{appx:sentinel-stac}).
This finding is corroborated in \cref{tab:sentinel-erratic-failures-full}, where all variants of STAC attain a detection accuracy of over 90\% in the \textbf{Close Box} domain.
In contrast, we observe a large performance drop when using a non-statistical distance function (``Temporal Non-Distr. Min''; \cref{appx:sentinel-stac-baselines}) to measure temporal action consistency.
This performance drop can be attributed to stochastic multimodality of the generative policy, which makes it challenging to sample individual actions that are similar to those at preceding timesteps during policy rollout.
Because the non-statistical distance function is more sensitive than statistical distance functions to stochasticity in action sampling, we see an increased occurrence of false alarms.

\subsection{Extended Results: VLM Runtime Monitor}\label{appx:sentinel-extended-vlm-results}
To supplement the analysis provided in \cref{sec:sentinel-results}, we herein focus on the performance of our VLM runtime monitor and its complementary role to STAC for the detection of erratic and task progression failures.

\paragraph{Erratic Failure Analysis}
\cref{tab:sentinel-erratic-failures-full} presents the extended results of our experiments on the \textbf{Close Box} domain, where we aim to detect erratic policy failures that result from OOD scaling of the box.
We run several evaluations of our VLM runtime monitor, varying the choice of VLM (GPT-4o, Claude 3.5 Sonnet, Gemini 1.5 Pro~\cite{reid2024gemini}) and prompt template (Video QA, Image QA).
Since erratic failures in this domain are assigned to STAC---which detects 99\% of them---we would like the VLM to avoid raising false alarms so as to keep the overall FPR of Sentinel low when the two detectors are combined.

\begin{table*}[h!]
    \centering
    \caption[Extended results on simulated box closing erratic failures (Sentinel)]{\small
        \textbf{Extended results on detecting erratic failures in the Close Box domain.} 
        Our temporal consistency detector, STAC, detects 99\% of erratic failures exhibited by diffusion policies. 
        VLMs raise many false alarms when prompted with just a single image (Image QA), whereas performing Video QA (Ours) leads to a stark increase in TNR across all models.
        Performance also varies across the choice of VLM; GPT-4o is the most reliable in this domain, in contrast to other VLMs that struggle to accurately characterize the robot's task progress in the video. 
        Overall, Sentinel detects 100\% of failures, while combining STAC and the VLM increases false alarms to 13\%.
    } 
    \label{tab:sentinel-erratic-failures-full}
    \adjustbox{max width=\textwidth}{\begin{tabular}{clcccccccc|cccc}
        \toprule
        & \textbf{Category 1: Erratic Failures} & \multicolumn{3}{c}{\textbf{Close Box: In-Distribution}} & & \multicolumn{3}{c}{\textbf{Close Box: Out-of-Distribution}} & & & \multicolumn{3}{c}{\textbf{Close Box: Combined}} \\
        & & \multicolumn{3}{c}{(Policy Success Rate: 91\%)} & & \multicolumn{3}{c}{(Policy Success Rate: 41\%)} & & & \multicolumn{3}{c}{(Policy Success Rate: 67\%)} \\
        \cmidrule{3-5} \cmidrule{7-9} \cmidrule{12-14}
        & \textbf{Failure Detector} & TPR $\uparrow$ & TNR $\uparrow$ & Det. Time (s) $\downarrow$ & & TPR $\uparrow$ & TNR $\uparrow$ & Det. Time (s) $\downarrow$ & & & TPR $\uparrow$ & TNR $\uparrow$ & Accuracy $\uparrow$ \\
        \midrule
        \multirow{6}{*}{\STAB{\rotatebox[origin=c]{90}{\small \textbf{Diffusion}}}}
            & Temporal Non-Distr. Min. & 1.00 & 0.97 & 5.00 & & 1.00 & 0.27 & 12.35 & & & 1.00 & 0.77 & 0.85 \\
            & Diffusion Recon.~\cite{graham2023denoising} & 0.33 & 0.95 & 13.60 & & 0.40 & 1.00 & 17.08 & & & 0.37 & 0.96 & 0.76 \\
            & Temporal Diffusion Recon. & 1.00 & 0.96 & 8.47 & & 0.92 & 1.00 & 15.75 & & & 0.92 & 0.97 & 0.95 \\
            & DDPM Loss (\cref{eq:sentinel-ddpm-loss}) & 1.00 & 0.90 & 8.27 & & 1.00 & 0.94 & 14.54 & & & 1.00 & 0.91 & 0.94 \\
            & Temporal DDPM Loss & 1.00 & 0.95 & 7.53 & & 1.00 & 0.37 & 13.66 & & & 1.00 & 0.79 & 0.86 \\
            & Diffusion Output Variance & 0.33 & 0.94 & 14.00 & & 0.28 & 1.00 & 17.27 & & & 0.26 & 0.96 & 0.72 \\
        \midrule
        \multirow{3}{*}{\STAB{\rotatebox[origin=c]{90}{\footnotesize \textbf{Embed.}}}}
            & Policy Encoder & 0.25 & 0.98 & 16.27 & & 1.00 & 0.00 & 1.59 & & & 0.94 & 0.70 & 0.78 \\
            & CLIP Pretrained & 1.00 & 0.95 & 15.73 & & 1.00 & 0.00 & 8.20 & & & 1.00 & 0.68 & 0.79 \\
            & ResNet Pretrained & 1.00 & 0.95 & 17.87 & & 1.00 & 0.00 & 15.51 & & & 1.00 & 0.68 & 0.79 \\
        \midrule
        \multirow{3}{*}{\STAB{\rotatebox[origin=c]{90}{\small \textbf{STAC}}}}
            & STAC For. KL & 1.00 & 0.90 & 6.60 & & 0.99 & 0.85 & 14.04 & & & 0.99 & 0.89 & 0.92 \\
            & STAC Rev. KL & 1.00 & 0.95 & 7.60 & & 0.93 & 0.97 & 15.12 & & & 0.93 & 0.96 & 0.95 \\
            & \gcl{\textbf{STAC MMD*}} & \gcl{1.00} & \gcl{0.94} & \gcl{7.20} & \gcl{} & \gcl{0.99} & \gcl{0.93} & \gcl{14.72} & \gcl{} & \gcl{} & \gcl{0.99} & \gcl{0.94} & \gcl{0.96} \\
        \midrule
        \multirow{6}{*}{\STAB{\rotatebox[origin=c]{90}{\small \textbf{VLM}}}}
            & GPT-4o Image QA & 1.00 & 0.00 & 23.20 & & 1.00 & 0.00 & 23.20 & & & 1.00 & 0.00 & 0.29 \\
            & \gcl{\textbf{GPT-4o Video QA*}} & \gcl{1.00} & \gcl{0.89} & \gcl{21.20} & \gcl{} & \gcl{0.69} & \gcl{0.95} & \gcl{21.02} & \gcl{} & \gcl{} & \gcl{0.77} & \gcl{0.91} & \gcl{0.87} \\
            & Gemini 1.5 Pro Image QA & 1.00 & 0.00 & 21.20 & & 1.00 & 0.00 & 23.20 & & & 1.00 & 0.00 & 0.29 \\
            & Gemini 1.5 Pro Video QA & 1.00 & 0.57 & 17.20 & & 1.00 & 0.50 & 20.20 & & & 1.00 & 0.55 & 0.68 \\
            & Claude 3.5 Sonnet Image QA & 1.00 & 0.06 & 23.20 & & 0.69 & 0.10 & 23.20 & & & 0.77 & 0.07 & 0.27 \\
            & Claude 3.5 Sonnet Video QA & 0.83 & 0.31 & 23.20 & & 0.44 & 0.40 & 23.20 & & & 0.55 & 0.35 & 0.40 \\
        \midrule
        \midrule
        \rowcolor{green!13}
        \multicolumn{2}{l}{\textbf{Sentinel (STAC MMD* + GPT-4o Video QA*)}} & 1.00 & 0.86 & 5.47 & & 1.00 & 0.90 & 14.25 & & & 1.00 & 0.87 & 0.91 \\
        \bottomrule
    \end{tabular}}
\end{table*}

We first discuss the poor performance of the Image QA baseline. 
When prompted with just a single image, we find that all VLMs struggle to distinguish policy successes from failures and thereby do not exceed a TNR of 7\%. 
Without observing the initial state (i.e., the box with its lids open) and the actions of the robot, the VLM is unable to identify the lids of the box and whether they have been closed.
Thus, once the task's time limit is exceeded, the VLM simply declares failure. 
We iterated on several prompts that included detailed questions in attempt to coerce the VLM to reason about the location of the box and its lids, but this yielded negligible changes in performance.
As a result, all the outputs of the Image QA baseline resemble the following example of a false positive: 
 \begin{mdframed}[backgroundcolor=light-gray, roundcorner=10pt, leftmargin=1, rightmargin=1, innerleftmargin=14, innertopmargin=12,innerbottommargin=0, outerlinewidth=1, linecolor=light-gray]
\begin{lstlisting}[linewidth=\columnwidth,breaklines=true]
Analysis: The current observation shows the manipulator's arms positioned near the white box, with the grippers open and not grasping the lids. The two smaller white side lids and the bigger white back lid of the box are not visible, suggesting they are not yet folded. The elapsed time is 30 out of 30 seconds, which means the robot has reached the time limit for completing the task. Given that the lids are not folded and the task is not completed, the robot is clearly failing the task.

Overall assessment: failure
\end{lstlisting}
\end{mdframed}

Prompting the VLMs in a Video QA setup (\cref{appx:sentinel-vlm}) results in a strict increase in TNR across all models.
Although, depending on the domain, we clearly observe that some VLMs show better visual reasoning performance than others.
On the \textbf{Close Box} domain, GPT-4o generates relatively coherent descriptions of the videos, whereas Claude 3.5 Sonnet and Gemini 1.5 Pro often e.g., confuse closed lids for open lids or fail to recognize that the robot has made any significant progress. 
While this performance discrepancy is difficult to explain, two potential reasons are: a) the \textbf{Close Box} task can require reasoning over long videos (i.e., 20-30 image frames) and, while the VLMs' large context windows are accommodating, the VLMs may still be susceptible to recency bias~\cite{zhao2021calibrate}; b) the images rendered in this domain might be better represented in the training data of one VLM (GPT-4o, in this case) compared to others.

\paragraph{Task Progression Failure Analysis}
In the context of VLM runtime monitoring, task progression failures differ from erratic failures in that they are more visually apparent and hence simpler for the VLM to interpret. 
For example, the policy takes more obviously incorrect actions, e.g., clearly misplacing the cover in the \textbf{Cover Object} domain (see \cref{fig:sentinel-environments}), but it does so in a temporally consistent manner that goes unnoticed by STAC.
Therefore, under task progression failures, we require the VLM to attain both a high TPR and TNR, whereas we are mainly concerned with TNR under erratic failures.
The extended results of our task progression failure experiments are shown in \cref{tab:sentinel-task-progression-failures-cover} for the \textbf{Cover Object} domain and \cref{tab:sentinel-task-progression-failures-close} for the \textbf{Close Box} domain, the two of which are aggregated in \cref{fig:sentinel-full-system-result}.

\begin{table*}[h!]
    \centering
    \caption[Extended results on simulated object covering task-progress failures (Sentinel)]{\small
        \textbf{Extended results on detecting task progression failures in the Cover Object domain.}
        STAC only detects 12\% of task progression failures (i.e., when the policy fails in a temporally consistent manner), which highlights the need for VLM runtime monitoring.
        Claude 3.5 Sonnet exhibits the most reliable detection performance in this domain, however, we use a prompt ensembling technique (details in \cref{appx:sentinel-prompt-ensemble}) to reduce the number of false positives.
        Overall, Sentinel detects 88\% of failures whilst raising 11\% false alarms.
    } 
    \label{tab:sentinel-task-progression-failures-cover}
    \adjustbox{max width=\textwidth}{\begin{tabular}{clcccccccc|cccc}
        \toprule
        & \textbf{Category 2: Task Progression Failures} & \multicolumn{3}{c}{\textbf{Cover Object: In-Distribution}} & & \multicolumn{3}{c}{\textbf{Cover Object: Out-of-Distribution}} & & & \multicolumn{3}{c}{\textbf{Cover Object: Combined}} \\
        & & \multicolumn{3}{c}{(Policy Success Rate: 98\%)} & & \multicolumn{3}{c}{(Policy Success Rate: 3\%)} & & & \multicolumn{3}{c}{(Policy Success Rate: 56\%)} \\
        \cmidrule{3-5} \cmidrule{7-9} \cmidrule{12-14}
        & \textbf{Failure Detector} & TPR $\uparrow$ & TNR $\uparrow$ & Det. Time (s) $\downarrow$ & & TPR $\uparrow$ & TNR $\uparrow$ & Det. Time (s) $\downarrow$ & & & TPR $\uparrow$ & TNR $\uparrow$ & Accuracy $\uparrow$ \\
        \midrule
        \multirow{2}{*}{\STAB{\rotatebox[origin=c]{90}{\small \textbf{Diff.}}}}
            & Temporal Non-Distr. Min. & 1.00 & 0.93 & 2.40 & & 0.06 & 1.00 & 7.60 & & & 0.09 & 0.93 & 0.56 \\
            & Diffusion Output Variance & 0.00 & 0.77 & - & & 0.00 & 1.00 & - & & & 0.00 & 0.77 & 0.44 \\
        \midrule
        \multirow{3}{*}{\STAB{\rotatebox[origin=c]{90}{\small \textbf{STAC}}}}
            & STAC For. KL & 1.00 & 0.95 & 2.40 & & 0.03 & 1.00 & 6.40 & & & 0.06 & 0.95 & 0.56 \\
            & STAC Rev. KL & 1.00 & 0.93 & 2.40 & & 0.03 & 1.00 & 6.40 & & & 0.06 & 0.93 & 0.55 \\
            & \gcl{\textbf{STAC MMD*}} & \gcl{1.00} & \gcl{0.93} & \gcl{2.40} & \gcl{} & \gcl{0.09} & \gcl{1.00} & \gcl{7.73} & \gcl{} & \gcl{} & \gcl{0.12} & \gcl{0.93} & \gcl{0.58} \\
        \midrule
        \multirow{9}{*}{\STAB{\rotatebox[origin=c]{90}{\small \textbf{VLM}}}}            
            & GPT-4o Image QA & 1.00 & 0.07 & 12.00 & & 1.00 & 0.00 & 11.03 & & & 1.00 & 0.07 & 0.47 \\
            & GPT-4o Video QA & 1.00 & 0.05 & 5.60 & & 1.00 & 0.00 & 10.06 & & & 1.00 & 0.05 & 0.46 \\
            & Gemini 1.5 Pro Image QA & 1.00 & 0.05 & 12.00 & & 0.91 & 0.00 & 11.15 & & & 0.91 & 0.05 & 0.42 \\
            & Gemini 1.5 Pro Video QA & 1.00 & 0.00 & 5.60 & & 1.00 & 0.00 & 7.54 & & & 1.00 & 0.00 & 0.44 \\
            & Claude 3.5 Sonnet Image QA & 1.00 & 0.81 & 12.00 & & 0.70 & 1.00 & 12.00 & & & 0.71 & 0.82 & 0.77 \\
            & Claude 3.5 Sonnet Video QA & 1.00 & 0.84 & 12.00 & & 0.79 & 1.00 & 12.00 & & & 0.79 & 0.84 & 0.82 \\
            & Claude 3.5 Sonnet Video QA + Success Video & 1.00 & 0.77 & 12.00 & & 0.94 & 1.00 & 11.59 & & & 0.94 & 0.77 & 0.85 \\
            & Claude 3.5 Sonnet Video QA + Goal Images & 1.00 & 0.93 & 5.60 & & 0.76 & 1.00 & 11.74 & & & 0.76 & 0.93 & 0.86 \\
            & \gcl{\textbf{Claude 3.5 Sonnet Prompt Ensemble* (see \cref{appx:sentinel-prompt-ensemble})}} & \gcl{1.00} & \gcl{0.93} & \gcl{12.00} & \gcl{} & \gcl{0.85} & \gcl{1.00} & \gcl{12.00} & \gcl{} & \gcl{} & \gcl{0.85} & \gcl{0.93} & \gcl{0.90} \\
        \midrule
        \midrule
        \multicolumn{2}{l}{Sentinel (STAC MMD* + Claude 3.5 Sonnet Video QA)} & 1.00 & 0.79 & 2.40 & & 0.82 & 1.00 & 8.68 & & & 0.82 & 0.80 & 0.81 \\
        \rowcolor{green!13}
        \multicolumn{2}{l}{\textbf{Sentinel (STAC MMD* + Claude 3.5 Sonnet Prompt Ensemble*)}} & 1.00 & 0.88 & 2.40 & & 0.88 & 1.00 & 8.69 & & & 0.88 & 0.89 & 0.88 \\
        \bottomrule
    \end{tabular}}
\end{table*}

Among the VLMs considered, we find that Claude 3.5 Sonnet exhibits the best performance in the \textbf{Cover Object} domain (\cref{tab:sentinel-task-progression-failures-cover}), achieving a 79\% TPR and an 84\% TNR when prompted in a Video QA setup.
Qualitative analysis of the responses from GPT-4o and Gemini 1.5 Pro reveals that they misinterpret the videos and thus raise an excessive number of false alarms.
As expected, STAC achieves an appreciable 93\% TNR, but only detects 12\% of task progression failures.
The combination of STAC and Claude 3.5 Sonnet Video QA performs amicably (82\% TPR, 80\% TNR) but can be improved in terms of reducing the number of false positives, the majority of which are raised by the VLM. 
Here, we show that the prompt ensembling strategy discussed in \cref{appx:sentinel-prompt-ensemble} strengthens the reliability of our VLM runtime monitor (``Claude 3.5 Sonnet Prompt Ensemble''), increasing the TNR to 93\%.
Finally, the combination of this improved VLM runtime monitor with STAC results in a version of Sentinel that detects 88\% of failures whilst raising a more acceptable number of false alarms (11\%).

\begin{table*}[h!]
    \centering
    \caption[Extended results on simulated box closing task-progress failures (Sentinel)]{\small
        \textbf{Extended results on detecting task progression failures in the Close Box domain.}
        Here, STAC detects considerably more task progression failures than in the \textbf{Cover Object} domain, yet 35\% of failures are left undetected. 
        As in \cref{tab:sentinel-erratic-failures-full}, we find that video-based reasoning is necessary for VLMs to attain high TNRs, with GPT-4o showing the best performance. 
        Overall, Sentinel detects 96\% of failures whilst raising 14\% false alarms.
    }
    \label{tab:sentinel-task-progression-failures-close}
    \adjustbox{max width=\textwidth}{\begin{tabular}{clcccccccc|cccc}
        \toprule
        & \textbf{Category 2: Task Progression Failures} & \multicolumn{3}{c}{\textbf{Close Box: In-Distribution}} & & \multicolumn{3}{c}{\textbf{Close Box: Out-of-Distribution}} & & & \multicolumn{3}{c}{\textbf{Close Box: Combined}} \\
        & & \multicolumn{3}{c}{(Policy Success Rate: 85\%)} & & \multicolumn{3}{c}{(Policy Success Rate: 0\%)} & & & \multicolumn{3}{c}{(Policy Success Rate: 40\%)} \\
        \cmidrule{3-5} \cmidrule{7-9} \cmidrule{12-14}
        & \textbf{Failure Detector} & TPR $\uparrow$ & TNR $\uparrow$ & Det. Time (s) $\downarrow$ & & TPR $\uparrow$ & TNR $\uparrow$ & Det. Time (s) $\downarrow$ & & & TPR $\uparrow$ & TNR $\uparrow$ & Accuracy $\uparrow$ \\
        \midrule
        \multirow{2}{*}{\STAB{\rotatebox[origin=c]{90}{\small \textbf{Diff.}}}}
            & Temporal Non-Distr. Min. & 1.00 & 0.97 & 4.67 & & 0.67 & - & 7.46 & & & 0.71 & 0.97 & 0.82 \\
            & Diffusion Output Variance & 0.00 & 1.00 & - & & 0.28 & - & 11.57 & & & 0.25 & 1.00 & 0.55 \\
        \midrule
        \multirow{3}{*}{\STAB{\rotatebox[origin=c]{90}{\small \textbf{STAC}}}}
            & STAC For. KL & 1.00 & 0.97 & 5.07 & & 0.61 & - & 8.14 & & & 0.65 & 0.97 & 0.78 \\
            & STAC Rev. KL & 1.00 & 0.97 & 6.13 & & 0.61 & - & 10.11 & & & 0.65 & 0.97 & 0.78 \\
            & \gcl{\textbf{STAC MMD*}} & \gcl{1.00} & \gcl{0.97} & \gcl{5.47} & \gcl{} & \gcl{0.61} & \gcl{-} & \gcl{9.06} & \gcl{} & \gcl{} & \gcl{0.65} & \gcl{0.97} & \gcl{0.78} \\
        \midrule
        \multirow{6}{*}{\STAB{\rotatebox[origin=c]{90}{\small \textbf{VLM}}}}            
            & GPT-4o Image QA & 1.00 & 0.00 & 23.20 & & 1.00 & - & 22.68 & & & 1.00 & 0.00 & 0.60 \\
            & \gcl{\textbf{GPT-4o Video QA*}} & \gcl{1.00} & \gcl{0.89} & \gcl{21.20} & \gcl{} & \gcl{0.87} & \gcl{-} & \gcl{22.00} & \gcl{} & \gcl{} & \gcl{0.88} & \gcl{0.89} & \gcl{0.89} \\
            & Gemini 1.5 Pro Image QA & 1.00 & 0.00 & 21.20 & & 0.96 & - & 23.20 & & & 0.96 & 0.00 & 0.57 \\
            & Gemini 1.5 Pro Video QA & 1.00 & 0.57 & 17.20 & & 0.98 & - & 15.47 & & & 0.98 & 0.57 & 0.82 \\
            & Claude 3.5 Sonnet Image QA & 1.00 & 0.06 & 23.20 & & 0.78 & - & 23.20 & & & 0.81 & 0.06 & 0.51 \\
            & Claude 3.5 Sonnet Video QA & 0.83 & 0.31 & 23.20 & & 0.80 & - & 23.20 & & & 0.81 & 0.31 & 0.61 \\
        \midrule
        \midrule
        \rowcolor{green!13}
        \multicolumn{2}{l}{\textbf{Sentinel (STAC MMD* + GPT-4o Video QA*)}} & 1.00 & 0.86 & 5.47 & & 0.96 & - & 12.20 & & & 0.96 & 0.86 & 0.92 \\
        \bottomrule
    \end{tabular}}
\end{table*}

The results in \cref{tab:sentinel-task-progression-failures-close} reaffirm the following \textbf{key takeaways}: a) image-based VLM reasoning is insufficient for understanding the robot's task progress, thus resulting in low TNRs; b) the VLMs' performances vary across domains, with GPT-4o and Claude 3.5 Sonnet performing the best in the \textbf{Close Box} and \textbf{Cover Object} domains, respectively; c) STAC and the VLM runtime monitor play complementary roles toward a performant overall failure detector across domains.
We expect the performance of our VLM runtime monitor (and thus Sentinel) to improve with the future release of more capable VLMs~\cite{chen2024spatialvlm}, which may also eliminate the discrepancies among VLMs noted above. 

\textbf{Discussion: Why combine failure detectors by taking the union of their predictions?}

Our full failure detector, Sentinel, combines STAC and the VLM by taking the union of their predictions (i.e., the ``Logical OR'' in \cref{fig:sentinel-full-system}), which, in the worst case, compounds their false positive rates by applying the union bound.
However, doing so follows from several design considerations:
\begin{itemize}
    \item \textbf{Importance weighting:} We explicitly define one failure category as the complement of the other and assign a specialized detector to each because it is extremely difficult to design a single detector that captures highly heterogeneous failure modes. Thus, by using the ``OR'' operation, we are placing \textit{equal importance} on the two proposed failure categories.
    We note that our failure detector's primary purpose is to detect unseen failures at deployment time: i.e., we do not assume any data of robot failures to calibrate the detector, which may be necessary to tune importance weights for different detectors.
    Furthermore, importance weights that are optimal on one dataset may perform poorly on failure modes not represented in that data.
    \item \textbf{Performance interpretability:} Using a simple scheme to combine detectors makes it easy to interpret a) the runtime behavior of the combined detector and b) which individual detectors are contributing to performance and when.
    For example, using the logical ``OR'' implies that Sentinel's overall FPR remains acceptably low if both STAC and the VLM runtime monitor have low FPRs.
    STAC achieves a provably low FPR (\cref{appx:sentinel-derivations}), while strategies exist to reduce the VLM's FPR through e.g., prompt ensembling (\cref{appx:sentinel-prompt-ensemble}) or conformal calibration~\cite{knowno2023}.
    Thus, we can expect a low FPR when the detectors are combined.
    We can similarly interpret Sentinel's TPR performance, as exemplified in our experimental analysis.
    Ease of performance interpretation may not hold true with more sophisticated schemes for combining detectors.
    
    \item \textbf{Runtime constraints:} 
    In practice, STAC (fast) and the VLM runtime monitor (slow) come with different inference-time latencies and may need to run at distinct timescales.
    Thus, our logical ``OR'' combination only applies at overlapping timesteps and otherwise allows each failure detector to flag independently, i.e., without synchronizing their detection rates.
    This flexibility is crucial, as more sophisticated combination schemes could encounter issues if e.g., unexpected network latencies result in delayed responses from a cloud-hosted VLM.
\end{itemize}
Nevertheless, we note that more sophisticated combination schemes might offer advantages, such as improved detection performance or scalability when integrating additional detectors. 
Exploring such schemes that align with the above design considerations represents a valuable direction for future work.

\section{Derivations}\label{appx:sentinel-derivations}
To validate our design choices, we show that STAC's score function and calibration procedure in \cref{sec:sentinel-temporal-consistency} provably result in a low FPR. To do so, we apply recently popularized tools from conformal prediction because they are sample efficient and distribution free, meaning that they do not require distributional assumptions on the trajectory rollouts. Our guarantee is a direct application of the standard results in split conformal prediction \cite{angelopoulos2021gentle}, but to ensure the self-containedness of this manuscript, we first briefly reintroduce the core concepts in conformal prediction (taken from \cite{angelopoulos2021gentle}) using the notation in \cref{chapter:sentinel}.

\paragraph{Background on Conformal Inference} In its most basic form, conformal prediction aims to construct a prediction set $\calC$ that will contain the true value of a new test point $X_{\mathrm{test}}$ with a user defined probability of at least $1 - \delta$ \cite{angelopoulos2021gentle}. To do so, a conformal algorithm requires 1) a sequence of calibration samples $\{X^i\}_{i=1}^M$ with all samples $X^1, \dots, X^M, X_{\mathrm{test}}$ i.i.d. and 2) a conformity score function $\eta(X) \in \R$. Intuitively, conformal methods use $\{\eta(X^i)\}_{i=1}^M$ to identify how likely $\eta(X_{\mathrm{test}})$ is to lie within the range of a $1-\delta$ fraction of the calibration samples (i.e., how well $X_{\mathrm{test}}$ conforms to the calibration data). We emphasize that this approach ensures that we construct a valid prediction set $\calC$, regardless of the choice of conformity score and without knowing any properties of the data generating distribution:

\begin{theorem}[Adapted from Thm. D.1 in \cite{angelopoulos2021gentle}]\label{thm:sentinel-conformal}
    Let $\calD_{\mathrm{calib}} = \{X^1, \dots, X^M\}$ be a calibration dataset and let $X_{\mathrm{test}}$ be a test sample. Suppose that the samples in $\calD_{\mathrm{calib}}$ and $X_{\mathrm{test}}$ are independent and identically distributed (i.i.d.). Then, defining 
    \begin{equation*}
        \gamma := \inf \bigg\{ \xi \in \R \ : \ \frac{|\{i : \eta(X^i) \leq \xi\}|}{M} \geq \frac{\lceil (M + 1)(1 - \delta)\rceil}{M}\bigg\}
    \end{equation*}
    as the $\frac{\lceil (M + 1)(1 - \delta)\rceil}{M}$ empirical quantile of the calibration data ensures that
    \begin{equation*}
        \prob \big(\eta(X_{\mathrm{test}}) \leq \gamma\big) \geq 1 - \delta.
    \end{equation*} 
    Here, $\lceil\cdot\rceil$ denotes the ceiling function.
\end{theorem}

\paragraph{Conformal guarantee of STAC} The base split conformal procedure outlined by \cref{thm:sentinel-conformal} requires that the samples used for calibration and test are i.i.d. This is not the case for states and actions observed sequentially within a trajectory, complicating the analysis of applying STAC at each timestep within a trajectory. 
Thus, to resolve this issue and provide a guarantee when we sequentially apply STAC on the correlated state-action pairs within a trajectory, we calibrate the detector using the consistency scores generated across full trajectories in \cref{sec:sentinel-temporal-consistency}. This allows us to rigorously bound the FPR using \cref{thm:sentinel-conformal}.
\begin{proposition}[STAC has low FPR]\label{prop:sentinel-stac}
    Let $\calD_\tau = \{\tau^i\}_{i=1}^M \iid P_\tau$ be the validation dataset of successful trajectories, each consisting of $H_i \leq H$ timesteps and drawn i.i.d. from the closed-loop nominal distribution $P_\tau$. For notational simplicity, assume that any trajectory has a length divisible by $k$ (i.e., $H_i \bmod k = 0$). Moreover, let $\eta_t$ be defined as the STAC temporal consistency score at some timestep $0 \leq t \leq H$ in \cref{eq:sentinel-cum-score-fn} and set $\gamma$ equal to the empirical $\frac{\lceil (M + 1)(1 - \delta)\rceil}{M}$ quantile of the terminal STAC scores $\{\eta_{H_i}^i\}_{i=1}^M$ of the trajectories in $\calD_\tau$. 
    Then, the \emph{false positive rate}---that is, the probability that we raise a false alarm at any point during a new successful test trajectory $\tau \sim P_\tau$ of length $H' \leq H$---is at most $\delta$:
    \begin{equation}\label{eq:sentinel-stac-fpr}
        \mathrm{FPR} := \prob_{P_\tau} \big(\exists \ 0 \leq t \leq H' \ \mathrm{s.t.} \ \eta_t > \gamma \big) \leq \delta.
    \end{equation}
\end{proposition}
\begin{proof}
    Let $H' \leq H$ be the length of the test trajectory $\tau$. If there is no distribution shift, i.e., when the test trajectory $\tau$ is i.i.d. with respect to $\calD_\tau \iid P_\tau$, it holds that $\eta_{H'}$ and $\{\eta_{H_i}^i\}_{i=1}^M$ are i.i.d. Therefore, by \cref{thm:sentinel-conformal}, we have that 
    \begin{equation*}
        \prob_{P_\tau}\big(\eta_{H'} > \gamma \big) \leq \delta.
    \end{equation*}
    Moreover, since we define $\eta_t = \sum_{i=0}^{j-1} \hat{D}(\bar{\pi}_{ik}, \  \tilde{\pi}_{(i+1)k})$ for $t = jk$ in \cref{eq:sentinel-cum-score-fn} and since $\hat{D}(\cdot, \cdot) \geq 0$ because it is a statistical distance, it follows that $\eta_t$ is increasing. That is, $\eta_{0} \leq \eta_{k} \leq \eta_{2k} \leq \dots \leq \eta_{H'}$. Therefore, if $\eta_t$ crosses the threshold $\gamma$ at any time, it also holds that $\eta_{H'} > \gamma$. This immediately implies the proposition, as we then have that 
    \begin{equation*}
        \prob_{P_\tau} \big(\exists \ 0 \leq t \leq {H'} \ \mathrm{s.t.} \ \eta_t > \gamma \big) = \prob_{P_\tau}\big(\eta_{H'} > \gamma \big) \leq \delta.
    \end{equation*}
\end{proof}

We conclude this section with three remarks:
\begin{enumerate}
    \item We only bound the FPR, which ensures that our algorithm does not raise a false alarm with high probability, so that any warnings likely correspond to an OOD scenario. We do so because a system that frequently raises false alarms is impractical to use. Our calibration approach does not guarantee the detection of failures, nor does it guarantee that we do not issue false alarms on OOD successes, as this is not possible without any distributional assumptions on the OOD scenarios or without using failure data for calibration \cite{luo2024online}. Instead, we empirically find that our temporal consistency score performs amicably at detecting failures in our experiments. 
    \item \cref{prop:sentinel-stac} only certifies that the FPR of STAC is low. We make no claims on combined performance of STAC and the VLM, as VLM represents a black-box classifier.
    Future work could investigate methodologies to jointly calibrate an ensemble of failure detectors. 
    \item Conformal guarantees, like those in \cref{thm:sentinel-conformal} and \cref{prop:sentinel-stac}, are \emph{marginal} with respect to the calibration data. That is, they may not hold exactly when given a particular calibration dataset, but if we were to sample thousands of calibration datasets, the guarantees would hold on average. Thus, as expected, STAC does not exactly satisfy \cref{eq:sentinel-stac-fpr} in our results, as compute budgets restricted our experiments to repetitions on a limited number of random seeds.
\end{enumerate}

\chapter{Additional Details: CUPID}\label{chapter:cupid-appx}

\section*{Appendix Overview -- Curating Data your Robot Loves with Influence Functions}
The appendix offers additional details \textit{w.r.t.} the implementation of \basemethod{} (\cref{appx:cupid-method}), the experiments conducted (\cref{appx:cupid-experiments}), along with extended results and analysis (\cref{appx:cupid-results}), and finally, supporting derivations for our data curation methods (\cref{appx:cupid-derivations}). Videos and code are made available at: \url{https://cupid-curation.github.io}.

\startcontents[sections]
\printcontents[sections]{l}{1}{\setcounter{tocdepth}{2}}

\section{Implementation Details}\label{appx:cupid-method}

\subsection{Influence Functions for Diffusion Policies}\label{appx:cupid-actinf-dp}

For ease of reference in this section, we restate the definition of the action influence (\cref{def:actinf}) and the proposition establishing performance influence (\cref{prop:cupid-polinf}), both originally introduced in \cref{sec:cupid-method}.

\paragraph{Restatement of \cref{def:actinf}.}\hspace{-0.8em} 
\textit{The \emph{\textbf{action influence}} of a state-action pair $(s,a)$ on a test state-action pair $(s',a')$ is the influence of $(s,a)$ on the policy's log-likelihood $ \log \pi_\theta(a'|s')$. That is,}
\begin{equation*}
    \actinf((s',a'), (s,a)) := -\nabla_\theta \log\pi_\theta(a'|s')^\top H_{\mathrm{bc}}^{-1}\nabla_\theta \ell(s, a; \pi_\theta).
\end{equation*}

\paragraph{Restatement of \cref{prop:cupid-polinf}.}\hspace{-0.8em}
\textit{Assume that $\theta(\calD) = \arg\min_{\theta'} \calL_{\mathrm{bc}}(\theta'; \calD)$, that $\calL_{\mathrm{bc}}$ is twice differentiable in $\theta$, and that $H_{\mathrm{bc}} \succ 0$ is positive definite (i.e., $\theta(\calD)$ is not a saddle point)\footnoteref{fn:cupid-track}. Then, it holds that}
\begin{equation*}
    \polinf(\xi) = \E_{\tau\sim p(\tau|\pi_\theta)}\bigg[\frac{R(\tau)}{H} \sum_{(s',a')\in\tau}\sum_{(s,a)\in\xi}\actinf\big((s',a'),(s,a)\big)\bigg].
\end{equation*}
\textit{where $\polinf(\xi)$ is the \emph{\textbf{performance influence}} of a demonstration $\xi$ (as introduced in \cref{def:polinf}).}

\subsubsection*{Computing the Action Influence}
Although \cref{prop:cupid-polinf} provides a clean mechanism to attribute policy performance to its training data by leveraging influence scores on action log-likelihoods, computing $\nabla_\theta \log\pi_\theta(a'|s')$ (in the action influence $\actinf$) for diffusion-based policy architectures is nontrivial due to the iterative denoising process~\cite{ho2020denoising, song2020score}. 
Instead, various works outside robotics propose to approximate the log-likelihood with the denoising loss $\ell(s', a'; \pi_{\theta})$ for the purpose of data attribution~\cite{georgiev2023journey}, because the denoising loss is proportionate to the variational lower bound on $\log \pi_\theta(a'|s')$. In \cref{sec:cupid-experiments}, we apply a similar approximation to perform data attribution on state-of-the-art diffusion policies~\cite{chi2023diffusion}, which we describe below.

\textbf{Diffusion Policy:} Consider the standard diffusion policy architecture~\cite{chi2023diffusion}. An action $a := a^0$ is generated by iteratively denoising an initially random action $a^T \sim \calN(0, 1)$ over $T$ steps as $a^T, \ldots, a^0$ using a noise prediction network $\epsilon_\theta$, where $a^i$ denotes the generated action at the $i$-th denoising iteration. Following the imitation learning setting described in \cref{sec:cupid-formulation}, the parameters $\theta$ of the noise prediction network $\epsilon_\theta$ are fit to the BC objective as $\theta = \arg\min_{\theta'} \{\mathcal{L}_{\text{bc}}(\theta'; \mathcal{D}) := \frac{1}{|\calD| H}\sum_{\xi^i\in\calD}\sum_{(s, a) \in \xi^i} \ell(s, a; \pi_{\theta'})\}$.
Here, the noise prediction network $\epsilon_\theta$ is trained to predict random noise $\epsilon^i \sim \calN(0, 1)$ added to the action $a$ at randomly sampled timesteps $i\sim \calU[0, T)$ of the diffusion process using the loss function $\ell$ defined as
\begin{equation}\label{eq:cupid-dp-loss}
    \ell(s, a; \pi_{\theta'}) := \mathbb{E}_{\epsilon^i, i} \left[||\epsilon^i - \epsilon_{\theta'}(\sqrt{\bar\alpha_i} a + \sqrt{1 - \bar{\alpha}_i}\epsilon^i, s, i)||^2 \right],
\end{equation}
where the constants $\bar \alpha_i$ depend on the chosen noise schedule of the diffusion process. 

\textbf{Influence Approximations:} Since the denoising loss $\ell$ in \cref{eq:cupid-dp-loss} is proportionate to the variational lower bound on the action log-likelihood $\log \pi_\theta(a|s)$, it may seem intuitive to substitute $\nabla_\theta\log \pi_\theta(a'|s')$ with $-\nabla_\theta\ell(s', a'; \pi_\theta)$---assuming gradient alignment---to approximate the action influence (\cref{eq:cupid-actinf}) as
\begin{equation}\label{eq:cupid-actinf-approx-loss}
    \actinf((s',a'), (s,a)) \approx \nabla_\theta\ell(s', a'; \pi_\theta)^\top H_{\mathrm{bc}}^{-1}\nabla_\theta \ell(s, a; \pi_\theta).
\end{equation}
A similar approach is taken by \citet{georgiev2023journey} for attributing the generations of image-based diffusion models. However, consistent with more recent results in the data attribution literature~\cite{zheng2023intriguing, lin2024diffusion}, we find this approximation to work poorly in practice, with highly influential training samples $(s, a) \in \mathcal{D}$ rarely reflecting the test-time transitions $(s', a') \in \tau$ over which the action influences are computed. 
Instead, we follow the approach of \citet{zheng2023intriguing}, which entails replacing both $\log \pi_\theta(a'|s')$ and $\ell(s, a; \pi_\theta)$ in \cref{eq:cupid-actinf} with a surrogate, label-agnostic output function $\ell_{\mathrm{square}}(s, a; \pi_\theta) := \mathbb{E}_{\epsilon^i, i} [||\epsilon_{\theta}(\sqrt{\bar\alpha_i} a + \sqrt{1 - \bar{\alpha}_i}\epsilon^i, s, i)||^2]$, making our final approximation of the action influence
\begin{equation}\label{eq:cupid-actinf-approx-square}
    \actinf((s',a'), (s,a)) \approx \nabla_\theta \ell_{\mathrm{square}}(s', a'; \pi_\theta)^\top  H_{\mathrm{square}}^{-1}\nabla_\theta \ell_{\mathrm{square}}(s, a; \pi_\theta).
\end{equation}
Here, $H_{\mathrm{square}} = \frac{1}{|\calD| H}\sum_{\xi^i\in\calD}\sum_{(s, a) \in \xi^i} \nabla_\theta \ell_{\mathrm{square}}(s, a; \pi_\theta) \nabla_\theta \ell_{\mathrm{square}}(s, a; \pi_\theta)^\top$ is the Gauss-Newton approximation of the Hessian---as introduced by \citet{martens2020insights} and applied for stable and efficient influence estimation in \cite{park2023trak, bae2022if}---under the surrogate output function $\ell_{\mathrm{square}}$. 

\textbf{Additional Remarks:} While the use of $\ell_{\mathrm{square}}$ may seem counterintuitive at first, it offers three key advantages for computing action influences: 
\begin{enumerate}
    \item Leave-one-out influences (\cref{sec:cupid-background}) computed using $\ell_{\mathrm{square}}$ (\cref{eq:cupid-actinf-approx-square}) are empirically found to correlate better with actual changes in a diffusion model's loss---i.e., the difference $\ell(s', a'; \pi_{\theta(\mathcal{D} \setminus (s, a))}) - \ell(s', a'; \pi_{\theta(\mathcal{D})})$---than those computed using the loss $\ell$ (\cref{eq:cupid-actinf-approx-loss})~\cite{zheng2023intriguing}.
    \item Theoretical analysis also shows that $\ell_{\mathrm{square}}$ more closely aligns with a distributional formulation of the leave-one-out influence compared to the loss $\ell$~\cite{lin2024diffusion}. In the case of diffusion policies, this distributional formulation would seek to design $\actinf$ such that it approximates the \textit{leave-one-out divergence} $\actinf((s', a')), (s, a)) \approx D_{\mathrm{KL}}( \pi_{\theta(\mathcal{D})} (a' | s') || \pi_{\theta(\mathcal{D} \setminus (s, a))}(a' | s'))$.
    \item Using $\ell_{\mathrm{square}}$ significantly reduces the computational cost of computing action influences for policies with high-dimensional action spaces, because the $\ell^2$-norm collapses the model's prediction into a scalar $||\epsilon_{\theta}(\sqrt{\bar\alpha_i} a + \sqrt{1 - \bar{\alpha}_i}\epsilon^i, s, i)||^2$. As a result, computing \cref{eq:cupid-actinf-approx-square} requires only a single model gradient $\nabla_\theta \ell_{\mathrm{square}}$ per training and test sample. In contrast, while the technique proposed by \citet{lin2024diffusion} offers a more accurate estimate of the leave-one-out divergence $D_{\mathrm{KL}}( \pi_{\theta(\mathcal{D})} (a' | s') || \pi_{\theta(\mathcal{D} \setminus (s, a))}(a' | s'))$, its computational cost scales linearly with the dimensionality of the model's output, which may be prohibitive.
\end{enumerate}

\textbf{Accuracy-Efficiency Tradeoff:} We note that our approach for computing the performance influence of a demonstration (\cref{eq:cupid-policy-inf-deriv}) is agnostic to the choice of influence estimation technique~\cite{georgiev2023journey, zheng2023intriguing, lin2024diffusion, mlodozeniec2024influence, xie2024data}, allowing practitioners to trade off between accuracy and efficiency based on available computational resources, and enabling integration of improved data attribution methods (e.g., \cite{ilyas2025magic}) in the future.

\subsection{CUPID Hyperparameters}\label{appx:cupid-hyperparameters}
We use the same set of hyperparameters for \basemethod{} and \qualitymethod{} across all experiments.

\textbf{Performance Influence (\cref{eq:cupid-policy-inf-deriv}):} For all tasks, we define the trajectory return to be $R(\tau) = 1$ if $\tau$ completes the task and $R(\tau) = -1$ otherwise. As a result, every rollout trajectory $\tau \sim p(\cdot|\pi_\theta)$ provides information on the utility of each demonstration toward the policy's closed-loop performance. We also found \basemethod{} to work with alternative return definitions---for example, focusing solely on successful rollouts by setting $R(\tau) = 0$ when $\tau$ fails. However, such choices may increase sample complexity.

\textbf{Action Influence (\cref{eq:cupid-actinf-approx-square}):} The action influence requires computing the gradient of an expectation $\nabla_\theta\ell_{\mathrm{square}}(s, a; \pi_\theta) = \nabla_\theta\mathbb{E}_{\epsilon^i, i} [||\epsilon_{\theta}(\sqrt{\bar\alpha_i} a + \sqrt{1 - \bar{\alpha}_i}\epsilon^i, s, i)||^2]$. For all tasks, we approximate the expectation using a batch of $B = 64$ samples $(\epsilon^{(b)}, i^{(b)})$, where $\epsilon^{(b)} \sim \mathcal{N}(0, 1)$ and $i^{(b)} \sim \mathcal{U}[0, T)$ are sampled independently.

\textbf{Data Attribution:} We leverage TRAK~\cite{park2023trak} to efficiently compute action influences as defined in \cref{eq:cupid-actinf-approx-square}. First, TRAK uses random projections $\mathbf{P} \sim \calN(0, 1)^{p \times d}$, where $p$ is the number of model parameters and $d << p$ is the specified projection dimension, to reduce the dimensionality of the gradients as $g_\theta = \mathbf{P}^\top\nabla_\theta \ell_{\mathrm{square}}$ while preserving their inner products $g_\theta \cdot g_\theta \approx \nabla_\theta \ell_{\mathrm{square}} \cdot \nabla_\theta \ell_{\mathrm{square}}$~\cite{johnson1984extensions}. Second, TRAK ensembles influence scores over $C$ independently trained models (i.e., from different seeds) to account for non-determinism in learning. In our experiments, we use the standard projection dimension $d = 4000$ and minimize computational cost by using only a single policy checkpoint $C = 1$, noting that ensembling over $C > 1$ policy checkpoints is likely to improve the accuracy of our influence scores.

\subsection{Combining Score Functions}\label{appx:cupid-comb-score-fns}
For ease of exposition in \cref{sec:cupid-methods-quality}, we express the overall score of a demonstration as the convex combination of its performance influence and its quality score $\alpha \polinf + (1-\alpha)\Psi_{\mathrm{qual}}$, where $\alpha = 1$ and $\alpha \in [0, 1)$ instantiates \basemethod{} and \qualitymethod{}, respectively. Here, we additionally note that taking weighted combinations of score functions requires first normalizing them to equivalent scales. 
Hence, our implementation uniformly normalizes demonstration scores within the range $[0, 1]$ (i.e., producing an absolute ranking of demonstrations) for each score function $\polinf$ and $\Psi_{\mathrm{qual}}$ before combining them.
This simple approach can be applied to combine an arbitrary number of demonstration score functions.

\section{Experimental Setup}\label{appx:cupid-experiments}

\subsection{Hardware Setup}
As depicted in \cref{fig:cupid-franka-dp-results}, our hardware experiments involve a Franka FR3 manipulator robot. 
We use a single ZED 2 camera to capture RGB-D observations and disregard the depth information. 
Our image-based policies process $256\times256$ downsampled RGB observations and predict sequences of end-effector poses for the manipulator, which are tracked using operational space control~\cite{khatib2003unified}.

\subsection{Policy Architectures}

\textbf{Diffusion Policy (DP):} We use the original diffusion policy implementation\footnote{DP's open-source implementation: \url{https://github.com/real-stanford/diffusion_policy}.} from \citet{chi2023diffusion}. Specifically, we use the convolutional-based diffusion policy architecture for efficiency. For state-based tasks (e.g., in RoboMimic; \cref{fig:cupid-robomimic-dp-results}), actions are generated solely using the noise prediction network $\epsilon_\theta$ as described in \cref{appx:cupid-actinf-dp}. However, for image-based tasks (e.g., on hardware; \cref{fig:cupid-franka-dp-results}), the policy $\pi_\theta$ contains two sets of parameters $\theta = (\theta_o, \theta_a)$ corresponding to a ResNet-18 encoder $E_{\theta_o}$ and the noise prediction network $\epsilon_{\theta_a}$. When scoring demonstrations, we compute action influences (\cref{eq:cupid-actinf-approx-square}) over all available policy parameters $\theta$, noting that one might also consider using a subset of the parameters, e.g., those of the noise prediction network or an alternative action head, under reduced computational budgets. 

\textit{Other optimizations:} In preliminary experiments, we found that the original diffusion policy (a) was heavily over-parameterized and (b) converged in performance much earlier in training than the specified maximum number of epochs. Thus, to accelerate experimentation in RoboMimic (\cref{fig:cupid-robomimic-dp-results}), we (a) manually determined the smallest model size that performed similarly to the original policy and (b) adjusted the maximum number of epochs to the point where additional training would result in no further performance gains. Importantly, we keep the model size and training epochs consistent across all curation methods for a given RoboMimic task. For real-world hardware experiments, we use the same model size and limit the number of training steps to 200K across all tasks, similar to \citet{hejna2025robotdatacurationmutual}. All other  diffusion policy hyperparameters are consistent with the original implementation~\cite{chi2023diffusion}. 

\begin{wraptable}{r}{0.34\linewidth}
    \small
    \caption[Hyperparameters for PI-0 multi-task policy fine-tuning (CUPID)]{\small
        \textbf{Hyperparameter configuration} used for $\pi_0$~\cite{black2410pi0} post-training.
    }
    \label{tab:cupid-pi0-params}
    \adjustbox{max width=\linewidth}{
    \begin{tabular}{l|l}
        \toprule
        \textbf{Hyperparameter}        & \textbf{Value}                    \\
        \midrule
        Training steps                 & 30{,}000                          \\
        Batch size                     & 16                                \\
        Optimizer                      & AdamW                             \\
        Learning rate schedule         & Cosine decay                      \\
        EMA                            & Disabled                          \\
        Action chunk length            & 50 steps                          \\
        Control frequency              & 10 Hz                             \\
        Image resolution               & $224 \times 224$                  \\
        Observation history            & 1 frame                           \\
        \midrule
        VLM backbone LoRA              & Rank = 16, $\alpha = 16$          \\
        Action expert LoRA             & Rank = 32, $\alpha = 32$          \\
        \bottomrule
    \end{tabular}}
\end{wraptable}

\textbf{Generalist Robot Policy ($\pi_0$):} 
We fine-tune \texttt{Physical Intelligence}'s $\pi_0$ Vision-Language-Action (VLA) policy\footnote{$\pi_0$'s open-source implementation: \url{https://github.com/Physical-Intelligence/openpi}.} via Low-Rank Adaptation (LoRA)~\cite{hu2022lora} on the ``Figure-8'' and ``TuckBox'' tasks. To ensure the post-trained policy's performance is solely a result of the properties of the curated dataset used for training, we use the standard fine-tuning parameter configuration from \citet{black2410pi0} and keep all hyperparameters fixed across experiments (see \cref{tab:cupid-pi0-params}). We trained on 2 NVIDIA RTX 4090 GPUs, which took approximately 15 hours under the configuration in \cref{tab:cupid-pi0-params}. In initial experiments, we found that training for 30K steps was necessary to compensate for mismatch between our robot's action space (target end-effector poses tracked via operational space control) and the action spaces used to pre-train the base $\pi_0$ policy (absolute joint angles). In addition, we found that using a descriptive prompt for the task was necessary to yield performant policies.
We kept these prompts fixed across training, evaluation, and all curation settings. For the ``TuckBox'' task, we used the instruction ``Move the blue box underneath the white shelf'' to avoid biasing the policy towards a particular behavior mode (e.g., ``sliding'' or ``pick-and-place''). For the ``Figure-8'' task, we used the instruction ``Pick up the red rope, then tie a figure 8,'' where we found the two-step instruction to increase performance over shorter instructions like ``Tie the cleat.'' Similar to the diffusion policy experiment, we fine-tune a separate $\pi_0$ model for each curation task---filter-$k$ (\cref{task:filter-k}) and select-$k$ (\cref{task:select-k})---using their corresponding base demonstration datasets.
We then fine-tune additional $\pi_0$ models on datasets curated by our methods.

\subsection{Tasks and Datasets}\label{appx:cupid-tasks}
Here, we provide additional details regarding our real-world hardware tasks and their corresponding datasets. We refer to \citet{pmlr-v164-mandlekar22a} for details on the simulated RoboMimic benchmark. 

\textbf{Figure-8:}
A brief description of the task is provided in \cref{sec:cupid-exp-quality}. The ``Figure-8'' dataset contains 160 demonstrations evenly split across four \textit{quality tiers}. Higher quality demonstrations complete the task at a constant rate without errors, while lower-quality demonstrations vary in progression rate~\cite{agia2024unpacking} and include retry or recovery behaviors. Therefore, the ``Figure-8'' task intends to reflect a practical setting where demonstrations of varying properties are introduced during data collection, whether organically or deliberately, e.g., to improve policy robustness to recoverable failures~\cite{dai2024racer}. Therefore, we expect curation algorithms that distinguish demonstrations upon notions of quality (e.g., predictability~\cite{hejna2025robotdatacurationmutual}) to perform well on this task, which is consistent with our findings in \cref{fig:cupid-franka-dp-results}(a) and \cref{fig:cupid-franka-pi0-transfer-results}(a).

\textbf{TuckBox:}
A brief description of the task is provided in \cref{sec:cupid-exp-strategies}. As mentioned, the ``TuckBox'' dataset contains 120 demonstrations split 2:1 between two subsets: 80 demonstrations solve the task by sliding the box under the receptacle, while 40 demonstrations first reposition the box in front of the receptacle via pick-and-place. Although the sliding strategy appears more smooth and involves just a single step, it is rendered unreliable by imperceptible test-time distribution shifts to the box's mass distribution. As such, ``TuckBox'' stands conceptually opposite to ``Figure-8,'' whereby attending to heuristic properties of demonstrations (e.g., quality) may result in poor curation performance (as shown in \cref{fig:cupid-franka-dp-results}(b)).

\textbf{Bookshelf:}
A brief description of the task is provided in \cref{sec:cupid-exp-correlations}. To summarize, the robot must extract a target book that is either shelved alone---affording a simple, horizontal pulling motion---or with another book stacked on top of it (i.e., a \textit{bookstack}). In the bookstack case, the robot must extract the target book using a vertical pulling motion, such that the stacked book does not fall off the shelf in the process (see \cref{fig:cupid-franka-dp-results}(c)). In total, the ``Bookshelf'' dataset contains 120 demonstrations split across three subsets: (a) 60 demonstrations feature the target book shelved alone with a white background, (b) 20 demonstrations feature the bookstack with a white background, and (c) 40 demonstrations feature the bookstack with a dark background. All subsets feature task-irrelevant distractor books on other shelves.

\textit{Spurious correlations in training data:} Although the vertical pulling solution to the bookstack case is demonstrated in scenes with both white and dark backgrounds, the disproporionate number of demonstrations in subset (a) versus subset (b) spuriously correlates the horizontal pulling motion with the white background. Such spurious correlations may result in \textit{causal confusion}~\cite{de2019causal}, where the policy ignores the bookstack, attends the white background, and executes the failing horizontal strategy.

\textit{Spurious correlations in rollout data:} Like ``TuckBox,'' ``Bookshelf'' represents another limiting case for curating data with quality metrics~\cite{hejna2025robotdatacurationmutual}. However, it also presents an additional challenge for methods that seek to curate data using online experience~\cite{chen2025curating}. For example, approaches that attend to differences in states between successful and failed policy rollouts may be susceptible to spurious correlations in the rollout data. Consider the simple case: if we were to observe successful rollouts when the target book is shelved alone and failed rollouts when another book is stacked above the target, then training a classifier (i.e., as in Demo-SCORE~\cite{chen2025curating}) to distinguish successful from failed states may wrongly attribute failures to the presence of the stacked book. Curating demonstrations with such a classifier would, in turn, worsen the spurious correlation in the training data. Thus, we posit that handling more challenging cases of spurious correlations in real-world data will require methods that \textit{causally attribute} the outcomes of observed test-time experiences to the training data, such as \basemethod{}.

\subsection{Baseline Details}\label{appx:cupid-baselines}

\textbf{DemInf:} We use the official implementation\footnote{DemInf open-source implementation: \url{https://github.com/jhejna/demonstration-information}.} provided by \citet{hejna2025robotdatacurationmutual}. We note that DemInf curates data offline---that is, without using any policy rollouts---and is at present only applicable to the demonstration filtering setting (i.e., filter-$k$, as defined in \cref{task:filter-k}). 

\textbf{Demo-SCORE:} We construct our own implementation based on the description provided by the authors~\cite{chen2025curating}. Given our assumed fixed budget of $m = 100$ rollouts for RoboMimic experiments (\cref{sec:cupid-experiments}), we collect 25 rollouts from $C = 4$ policy checkpoints throughout training. We train three-layer MLP classifiers with hidden dimensions $[16, 16, 16]$ on the first three rollout sets, and select the best classifier via cross-validation on the last 25 rollouts, as described in \cite{chen2025curating}. Since we reduce the rollout budget to $m = 25$ rollouts for hardware experiments (\cref{sec:cupid-experiments}), we collect 25 rollouts from the last $C = 1$ policy checkpoint. We then train a single ResNet-18 encoder and three-layer classification head with hidden dimensions $[32, 32, 32]$ on 20 of the rollouts, leaving 5 validation rollouts to monitor for overfitting. We train all classifiers with a heavy dropout of $0.3$ and an AdamW weight decay of $0.1$ to prevent overfitting, in alignment with \cite{chen2025curating}.
Although \citet{chen2025curating} only test Demo-SCORE for demonstration filtering, we extend its use for demonstration selection (i.e., select-$k$, as defined in \cref{task:select-k}).

\textbf{Success Similarity:}
We design a custom robot data curation algorithm  that, similar to Demo-SCORE, valuates demonstrations based on a heuristic measure of similarity \textit{w.r.t.} successful policy rollouts. Instead of training classifiers, Success Similarity measures the average state-embedding similarity of a demonstration \textit{w.r.t.} all successful rollouts as 
\begin{equation*}
    S(\xi; \calD_\tau) = -\sum_{\tau \in \calD_\tau} \bigg[\mathbf{1}(R(\tau) = 1) \cdot \frac{1}{H^2} \sum_{s'\in\tau}\sum_{s\in\xi} D\big(\phi(s'), \phi(s)\big) \bigg],
\end{equation*}
where the indicator function $\mathbf{1}$ evaluates to 1 if rollout $\tau$ is successful and 0 otherwise, $H$ is the assumed length of all demonstrations $\xi \in \calD$ and rollouts $\tau \in \calD_\tau$ for notational simplicity, $\phi$ is the state embedding function, and $D$ is a specified distance function over state embeddings~\cite{sinha2024real}, such as the Mahalanobis, L2, or cosine distance. 
For image-based states, we experimented with various embedding functions $\phi$, including ResNet~\cite{he2016deep}, DINOv2~\cite{oquab2023dinov2}, and the policy's vision encoder~\cite{agia2024unpacking}, and ultimately found the policy's vision encoder to work best in RoboMimic. The embedding function is set to identity for low-dimensional states (i.e., $\phi(s) = s$). 
Lastly, the distance function $D$ is chosen for compatibility with $\phi$: e.g., L2 distance for policy encoder embeddings and cosine distance for DINOv2 embeddings. 

\textit{Comparison to Performance Influence (\basemethod{}):} One can interpret Success Similarity as replacing the action influence $\actinf((s',a'),(s,a))$ (\cref{eq:cupid-actinf}) with a state-based proxy $-D(\phi(s'), \phi(s))$ in an attempt to estimate the performance contribution of a demonstration (\cref{eq:cupid-policy-inf-deriv}). In our RoboMimic experiments (\cref{fig:cupid-robomimic-dp-results}), this approach performs comparably to Demo-SCORE and, in some cases, even outperforms it---without requiring the training of any additional models. However, Success Similarity performs consistently worse than \basemethod{} across all tasks, supporting prior findings that influence functions offer a substantially stronger causal signal than heuristic measures of similarity~\cite{park2023trak}.

\textbf{Oracle:} For each task, the Oracle method represents a best attempt to curate data assuming privileged access to ground-truth demonstration labels. For the RoboMimic and ``Figure-8'' tasks, the Oracle ranks demonstrations in descending order of quality, choosing high-quality demonstrations before low-quality demonstrations. For the ``TuckBox'' task, the Oracle first chooses all demonstrations exhibiting the more robust pick-and-place strategy before any demonstration exhibiting the more brittle sliding strategy. Lastly, for the ``Bookshelf'' task, the Oracle chooses demonstrations to minimize the effect of the \textit{known} spurious correlation (i.e., horizontal pulling motion in the presence of a white background), resulting in a more balanced curated dataset. These definitions of the Oracle apply identically to the filter-$k$ (\cref{task:filter-k}) and select-$k$ (\cref{task:select-k}) curation tasks studied throughout this chapter.

\textbf{Additional baselines:} We implement a number of additional custom baselines that one might try in practice, such as curating data based on policy loss, policy uncertainty, state diversity, and action diversity. However, we exclude them from our experiments given their relatively poor performance.

\section{Additional Results and Analysis}\label{appx:cupid-results}

We present additional results and ablations for our RoboMimic and Franka real-world tasks that were cut from the main text due to space constraints. 

\subsection{Extended Discussion on RoboMimic Results (\cref{sec:cupid-discussion-properties})}\label{appx:cupid-results-robomimic-discussion}

We provide an extended discussion on \cref{sec:cupid-discussion-properties} for our RoboMimic simulation results. 

\textit{Performance versus Data Quality:} One of our key findings is that the performance of a state-of-the-art policy does not strictly correlate with the \textit{perceived quality} of its training data. Factors such as redundancy, balance, and coverage of the dataset all play a role in determining policy performance. This is illustrated in the Oracle filter-$k$ results (left three plots of \cref{fig:cupid-robomimic-dp-results}). While the top row shows a monotonic increase in average dataset quality as lower-quality demonstrations are filtered out, the bottom row reveals (1) a consistent performance drop for diffusion policies on 2 out of 3 tasks, and (2) as expected, performance degradation when too many demonstrations are removed. Similar analysis applies to the select-$k$ setting. These results highlight two important points: First, the impact of dataset curation should not be judged by quality labels alone, but by the downstream performance of models trained on curated datasets. Second, determining how much data to curate (i.e., the $k$ in filter-$k$ and select-$k$) remains another key challenge for effective data curation in practice.

\textit{Performance versus Task Complexity:} We evaluate curation performance across three RoboMimic tasks of increasing complexity---``Lift MH,'' ``Square MH,'' and ``Transport MH.'' On the simplest task, “Lift MH,” diffusion policies achieve 100\% success despite training on all demonstrations, indicating that low-quality demonstrations have minimal impact and can be safely filtered. We observe a similar trend for the moderately difficult ``Square MH'' task, where the policy benefits from access to all demonstrations regardless of their quality. However, performance degrades more quickly as demonstrations are filtered, suggesting increased sensitivity to data quantity due to the task’s higher complexity relative to ``Lift MH.'' Finally, on the challenging ``Transport MH'' task, which requires precise bi-manual coordination, both \basemethod{} and \qualitymethod{} significantly outperform the base policy. These results suggest that curation of mixed-quality datasets is most beneficial for complex, precision-critical tasks, where training on lower-quality data is more likely to hinder performance.

\subsection{Ablation on Number of Policy Rollouts in RoboMimic (\cref{sec:cupid-discussion-rollouts})}

We conduct an ablation study in RoboMimic evaluating the quality of datasets curated by \basemethod{} and \qualitymethod{} under varying numbers of rollouts, $m \in \{1, 5, 10, 25, 50, 100\}$. The results for state-based and image-based diffusion policies are shown in \cref{fig:cupid-robomimic-state-data-quality-rollout-appx} and \cref{fig:cupid-robomimic-image-data-quality-rollout-appx}, respectively. For ``Lift MH'' and ``Square MH,'' performance influences (\cref{eq:cupid-policy-inf-deriv}) stabilize around $m \in [25, 50]$, yielding quality trends similar to those obtained with $m = 100$. For ``Transport MH,'' quality trends continue to evolve until approximately $m \in [50, 100]$ rollouts, indicating that more rollouts are beneficial for accurate influence estimation in complex task settings---where curation has the greatest effect on performance.

\begin{figure}[H]
    \centering
    \includegraphics[width=\linewidth]{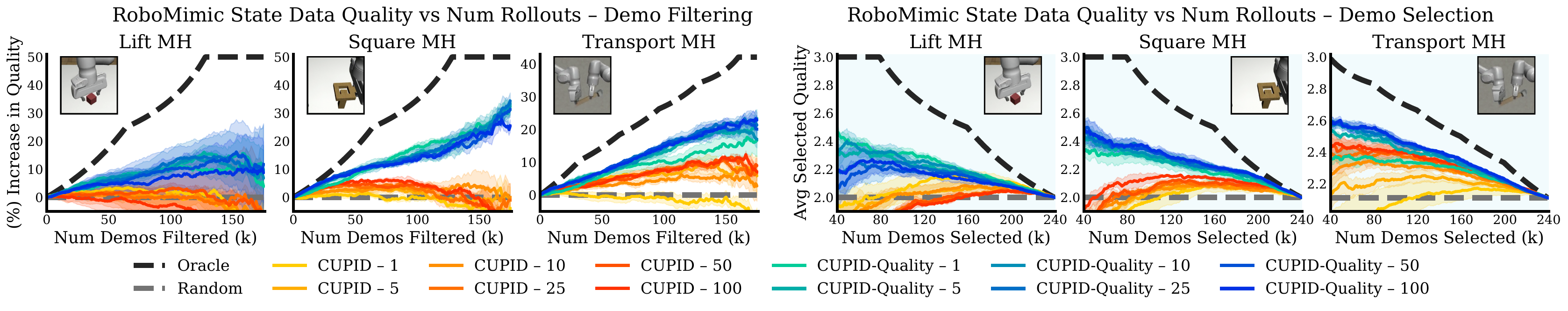}
    \caption[Ablation result on number of policy rollouts in RoboMimic state (CUPID)]{\small
        RoboMimic state ablation: Data quality trends under varying number of rollouts. Performance influences (\cref{eq:cupid-policy-inf-deriv}) converge around $m \in [25, 50]$ rollouts for ``Lift MH'' and ``Square MH'' (yielding similar quality trends), but continue to evolve until $m \in [50, 100]$ rollouts for ``Transport MH.'' 
        Curation performed on state-based diffusion policies.
        Results are averaged over 3 random seeds.
        Errors bars represent the standard error.
    }
    \label{fig:cupid-robomimic-state-data-quality-rollout-appx}
\end{figure}

\begin{figure}[H]
    \centering
    \includegraphics[width=\linewidth]{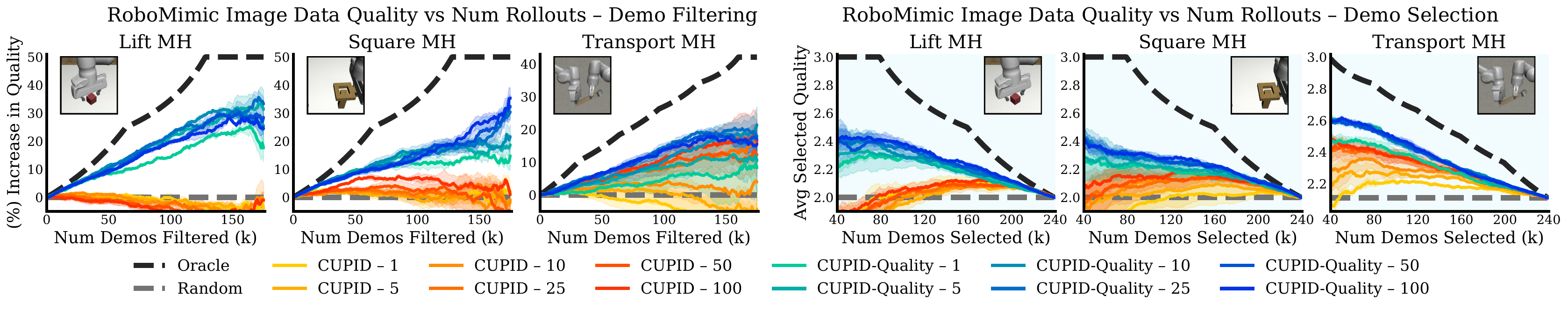}
    \caption[Ablation result on number of policy rollouts in RoboMimic image (CUPID)]{\small
        RoboMimic image ablation: Data quality trends under varying number of rollouts. Performance influences (\cref{eq:cupid-policy-inf-deriv}) converge around $m \in [25, 50]$ rollouts for ``Lift MH'' and ``Square MH'' (yielding similar quality trends), but continue to evolve until $m \in [50, 100]$ rollouts for ``Transport MH.''
        Curation performed on image-based diffusion policies.
        Results are averaged over 3 random seeds. 
        Errors bars represent the standard error.
    }
    \label{fig:cupid-robomimic-image-data-quality-rollout-appx}
\end{figure}

\subsection{Additional Data Quality Results in RoboMimic}

We provide full data quality results in RoboMimic. \cref{fig:cupid-robomimic-state-data-quality-appx} is identical to the top row of \cref{fig:cupid-robomimic-dp-results} in the main text, but also includes data quality trends for select-$k$ curation on ``Lift MH.'' \cref{fig:cupid-robomimic-image-data-quality-appx} shows data quality results for image-based diffusion policies. We do not retrain image-based policies on curated datasets (as in the bottom row of \cref{fig:cupid-robomimic-dp-results}) due to the substantial computational resources required.

\begin{figure}[H]
    \centering
    \includegraphics[width=\linewidth]{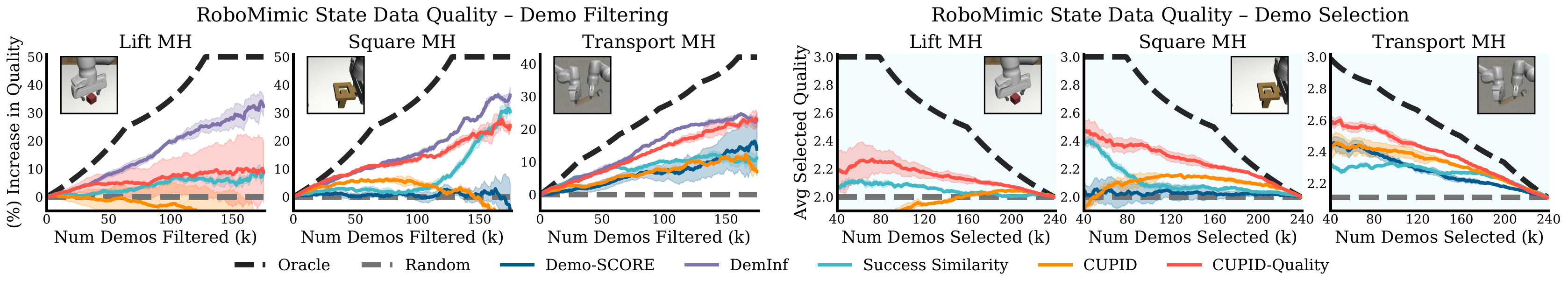}
    \caption[Extended results on curated data quality in RoboMimic state (CUPID)]{\small
        RoboMimic state data quality results. Curation performed on state-based diffusion policies. Results are averaged over 3 random seeds. Errors bars represent the standard error.
    }
    \label{fig:cupid-robomimic-state-data-quality-appx}
\end{figure}

\begin{figure}[H]
    \centering
    \includegraphics[width=\linewidth]{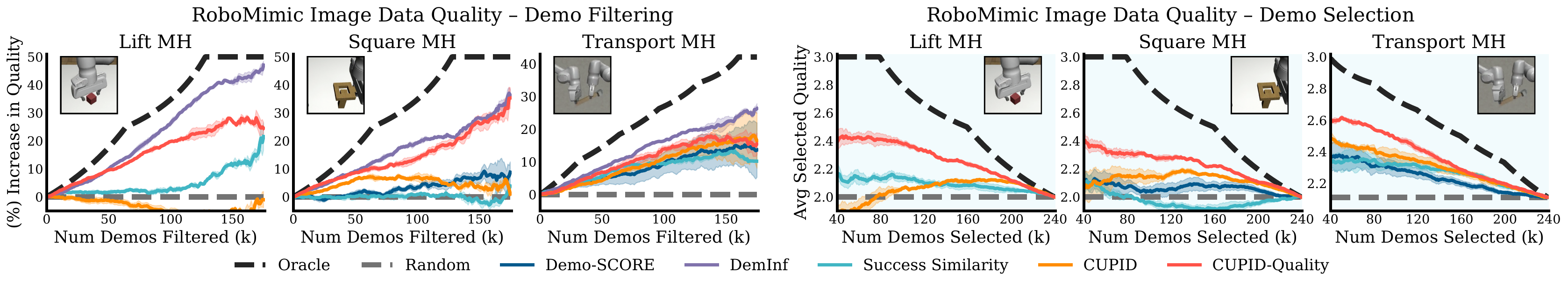}
    \caption[Extended results on curated data quality in RoboMimic image (CUPID)]{\small
        RoboMimic image data quality results. Curation performed on image-based diffusion policies. Results are averaged over 3 random seeds. Errors bars represent the standard error.
    }
    \label{fig:cupid-robomimic-image-data-quality-appx}
\end{figure}

\subsection{Data Filtering Curation Distributions in Franka Real-World}\label{appx:cupid-filter-distr}

\begin{figure}[H]
    \centering
    
    \begin{subfigure}[b]{\linewidth}
        \centering
        \includegraphics[width=0.95\linewidth]{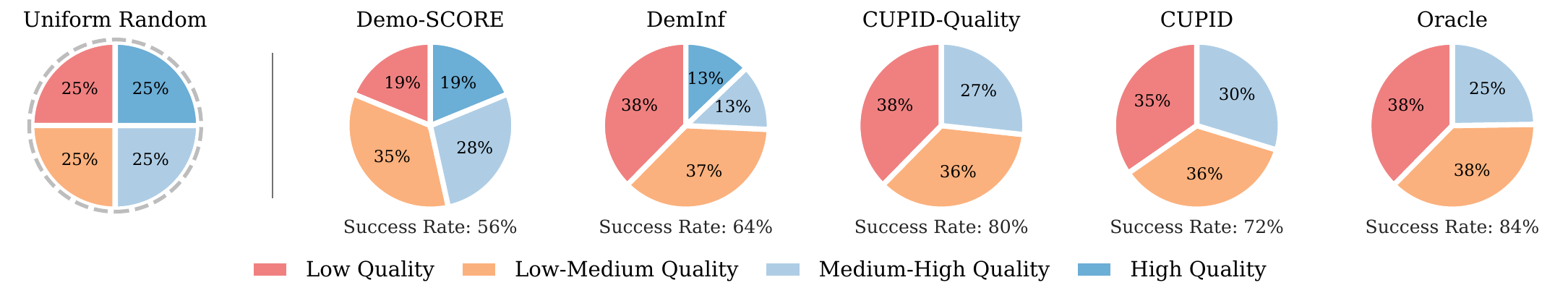}
        \vspace{-6pt}
        \caption{\footnotesize 
            \textbf{Figure-8:} Distribution of demonstrations \textit{filtered}. Filtering lower-quality demos is better.
        }
        \label{fig:cupid-franka-dp-distr-filter-curated-results-figure8}
    \end{subfigure}

    \vspace{18pt}

    \begin{subfigure}[b]{\linewidth}
        \centering
        \includegraphics[width=0.95\linewidth]{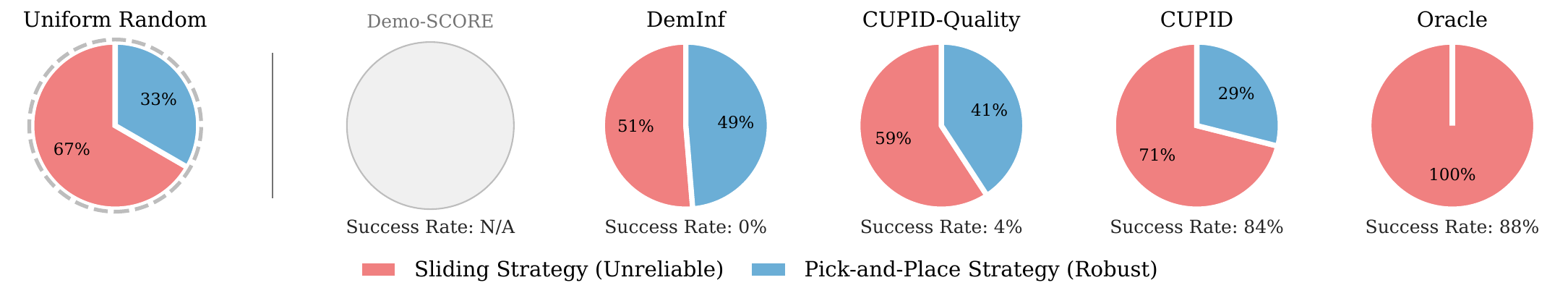}
        \vspace{-6pt}
        \caption{\footnotesize
            \textbf{TuckBox:} Distribution of demonstrations \textit{filtered}. Filtering sliding demos is better.
        }
        \label{fig:cupid-franka-dp-distr-filter-curated-results-tuckbox}
    \end{subfigure}
    
    \vspace{18pt}
    
    \begin{subfigure}[b]{\linewidth}
        \centering
        \includegraphics[width=0.95\linewidth]{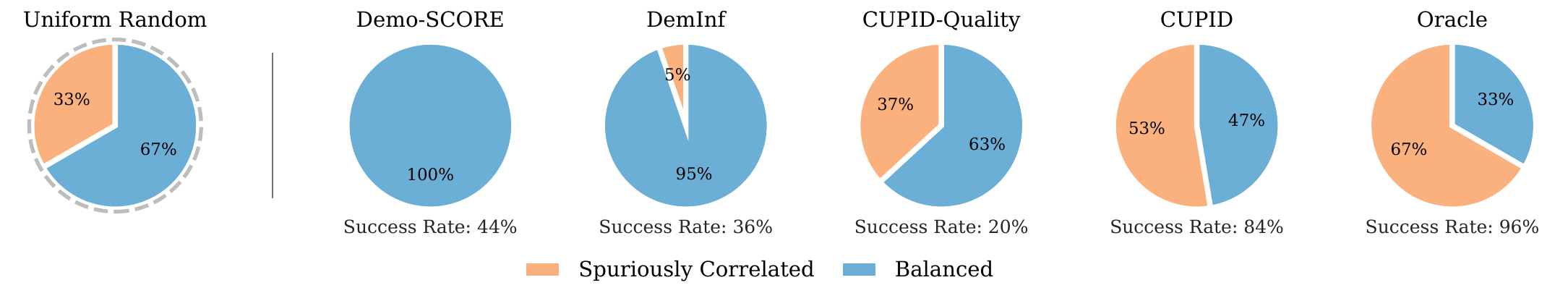}
        \vspace{-6pt}
        \caption{\footnotesize
            \textbf{Bookshelf:} Distribution of demonstrations \textit{filtered}. Filtering spurious correlations is better.
        }
        \label{fig:cupid-franka-dp-distr-filter-curated-results-bookshelf}
    \end{subfigure}
    
    \vspace{8pt}   
    
    \caption[Distributions of demos filtered on real-world tasks (CUPID)]{\small
        \textbf{Franka diffusion policy -- distribution of demonstrations filtered ($S^\star$ in \cref{task:filter-k}).} See \cref{fig:cupid-franka-dp-distr-filter-dataset-results} for distributions of the corresponding curated datasets used for policy training.\\
    }
    \label{fig:cupid-franka-dp-distr-filter-curated-results}
\end{figure}

\subsection{Data Selection Curation Distributions in Franka Real-World}\label{appx:cupid-select-distr}

\begin{figure}[H]
    \centering
    
    \begin{subfigure}[b]{\linewidth}
        \centering
        \includegraphics[width=0.90\linewidth]{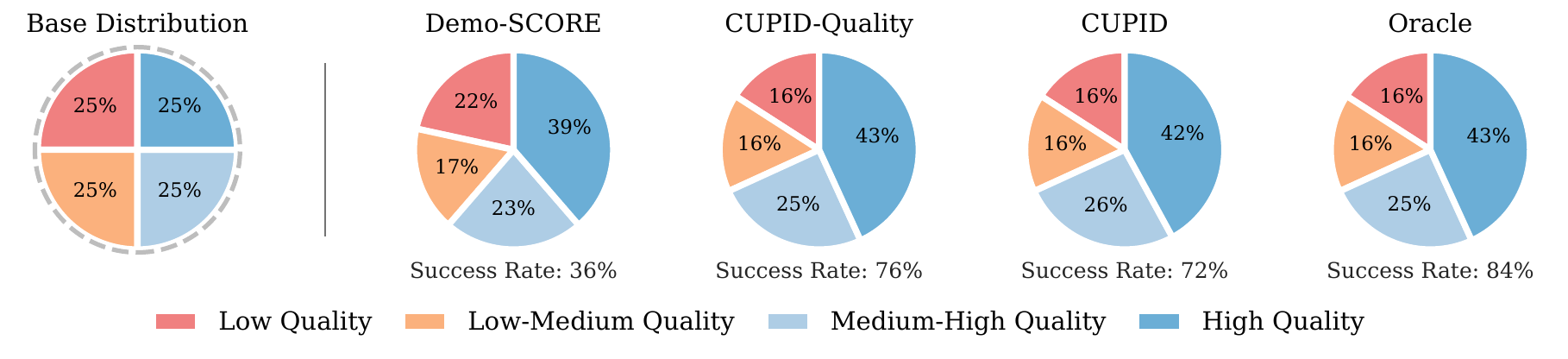}
        \vspace{-6pt}
        \caption{\footnotesize 
            \textbf{Figure-8:} Distribution of curated demonstrations after \textit{selecting} 33\%. Higher-quality demos are better.
        }
        \label{fig:cupid-franka-dp-distr-select-dataset-results-figure8}
    \end{subfigure}

    \vspace{18pt}

    \begin{subfigure}[b]{\linewidth}
        \centering
        \includegraphics[width=0.90\linewidth]{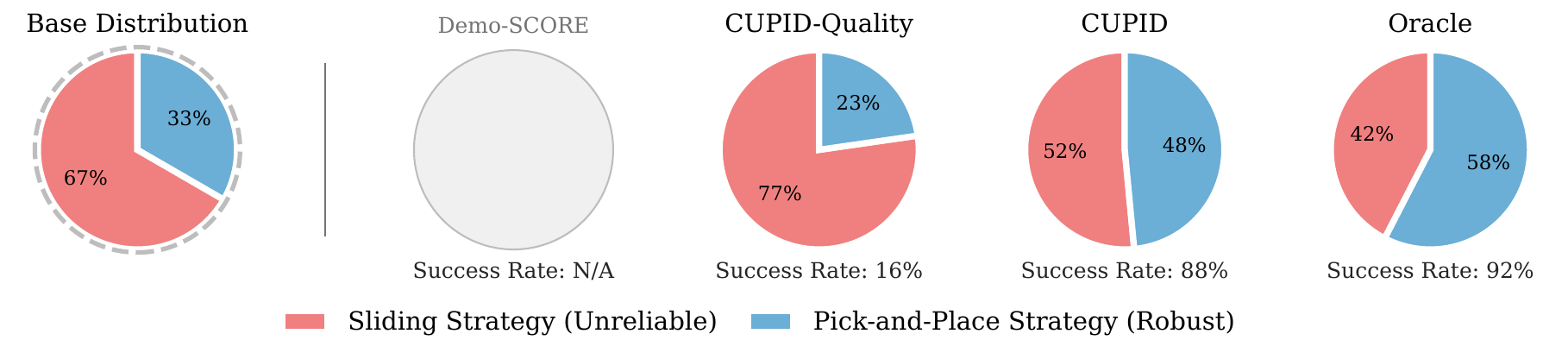}
        \vspace{-6pt}
        \caption{\footnotesize            
            \textbf{TuckBox:} Distribution of curated demonstrations after \textit{selecting} 33\%. Pick-and-place demos are better.
        }
        \label{fig:cupid-franka-dp-distr-select-dataset-results-tuckbox}
    \end{subfigure}
    
    \vspace{8pt}   

    \caption[Curated dataset distributions for demo selection on real-world tasks (CUPID)]{\small
        \textbf{Franka diffusion policy curated dataset distributions for selection (\cref{task:select-k}).} \basemethod{} selects higher-quality demonstrations (Figure-8) and robust strategies (TuckBox), improving policy performance across tasks (see \cref{fig:cupid-franka-dp-results}). While curation heuristics employed by baselines may be effective in some cases (e.g., \qualitymethod{} in Figure-8), they can lead to suboptimal selection in others. \\
    }
    \label{fig:cupid-franka-dp-distr-select-dataset-results}
\end{figure}

\begin{figure}[H]
    \centering
    
    \begin{subfigure}[b]{\linewidth}
        \centering
        \includegraphics[width=0.90\linewidth]{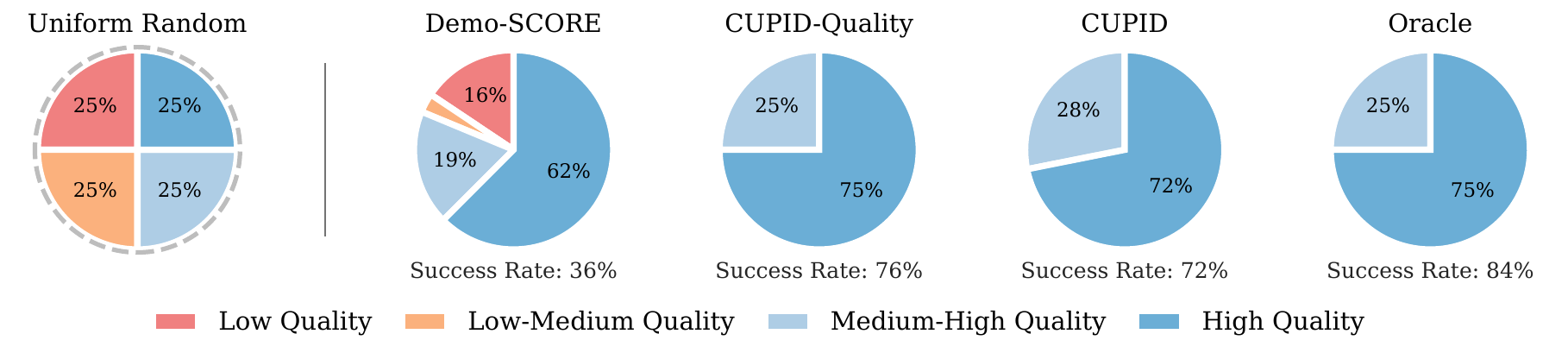}
        \vspace{-6pt}
        \caption{\footnotesize 
            \textbf{Figure-8:} Distribution of demonstrations \textit{selected}. Selecting higher-quality demos is better.
        }
        \label{fig:cupid-franka-dp-distr-select-curated-results-figure8}
    \end{subfigure}

    \vspace{18pt}

    \begin{subfigure}[b]{\linewidth}
        \centering
        \includegraphics[width=0.90\linewidth]{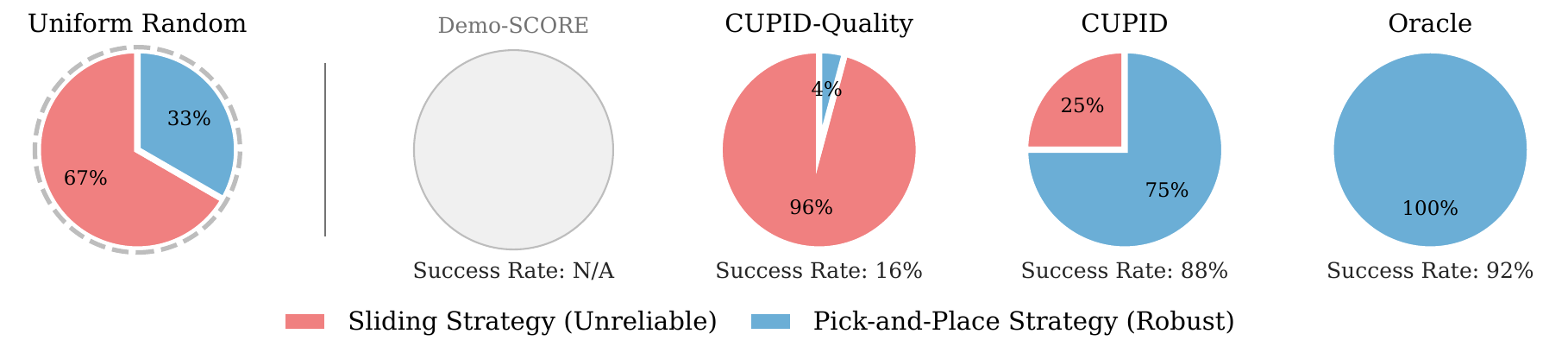}
        \vspace{-6pt}
        \caption{\footnotesize
            \textbf{TuckBox:} Distribution of demonstrations \textit{selected}. Selecting pick-and-place demos is better.
        }
        \label{fig:cupid-franka-dp-distr-select-curated-results-tuckbox}
    \end{subfigure}
    
    \vspace{8pt}   
    
    \caption[Distributions of demos selected for real-world tasks (CUPID)]{\small
        \textbf{Franka diffusion policy -- distribution of demonstrations selected ($S^\star$ in \cref{task:select-k}).} See \cref{fig:cupid-franka-dp-distr-select-dataset-results} for distributions of the corresponding curated datasets used for policy training.\\
    }
    \label{fig:cupid-franka-dp-distr-select-curated-results}
\end{figure}

\subsection{Additional Results for Franka PI-0: Curated Dataset Transfer (\cref{sec:cupid-discussion-pi0-transfer})}\label{appx:cupid-results-pio-transfer}

\cref{fig:cupid-franka-pi0-results-appx} contains the full results of our $\pi_0$ ablation (\cref{fig:cupid-franka-pi0-transfer-results}), including the performance of $\pi_0$~\cite{black2410pi0} trained on datasets curated by \basemethod{} and \qualitymethod{} for both the ``Figure-8'' and ``TuckBox'' tasks. 

\begin{figure}[H]
    \centering
    \includegraphics[width=0.70\linewidth]{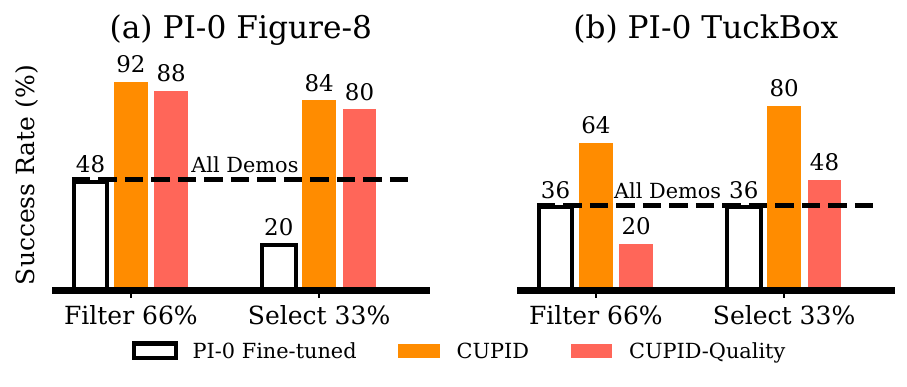}
    \caption[Extended results on real-world PI-0 multi-task policy performance (CUPID)]{\small 
        Data curated for single-task diffusion policies improves $\pi_0$~\cite{black2410pi0} post-training performance. As in \cref{fig:cupid-franka-dp-results}, quality measures (\qualitymethod{}) may degrade performance when higher-quality demonstrations induce brittle strategies at test time (TuckBox), whereas curating based on performance (\basemethod{}) is robust across settings.
    }
    \label{fig:cupid-franka-pi0-results-appx}
\end{figure}

In this experiment, we investigate two questions: (1) Can datasets curated with one policy architecture result in increased performance when used to train another policy with a different architecture? (2) How influential is curation for policies that have been pre-trained on large-scale multi-task datasets?

\textit{Curation Transfer:} Towards the first question, \cref{fig:cupid-franka-pi0-results-appx} shows that datasets curated using diffusion policies significantly increase the performance of fine-tuned $\pi_0$ policies relative to fine-tuning on the base, uncurated datasets. We attribute these results to two causes: First, we find that both the diffusion policy and $\pi_0$ have sufficient capacity to accurately fit the training data distribution, and thus, they should learn a similar behavior distribution from the training data. This implies that the observed performance gains in \cref{fig:cupid-franka-pi0-results-appx} result from curation transfer between policies. Second, as the ``TuckBox'' experiment shows in \cref{fig:cupid-franka-dp-results}(b), our method is able to effectively identify behaviors in the demonstration data that are not robust. While on-policy evaluations (i.e., rollouts) are necessary to identify such brittle behaviors, these are purely properties of the training demonstration data. Therefore, filtering out poor behaviors will increase the performance of any policy.
Similarly, on the high-precision ``Figure-8'' task, filtering out more noisy, low-quality demonstrations is likely to improve performance for any policy. 

\textit{VLA Robustness:} Towards the second question, we find that even when the base policy is pre-trained on a large, diverse, multi-task dataset, curation is still essential to yield strong fine-tuned performance. As shown in \cref{fig:cupid-franka-pi0-results-appx}, $\pi_0$ policies trained on the base demonstration datasets are unable to reliably complete our tasks. In contrast, policies trained on curated datasets attain significantly higher success rates. As such, our results indicate that simply training VLM-based policies on more data and more tasks does not strictly result in pre-conditioned policies that use their generalist knowledge to ``ignore'' low-quality behaviors or brittle strategies in demonstration data---i.e., data curation still appears essential. 

\textit{Concluding Remarks:} Overall, these results indicate that using smaller, single-task policies to curate individual datasets, which may then benefit a larger, multi-task policy is a promising direction to alleviate the computational cost of applying our method to generalist policies. Still, we emphasize that datasets curated using our method are not completely \emph{model agnostic}, as the same demonstrations may influence different models in different ways. As such, while $\pi_0$ achieves a higher base performance than the diffusion policy, the $\pi_0$ policies trained on curated datasets perform similarly to or slightly worse than the diffusion policies (for which those datasets were curated).

\section{Derivations}\label{appx:cupid-derivations}

\subsection{Proof of \cref{prop:cupid-polinf}}\label{appx:cupid-proof}
\begin{proof}
    As presented in \cref{sec:cupid-background},  applying the basic derivation of the influence function\footnoteref{fn:cupid-track} in \cite{koh2017understanding} gives us that
    \begin{align*}
     \polinf(\xi) &:= \frac{dJ(\pi_\theta)}{d\epsilon}\bigg|_{\epsilon=0} \\
                  &= -\nabla_\theta J(\pi_\theta)^\top \nabla^2_\theta\calL_{\mathrm{bc}}(\theta ; \calD)^{-1} \nabla_\theta \ell_{\mathrm{traj}}(\xi; \pi_\theta).
     \end{align*}
     Next, note that the standard log-derivative trick underlying policy gradient methods \cite{sutton1999policy, williams1992simple} tells us that 
     \begin{align*}
         \nabla_\theta J(\pi_\theta) = \E_{\tau \sim p(\tau | \pi_\theta)} \big[R(\tau)\sum_{(s',a')\in \tau}\nabla_\theta\log \pi_\theta(a'|s')\big].
     \end{align*}
     Therefore, since $\calL_{\mathrm{bc}}$ and $\ell_{\mathrm{traj}}$ are deterministic functions of $\theta$, $\xi$, and $\calD$, it holds that 
     \begin{align*}
         \polinf(\xi) = \E_{\tau\sim p(\tau | \pi_\theta)} \big[ R(\tau) \sum_{(s',a')\in\tau} -\nabla_\theta\log \pi_\theta(a'|s')^\top H_{\mathrm{bc}}^{-1}\nabla_\theta\ell_{\mathrm{traj}}(\xi; \pi_\theta)\big]
     \end{align*}
     by linearity of expectation.
     Finally, by simply noting that $\ell_{\mathrm{traj}}(\xi; \pi_\theta) = \frac{1}{H}\sum_{(s,a)\in\xi}\ell(s, a;\theta)$ and applying the definition of $\actinf$, we have the result:
     \begin{align*}
     \polinf(\xi) = \E_{\tau\sim p(\tau|\pi_\theta)}\bigg[\frac{R(\tau)}{H} \sum_{(s',a')\in\tau}\sum_{(s,a)\in\xi}\actinf\big((s',a'),(s,a)\big)\bigg].
     \end{align*}
\end{proof}

\subsection{Derivation of Performance Influence for Variable Length Trajectories}\label{appx:cupid-proof-length}
In \cref{sec:cupid-formulation} and \cref{sec:cupid-method}, we assumed that all trajectories in the demonstration dataset $\calD$ were of an equal length $H$ for notational simplicity. Here, we show that without loss of generality, our analysis extends to the case where the length of demonstration trajectories vary. Suppose each demonstration $\xi^i \in \calD$ has length $H^i$, so that the base policy $\pi_\theta$ minimizes the average loss across all samples in the demonstration data, i.e., 
\begin{equation}\label{eq:cupid-bclossvary}
    \theta = \arg\min_{\theta'} \{\tilde{\mathcal{L}}_{\text{bc}}(\theta'; \mathcal{D}) := \frac{1}{(\sum_{i=1}^nH^i
    )
    }\sum_{\xi^i\in\calD}\sum_{(s, a) \in \xi^i} \ell(s, a; \pi_{\theta'})\}.
\end{equation}
Note that the objective in \cref{eq:cupid-bclossvary} is equivalent to an unweighted BC loss 
\begin{equation*}
    {\mathcal{L}'}_{\text{bc}}(\theta'; \mathcal{D}) := \sum_{\xi^i\in\calD}\sum_{(s, a) \in \xi^i} \ell(s, a; \pi_{\theta'}),
\end{equation*}
which decomposes into its unweighted trajectory losses $\ell_{\mathrm{traj}}'(\xi ; \pi_{\theta'}) := \sum_{(s,a) \in \xi}\ell(s, a; \pi_{\theta'})$, so that $\calL_{\mathrm{bc}}'(\theta', \calD) = \sum_{\xi^i \in \calD}\ell_{\mathrm{traj}}'(\xi^i ; \pi_{\theta'}).$ We can then derive an equivalent statement to \cref{prop:cupid-polinf} for the unweighted loss functions that applies when the demonstrations have variable length.

\begin{proposition}\label{prop:cupid-polinf-length}
Assume that $\theta(\calD) = \arg\min_{\theta'} \calL_{\mathrm{bc}}'(\theta'; \calD)$, that $\calL_{\mathrm{bc}}'$ is twice differentiable in $\theta$, and that $H_{\mathrm{bc}} \succ 0$ is positive definite (i.e., $\theta(\calD)$ is not a saddle point)\footnoteref{fn:cupid-track}. Then, it holds that 
\begin{equation}\label{eq:cupid-policy-inf-deriv-vary}
    \polinf(\xi) = \E_{\tau\sim p(\tau|\pi_\theta)}\bigg[R(\tau)\sum_{(s',a')\in\tau}\sum_{(s,a)\in\xi}\actinf\big((s',a'),(s,a)\big)\bigg].
\end{equation}
\end{proposition}
\begin{proof}
    As presented in \cref{sec:cupid-background},  applying the basic derivation of the influence function\footnoteref{fn:cupid-track} in \cite{koh2017understanding} gives us that
    \begin{align*}
     \polinf(\xi) &:= \frac{dJ(\pi_\theta)}{d\epsilon}\bigg|_{\epsilon=0} \\
                  &= -\nabla_\theta J(\pi_\theta)^\top \nabla^2_\theta\calL_{\mathrm{bc}}'(\theta ; \calD)^{-1} \nabla_\theta \ell_{\mathrm{traj}}'(\xi; \pi_\theta).
     \end{align*}
     Next, note that the standard log-derivative trick underlying policy gradient methods \cite{sutton1999policy, williams1992simple} tells us that 
     \begin{align*}
         \nabla_\theta J(\pi_\theta) = \E_{\tau \sim p(\tau | \pi_\theta)} \big[R(\tau)\sum_{(s',a')\in \tau}\nabla_\theta\log \pi_\theta(a'|s')\big].
     \end{align*}
     Therefore, since $\calL_{\mathrm{bc}}'$ and $\ell_{\mathrm{traj}}'$ are deterministic functions of $\theta$, $\xi$, and $\calD$, it holds that 
     \begin{align*}
         \polinf(\xi) = \E_{\tau\sim p(\tau | \pi_\theta)} \big[ R(\tau) \sum_{(s',a')\in\tau} -\nabla_\theta\log \pi_\theta(a'|s')^\top H_{\mathrm{bc}}^{-1}\nabla_\theta\ell_{\mathrm{traj}}'(\xi; \pi_\theta)\big]
     \end{align*}
     by linearity of expectation.
     Finally, by simply noting that $\ell_{\mathrm{traj}}'(\xi; \pi_\theta) = \sum_{(s,a)\in\xi}\ell(s, a;\theta)$ and applying the definition of $\actinf$, we have the result:
     \begin{align*}
     \polinf(\xi) = \E_{\tau\sim p(\tau|\pi_\theta)}\bigg[R(\tau)\sum_{(s',a')\in\tau}\sum_{(s,a)\in\xi}\actinf\big((s',a'),(s,a)\big)\bigg].
     \end{align*}
\end{proof}

\chapter{Additional Details: STAP}\label{chapter:stap-appx}

\section*{Appendix -- \textbf{STAP}: Sequencing Task-Agnostic Policies}

The appendix discusses commonly asked questions about our planning framework and provides details on the manipulation skill library used for planning.
Qualitative results and code are made available at \url{https://sites.google.com/stanford.edu/stap}.

\startcontents[sections]
\printcontents[sections]{l}{1}{\setcounter{tocdepth}{2}}

\section{Discussion and Limitations}
\label{appx:stap-faq}

\noindent\textbf{Q1}: \textit{What enables our planning framework to generalize to unseen tasks?}

A core hypothesis of this investigation is that it is easier to learn \textit{task-agnostic skills} that support long-horizon reasoning than it is to learn a single \textit{long-horizon policy} that can generalize to arbitrary tasks (i.e. skill sequences).
This hypothesis is motivated by two facts: 1) the state-action space associated with all possible long-horizon skill sequences grows exponentially $O(c^H)$ with the length $H$ of the skill sequences considered, where $c$ is some constant; 2) adequately exploring this exponentially sized state-action space during the training of a single long-horizon policy, as required to solve arbitrary tasks, is challenging from both a methodology and engineering standpoint.
In contrast, the state-action that must be sufficiently explored to plan with STAP grows only linearly $O(K)$ with the number of task-agnostic skills $K$ in the skill library.
However, it is essential that STAP's task-agnostic skills are trained (at least in-part) on states that are likely to occur when solving tasks of interest.

Having learned a library of task-agnostic skills, our method plans with the skills at test-time to maximize the feasibility of a specified sequence of skills.
Because the skills have been trained independent of each other and of the planner, every specified sequence of skills can be regarded as an \textit{unseen task} that STAP must generalize to.
Our method accomplishes this via optimization of a planning objective (see \cref{sec:stap-taps-grounding}) in a process that involves the skills' policies, Q-functions, and dynamics models. 
Thus, the generality of our method to unseen tasks stems from the compositionality of independently learned skills.

\noindent\textbf{Q2}: \textit{Why is STAP useful for Integrated Task and Motion Planning?}

Task and Motion Planning (TAMP) seeks to solve long-horizon tasks by integrating symbolic and geometric reasoning.
Symbolic task planners and PDDL are commonly used to produce candidate plan skeletons or skill sequences that satisfy a user-specified symbolic goal.
Robotics subroutines, e.g. motion planners, collision checkers, inverse kinematics solvers, are then procedurally invoked to verify the feasibility of the plan skeleton and return a corresponding motion plan.

Different than prior work, STAP presents an avenue to develop general TAMP algorithms centered around learned skills.
While in \cref{sec:stap-planning-tamp}, we use STAP to perform geometric reasoning on candidate skill sequences proposed by a symbolic planner, many other instantiations are possible.
For example, success probabilities (\cref{eq:stap-planning-objective}) predicted by STAP can serve as a geometric feasibility heuristic to guide task planning with classical search algorithms~\cite{hoffmann2001-ff} or foundation models~\cite{lin2023text2motion}.
Such TAMP algorithms would inherit the efficiency of planning with STAP, as shown in \cref{fig:stap-exp_a}.
They would also benefit from the modularity associated with skill libraries, where new skills can be added to support a larger set of tasks and old skills can be updated to improve overall planning performance without the need to modify any other components of the TAMP framework.

\noindent\textbf{Q3}: \textit{What would it take to scale STAP to high-dimensional observation spaces?}

Our framework is not limited to the low-dimensional state space described in \cref{sec:stap-experiments}, however, several challenges must be addressed in order to use STAP in conjunction with high-dimensional sensory data such as images or 3D point clouds.
\vspace{-8pt}
\begin{itemize}
    \item \textbf{Skills:} Q-functions must accurately characterize the skill's success probability given the high-dimensional observation. 
    Challenge: the fidelity of skill Q-functions obtained via model-free Reinforcement Learning (RL)~\cite{fujimoto2018-td3,haarnoja2018-sac} may be too low for long-horizon planning. 
    Potential solution(s): acquire skills from large-scale datasets and couple controllable, data-driven learning methods (e.g. offline RL, imitation learning, supervised learning) with data augmentation techniques.
    
    \item \textbf{Dynamics:} The planning state space (\cref{eq:stap-long-horizon-domain}) must be amenable to accurate forward prediction over long-horizon skill sequences.
    Challenge: predicted high-dimensional states become coarse over long-horizons~\cite{finn2017-dvf,wu2021greedy} which complicates their use in manipulation planning settings that demand fine-grained geometric detail.
    Potential solution(s): leverage pretrained representations for robotics~\cite{nair2022r3m, karamcheti2023language} and learn latent dynamic models~\cite{hafner2019-dreamer,xu2021-daf} for forward prediction.
\end{itemize}
\vspace{-8pt}
Since STAP is reliant on the quality of the underlying skill library, we expect the capabilities of our method to improve with advancements in robot skill acquisition and visuomotor policies, representation learning, and video prediction for robotics.

\noindent\textbf{Q4}: \textit{How else can uncertainty quantification be incorporated into planning?}

In our TAMP experiments (\cref{sec:stap-planning}), we take a filtering-based approach to robustify STAP planning in out-of-distribution scenarios.
Specifically, we use Sketching Curvature for Out-of-Distribution Detection (SCOD)~\cite{SharmaAzizanEtAl2021} for uncertainty quantification (UQ) of Q-functions and disregard the $n$ plans with the most uncertain Q-values at each planning iteration.

While SCOD imposes no train-time dependencies on any algorithms in our framework, its forward pass is computationally and memory intensive and slows planning as a result.
For faster planning, UQ alternatives such as deep ensembles~\cite{ensemblereview2021} or Monte-Carlo dropout~\cite{GalZoubin2016} can be employed which, in contrast to SCOD, require modifying the algorithms used to learn skills. 
We further note that the described filtering-based optimization approach can be substituted with more sophisticated planning techniques, several of which have been implemented and verified to work with STAP. 
For example, we could formulate a distributionally robust variant of our planning objective (\cref{eq:stap-planning-objective}) to optimize a lower-confidence bound of the Q-values:
\begin{equation*}
    J(a_{1:H}; \overline{s}_1)%
        = \Estap{\overline{s}_{2:H} \sim \overline{T}_{1:H-1}}{\Pi_{h=1}^{H} \func{Q_h}{\!\func{\Gamma_h}{\overline{s}_h}, a_h} - \alpha F_\text{unc}(\overline{s}_h, a_h; Q_h, \Gamma_h)},
\end{equation*}
where $F_\text{unc}$ quantifies the uncertainty of Q-function $Q_h$ evaluated at $s_h=\Gamma(\overline{s}_h)$ and $a_h$, and $\alpha$ is a scalar hyperparameter (e.g. the z-score). 
The uncertainty function $F_\text{unc}$ is determined by the choice of UQ method.
For deep ensembles, this would correspond to the empirical variance of the ensemble predictions $F_\text{unc}(\overline{s}_h, a_h; Q_h, \Gamma_h) = \Vstap{}{Q_h(\Gamma_h(\overline{s}_h), a_h)}$.
Other choices of $F_\text{unc}$ include posterior predictive variances provided by SCOD or coherent risk measures such as Conditional Value at Risk (CVaR)~\cite{brown2020bayesian}. 
We have experimetally validated such formulations of STAP, and ultimately, the appropriate choice of planning objective and optimization scheme should correspond to the evaluation tasks and domains of interest.

\noindent\textbf{Q5}: \textit{What are the main limitations of our method?}

Currently, our method is limited in settings with high degrees of partial observability and stochastic dynamics.
Examples include mobile manipulation in unknown environments and interacting with dynamic agents, where it may be intractable to predict future states necessary for look-ahead planning.
While these tasks are beyond the scope of this investigation, we note that our method is agnostic to the specific choice of skills and dynamics models, and thus, components that account for stochasticity could be used interchangeably. 
Instead, we focus on generalization to geometrically challenging manipulation problems.

\section{Manipulation Skill Library}
\label{appx:stap-skill-library}

\subsection{Parameterized Manipulation Primitives} 
All variants of STAP interface with a library of manipulation skills $\mathcal{L}=\{\psi^1, \ldots, \psi^K\}$.
Each skill $\psi$ consists of a learned policy $\func{\pi}{a \given s}$ and a parameterized manipulation primitive~\cite{felip2013manipulation} $\phi(a)$.
The policy is trained to output parameters $a \sim \func{\pi}{a \given s}$ that results in successful actuation of the primitive $\phi(a)$ in a contextual bandit setting (\cref{eq:stap-skill-mdp}) with a binary reward function $R(s, a, s')$.
Our library consists of four skills, $\mathcal{L}=\{\psi^{\text{Pick}}$, $\psi^{\text{Place}}$, $\psi^{\text{Pull}}$, $\psi^{\text{Push}}\}$, used to solve tasks in simulation and in the real-world.
The policies must learn to manipulate objects with different geometries (e.g. $\pi^{\text{Pick}}$ is used for both $\actioncall{Pick}{box}$ and $\actioncall{Pick}{hook}$).  
We describe the parameterization and reward function of each skill below.
A reward of $r = 0$ is provided if any collision occurs with a non-argument object.
For example, if $\psi^{\text{Place}}$ collides with \textit{rack} while executing $\actioncall{Place}{box}{table}$.
\vspace{-8pt}
\begin{itemize}
    \item \textbf{$\actioncall{Pick}{obj}$}: the parameter $a$ represents the grasp pose of \textit{obj} in the coordinate frame of \textit{obj}. The policy $\pi^{\text{Pick}}$ receives a reward of $r = 1$ if the primitive $\phi^{\text{Pick}}$ successfully grasps and picks up \textit{obj}.
    \item \textbf{$\actioncall{Place}{obj}{rec}$}: the parameter $a$ represents the placement pose of \textit{obj} in the coordinate frame of \textit{rec}. The policy $\pi^{\text{Place}}$ receives a reward of $r = 1$ if the primitive $\phi^{\text{Place}}$ places \textit{obj} stable atop \textit{rec}.
    \item \textbf{$\actioncall{Pull}{obj}{tool}$}: the parameter $a$ represents the initial position, direction, and distance of a pull on \textit{obj} with \textit{tool} in the coordinate frame of \textit{obj}. The policy $\pi^{\text{Pull}}$ receives a reward of $r = 1$ if the primitive $\phi^{\text{Pull}}$ moves \textit{obj} toward the robot by a minimum of $0.05m$. 
    \item \textbf{$\actioncall{Push}{obj}{tool}{rec}$}: the parameter $a$ represents the initial position, direction, and distance of a push on \textit{obj} with \textit{tool} in the coordinate frame of \textit{obj}. The policy $\pi^{\text{Push}}$ receives a reward of $r = 1$ if the primitive $\phi^{\text{Push}}$ moves \textit{obj} away from the robot by a minimum of $0.05m$ and if the final pose of \textit{obj} is underneath \textit{rec}.
\end{itemize}

\subsection{Training Manipulation Skills}
We use the Soft Actor-Critic~\cite{haarnoja2018-sac} (SAC) algorithm with original hyperparameters to simultaneously learn a stochastic policy $\func{\pi}{a \given s}$ and Q-function $Q^\pi(s, a)$ for each skill $\psi$. 
All models are trained for $200k$ single-step episodes and the $(s, a, s')$ transitions are stored in a replay buffer $\mathcal{D}^\psi=\{(s^i, a^i, s'^i)\}_{i=1}^{200k}$ for each skill $\psi$.
The replay buffer data is later used to train the dynamics models $\func{\overline{T}^k}{\overline{s}, a^k}$ (\cref{sec:stap-training-dynamics}) and calibrate the weights $w^k$ used by SCOD for UQ (\cref{sec:stap-training-scod}).

\chapter{Additional Details: Text2Motion}\label{chapter:text2motion-appx}

\section*{Overview}
The appendix offers additional details with respect to the implementation of \ttm{} and language planning baselines (\cref{appx:text2motion-implementation-details}), the experiments conducted (\cref{appx:text2motion-experiment-details}), derivations supporting the design of our algorithms (\cref{appx:text2motion-derivations}), and the \textbf{real-world planning demonstrations} (\cref{appx:text2motion-demos}). Qualitative results are made available at \url{https://sites.google.com/stanford.edu/text2motion}.

\startcontents[sections]
\printcontents[sections]{l}{1}{\setcounter{tocdepth}{2}}

\section{Implementation Details}
\label{appx:text2motion-implementation-details}
The \ttm{} planner integrates both \sh{} and \se{} to construct skill sequences that are feasible for the robot to execute in the environment. 
The planning procedure relies on four core components: 1) a library of learned robot skills, 2) a method for detecting when a skill is out-of-distribution (OOD), 3) a large language model (LLM) to perform task-level planning, and 4) a geometric feasibility planner that is compatible with the learned robot skills. 
All evaluated language-based planners use the above components, while \scgs{} and \imgs{} are myopic agents that do not perform geometric feasibility planning.
We provide implementation details of these components in the following subsections.

\subsection{Learning Robot Skills and Dynamics}
\label{appx-sub:text2motion-learning-models}
\textbf{Skill library overview:} All evaluated language planners interface an LLM with a library of robot skills $\mathcal{L}=\{\psi^1, \ldots, \psi^N\}$.
Each skill $\psi$ has a language description (e.g. \graytext{Pick(a)}) and is associated with a parameterized manipulation primitive~\cite{felip2013manipulation} $\phi(a)$.
A primitive $\phi(a)$ is controllable via its \textit{parameter} $a$ which determines the motion~\cite{khatib2003unified} of the robot's end-effector through a series of waypoints.
For each skill $\psi$, we train a policy $\pi(a|s)$ to output parameters $a\in\mathcal{A}$ that maximize primitive's $\phi(a)$ probability of success in a contextual bandit setting (\cref{eq:text2motion-skill-mdp}) with a skill-specific binary reward function $R(s, a, s')$.
We also train an ensemble of Q-functions $Q_{1:B}^\pi(s, a)$ and a dynamics model $T^\pi(s' | s, a)$ for each skill, both of which are required for geometric feasibility planning. 
We discuss the calibration of Q-function ensembles for OOD detection of skills in \cref{appx-sub:text2motion-ood-calibration}.

We learn four manipulation skills to solve tasks in simulation and in the real-world: $\psi^{\text{Pick}}$, $\psi^{\text{Place}}$, $\psi^{\text{Pull}}$, $\psi^{\text{Push}}$. 
Only a single policy per skill is trained, and thus, the policy must learn to engage the primitive over objects with differing geometries (e.g. $\pi^{\text{Pick}}$ is used for both \graytext{Pick(box)} and \graytext{Pick(hook)}).
The state space $\mathcal{S}$ for each policy is defined as the concatenation of geometric state features (e.g. pose, size) of all objects in the scene, where the first $n$ object states correspond to the $n$ skill arguments and the rest are randomized.
For example, the state for the skill \graytext{Pick(hook)} would have be a vector of all objects' geometric state features with the first component of the state corresponding to the \graytext{hook}.  

\textbf{Parameterized manipulation primitives:}
We describe the parameters $a$ and reward function $R(s, a, s')$ of each parameterized manipulation primitive $\phi(a)$ below.
A collision with a non-argument object constitutes an execution failure for all skills, and as a result, the policy receives a reward of $0$.
For example, $\pi^{\text{Pick}}$ would receive a reward of $0$ if the robot collided with \graytext{box} during the execution of \graytext{Pick(hook)}. 
\begin{itemize}
    \item \graytext{Pick(obj)}: $a\sim\pi^{\text{Pick}}(a \vert s)$ denotes the grasp pose of \graytext{obj} \textit{w.r.t} the coordinate frame of \graytext{obj}. A reward of $1$ is received if the robot successfully grasps \graytext{obj}.
    \item \graytext{Place(obj, rec)}: $a\sim\pi^{\text{Place}}(a \vert s)$ denotes the placement pose of \graytext{obj} \textit{w.r.t} the coordinate frame of \graytext{rec}. A reward of $1$ is received if \graytext{obj} is stably placed on \graytext{rec}.
    \item \graytext{Pull(obj, tool)}: $a\sim\pi^{\text{Pull}}(a \vert s)$ denotes the initial position, direction, and distance of a pull on \graytext{obj} with \graytext{tool} \textit{w.r.t} the coordinate frame of \graytext{obj}. A reward of $1$ is received if \graytext{obj} moves toward the robot by a minimum of $d_{\text{Pull}}=0.05m$. 
    \item \graytext{Push(obj, tool, rec)}: $a\sim\pi^{\text{Push}}(a \vert s)$ denotes the initial position, direction, and distance of a push on \graytext{obj} with \graytext{tool} \textit{w.r.t} the coordinate frame of \graytext{obj}. A reward of $1$ is received if \graytext{obj} moves away from the robot by a minimum of $d_{\text{Push}}=0.05m$ and if \graytext{obj} ends up underneath \graytext{rec}.
\end{itemize}

\textbf{Dataset generation:} 
All planners considered in \cref{chapter:text2motion} rely on accurate Q-functions $Q^\pi(s, a)$ to estimate the feasibility of skills proposed by the LLM. 
This places a higher fidelity requirement on the Q-functions than typically needed to learn a reliable policy, as the Q-functions must characterize both skill success (feasibility) and failure (infeasibility) at a given state.
Because the primitives $\phi(a)$ reduce the horizon of policies $\pi(a|s)$ to a single timestep, and the reward functions are $R(s, a, s')\in\{0, 1\}$, the Q-functions can be interpreted as binary classifiers of state-action pairs.
Thus, we take a staged approach to learning the Q-functions $Q^\pi$, followed by the policies $\pi$, and lastly the dynamics models $T^\pi$.

Scenes in our simulated environment are instantiated from a symbolic specification of objects and spatial relations, which together form a symbolic state.
The goal is to learn a Q-function that sufficiently covers the state-action space of each skill.
We generate a dataset that meets this requirement in four steps: a) enumerate all valid symbolic states; b) sample geometric scene instances $s$ per symbolic state; c) uniformly sample actions over the action space $a\sim\mathcal{U}^{[0,1]^d}$; (d) simulate the states and actions to acquire next states $s'$ and compute rewards $R(s, a, s')$.
We slightly modify this sampling strategy to maintain a minimum success-failure ratio of 40\%, as uniform sampling for more challenging skills such as \graytext{Pull} and \graytext{Push} seldom emits a success ($\sim$3\%).
We collect 1M $(s, a, s', r)$ tuples per skill of which 800K of them are used for training ($\mathcal{D}_t$), while the remaining 200K are used for validation ($\mathcal{D}_v$).
We use the same datasets to learn the Q-functions $Q^\pi$, policies $\pi$, and dynamics models $T^\pi$ for each skill.
 
\textbf{Model training:}
We train an ensemble of Q-functions with mini-batch gradient descent and logistic regression loss.
Once the Q-functions have converged, we distill their returns into stochastic policies $\pi$ through the maximum-entropy update~\cite{pmlr-v80-haarnoja18b}:
\begin{equation*}
    \begin{split}
        \pi^* \leftarrow \arg \max_{\pi} \, \operatorname{E}_{(s,a)\sim \mathcal{D}_t} [\min(Q_{1:B}^\pi(s, a)) \\ - \alpha \log\pi(a|s) ].
    \end{split}
\end{equation*}
Instead of evaluating the policies on $\mathcal{D}_v$, which contains states for which no feasible action exists, the policies are synchronously evaluated in an environment that exhibits only feasible states. 
This simplifies model selection and standardizes skill capabilities across primitives. 
All Q-functions achieve precision and recall rates of over 95\%. 
The average success rates of the converged policies over 100 evaluation episodes are: $\pi_{\text{Pick}}$ with 99\%, $\pi_{\text{Place}}$ with 90\%, $\pi_{\text{Pull}}$ with 86\%, $\pi_{\text{Push}}$ with 97\%.

We train a deterministic dynamics model per skill using the forward prediction loss:
\begin{equation*}
        L_{\text{dynamics}}\left(T^\pi; \mathcal{D}_t \right) = \operatorname{E}_{(s,a,s')\sim \mathcal{D}_t}||T^\pi(s, a) - s'||_2^2.
\end{equation*}
The dynamics models converge to within millimeter accuracy on the validation split.

\textbf{Hyperparameters:} The Q-functions, policies, and dynamics models are MLPs with hidden dimensions of size [256, 256] and ReLU activations.
We train an ensemble of $B=8$ Q-functions with a batch size of 128 and a learning rate of 1$e^{-4}$ with a cosine annealing decay~\cite{loshchilov2017sgdr}.
The Q-functions for \graytext{Pick}, \graytext{Pull}, and \graytext{Push} converged on $\mathcal{D}_v$ in 3M iterations, while the Q-function for \graytext{Place} required 5M iterations.
We hypothesize that this is because classifying successful placements demands carefully attending to the poses and shapes of all objects in the scene so as to avoid collisions.
The policies are trained for 250K iterations with a batch size of 128 and a learning rate of 1$e^{-4}$, leaving all other parameters the same as \cite{pmlr-v80-haarnoja18b}.
The dynamics models are trained for 750K iterations with a batch size of 512 and a learning rate of 5$e^{-4}$; only on successful transitions to avoid the noise associated with collisions and truncated episodes.
The parallelized training of all models takes approximately 12 hours on an Nvidia Quadro P5000 GPU and 2 CPUs per job.

\subsection{Out-of-Distribution Detection}
\label{appx-sub:text2motion-ood-calibration}
The training dataset described in \cref{appx-sub:text2motion-learning-models} contain both successes and failures for symbolically valid skills like \graytext{Pick(box)}.
However, when using LLMs for robot task planning, it is often the case that the LLM will propose symbolically invalid skills, such as \graytext{Pick(table)}, that neither the skill's policy, Q-functions, or dynamics model have observed in training.
We found that a percentage of out-of-distribution (OOD) queries would result in erroneously high Q-values, causing the invalid skill to be selected. 
Attempting to execute such a skill leads to control exceptions or other failures.

Whilst there are many existing techniques for OOD detection of deep neural networks, we opt to detect OOD queries on the learned Q-functions via deep ensembles due to their ease of calibration~\cite{LakshminarayananPritzelEtAll2017}.
A state-action pair is classified as OOD if the empirical variance of the predicted Q-values is above a determined threshold:
\begin{equation*}
    \func{F_{\text{OOD}}}{\psi} = \mathbbm{1} \left(\Vstap{i\sim1:B}{Q^{\pi}_i(s, a)} \geq \epsilon^{\psi} \right),
\end{equation*}
where each threshold $\epsilon^\psi$ is unique to skill $\psi$.

To determine the threshold values, we generate an a calibration dataset of 100K symbolically invalid states and actions for each skill.
The process takes less than an hour on a single CPU as the actions are infeasible and need not be simulated in the environment (i.e. rewards are known to be $0$).
We compute the mean and variance of the Q-ensemble for each $(s, a)$ sample in both the training dataset (in-distribution inputs) and the calibration dataset (out-of-distribution inputs), and produce two histograms by binning the computed ensemble variances by the ensemble means. 
We observe that the histogram of variances corresponding to OOD inputs is uniform across all Q-value bins and is an order of magnitude large than the ensemble variances computed over in-distribution inputs.
This allows us to select thresholds $\epsilon^\psi$ which are low enough to reliably detect OOD inputs, yet will not be triggered for in-distribution inputs: $\epsilon^{\text{Pick}} = 0.10$, $\epsilon^{\text{Place}} = 0.12$, $\epsilon^{\text{Pull}} = 0.10$, and $\epsilon^{\text{Push}} = 0.06$.

\subsection{Task Planning with LLMs}
\label{appx-sub:text2motion-llm-task-planning}
\ttm{}, \se{}, and the myopic planning baselines \scgs{} and \imgs{} use \texttt{code-davinci-002}~\cite{chen2021evaluating} to generate and score skills, while \sh{} queries \texttt{text-davinci-003}~\cite{ouyang2022training} to directly output full skill sequences.
In our experiments, we used a temperature setting of $0$ for all LLM queries. 

To maintain consistency in the evaluation of various planners, we allow \ttm{}, \scgs{}, and \imgs{} to generate $K=5$ skills $\{\psi^1_t, \ldots, \psi^K_t\}$ at each timestep $t$.
Thus, every search iteration of \se{} considers five possible extensions to the current running sequence of skills $\psi_{1:t-1}$.
Similarly, \sh{} generates $K=5$ skill sequences.

As described in \cref{sec:text2motion-greedy-search}, skills are selected at each timestep $t$ via a combined usefulness and geometric feasibility score:
\begin{align*}
    S_{\text{skill}}(\psi_t) &= S_{\text{llm}}(\psi_t) \cdot S_{\text{geo}}(\psi_t) \\
    &\approx p(\psi_t \mid i, s_{1:t}, \psi_{1:t-1}) \cdot Q^{\pi_t}(s_t, a^*_t),
\end{align*} 
where \ttm{}, \se{}, and \sh{} use geometric feasilibity planning (details below in \cref{appx-sub:text2motion-taps-geometric-planning}) to compute $S_{\text{geo}}(\psi_t)$, while \scgs{} and \imgs{} use the current value function estimate $V^{\pi_t}(s_t) = \operatorname{E}_{a_t\sim\pi_t}[Q^{\pi_t}(s_t, a_t)]$.
We find that in both cases, taking $S_{\text{llm}}(\psi_t)$ to be the SoftMax log-probability score produces a winner-takes-all effect, causing the planner to omit highly feasible skills simply because their associated log-probability was marginally lower than the LLM-likelihood of another skill.
Thus, we dampen the SoftMax operation with a $\beta$-coefficient to balance the ranking of skills based on both feasibility and usefulness. 
We found $\beta=0.3$ to work well.

\subsection{Geometric Feasibility Planning}
\label{appx-sub:text2motion-taps-geometric-planning}
Given a sequence of skills $\psi_{1:H}$, geometric feasibility planning computes parameters $a_{1:H}$ that maximizes the success probability of the underlying sequence of primitives $\phi_{1:H}$. 
For example, given a skill sequence \graytext{Pick(hook)}, \graytext{Pull(box, hook)}, geometric feasibility planning would compute a 3D grasp position on the hook that enables a successful pull on the box thereafter.

\ttm{} is agnostic to the method that fulfils the role of geometric feasibility planning. 
In our experiments we leverage Sequencing Task-Agnostic Policies (STAP)~\cite{taps-2022}.
Specifically, we consider the PolicyCEM variant of STAP, where optimization of the skill sequence's success probability (\cref{eq:text2motion-taps-objective}) is warm started with parameters sampled from the policies $a_{1:H}\sim \pi_{1:H}$. 
We perform ten iterations of the Cross-Entropy Method~\cite{rubinstein1999-cem}, sampling 10K trajectories at each iteration and selecting 10 elites to update the mean of the sampling distribution for the following iteration. 
The standard deviation of the sampling distribution is held constant at 0.3 for all iterations.

\section{Experiment Details}
\label{appx:text2motion-experiment-details}

\begin{table*}[t]
    \small
    \centering
    \caption[Simulated TableEnv manipulation planning task suite details (Text2Motion)]{\textbf{TableEnv manipulation task suite}. We use the following shorthands as defined in \cref{chapter:text2motion}: LH: Long-Horizon, LG: Lifted Goals, PAP: Partial Affordance Perception.
    \copyright{} 2023 Springer Nature.
    }
    \begin{tabular}{@{}l|l|l@{}}
    \toprule
    \textbf{Task ID} & \textbf{Properties} & \textbf{Instruction}\\
    \midrule
    \textbf{Task 1} & LH & How would you pick and place all of the boxes onto the rack?” \\
    \textbf{Task 2} & LH + LG & How would you pick and place the yellow box and blue box onto the table, \\ & & then use the hook to push the cyan box under the rack?”

 \\
    \textbf{Task 3} & LH + PAP & How would you move three of the boxes to the rack?”
 \\
    \textbf{Task 4} & LG + PAP & How would you put one box on the rack?”

 \\
    \textbf{Task 5} & LH + LG + PAP & How would you get two boxes onto the rack?”

 \\
    \textbf{Task 6} & LH + LG + PAP & How would you move two primary colored boxes to the rack?”

 \\
    \bottomrule
    \end{tabular}
    \label{tab:text2motion-text2motion-domains}
\end{table*}

We refer to \cref{tab:text2motion-text2motion-domains} for an overview of the tasks in the TableEnv Manipulation suite.

\subsection{Scene Descriptions as Symbolic States}
\label{appx-sub:text2motion-scene-descr-symbolic}
For the remainder of this section, we use the following definitions of terms:
\begin{itemize}
    \item \textbf{Predicate:} a binary-valued function over objects that evaluates to true or false (e.g. \graytext{on(a, b)})
    \item \textbf{Spatial Relation:} a predicate grounded over objects that evaluates to true (e.g. \graytext{on(rack, table)})
    \item \textbf{Predicate Classifier:} a function that implements whether a predicate is true or false in the scene. We use hand-crafted predicate classifiers for each spatial relation we model
    \item \textbf{Symbolic State:} the set of all predicates that hold true in the scene
    \item \textbf{Satisfaction Function:} a binary-valued function that takes as input a geometric state, uses the predicate classifiers to detect what predicates hold true in the geometric state, and collects those predicates into a set to form a symbolic state. The satisfaction function evaluates to true if the predicted goals (predicates) hold in the symbolic state
\end{itemize}

To provide scene context to \ttm{} and the baselines, we take a heuristic approach to converting a geometric state $s$ into a basic symbolic state.
Symbolic states consist of combinations of one or more of the following predicates: \graytext{on(a, b)}, \graytext{under(a, b)}, and \graytext{inhand(a)}. 
\graytext{inhand(a) = True} when the height of object \graytext{a} is above a predefined threshold.
\graytext{on(a, b) = True} when i) object \graytext{a} is above \graytext{b} (determined by checking if the centroid of \graytext{a}'s axis-aligned bounding box is greater than \graytext{b}'s axis-aligned bounding box), ii) \graytext{a}'s bounding box intersects \graytext{b}'s bounding box, and iii) \graytext{inhand(a) = False}. \graytext{under(a, b) = True} when \graytext{on(a, b) = False} and \graytext{a}'s bounding box intersects \graytext{b}'s bounding box.

\rone{The proposed \textbf{goal prediction method} (\cref{subsec:text2motion-goal-prediction}) outputs goal propositions consisting of combinations of the predicates above which have been grounded over objects (i.e. spatial relations). As an example, for the natural language instruction ``Put two of the boxes under the rack'' and a symbolic state \texttt{[on(red box, table), on(green box, rack), on(hook, rack), on(blue box, rack)]}, the LLM might predict the set of three goals \texttt{\{[under(red box, rack), under(blue box, rack)], [under(red box, rack), under(green box, rack)], [under(green box, rack), under(blue box, rack)]\}}. }
We note that objects are neither specified as within or beyond the robot workspace, as we leave it to the skill's Q-functions to determine feasibility (\cref{appx-sub:text2motion-learning-models}).

Since planning in high-dimensional observation spaces is not the focus of this investigation, we assume knowledge of objects in the scene and use hand-crafted heuristics to detect spatial relations between objects. 
There exists several techniques to convert high-dimensional observations into scene descriptions, such as the one used in \cite{zeng2022socratic}.
We leave exploration of these options to future work.

\subsection{In-Context Examples}
\label{appx-sub:text2motion-suppl-incontext-examples}

For all experiments and methods, we use the following in-context examples to construct the prompt passed to the LLMs.

\vspace{0.3cm}
\noindent\fbox{\parbox{0.97\linewidth}{\small{\texttt{{\\Available scene objects: ['table', 'hook', 'rack', 'yellow box', 'blue box', 'red box']\\
Object relationships: ['inhand(hook)', 'on(yellow box, table)', 'on(rack, table)', 'on(blue box, table)']\\
Human instruction: How would you push two of the boxes to be under the rack?\\
Goal predicate set: [['under(yellow box, rack)', 'under(blue box, rack)'], ['under(blue box, rack)', 'under(red box, rack)'], ['under(yellow box, rack)', 'under(red box, rack)']]\\
Top 1 robot action sequences: ['push(yellow box, hook, rack)', 'push(red box, hook, rack)']\\}}}}}

\noindent\fbox{\parbox{0.97\linewidth}{\small{\texttt{{\\Available scene objects: ['table', 'cyan box', 'hook', 'blue box', 'rack', 'red box']\\
Object relationships: ['on(hook, table)', 'on(rack, table)', 'on(blue box, table)', 'on(cyan box, table)', 'on(red box, table)']\\
Human instruction: How would you push all the boxes under the rack?\\
Goal predicate set: [['under(blue box, rack)', 'under(cyan box, rack)', 'under(red box, rack)']]\\
Top 1 robot action sequences: ['pick(blue box)', 'place(blue box, table)', 'pick(hook)', 'push(cyan box, hook, rack)', 'place(hook, table)', 'pick(blue box)', 'place(blue box, table)', 'pick(hook)', 'push(blue box, hook, rack)', 'push(red box, hook, rack)']\\}}}}}

\noindent\fbox{\parbox{0.97\linewidth}{\small{\texttt{{\\Available scene objects: ['table', 'cyan box', 'red box', 'hook', 'rack']\\
Object relationships: ['on(hook, table)', 'on(rack, table)', 'on(cyan box, rack)', 'on(red box, rack)']\\
Human instruction: put the hook on the rack and stack the cyan box above the rack - thanks\\
Goal predicate set: [['on(hook, rack)', 'on(cyan box, rack)']]\\
Top 1 robot action sequences: ['pick(hook)', 'pull(cyan box, hook)', 'place(hook, rack)', 'pick(cyan box)', 'place(cyan box, rack)']\\}}}}}

\noindent\fbox{\parbox{0.97\linewidth}{\small{\texttt{{\\Available scene objects: ['table', 'rack', 'hook', 'cyan box', 'yellow box', 'red box']\\
Object relationships: ['on(yellow box, table)', 'on(rack, table)', 'on(cyan box, table)', 'on(hook, table)', 'on(red box, rack)']\\
Human instruction: Pick up any box.\\
Goal predicate set: [['inhand(yellow box)'], ['inhand(cyan box)']]\\
Top 1 robot action sequences: ['pick(yellow box)']\\}}}}}

\noindent\fbox{\parbox{0.97\linewidth}{\small{\texttt{{\\Available scene objects: ['table', 'blue box', 'cyan box', 'hook', 'rack', 'red box', 'yellow box']\\
Object relationships: ['inhand(hook)', 'on(red box, rack)', 'on(yellow box, table)', 'on(blue box, table)', 'on(cyan box, rack)', 'on(rack, table)']\\
Human instruction: could you move all the boxes onto the rack?\\
Goal predicate set: [['on(yellow box, rack)', 'on(blue box, rack)']]\\
Top 1 robot action sequences: ['pull(yellow box, hook)', 'place(hook, table)', 'pick(yellow box)', 'place(yellow box, rack)', 'pick(blue box)', 'place(blue box, rack)']\\}}}}}

\noindent\fbox{\parbox{0.97\linewidth}{\small{\texttt{{\\Available scene objects: ['table', 'blue box', 'red box', 'hook', 'rack', 'yellow box']\\
Object relationships: ['on(hook, table)', 'on(blue box, table)', 'on(rack, table)', 'on(red box, table)', 'on(yellow box, table)']\\
Human instruction: situate an odd number greater than 1 of the boxes above the rack\\
Goal predicate set: [['on(blue box, rack)', 'on(red box, rack)', 'on(yellow box, rack)']]\\
Top 1 robot action sequences: ['pick(hook)', 'pull(blue box, hook)', 'place(hook, table)', 'pick(blue box)', 'place(blue box, rack)', 'pick(red box)', 'place(red box, rack)', 'pick(yellow box)', 'place(yellow box, rack)']\\}}}}}

\noindent\fbox{\parbox{0.97\linewidth}{\small{\texttt{{\\Available scene objects: ['table', 'cyan box', 'hook', 'red box', 'yellow box', 'rack', 'blue box']\\
Object relationships: ['on(hook, table)', 'on(red box, table)', 'on(blue box, table)', 'on(cyan box, table)', 'on(rack, table)', 'under(yellow box, rack)']\\
Human instruction: How would you get the cyan box under the rack and then ensure the hook is on the table?\\
Goal predicate set: [['under(cyan box, rack)', 'on(hook, table)']]\\
Top 1 robot action sequences: ['pick(blue box)', 'place(blue box, table)', 'pick(red box)', 'place(red box, table)', 'pick(hook)', 'push(cyan box, hook, rack)', 'place(hook, table)']\\}}}}}

\noindent\fbox{\parbox{0.97\linewidth}{\small{\texttt{{\\Available scene objects: ['table', 'cyan box', 'hook', 'yellow box', 'blue box', 'rack']\\
Object relationships: ['on(hook, table)', 'on(yellow box, rack)', 'on(rack, table)', 'on(cyan box, rack)']\\
Human instruction: set the hook on the rack and stack the yellow box onto the table and set the cyan box on the rack\\
Goal predicate set: [['on(hook, rack)', 'on(yellow box, table)', 'on(cyan box, rack)']]\\
Top 1 robot action sequences: ['pick(yellow box)', 'place(yellow box, table)', 'pick(hook)', 'pull(yellow box, hook)', 'place(hook, table)']\\}}}}}

\noindent\fbox{\parbox{0.97\linewidth}{\small{\texttt{{\\Available scene objects: ['table', 'cyan box', 'hook', 'rack', 'red box', 'blue box']\\
Object relationships: ['on(hook, table)', 'on(blue box, rack)', 'on(cyan box, table)', 'on(red box, table)', 'on(rack, table)']\\
Human instruction: Move the warm colored box to be underneath the rack.\\
Goal predicate set: [['under(red box, rack)']]\\
Top 1 robot action sequences: ['pick(blue box)', 'place(blue box, table)', 'pick(red box)', 'place(red box, table)', 'pick(hook)', 'push(red box, hook, rack)']\\}}}}}

\noindent\fbox{\parbox{0.97\linewidth}{\small{\texttt{{\\Available scene objects: ['table', 'blue box', 'hook', 'rack', 'red box', 'yellow box']\\
Object relationships: ['on(hook, table)', 'on(red box, table)', 'on(blue box, table)', 'on(yellow box, rack)', 'on(rack, table)']\\
Human instruction: Move the ocean colored box to be under the rack and ensure the hook ends up on the table.\\
Goal predicate set: [['under(blue box, rack)']]\\
Top 1 robot action sequences: ['pick(red box)', 'place(red box, table)', 'pick(yellow box)', 'place(yellow box, rack)', 'pick(hook)', 'push(blue box, hook, rack)', 'place(hook, table)']\\}}}}}

\noindent\fbox{\parbox{0.97\linewidth}{\small{\texttt{{\\Available scene objects: ['table', 'cyan box', 'hook', 'rack', 'red box', 'blue box']\\
Object relationships: ['on(hook, table)', 'on(cyan box, rack)', 'on(rack, table)', 'on(red box, table)', 'inhand(blue box)']\\
Human instruction: How would you set the red box to be the only box on the rack?\\
Goal predicate set: [['on(red box, rack)', 'on(blue box, table)', 'on(cyan box, table)']]\\
Top 1 robot action sequences: ['place(blue box, table)', 'pick(hook)', 'pull(red box, hook)', 'place(hook, table)', 'pick(red box)', 'place(red box, rack)', 'pick(cyan box)', 'place(cyan box, table)']\\}}}}}

\section{Derivations}
\label{appx:text2motion-derivations}

We provide two derivations to support our approximation of the skill score $S_{\text{skill}}$ (used to select skills while planning with \se{} and \ttm{}) defined in \cref{eq:text2motion-tamp-step-score-factor}.
The skill score is expressed as a product of two terms:
\begin{equation}
    \label{eq:text2motion-der-tamp-step-score-factor}
    \begin{split}
        S_{\text{skill}}(\psi_t) &= p(\psi_t \mid i, s_1, \psi_{1:t-1}, r_{1:t-1}) \\ 
        &\quad\quad\quad\quad p(r_t \mid i, s_1, \psi_{1:t}, r_{1:t-1}).
    \end{split}
\end{equation}

\subsection{Skill Usefulness Derivation}
\label{appx-sub:text2motion-skill-usefulness}
\cref{eq:text2motion-llm-step-score} defines the first term in the skill score product to be the skill \textit{usefulness} score $S_{\text{llm}}$.
We derive the approximation of $S_{\text{llm}}$ given in \cref{eq:text2motion-llm-step-score-decomp}, which corresponds to quantity we use in our experiments.
\begin{align}
    S_{\text{llm}}(\psi_t) 
        &= p(\psi_t \mid i, s_1, \psi_{1:t-1}, r_{1:t-1}) \nonumber \\
        \begin{split}
            &= \int p(\psi_t \mid i, s_{1:t}, \psi_{1:t-1}, r_{1:t-1}) \\
            &\quad\quad\quad\quad p(s_{2:t} \mid i, s_1, \psi_{1:t-1}, r_{1:t-1}) \,d s_{2:t} \nonumber
        \end{split} \\
        &= \Estap{s_{2:t}}{p(\psi_t \mid i, s_{1:t}, \psi_{1:t-1}, r_{1:t-1})} \label{eq:text2motion-der-llm-step-score-dep} \\
        &\approx \Estap{s_{2:t}}{p(\psi_t \mid i, s_{1:t}, \psi_{1:t-1})} \label{eq:text2motion-der-llm-step-score-indep} \\
        &\approx p(\psi_t \mid i, s_{1:t}, \psi_{1:t-1}) \label{eq:text2motion-der-llm-step-score-decomp}
\end{align}

The final expression is given in \cref{eq:text2motion-der-llm-step-score-decomp}. Here, we compute a single sample Monte-Carlo estimate of \cref{eq:text2motion-der-llm-step-score-indep} under the future state trajectory $s_{2} \sim T^{\pi_{1}}(\cdot | s_{1}, a^*_{1}), \ldots, s_{t} \sim T^{\pi_{t-1}}(\cdot | s_{t-1}, a^*_{t-1})$, where $a^*_{1:t-1}$ is computed by STAP~\cite{taps-2022}.
The key insight is that future state trajectories $s_{2:t}$ are only ever sampled after STAP has performed geometric feasibility planning to maximize the \textit{success probability} (\cref{eq:text2motion-motion-parameter-score}) of the running plan $\psi_{1:t-1}$.
By doing so, we ensure that the future states $s_{2:t}$ correspond to 
a successful execution of the running plan $\psi_{1:t-1}$, i.e. achieving positive rewards $r_{1:t-1}$.
This supports the independence assumption on rewards $r_{1:t-1}$ used to derive \cref{eq:text2motion-der-llm-step-score-indep} from \cref{eq:text2motion-der-llm-step-score-dep}.

\subsection{Skill Feasibility Derivation}
\label{appx-sub:text2motion-skill-feasibility}
\cref{eq:text2motion-motion-step-score} defines the second term in the skill score product (\cref{eq:text2motion-der-tamp-step-score-factor}) as the skill \textit{feasibility} score $S_{\text{geo}}$.
We derive the approximation provided in \cref{eq:text2motion-motion-step-score-decomp}, which is the quantity we use in our experiments.
\begin{align}
    S_{\text{geo}}(\psi_t) 
        &= p(r_t \mid i, s_1, \psi_{1:t}, r_{1:t-1}) \label{eq:text2motion-der-geo-step-score} \\
        &= p(r_t \mid s_1, \psi_{1:t}, r_{1:t-1}) \label{eq:text2motion-der-geo-step-score-indep-i} \\
        \begin{split}
            &= \int p(r_t \mid s_{1:t}, \psi_{1:t}, r_{1:t-1}) \\
            &\quad\quad\quad\quad p(s_{2:t} \mid s_1, \psi_{1:t}, r_{1:t-1}) \,d s_{2:t} \nonumber
        \end{split} \\
        &= \Estap{s_{2:t}}{p(r_t \mid s_{1:t}, \psi_{1:t}, r_{1:t-1})} \label{eq:text2motion-der-geo-step-score-dep-r} \\
        &\approx \Estap{s_{2:t}}{p(r_t \mid s_{1:t}, \psi_{1:t})} \label{eq:text2motion-der-geo-step-score-indep-r} \\
        &= \Estap{s_{2:t}}{p(r_t \mid s_{1:t}, a^*_{1:t})} \label{eq:text2motion-der-geo-step-param-score-indep-r} \\
        &= \Estap{s_{2:t}}{p(r_t \mid s_t, a^*_t)} \label{eq:text2motion-der-geo-step-param-score-markov} \\
        &= \Estap{s_{2:t}}{Q^{\pi_t}(s_t, a^*_t)} \label{eq:text2motion-der-geo-step-Q-score} \\
        &\approx Q^{\pi_t}(s_t, a^*_t) \label{eq:text2motion-der-geo-step-Q-score-decomp}
\end{align}

From \cref{eq:text2motion-der-geo-step-score} to \cref{eq:text2motion-der-geo-step-score-indep-i}, the reward $r_t$ is conditionally independent of the instruction $i$ given the initial state $s_1$, running plan $\psi_{1:t}$, and previous rewards $r_{1:t-1}$.
As described in \cref{appx-sub:text2motion-skill-usefulness}, we can use STAP to make an independence assumption on the previous rewards $r_{1:t-1}$ between \cref{eq:text2motion-der-geo-step-score-dep-r} and \cref{eq:text2motion-der-geo-step-score-indep-r}. 
The reward probability in \cref{eq:text2motion-der-geo-step-score-indep-r} depends on the parameters $a^*_{1:t}$ computed by STAP and fed to the underlying primitive sequence $\phi_{1:t}$, which gives \cref{eq:text2motion-der-geo-step-param-score-indep-r}.
\cref{eq:text2motion-der-geo-step-param-score-markov} comes from the Markov assumption, and can be reduced to \cref{eq:text2motion-der-geo-step-Q-score} by observing that the reward probability $p(r_t \mid s_t, a^*_t)$ is equal to the Q-value $Q^{\pi_t}(s_t, a^*_t)$ in the contextual bandit setting we consider.
The final expression given in \cref{eq:text2motion-der-geo-step-Q-score-decomp}, which represents a single sample Monte-Carlo estimate of \cref{eq:text2motion-der-geo-step-Q-score} under a sampled future state trajectory $s_{2} \sim T^{\pi_{1}}(\cdot | s_{1}, a^*_{1}), \ldots, s_{t} \sim T^{\pi_{t-1}}(\cdot | s_{t-1}, a^*_{t-1})$.

\section{Real World Demonstration}
\label{appx:text2motion-demos}

\subsection{Hardware Setup}
\label{appx-sub:text2motion-hardware-setup}
We use a Kinect V2 camera for RGB-D image capture and manually adjust the color thresholds to segment objects in the scene. 
Given the segmentation masks and the depth image, we can estimate object poses to construct the geometric state of the environment. 
\rtwo{For the skill library, we use the same set of policies, Q-functions, and dynamics models trained in simulation.}
We run robot experiments on a Franka Panda robot manipulator.
 
\subsection{Real-World Robot Demonstration}
Demonstrations of \ttm{} operating on a real robot are made available at \url{https://sites.google.com/stanford.edu/text2motion}.

\bibliographystyle{plainnat}
\bibliography{references}
\end{document}